%% file: main.tex
\documentclass[11pt, letterpaper]{article}
\usepackage{arxiv}
\usepackage{tikz}
\usepackage{adjustbox}
\usetikzlibrary{positioning, arrows.meta}

\title{\LARGE Learning under Quantization for High-Dimensional Linear Regression}

\author{
Dechen Zhang\thanks{Institute of Data Science, The University of Hong Kong. Email: \email{dechenzhang}{connect.hku.hk}}
\qquad Junwei Su\thanks{School of Computing \& Data Science, The University of Hong Kong. Email: \email{junweisu}{connect.hku.hk}}
\qquad Difan Zou\thanks{School of Computing \& Data Science and Institute of Data Science, The University of Hong Kong. Email: \email{dzou}{hku.hk} }
}
\date{\small{\today}}

\begin{document}


\maketitle

\begin{abstract}
    The use of low-bit quantization has emerged as an indispensable technique for enabling the efficient training of large-scale models. Despite its widespread empirical success, a rigorous theoretical understanding of its impact on learning performance remains notably absent, even in the simplest linear regression setting. We present the first systematic theoretical study of this fundamental question, analyzing finite-step stochastic gradient descent (SGD) for high-dimensional linear regression under a comprehensive range of quantization targets: data, label, parameter, activation, and gradient. 
    Our novel analytical framework establishes precise algorithm-dependent and data-dependent excess risk bounds that characterize how different quantization affects learning: parameter, activation, and gradient quantization amplify noise during training; data quantization distorts the data spectrum and introduces additional approximation error. Crucially, we distinguish the effects of two quantization schemes: we prove that for additive quantization (with constant quantization steps), the noise amplification benefits from a suppression effect scaled by the batch size, while multiplicative quantization (with input-dependent quantization steps) largely preserves the spectral structure, thereby reducing the spectral distortion.
    Furthermore, under common polynomial-decay data spectra, we quantitatively compare the risks of multiplicative and additive quantization, drawing a parallel to the comparison between FP and integer quantization methods. Our theory provides a powerful lens to characterize how quantization shapes the learning dynamics of optimization algorithms, paving the way to further explore learning theory under practical hardware constraints.
\end{abstract}




\settocdepth{part}
\input{1_intro}
\input{2_related}
\input{3_main}
\input{4_proof_sketch}

\newpage
\bibliography{reference}

\bibliographystyle{ims}

\newpage
\appendix 

\input{5_appendix}
\input{additional_experiment_detail}
\end{document}

%% file: 1_intro.tex
\section{Introduction}
Quantization has garnered widespread attention as an essential technique for deploying large-scale deep learning models, particularly large language models (LLMs) \citep{lang2024comprehensive,shen2024exploring}. In line with this low-precision paradigm, a new frontier of research has emerged: quantization scaling laws, which seek to formalize the trade-offs between model size, dataset size, and computational bit-width. Seminal work by \citet{kumar2024scaling} treated bit-width as a discrete measure of precision. This was extended by \citet{sun2025scaling}, who established a more comprehensive scaling law for floating-point (FP) quantization \citep{kuzmin2022fp8} by separately accounting for the distinct roles of exponent and mantissa bits. Going further, \citet{chen2025scaling} proposed a unified scaling law that models quantized error as a function of model size, training data volume, and quantization group size. Collectively, these studies provide rigorous understandings to guide the joint allocation of fixed compute or memory budgets across data size, model size and precision (quantization bit-width).

The empirical understanding of low-precision training has advanced rapidly, yet a significant theory-practice gap persists. Theoretical research remains predominantly restricted to analyzing \textit{convergence guarantees on the training loss} for quantized optimizers \citep{nadiradze2021asynchronous,9834296,markov2023quantized,xin2025global}. For example, \citet{markov2023quantized} derived convergence guarantee for the communication-efficient variant of Fully-Shared Data-Parallel distributed training under parameter and gradient quantization.
While these studies offer crucial insights into optimization, they overlook a more fundamental question: \textit{how does quantization affect the model's learning performance?} Specifically, a rigorous characterization of the interplay between quantization, model dimension, dataset size, and their joint effect on the \textit{population risk} remains largely unexplored. A notable step in this direction is \citet{zhang2022quantization}, which analyzed the generalization of quantized two-layer networks through the lens of neural tangent kernel (NTK). However, their work is limited in three key aspects: it only considers parameter quantization; its analysis is confined to the lazy-training regime; and it fails to provide explicit generalization bounds in terms of core parameters like sample size, dimension, and quantization error. These limitations restrict its applicability to modern low-precision training practices. 

Motivated by recent theoretical advances in scaling laws \citep{lin2024scaling,lin2025improved,li2025functional}, we analyze the learning performance of quantized training using a high-dimensional linear model. This model serves as a powerful and well-established testbed for isolating phenomena like learning rate and batch size effects \citep{kunstner2025scaling,luo2025multi,zhang2024does,xiao2024rethinking, ren2025emergence,bordelon2025feature}. Its simplicity provides the analytical flexibility necessary to derive precise relationships between generalization error and critical parameters such as dimension, sample size, and quantization error (or bit-width).


\textbf{Our setting.} In this paper, we consider SGD for linear regression under quantization. We first iterate the standard linear regression problem as follows:
\begin{equation*}
    \min_{\mathbf{w}}L(\mathbf{w}),\ \text{where} \ L(\mathbf{w})=\frac{1}{2}\mathbb{E}_{\mathbf{x},y}\left[\left(y-\langle \mathbf{w},\mathbf{x} \rangle \right)^2\right].
\end{equation*}
Here $\mathbf{x}\in \mathcal{H}$ is the feature vector, $\mathcal{H}$ is some (finite $d$-dimensional or countably infinite dimensional) Hilbert space, $y\in \mathbb{R}$ is the response, $\mathcal{D}$ is an unknown distribution over $\mathbf{x}$ and $y$, and $\mathbf{w}\in \mathcal{H}$ is the weight vector to be optimized. We consider the constant step size SGD under quantization: at each iteration $t$, an i.i.d. batch (with batch size $B$) of examples $(\mathbf{X}_t, \mathbf{y}_t)\in \mathbb{R}^{B\times d} \times \mathbb{R}^{B}$ is observed, and the weight $\mathbf{w}_t \in \mathbb{R}^d$ is updated according to following quantized SGD algorithm.
\begin{align}\label{eq:SGD_quantized}
    \mathbf{w}_t=\mathbf{w}_{t-1}+\gamma \frac{1}{B} \mathcal{Q}_d(\mathbf{X}_t)^\top \mathcal{Q}_o\Big(\mathcal{Q}_l(\mathbf{y}_t)-\mathcal{Q}_a\big(\mathcal{Q}_d(\mathbf{X}_t)\mathcal{Q}_p(\mathbf{w}_{t-1})\big)\Big), \quad t=1,...,N, \tag{quantized SGD}
\end{align}
where $\gamma>0$ is a constant stepsize, $N$ is the number of sample batches observed, the master weights is initialized at $\mathbf{w}_0$, and $\mathcal{Q}_d, \mathcal{Q}_l, \mathcal{Q}_p, \mathcal{Q}_a, \mathcal{Q}_o$ are independent general quantization operations for data feature, label, model parameter, activation and output gradient respectively. Notably, for theoretical simplicity, we assume all matrix operations (e.g., addition and multiplication) are computed in full precision, with quantization applied subsequently to obtain low-precision values. Then, we consider the average iterate as the algorithm output, i.e., $\overline{\mathbf{w}}_N:=\frac{1}{N}\sum_{t=0}^{N-1}\mathbf{w}_t$. Without loss of generality, we assume the initial parameter is $\mathbf{w}_0=0$.

The goal of this work is to characterize the learning performance of the quantized SGD via evaluating the population risk $L(\overline{\mathbf{w}}_N)$, and more importantly, its relationship with the quantization error. Let $\mathbf w^*=\arg\min L(\mathbf w)$, we define the following excess risk as a surrogate of the population risk:
\begin{align}\label{eq:excess_risk}
\mathcal{E}(\overline{\mathbf{w}}_N) = L(\overline{\mathbf{w}}_N) - L(\mathbf w^*).\tag{\text{Excess Risk}}
\end{align}


\textbf{Our contributions.} 
We perform a novel theoretical study on the learnability of the quantized SGD algorithm for high-dimensional linear regression problems. Our contributions are summarized as follows:

\begin{itemize}[leftmargin=*]
\item We perform systematic analysis and establish a theoretical bound for the excess risk of quantized SGD. This bound is explicitly formulated as a function of the full eigen-spectrum of the quantized data feature covariance, sample size, and quantization errors (see Theorem \ref{Theorem 4} for details). Our results precisely reveal how quantization applied to different model components impacts learning performance: quantization of data distorts the spectrum of effective data covariance and introduces an additional approximation error; while the quantization of parameter, activation, and output gradient amplify noise throughout the training process on the quantized feature space.

\item We analyze two standard quantization error models: additive and multiplicative, which conceptually relate to the integer and FP quantization techniques. For additive quantization, our theoretical bounds indicate that the noise amplification stemming from activation and output gradient quantization diminishes as the batch size increases, whereas the spectrum of the effective data covariance is distorted by a constant noise floor (see Corollary \ref{theorem new 2.5} for details). Conversely, for multiplicative quantization, our results demonstrate that data quantization preserves the intrinsic spectral structure of the effective covariance, thereby reducing spectral distortion; however, the resulting noise amplification remains independent of the batch size (see Theorem \ref{Theorem 1} for details).


\item We further derive the conditions on the quantization errors such that the learning performance of the full-precision SGD can be maintained (in orders). Our results indicate that compared with multiplicative quantization, additive quantization necessitates stricter spectral constraints on data quantization but allows for more relaxed conditions on activation and output gradient quantization, benefiting from the batch-averaging effect (see Corollary \ref{coro:comparison} for details). By applying our excess risk bounds to polynomial decay spectrum, we show that multiplicative quantization is applicable even in high-dimensional settings, whereas additive quantization is not (see Corollary \ref{coro: case study} for details). These simplified theoretical results also draw implications for comparing integer and FP quantization, allowing us to identify the conditions under which each type is likely to yield superior performance.


\end{itemize}

\textbf{Notations.} For two positive-valued functions $f(x)$ and $g(x)$, we write $f(x) \lesssim g(x)$ or $f(x) \gtrsim g(x)$ if $f(x) \leq cg(x)$ or $f(x) \geq cg(x)$ holds for some absolute (if not otherwise specified) constant $c > 0$ respectively. We write $f(x) \eqsim g(x)$ if $f(x) \lesssim g(x) \lesssim f(x)$. For two vectors $\mathbf{u}$ and $\mathbf{v}$ in a Hilbert space, we denote their inner product by $\langle \mathbf{u},\mathbf{v}\rangle$ or $\mathbf{u}^\top\mathbf{v}$. For two matrices $\mathbf{A}$ and $\mathbf{B}$ of appropriate dimensions, we define their inner product by $\langle \mathbf{A},\mathbf{B}\rangle := \mathrm{tr}\left(\mathbf{A}^\top\mathbf{B}\right)$. We use $\|\cdot \|$ to denote the operator norm for matrices and $\ell_2$-norm for vectors. For a positive semi-definite (PSD) matrix $\mathbf{A}$ and a vector $\mathbf{v}$ of appropriate dimension, we write $\|\mathbf{v}\|_\mathbf{A}^2=\mathbf{v}^\top \mathbf{A}\mathbf{v}$.

%% file: 2_related.tex
\section{Related Works}
\textbf{High-dimensional linear regression via SGD.} 
Theoretical guarantees for the generalization property have garnered significant attention in machine learning and deep learning. Seminal work by \citet{bartlett2020benign,tsigler2023benign} derived nearly tight upper and lower excess risk bounds in linear (ridge) regression for general regularization schemes. With regards to the classical underparameterized regime, a large number of works studied the learnability of iterate averaged SGD in linear regression \citep{polyak1992acceleration,defossez2015averaged,bach2013non,dieuleveut2017harder,jain2018parallelizing,jain2017markov}. With regards to modern overparameterized setting, one-pass SGD in linear regression has also been extensively studied \citep{dieuleveut2015non,berthier2020tight,varre2021last,zou2023benign,wu2022last,wu2022power,zhang2024optimality}, providing a framework to characterize how the optimization algorithm affects the generalization performance for various data distributions. 
Another line of work analyzed the behavior of multi-pass SGD on a high-dimensional $\ell^2$-regularized least-squares problem, characterizing excess risk bounds \citep{lei2021generalization,zou2022risk} and the exact dynamics of excess risk \citep{paquette2024homogenization}. 
From a technical perspective, our work builds on the sharp finite-sample and dimension-free analysis of SGD developed by \citet{zou2023benign}. However, these works did not concern the practical quantization operations. It remains unclear how quantization error affects the learning behavior of SGD for linear regression.

\textbf{Theoretical analysis for quantization.}
As a powerful technique for deploying large-scale deep learning models, quantization has attracted significant attention. From the theoretical perspective, a line of works focus on the convergence guarantee in both quantized training (SGD) algorithms \citep{de2015taming,alistarh2017qsgd,faghri2020adaptive,gorbunov2020unified,gandikota2021vqsgd,markov2023quantized,xin2025global} and post-training quantization methods \citep{lybrand2021greedy,zhang2023spfq,zhang2023post,zhang2025provable}. For low-precision training (SGD), \citet{de2015taming} was the first to consider the convergence guarantees. Assuming unbiased stochastic quantization, convexity, and gradient sparsity, they gave upper bounds on the error probability of SGD. \citet{alistarh2017qsgd} refined these results by focusing on the trade-off between communication and convergence and proposed Quantized SGD (QSGD). \citet{faghri2020adaptive} extended the fixed quantization scheme \citep{alistarh2017qsgd} to two adaptive quantization schemes, providing a more general convergence guarantee for quantized training. For post-training quantization, \citet{lybrand2021greedy} derived an error bound for ternary weight quantization under independent Gaussian data distribution. \citet{zhang2023post} extended these results to more general quantization grids and a wider range of data distributions using a different proof technique. More recently, \citet{zhang2025provable} presented the first
quantitative error bounds for OPTQ post-training algorithm framework. However, no prior work provides explicit generalization bounds. 

\textbf{Linear models for theory of scaling law.}
Several recent studies have sought to formalize and explain the empirical scaling laws using conceptually simplified linear models \citep{bahri2024explaining, atanasov2024scaling, paquette20244+, bordelon2024dynamical, lin2024scaling, lin2025improved}. Among them, \citet{bahri2024explaining} considered a linear teacher-student model with power-law spectrum and showed that the test loss of the ordinary least square estimator decreases following a power law in sample size $N$ (or model size $M$) when the other parameter goes to infinity. \citet{bordelon2024dynamical} analyzed the test error of the solution found by gradient flow in a linear random feature model and established power-law scaling in one of $N$, $M$ and training time $T$ while the other two parameters go to infinity. Building on the technique in \citet{zou2023benign}, \citet{lin2024scaling} analyzed the test error of the last iterate of one-pass SGD in a sketched linear model. They presented the first systematic study to establish a finite-sample joint scaling law (in $M$ and $N$) for linear models that aligns with empirical observations \citep{kaplan2020scaling}. More recently, \citet{lin2025improved} extended the scaling law analysis to the setting with data reuse (i.e., multi-pass SGD) in data-constrained regimes.

%% file: 3_main.tex
\section{Preliminary}

\subsection{Quantization operations} 
For all quantization operations in  (\ref{eq:SGD_quantized}), we employ the stochastic quantization method \citep{markov2023quantized}, which unbiasedly rounds values using randomly adjusted probabilities. This stochastic quantization is widely used in both empirical and theoretical analysis of quantization \citep{modoranu2024microadam,ozkara2025stochastic}.
We summarize this in the following assumption.
\begin{assumption}
    \label{ass1}
    Let $\mathcal Q_i, i\in \{d,l,p,a,o\}$ be the coordinate-wise quantization operation for data feature, label, model parameter, activation, and output gradient, respectively. We assume that the quantization operation is unbiased, i.e., for any $\mathbf{u}$,
    \begin{equation*}
        \mathbb{E}\left[\mathbf{\mathcal{Q}}_i(\mathbf{u})|\mathbf{u}\right]=\mathbf{u}.
    \end{equation*}
\end{assumption}

Furthermore, to better uncover the effect of quantization, we consider the following two types of quantization error: multiplicative quantization and additive quantization, which are motivated by abstracting the behavior of prevalent numerical formats used in practice.
\begin{definition}\label{def:type_quantization}
Let $\mathcal{Q}$ be an unbiased quantization operation. We categorize it based on the structure of its error variance:
\begin{itemize}[leftmargin=*,nosep]
\item \textbf{Multiplicative quantization.} We call the quantization is $\epsilon$-multiplicative if the conditional second moment of quantization error is proportional to the outer product of raw data itself, i.e.,
    \begin{equation*}
        \mathbb{E}\left[\left(\mathcal{Q}(\mathbf{x})-\mathbf{x}\right)\left(\mathcal{Q}(\mathbf{x})-\mathbf{x}\right)^\top\bigg|\mathbf{x}\right]=\epsilon \mathbf{x}\mathbf{x}^\top.
    \end{equation*}
\item \textbf{Additive quantization.} We call the quantization is $\epsilon$-additive if the conditional second moment of quantization error is proportional to identity, i.e.,
    \begin{equation*}
        \mathbb{E}\left[\left(\mathcal{Q}(\mathbf{x})-\mathbf{x}\right)\left(\mathcal{Q}(\mathbf{x})-\mathbf{x}\right)^\top\bigg|\mathbf{x}\right]=\epsilon \mathbf{I}.
    \end{equation*}
\end{itemize}
\end{definition}

This theoretical distinction is grounded in practical quantization schemes. For instance, integer quantization (e.g., INT8, INT16) uses a fixed bin length, resulting in an error that is largely independent of the value's magnitude. This characteristic aligns with our definition of additive quantization, where the error variance is uniform across coordinates. Conversely, floating-point quantization (e.g., FP8, FP32) employs a value-aware bin length via its exponent and mantissa bits (e.g., the E4M3 format in FP8). This structure causes the quantization error to scale with the magnitude of the value itself, corresponding to the model of multiplicative quantization. 

To precisely capture the quantization error, we further introduce some relevant notations on quantization errors during training. Denote the activation and output gradient at time $t$ as 
\begin{equation*}
    \begin{aligned}
        \mathbf{a}_t=\mathcal{Q}_d(\mathbf{X}_t)\mathcal{Q}_p(\mathbf{w}_{t-1}),\quad \mathbf{o}_t=\mathcal{Q}_l(\mathbf{y}_t)-\mathcal{Q}_a\left(\mathcal{Q}_d(\mathbf{X}_t)\mathcal{Q}_p(\mathbf{w}_{t-1})\right).
    \end{aligned}
\end{equation*}
Then we are ready to define quantization errors.
\begin{definition}
    The quantization error on data $\boldsymbol{\epsilon}^{(d)}$, on label $\epsilon^{(l)}$, on parameter $\boldsymbol{\epsilon}_t^{(p)}$ at time $t$, on activation $\boldsymbol{\epsilon}_t^{(a)}$ at time $t$ and on output gradient $\boldsymbol{\epsilon}_t^{(o)}$ at time $t$ are defined as follows.
    \begin{gather*}
        \boldsymbol{\epsilon}^{(d)}:=\mathcal{Q}_d(\mathbf{x})-\mathbf{x},\quad \epsilon^{(l)}:=\mathcal{Q}_l(y)-y,\quad \boldsymbol{\epsilon}_t^{(p)}:=\mathcal{Q}_p(\mathbf{w}_t)-\mathbf{w}_t,\\ \boldsymbol{\epsilon}_t^{(a)}:=\mathcal{Q}_a(\mathbf{a}_t)-\mathbf{a}_t,\quad \boldsymbol{\epsilon}_t^{(o)}:=\mathcal{Q}_o(\mathbf{o}_t)-\mathbf{o}_t.
    \end{gather*}
\end{definition}

\subsection{Data model} 
We then state the regularity assumptions on the data distribution, which align with those common in prior works \citep{zou2023benign, lin2024scaling}. A key distinction in our setting is that the training process is performed on quantized data, i.e., $\mathcal{Q}_d(\mathbf{x})$ and $\mathcal{Q}_l(y)$. Consequently, we formulate these assumptions directly on the quantized data rather than the full-precision versions.
\begin{assumption}[Data covariance]
    \label{ass: addd quantized}
    Let $\mathbf H=\mathbb E[\mathbf x\mathbf x^\top]$ be the data covariance matrix and 
    \begin{equation*}
        \mathbf{H}^{(q)}:=\mathbb{E}[\mathcal{Q}_d(\mathbf{x})\mathcal{Q}_d(\mathbf{x})^\top],\quad \mathbf{D}:=\mathbb{E}[(\mathcal{Q}_d(\mathbf{x})-\mathbf{x})(\mathcal{Q}_d(\mathbf{x})-\mathbf{x})^\top],
    \end{equation*}
    be the covariance matrices of the quantized data feature and quantization error of data covariance, respectively. Then we assume that $\mathrm{tr}(\mathbf{H})$ and $\mathrm{tr}(\mathbf{H}^{(q)})$ are finite.
\end{assumption}

Further let $\mathbf{H}= \sum_i \lambda_i\mathbf{v}_i\mathbf{v}_i^\top$ be the eigen-decomposition of $\mathbf H$, where $\{\lambda_i\}_{i=1}^\infty$ are the eigenvalues of $\mathbf{H}$ sorted in non-increasing order and $\mathbf{v}_i$ are the corresponding eigenvectors. As in \citet{zou2023benign}, we denote
\begin{equation*}    \mathbf{H}_{0:k}:=\sum_{i=1}^k\lambda_i\mathbf{v}_i\mathbf{v}_i^\top,\quad\mathbf{H}_{k:\infty}:=\sum_{i>k}\lambda_i\mathbf{v}_i\mathbf{v}_i^\top,\quad \mathbf{I}_{0:k}:=\sum_{i=1}^k\mathbf{v}_i\mathbf{v}_i^\top,\quad\mathbf{I}_{k:\infty}:=\sum_{i>k}\mathbf{v}_i\mathbf{v}_i^\top.
\end{equation*}
Similarly, we denote the eigen-decomposition of $\mathbf{H}^{(q)}$ as $\mathbf{H}^{(q)}= \sum_i \lambda_i^{(q)}\mathbf{v}_i^{(q)}{\mathbf{v}_i^{(q)}}^\top$ and correspondingly obtain $\mathbf{H}_{0:k}^{(q)},\mathbf{H}_{k:\infty}^{(q)},\mathbf{I}_{0:k}^{(q)},\mathbf{I}_{k:\infty}^{(q)}$. We then extend the fourth moment and noise assumptions in \citet{zou2023benign,lin2024scaling} to the low-precision setting.
\begin{assumption}[Fourth-order moment]
\label{ass2}
    Let $\mathbf{x}^{(q)}=\mathcal{Q}_d(\mathbf{x})$. Then for any PSD matrix $\mathbf A$, there exists a constant $\alpha_B>0$ such that
    \begin{equation*}
        \mathbb{E}\left[{\mathbf{x}^{(q)}}{\mathbf{x}^{(q)}}^\top\mathbf{A}{\mathbf{x}^{(q)}}{\mathbf{x}^{(q)}}^\top\right]\preceq\alpha_B\operatorname{tr}(\mathbf{H}^{(q)}\mathbf{A})\mathbf{H}^{(q)}.
    \end{equation*}
\end{assumption}

To extend the model noise assumption in \citet{zou2023benign} to the low-precision setting, we define the optimal model weight regarding the quantized data feature and label: 
\begin{equation*}
    {\mathbf{w}^{(q)}}^* = \mathrm{argmin}_\mathbf{w} \ \mathbb{E}\left[(\mathcal{Q}_l(y)-\langle \mathbf{w},\mathcal{Q}_d(\mathbf{x})\rangle)^2\right].
\end{equation*}
Then we are ready to make the assumption on the model noise $\xi:=\mathcal{Q}_l(y)-\langle{\mathbf{w}^{(q)}}^*,\mathcal{Q}_d(\mathbf{x})\rangle$.
\begin{assumption}
    \label{ass3}
    Assume there exists a positive constant $\sigma>0$ such that
    \begin{equation*}
        \mathbb{E}\left[\xi^2\mathcal{Q}_d(\mathbf{x})\mathcal{Q}_d(\mathbf{x})^\top\right]\preceq\sigma^{2}\mathbf{H}^{(q)}.
    \end{equation*}
\end{assumption}
In fact, Assumptions \ref{ass2} and \ref{ass3} can be directly inferred from the standard assumptions on the full-precision data (Assumptions 2.1 and 2.2 in \citet{zou2023benign}) under specific quantization schemes. We defer the discussion to Section \ref{Discussion of Assumptions}.

\section{Main Theoretical Results}
We first derive excess risk upper bounds for quantized SGD in Section \ref{Excess Risk Bounds}, then compare these rates with the full-precision SGD (in orders) in Section \ref{Comparisons} and perform specific case study in Section \ref{Case Study}. 

\subsection{Excess Risk Bounds}
\label{Excess Risk Bounds}
We now provide excess risk bounds under general quantization, multiplicative quantization and additive quantization. Denote the effective dimension for $\mathbf{H}^{(q)}$: $k^*=\max\left\{k: \lambda_k^{(q)} \geq \frac{1}{N\gamma}\right\}$.
\begin{theorem}[\rm \textbf{General quantization}]
    \label{Theorem 4}
    Consider general quantization. Denote $\mathbf{D}_1^{\mathbf{H}}=\mathbf{D}(\mathbf{H}+\mathbf{D})^{-1}\mathbf{H}(\mathbf{H}+\mathbf{D})^{-1}\mathbf{D}$, $\mathbf{D}_2^{\mathbf{H}}=\mathbf{H}(\mathbf{H}^{(q)})^{-1}\frac{1}{N\gamma}\left(\mathbf{I}-(\mathbf{I}-\gamma\mathbf{H}^{(q)})^N\right)(\mathbf{H}^{(q)})^{-1}\mathbf{D}(\mathbf{H}^{(q)})^{-1}\mathbf{H}$. Under Assumption \ref{ass1}, \ref{ass: addd quantized}, \ref{ass2} and \ref{ass3}, if the stepsize $\gamma < \frac{1}{\alpha_B \mathrm{tr}(\mathbf{H}^{(q)})}$, then it holds,
    \begin{equation*}
        \mathbb{E}[\mathcal E(\overline{\mathbf{w}}_N)]\leq 2\mathrm{VarErr}+2\mathrm{BiasErr}+\mathrm{ApproxErr},
    \end{equation*}
    where
    \begin{equation*}
        \begin{aligned}
            &\mathrm{VarErr}\leq\frac{2\alpha_B\left(\frac{\|{\mathbf{w}^{(q)}}^*\|_{\mathbf{I}_{0:k^*}^{(q)}}^2}{N\gamma}+\|{\mathbf{w}^{(q)}}^*\|_{\mathbf{H}_{k^*:\infty}^{(q)}}^2\right)+{\sigma_G^{(q)}}^2}{1-\gamma\alpha_B\mathrm{tr}(\mathbf{H}^{(q)})}\left(\frac{k^*}{N}+N\gamma^2\cdot\sum_{i>k^*}(\lambda_i^{(q)})^2\right),\\
            &\mathrm{BiasErr}\leq\frac{1}{\gamma^2N^2}\cdot\|{\mathbf{w}^{(q)}}^*\|_{(\mathbf{H}_{0:k^*}^{(q)})^{-1}}^2+\|{\mathbf{w}^{(q)}}^*\|_{\mathbf{H}_{k^*:\infty}^{(q)}}^2,\\
            &\mathrm{ApproxErr}\leq\left\|\mathbf{w}^*\right\|_{\mathbf{D}_1^{\mathbf{H}}}^2+\|\mathbf{w}^*\|_{\mathbf{D}_2^\mathbf{H}}^2,
        \end{aligned}
    \end{equation*}
    with ${\sigma_G^{(q)}}^2=\frac{\sigma^2+\sup_t \left\{\left\|\mathbb{E}\left[\boldsymbol{\epsilon}_t^{(o)}{\boldsymbol{\epsilon}_t^{(o)}}^\top|\mathbf{o}_t\right]+\mathbb{E}\left[\boldsymbol{\epsilon}_t^{(a)}{\boldsymbol{\epsilon}_t^{(a)}}^\top|\mathbf{a}_t\right] \right\|\right\}}{B}+\alpha_B\sup_t\mathbb{E}\left[\mathrm{tr}\left(\mathbf{H}^{(q)} \boldsymbol{\epsilon}_{t-1}^{(p)}{\boldsymbol{\epsilon}_{t-1}^{(p)}}^\top\right)\right]$.
        
\end{theorem}

Theorem \ref{Theorem 4} establishes the first excess risk bound for quantized SGD under a general quantization paradigm. The excess risk is decomposed into three components: variance error, bias error, and approximation error. Notably, the variance and bias errors mirror those of full-precision SGD \citep{zou2023benign} and exact equivalence is recovered when the quantization error vanishes. The key role that quantization plays is two-fold: data quantization significantly influences the effective (quantized) data covariance $\mathbf{H}^{(q)}$, while activation, output gradient and parameter quantization amplify the effective noise variance $\sigma_G^{(q)}$ (which will be further characterized in the subsequent theorems when given specific quantization type). Specifically, the quantized data covariance arises from performing SGD in quantized data feature space and the quantized noise variance corresponds to additional quantization error introduced in the parameter update rule. We also note that the additional approximation error, resulting from quantization of data, can be interpreted as the discrepancy between the global optimum in full-precision data space and quantized data feature space. 

Crucially, in the absence of quantization, our excess risk bound reduces exactly to the standard results presented in \citet{zou2023benign}. It is also worth noting that under the unbiased quantization assumption, the quantization of parameter, output gradient, and activation do not affect bias error \footnote{For theoretical tractability and simplicity, our framework employs the unbiased quantization assumption (Assumption \ref{ass1}). Without this assumption, the conditional expectations of the parameter, output gradient, and activation quantization errors (i.e., quantization bias) would contribute to the bias error. We believe our framework is readily extendable to this general biased quantization setting.}.

To further elucidate the effects of quantization, we examine two specific schemes: multiplicative and additive quantization. The result for additive quantization can be derived directly from Theorem \ref{Theorem 4} and is summarized below.
\begin{corollary}[\rm \textbf{Additive quantization}]
    \label{theorem new 2.5}
    Under Assumption \ref{ass1}, \ref{ass: addd quantized}, \ref{ass2} and \ref{ass3}, if there exist $\epsilon_d,\epsilon_l,\epsilon_p,\epsilon_a$ and $\epsilon_o$ such that for any $i\in \{d,l,p,a,o\}$, quantization $\mathcal{Q}_i$ is $\epsilon_i$-additive, and the stepsize satisfies $\gamma < \frac{1}{\alpha_B [\mathrm{tr}(\mathbf{H})+d\epsilon_d]}$, then
    \begin{equation*}
        \mathbb{E}[\mathcal{E}(\overline{\mathbf{w}}_N)] \lesssim \mathrm{ApproxErr}+\mathrm{ VarErr}+\mathrm{BiasErr},
    \end{equation*}
    where
    \begin{align*}
            &\mathrm{ApproxErr}\lesssim \frac{\epsilon_d}{\lambda_d+\epsilon_d}\left\|\mathbf{w}^*\right\|_\mathbf{H}^2,\quad \mathrm{BiasErr}\lesssim\frac{1}{\gamma^2N^2}\cdot\|{\mathbf{w}^{(q)}}^*\|_{(\mathbf{H}_{0:k^*}^{(q)})^{-1}}^2+\|{\mathbf{w}^{(q)}}^*\|_{\mathbf{H}_{k^*:\infty}^{(q)}}^2,\\
            &\mathrm{VarErr}\lesssim \frac{\alpha_B\|\mathbf{w}^*\|_\mathbf{H}^2+\frac{\sigma^2+\epsilon_o+\epsilon_a}{B}+\alpha_B \epsilon_p[\mathrm{tr}(\mathbf{H})+d\epsilon_d]}{1-\gamma\alpha_B[\mathrm{tr}(\mathbf{H})+d\epsilon_d]}\left(\frac{k^*}{N}+N\gamma^2\cdot\sum_{i>k^*}(\lambda_i+\epsilon_d)^2\right).
    \end{align*}
\end{corollary}

Corollary \ref{theorem new 2.5} explicitly demonstrates how data quantization distorts effective data covariance spectrum and how parameter, activation and output gradient quantization amplify noise during training under additive quantization scheme. A key observation concerns the scaling with respect to the batch size $B$. Consistent with the label noise $\sigma^2$, the noise amplification from activation and output gradient quantization ($\epsilon_a, \epsilon_o$) are scaled by a factor of $1/B$. In contrast, the noise amplification from parameter quantization ($\epsilon_p$) scales with the trace of the quantized data covariance and is independent of batch size.

The interpretation is that additive quantization imposes a constant bound on the conditional second moment of the quantization error. Consequently, the underlying data structure inherent within the activation quantization error $\boldsymbol{\epsilon}_t^{(a)}$ and output gradient quantization error $\boldsymbol{\epsilon}_t^{(o)}$ is effectively neutralized. Formally, the noise amplification from these terms is characterized as $\frac{1}{B^2}\mathbb{E}[{\mathbf{X}^q}^\top \boldsymbol{\epsilon}\boldsymbol{\epsilon}^\top \mathbf{X}^q]$. Under additive quantization, since the error variance is bounded by a constant, the dependency on data within $\boldsymbol{\epsilon}$ vanishes. However, the noise amplification from parameter quantization, which is characterized as $\frac{1}{B^2}\mathbb{E}[{\mathbf{X}^q}^\top \mathbf{X}^q\boldsymbol{\epsilon}^{(p)}{\boldsymbol{\epsilon}^{(p)}}^\top{\mathbf{X}^q}^\top \mathbf{X}^q]$, preserves the underlying dependency on data, even if the error variance itself is constant.

Moreover, a critical consequence of additive quantization is the distortion of the data covariance spectrum $\mathbf{H}^{(q)}$. Specifically, a fixed constant $\epsilon_d$ is added across the entire spectrum, effectively imposing a noise floor that prevents the tail eigenvalues from decaying. This spectral flattening severely impedes learnability, as it leads to substantial risk accumulation within the high-dimensional tail subspace.


We next examine the multiplicative quantization scheme. Unlike additive quantization, multiplicative quantization exhibits an inherent structural alignment with the full-precision dynamics, as the error scales relative to the signal magnitude. Exploiting this property allows us to derive a refined excess risk bound through a direct analysis, rather than relying on a generic application of the general result in Theorem \ref{Theorem 4}. Our theoretical findings are summarized below.
\begin{theorem}[\rm \textbf{Multiplicative quantization}]
    \label{Theorem 1}
    Under Assumption \ref{ass1}, \ref{ass: addd quantized}, \ref{ass2} and \ref{ass3}, if there exist $\epsilon_d,\epsilon_l,\epsilon_p,\epsilon_a$ and $\epsilon_o$ such that for any $i\in \{d,l,p,a,o\}$, quantization $\mathcal{Q}_i$ is $\epsilon_i$-multiplicative, and the stepsize satisfies $\gamma < \frac{1}{ \alpha_B (1+\epsilon_o)[1+\epsilon_p+\epsilon_a(1+\epsilon_p)](1+\epsilon_d)\mathrm{tr}(\mathbf{H})}$, then the excess risk can be upper bounded as follows.
    \begin{equation*}
        \mathbb{E}[\mathcal{E}(\overline{\mathbf{w}}_N)] \lesssim \mathrm{ApproxErr}+\mathrm{VarErr}+\mathrm{BiasErr},
    \end{equation*}
    where
    \begin{equation*}
        \begin{aligned}
            &\mathrm{ApproxErr}\lesssim \frac{\epsilon_d}{1+\epsilon_d}\left\|\mathbf{w}^*\right\|_\mathbf{H}^2,\quad \mathrm{BiasErr}\lesssim\frac{1}{\gamma^2N^2}\cdot\|{\mathbf{w}^{(q)}}^*\|_{(\mathbf{H}_{0:k^*}^{(q)})^{-1}}^2+\|{\mathbf{w}^{(q)}}^*\|_{\mathbf{H}_{k^*:\infty}^{(q)}}^2,\\
            &\mathrm{VarErr}\lesssim\left(\frac{k^*}{N}+N\gamma^2(1+\epsilon_d)^2\sum_{i>k^*}\lambda_i^2\right) \frac{\frac{(1+\epsilon_o)\sigma^2}{B}+\alpha_B(1+\epsilon_o) [1+\epsilon_p+\epsilon_a(1+\epsilon_p)]\left\|\mathbf{w}^*\right\|_\mathbf{H}^2}{1-\gamma\alpha_B(1+\epsilon_o)[1+\epsilon_p+\epsilon_a(1+\epsilon_p)](1+\epsilon_d)\mathrm{tr}\left(\mathbf{H}\right)}.
        \end{aligned}
    \end{equation*}
\end{theorem}

Theorem \ref{Theorem 1} characterizes the spectrum distortion and noise amplification effects induced by multiplicative quantization. Notably, in stark contrast to additive quantization, which severely flattens the tail spectrum by imposing a constant floor, multiplicative quantization largely preserves the intrinsic spectral structure. Specifically, it acts as a linear transformation that scales the entire spectrum by a factor of $(1+\epsilon_d)$ without altering the relative distribution of eigenvalues. This preservation of the spectral decay property ensures superior learnability compared to the additive quantization scheme.

Regarding noise amplification, Theorem \ref{Theorem 1} reveals a critical divergence from the additive quantization scheme. While the contribution from intrinsic label noise ($\sigma^2$) is still suppressed by the batch size factor $1/B$, the quantization noise stemming from activation and output gradients ($\epsilon_a, \epsilon_o$) is coupled with the model parameter $\|\mathbf{w}^*\|_\mathbf{H}^2$ and does not scale with $1/B$. This phenomenon arises because multiplicative quantization error (scales proportionally with the signal strength) is inherently signal-dependent and is intrinsically tied to the data structure.

We provide further analysis of quantized SGD with quantized master weights in Section \ref{sec: master weight}. Training with quantized master weights necessitates stricter step size conditions to ensure convergence and introduces additional error terms into the excess risk bounds, thereby degrading generalization performance.

\subsection{Comparisons with Standard Excess Risk Bound}
\label{Comparisons}
In this part, we will provide a detailed comparison with standard excess risk bounds and identify the conditions on the quantization error such that the excess risk bound will not be largely affected. First, let $k_0^*=\max\{k: \lambda_k \geq \frac{1}{N\gamma}\}$, we recall the standard excess risk bound \citep{zou2023benign}: 
\begin{equation*}
    \begin{aligned}
        R_0=&\left(\frac{k_0^*}{N}+N\gamma^2\cdot\sum_{i>k_0^*}\lambda_i^2\right)\frac{\alpha_B\left(\frac{1}{N\gamma}\|{\mathbf{w}}^*\|_{\mathbf{I}_{0:k_0^*}}^2+\|{\mathbf{w}}^*\|_{\mathbf{H}_{k_0^*:\infty}}^2\right)
        +\frac{\sigma^2}{B}}{1-\gamma\alpha_B\mathrm{tr}\left(\mathbf{H}\right)}\\
        +&\frac{1}{\gamma^2N^2}\cdot\|{\mathbf{w}}^*\|_{(\mathbf{H}_{0:k_0^*})^{-1}}^2+\|{\mathbf{w}}^*\|_{\mathbf{H}_{k_0^*:\infty}}^2.
    \end{aligned}
\end{equation*}
The following corollary derives the conditions on the quantization errors such that the learning performance of the full-precision SGD can be maintained (in orders).

\begin{corollary}\label{coro:comparison}
To ensure that $\mathbb E[\mathcal E(\overline{\mathbf w}_N)]\lesssim R_0$, conditions on the quantization error are as follows:
\begin{itemize}[leftmargin=*,nosep]
\item For multiplicative quantization, under the conditions in Theorem \ref{Theorem 1}, we require
    \begin{equation*}
        \epsilon_d\lesssim 1\wedge \frac{R_0}{\|\mathbf{w}^*\|_{\mathbf{H}}^2},\quad\epsilon_o,\epsilon_a,\epsilon_p \lesssim  \left(\frac{\sigma^2}{B\alpha_B\|\mathbf{w}^*\|_\mathbf{H}^2}+\frac{\frac{1}{N\gamma}\|{\mathbf{w}}^*\|_{\mathbf{I}_{0:k_0^*}}^2+\|{\mathbf{w}}^*\|_{\mathbf{H}_{k_0^*:\infty}}^2}{\|\mathbf{w}^*\|_\mathbf{H}^2}\right)\wedge 1.
    \end{equation*}
\item For additive quantization, under the conditions in Corollary \ref{theorem new 2.5}, we require
\begin{gather*}
        \epsilon_d\lesssim \sqrt{\frac{\frac{k_0^*}{N}+N\gamma^2\cdot\sum_{i>k_0^*}\lambda_i^2}{N\gamma^2(d-k_0^*)}}\wedge \frac{R_0\lambda_d}{\|\mathbf{w}^*\|_{\mathbf{H}}^2},\quad \epsilon_a,\epsilon_o\lesssim \sigma^2+B\alpha_B\left(\frac{\|{\mathbf{w}}^*\|_{\mathbf{I}_{0:k_0^*}}^2}{N\gamma}+\|{\mathbf{w}}^*\|_{\mathbf{H}_{k_0^*:\infty}}^2\right), \\ 
        \epsilon_p\lesssim \frac{\sigma^2}{B\alpha_B[\mathrm{tr}(\mathbf{H})+d\epsilon_d]}+\frac{\frac{\|{\mathbf{w}}^*\|_{\mathbf{I}_{0:k_0^*}}^2}{N\gamma}+\|{\mathbf{w}}^*\|_{\mathbf{H}_{k_0^*:\infty}}^2}{\mathrm{tr}(\mathbf{H})+d\epsilon_d}.
    \end{gather*}
\end{itemize}
\end{corollary}

Corollary \ref{coro:comparison} identifies the conditions under which the quantized excess risk matches the full-precision baseline $R_0$. 
Regarding data quantization ($\epsilon_d$), the additive scheme imposes stringent spectrum-dependent constraints compared to the multiplicative quantization scheme. Specifically, the precision requirements are notably strict to prevent the constant quantization noise floor from overwhelming weak spectral components. Conversely, for activation and output gradient quantization ($\epsilon_a,\epsilon_o$), the additive scheme exhibits a favorable dependence on the batch size. As indicated by the scaling with $B$ in the bounds for $\epsilon_a$ and $\epsilon_o$, larger batch sizes effectively relax the precision requirements for these components. In contrast, larger batch sizes may essentially tighten the requirements under the multiplicative quantization scheme.

These findings validate our core insights: (1) multiplicative data quantization is superior in maintaining the spectral structure of $\mathbf{H}$, thus tolerating larger data quantization errors; (2) additive quantization benefits from the fact that the noise variance in activation and output gradient is independent of the signal magnitude, allowing these errors to be effectively suppressed by increasing the batch size.

\subsection{Case Study on Data Distribution with Polynomial-decay Spectrum}
\label{Case Study}
Following \cite{lin2024scaling,lin2025improved}, we study the excess risk bounds assuming optimal parameter prior and the power-law spectrum for more concise theoretical results. In particular, we make the following assumption.
\begin{assumption}
\label{assumption 1}
    There exists $a > 1$ such that the eigenvalues of $\mathbf{H}$ satisfy $\lambda_i \eqsim i^{-a},\  i >0$. We also assume that $\mathbb{E}\left[\mathbf{w}^*{\mathbf{w}^*}^\top\right]=\mathbf{I}$ and $\sigma^2 \lesssim 1$.
\end{assumption}
\begin{corollary}\label{coro: case study}
    Taking expectation on $\mathbf{w}^*$, under Assumption \ref{assumption 1}, we have:
    \begin{itemize}[leftmargin=*,nosep]
        \item For multiplicative quantization, under the conditions in Theorem \ref{Theorem 1},
        \begin{equation*}
            \mathbb{E}\left[\mathcal{E}(\overline{\mathbf{w}}_N)\right]\lesssim\frac{\epsilon_d}{1+\epsilon_d}+d^{1-a}+N^{1/a-1}(1+\epsilon_o) [1+\epsilon_p+\epsilon_a(1+\epsilon_p)](1+\epsilon_d)^{1/a}.
        \end{equation*}
        \item For additive quantization, under the conditions in Corollary \ref{theorem new 2.5},
        \begin{equation*}
            \begin{aligned}
                \mathbb{E}[\mathcal{E}(\overline{\mathbf{w}}_N)]\lesssim &\left(1+\frac{\epsilon_o+\epsilon_a}{B}+\epsilon_p(1+d\epsilon_d)\right)\left(1+\frac{(d^a\epsilon_d)^2}{1+d^a\epsilon_d}\right)d^{1-a}+\frac{d^a\epsilon_d}{1+d^a\epsilon_d}\\
                +&\left(1+\frac{\epsilon_o+\epsilon_a}{B}+\epsilon_p(1+d\epsilon_d)\right)(1+d^a\epsilon_d)^{1/a}N^{1/a-1}.
            \end{aligned}
        \end{equation*}
    \end{itemize}
\end{corollary}

Our findings in polynomial-decay data spectrum scenarios reveal distinct scaling behaviors under multiplicative and additive quantization. Specifically, the excess risk induced by additive data quantization exhibits a detrimental dependency on data dimension $d$, whereas the risk under multiplicative data quantization remains dimension-independent. This dependence has critical implications for learnability: in high-dimensional regimes ($d \to \infty$), the risk bound for additive quantization diverges, rendering the generalization guarantee vacuous. In contrast, the dimension-free nature of multiplicative quantization ensures its applicability even in infinite-dimensional settings.

Intuitively, this disparity stems from how each scheme interacts with the spectral structure. Multiplicative quantization preserves the intrinsic spectral decay, thereby retaining the utility of the effective dimension ($k^*$) cut-off. This allows the learning complexity to be controlled by the intrinsic data properties rather than the data dimension. Conversely, additive quantization employs a uniform quantization strength across all dimensions. This constant noise floor prevents the tail spectrum from decaying effectively and accumulates across the entire high-dimensional tail, rendering the effective dimension mechanism failed.


\textbf{Implications for integer and FP quantization.} 
These simplified theoretical results (Corollary \ref{coro: case study}) draw critical implications for comparing integer and floating-point (FP) quantization, allowing us to identify the conditions under which each type yields superior performance. Specifically, in practical integer quantization with bit-width $b$ and FP quantization with mantissa bit-width $m$, the quantization step size for a value $x$ are approximately $\delta(x)\eqsim2^{-b}$ and $\delta(x)=|x|2^{-m}$ \footnote{We assume the exponent bits in FP quantization can cover the scaling of $x$. For integer quantization, we assume the dynamic range (i.e., $x_{\max}-x_{\min}$) is normalized to constant level.}, respectively. Since the conditional second moment of quantization error $\mathbb{E}[(\mathcal{Q}(x)-x)^2|x]$ is roughly proportional to the square of the quantization step size ($\delta(x)^2$), the quantization error parameters in our bounds can be characterized as $\epsilon_{\mathrm{add}} \approx 2^{-2b}$ for the additive (integer) quantization scheme and $\epsilon_{\mathrm{mul}} \approx 2^{-2m}$ for the multiplicative (FP) quantization scheme.

Equipped with this mapping, practitioners can directly apply Corollary \ref{coro: case study} to determine the optimal quantization scheme for specific scenarios. A notable observation concerns the distinct role of the dimension $d$ in data quantization. Roughly, FP quantization becomes preferable when $m_d\geq b_d-\frac{a}{2}\log_2 d$ whereas integer quantization is favored when $b_d \geq m_d+\frac{a}{2}\log_2 d$ \footnote{Here $b_d$ and $m_d$ are the bit-width for integer data quantization and the mantissa bit-width for FP data quantization respectively.}. This means FP quantization can outperform integer quantization even when its mantissa bit-width is smaller than the integer bit-width by $\frac{a}{2}\log_2 d$, highlighting the advantage of FP quantization in high-dimensional settings.

\begin{figure}[t!]
  \centering
  \subfigure[\textbf{Multiplicative} (FP-like)]{
    \includegraphics[width=.24\textwidth]{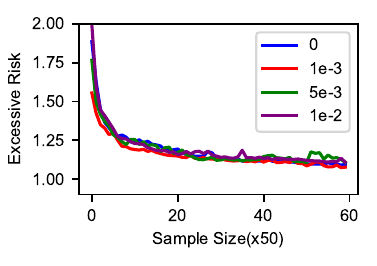}
  }\hspace{-.15in}
  \subfigure[\textbf{Additive} (INT-like) ]{
    \includegraphics[width=.24\textwidth]{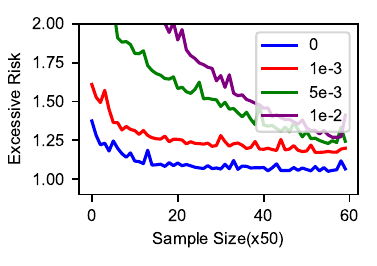}
  }\hspace{-.15in}
  \subfigure[\textbf{Multiplicative} (FP-like)]{
    \includegraphics[width=.24\textwidth]{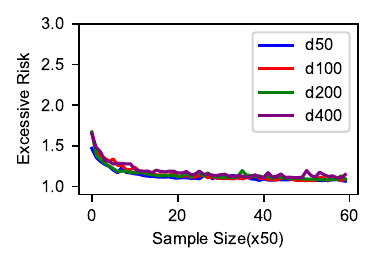}
  }\hspace{-.15in}
  \subfigure[\textbf{Additive} (INT-like)]{
    \includegraphics[width=.24\textwidth]{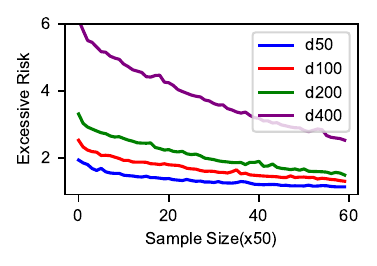}
  }
  \vskip -.15in
  \caption{\textbf{Generalization under quantization.}
  Population risk ($\mathbb{E}_{\mathbf{x},y}[(y-\langle \mathbf{w},\mathbf{x} \rangle )^2]$) for quantized SGD with iterate averaging under multiplicative (FP-like) vs.\ additive (INT-like) quantization.
(a) and (b): vary the quantization level at fixed dimension.
(c) and (d): vary dimension at fixed quantization level.}\vspace{-.2in}
  \label{fig:generalization_perf}
\end{figure}

\paragraph{Numerical experiments.}
\label{Empirical Study}
We evaluate constant–stepsize SGD with iterate averaging on a Gaussian least–squares model.
The feature distribution has covariance matrix with eigenvalues
$\lambda_i=i^{-2}$.
The ground–truth parameter is $\mathbf{w}^*$ with entries $\mathbf{w}^*[i]=1$,
and the observation noise variance is $\sigma^2=1$. This study answers two questions: \textbf{Q1}: How do \emph{additive} vs. \emph{multiplicative} quantization errors affect learning? \textbf{Q2}: How does \emph{dimension} $d$ interact with these two quantization types?

\noindent \textbf{Q1 (Quantization level).}
We fix $d=200$ and $B=1$, and vary the quantization error level
$\varepsilon\in\{0.001,\,0.005,\,0.01\}$ for each scheme.
Results are shown in Fig.~\ref{fig:generalization_perf}(a,b). Under multiplicative quantization, quantized SGD largely retains the generalization performance of full-precision SGD across a wide range of quantization levels. Conversely, under additive quantization, performance degrades as the quantization level increases. These empirical observations validate our theoretical findings: with a batch size of $B=1$, additive quantization requires stricter conditions (lower quantization level) to match the performance of full-precision SGD.


\noindent \textbf{Q2 (Dimension).}
We fix the quantization level at $\varepsilon=0.01$ and $B=1$,
and vary $d\in\{50,\,100,\,200,\,400\}$.
Results are shown in Fig.~\ref{fig:generalization_perf}(c,d). Under multiplicative quantization, generalization performance is preserved even in high-dimensional settings; conversely, under additive quantization, performance deteriorates as the data dimension increases. These empirical results corroborate our theoretical findings: multiplicative quantization remains effective in high-dimensional contexts, whereas additive quantization is ill-suited for such scenarios.

Furthermore, we conduct additional experiments on the real-world \texttt{Communities and Crime} dataset, as well as settings with larger batch sizes and exponential-decay spectra. These results, presented in Section \ref{sec: exp}, consistently align with our theoretical analysis.

%% file: 4_proof_sketch.tex
\section{Conclusion and Limitations}
In this work, we presented a comprehensive theoretical framework to analyze the excess risk of quantized SGD in high-dimensional linear regression. Our analysis disentangles the distinct impacts of various quantization targets: while parameter, activation, and gradient quantization primarily serve as noise amplifiers, data quantization fundamentally distorts the effective feature covariance spectrum. Crucially, we show that multiplicative quantization excels at preserving the spectral structure of the data, thereby maintaining learnability even in high-dimensional settings. In contrast, additive quantization leverages the independence of noise variance from signal magnitude, allowing activation and gradient noise to be effectively suppressed by large batch sizes. Furthermore, our theory establishes the conditions on quantization errors required to maintain full-precision SGD performance, and identifies the scenarios under which FP and integer quantization are each likely to yield superior performance under polynomial decay spectrum.

\textbf{Future work.} Our work lays a solid foundation for several promising research avenues. Firstly, developing a lower bound analysis for the excess risk of quantized SGD. Secondly, extending single-pass SGD to more practical training configurations, such as data reuse (i.e., multi-pass SGD), learning rate scheduling, momentum, and preconditioning. Thirdly, extending training the linear models to the training of over-parameterized neural networks. Fourthly, deriving scaling laws for low-precision training.

\section*{Acknowledgments} We would like to thank the anonymous reviewers and area chairs for their helpful comments. We acknowledge the support from NSFC 62306252, Hong Kong ECS award 27309624, Guangdong NSF
2024A1515012444, and the central fund from HKU IDS.



%% file: 5_appendix.tex
\section*{Appendix}
\tableofcontents
\resettocdepth{}
\newpage
The appendix is organized as follows. In Section \ref{Initial Study}, we begin the analysis of excess risk bounds for the iteratively averaged quantized SGD by firstly deriving the update rule for the parameter deviation $\mathbf{w}_t - {\mathbf{w}^{(q)}}^*$ (detailed in Section \ref{Deviation of the Update Rule}) and secondly performing an excess risk decomposition (detailed in Section \ref{Decomposition of the Excess Risk}):
\begin{equation*}
    \mathbb{E}[\mathcal{E}(\overline{\mathbf{w}}_N)]=\underbrace{\frac{1}{2}\langle\mathbf{H},\mathbb{E}[\overline{\boldsymbol{\eta}}_N\otimes\overline{\boldsymbol{\eta}}_N]\rangle}_{R_N}+\mathrm{ApproxErr}.
\end{equation*}
We then conduct a refined analysis for $\mathrm{ApproxErr}$ in Sections \ref{analyze R4}. 
For $R_N$, we extend techniques from \citet{zou2023benign} in Section \ref{analyze R2}. In particular, we first introduce useful notations in Section \ref{Preliminary} and then present a comprehensive analysis of the update rule for $\mathbb{E}[\boldsymbol{\eta}_t \boldsymbol{\eta}_t^\top]$ in Section \ref{Initial study R2}. This analysis is crucial for adapting previous proof techniques to the quantized SGD setting. Based on these results, we perform a bias–variance decomposition in Section \ref{Bias-Variance Decomposition}, and analyze the bias and variance errors separately in Section \ref{Bounding the Bias Error} and \ref{Bounding the Variance Error}. In Section \ref{sec: master weight}, we include bounds when master weight is quantized.

The following proof dependency graph visually encapsulates the logical structure and organizational architecture of the theoretical results in our paper. In particular, the arrow from element $X$ to element $Y$ means the proof of $Y$ relies on $X$.

\vspace{0.5cm}
\begin{adjustbox}{max width=\textwidth, center}
\begin{tikzpicture}[
    node distance=1.5cm and 3cm, 
    theorem/.style={rectangle, draw=blue!80, fill=blue!8, very thick, minimum width=6cm, minimum height=1.5cm, align=center, rounded corners=3pt, font=\Huge\bfseries\sffamily},
    mtheorem/.style={rectangle, draw=red!70!black, fill=red!12, very thick, minimum width=6cm, minimum height=1.5cm, align=center, rounded corners=3pt, font=\Huge\bfseries\sffamily},
    lemma/.style={rectangle, draw=green!70!black, fill=green!8, thick, minimum width=6cm, minimum height=1.5cm, align=center, rounded corners=3pt, font=\Huge\bfseries\sffamily},
    aux/.style={rectangle, draw=orange!80!black, fill=orange!8, thick, minimum width=6cm, minimum height=1.5cm, align=center, rounded corners=3pt, font=\Huge\bfseries\sffamily},
    prop/.style={rectangle, draw=purple!80, fill=purple!8, thick, minimum width=6cm, minimum height=1.5cm, align=center, rounded corners=3pt, font=\Huge\bfseries\sffamily},
    mArrow/.style={-Stealth, line width=1.5pt, red!60!black},
    aArrow/.style={-Stealth, line width=1.5pt, blue!50},
    bArrow/.style={-Stealth, line width=1.5pt, green!60!black, opacity=0.8},
    cArrow/.style={-Stealth, line width=1.5pt, orange!70!black, opacity=1}
]

\node[mtheorem] (C43) at (11, 8) {Corollary \ref{coro: case study}};
\node[mtheorem] (C41) at (21, 4) {Corollary \ref{theorem new 2.5}};
\node[mtheorem] (T42) at (1, 4) {Theorem \ref{Theorem 1}};

\node[mtheorem] (C42) at (11, 4) {Corollary \ref{coro:comparison}};
\node[mtheorem] (TD1) at (1, 0) {Theorem \ref{Theorem 1, proof}};
\node[mtheorem] (CD1) at (21, 0) {Corollary \ref{theorem new 2.5, proof}};

\node[mtheorem] (T41) at (25, -4) {Theorem \ref{Theorem 4}};

\node[theorem] (LC16) at (25, -8) {Lemma \ref{R_2 bound}};
\node[theorem] (LC17) at (-3, -8) {Lemma \ref{R_2 bound multi}};
\node[theorem] (LB1) at (5, -8) {Lemma \ref{lem: b1}};
\node[theorem] (LB2) at (17, -8) {Lemma \ref{lem: b2}};

\node[aux] (LA3) at (11, -4) {Lemma \ref{Refine excess risk decomposition}};

\node[lemma] (LC11) at (5, -12) {Lemma \ref{bias R2 bound multiplicative}};
\node[lemma] (LC15) at (-3, -12) {Lemma \ref{variance R2 bound multipicative}};

\node[lemma] (LC10) at (17, -12) {Lemma \ref{bias R2 bound}};
\node[lemma] (LC14) at (25, -12) {Lemma \ref{variance R2 bound}};

\node[lemma] (LC13) at (-8, -16) {Lemma \ref{A bound for M_B^{(q)} C_t^{(M)}}};
\node[lemma] (LC7) at (-0.5, -16) {Lemma \ref{(Initial Study of $S_t$) mu}};
\node[lemma] (LC9) at (7, -16) {Lemma \ref{A bound for M_B^{(q)} S_t^{(M)}}};
\node[lemma] (LC8) at (15, -16) {Lemma \ref{A bound for M_B^{(q)} S_t}};
\node[lemma] (LC6) at (22.5, -16) {Lemma \ref{(Initial Study of $S_t$)}};
\node[lemma] (LC12) at (30, -16) {Lemma \ref{A bound for M_B^{(q)} C_t}};

\node[aux] (LC5) at (1, -20) {Lemma \ref{Bias-Variance Decomposition under Multiplicative Quantization}};
\node[aux] (LC4) at (21, -20) {Lemma \ref{Bias-Variance Decomposition under General Quantization}};

\node[aux] (LC1) at (11, -24) {Lemma \ref{Initial study of R_2}};
\node[aux] (LC2) at (21, -24) {Lemma \ref{update rule for eta_t^2}};
\node[aux] (LC3) at (1, -24) {Lemma \ref{update rule for eta_t^2 multiplicative}};

\node[aux] (LA1) at (11, -28) {Lemma \ref{update rule for eta_t}};

\draw[cArrow] (LA1) -- (LC1);
\draw[cArrow] (LA1) -- (LC2);
\draw[cArrow] (LA1) -- (LC3);

\draw[cArrow] (LC1) -- (LC4);
\draw[cArrow] (LC1) -- (LC5);

\draw[cArrow] (LC2) -- (LC4);
\draw[cArrow] (LC3) -- (LC5);

\draw[cArrow] (LA3) -- (LC1);
\draw[aArrow] (LA3) -- (LB1);
\draw[aArrow] (LA3) -- (LB2);

\draw[bArrow] (LC5) -- (LC7);
\draw[bArrow] (LC5) -- (LC9);
\draw[bArrow] (LC5) -- (LC13);

\draw[bArrow] (LC4) -- (LC6);
\draw[bArrow] (LC4) -- (LC8);
\draw[bArrow] (LC4) -- (LC12);

\draw[bArrow] (LC13) -- (LC15);
\draw[bArrow] (LC7) -- (LC11);
\draw[bArrow] (LC9) -- (LC11);

\draw[bArrow] (LC12) -- (LC14);
\draw[bArrow] (LC6) -- (LC10);
\draw[bArrow] (LC8) -- (LC10);

\draw[aArrow] (LC15) -- (LC17);
\draw[aArrow] (LC11) -- (LC17);

\draw[aArrow] (LC10) -- (LC16);
\draw[aArrow] (LC14) -- (LC16);

\draw[mArrow] (LC16) -- (T41);

\draw[mArrow] (LC17) -- (TD1);
\draw[mArrow] (LB1) -- (TD1);

\draw[mArrow] (LB2) -- (CD1);
\draw[mArrow] (T41) -- (CD1);

\draw[mArrow] (CD1) -- (C42);
\draw[mArrow] (TD1) -- (C42);

\draw[mArrow] (TD1) -- (T42);
\draw[mArrow] (CD1) -- (C42);

\draw[mArrow] (T42) -- (C43);
\draw[mArrow] (C41) -- (C43);

\end{tikzpicture}
\end{adjustbox}

\newpage

\section{Initial Study}
\label{Initial Study}
For simplicity, we denote ${y}^{(q)}=\mathcal{Q}_l({y}), \mathbf{w}_t^{(q)}=\mathcal{Q}_p(\mathbf{w}_t),\mathbf{x}^{(q)}=\mathcal{Q}_d(\mathbf{x})$. For convenience, we assume that $\mathbf{H}$ is strictly positive definite and that $L(\mathbf{w})$ admits a unique global optimum as \citet{zou2023benign}. We first recall the definition of the global minima $\mathbf{w}^*$ and ${\mathbf{w}^{(q)}}^*$:
\begin{equation*}
    {\mathbf{w}}^* = \mathrm{argmin}_\mathbf{w} \ \mathbb{E}\left[(y-\langle \mathbf{w},\mathbf{x}\rangle)^2\right], \quad {\mathbf{w}^{(q)}}^* = \mathrm{argmin}_\mathbf{w} \ \mathbb{E}\left[(\mathcal{Q}_l(y)-\langle \mathbf{w},\mathcal{Q}_d(\mathbf{x})\rangle)^2\right].
\end{equation*}
The first order optimality shows that
\begin{equation}
    \label{definition w*}\mathbb{E}[(y-\langle{\mathbf{w}}^*,\mathbf{x}\rangle)\mathbf{x}]=\mathbf{0},\quad \mathbb{E}[(\mathcal{Q}_l(y)-\langle{\mathbf{w}^{(q)}}^*,\mathcal{Q}_d(\mathbf{x})\rangle)\mathcal{Q}_d(\mathbf{x})]=\mathbf{0},
\end{equation}
which implies that
\begin{equation*}
    \mathbf{w}^*=\mathbf{H}^{-1}\mathbb{E}_{(\mathbf{x},y)\sim\mathcal{D}}[y\mathbf{x}], \quad {\mathbf{w}^{(q)}}^*=(\mathbf{H}^{(q)})^{-1}\mathbb{E}\left[\mathcal{Q}_l(y)\mathcal{Q}_d(\mathbf{x})\right]=(\mathbf{H}^{(q)})^{-1}\mathbb{E}_{(\mathbf{x},y)\sim\mathcal{D}}[y\mathbf{x}].
\end{equation*}
Hence, by denoting $\mathbf{H}^{(q)}=\mathbf{H}+\mathbf{D}$, we can characterize the difference between $\mathbf{w}^{(q)^*}$ and $\mathbf{w}^*$ as:
\begin{equation}
\label{wq-w*}
    \begin{aligned}
        {\mathbf{w}^{(q)}}^*-\mathbf{w}^*=&\left[(\mathbf{H}^{(q)})^{-1}-\mathbf{H}^{-1}\right]\mathbb{E}_{(\mathbf{x},y)\sim\mathcal{D}}[y\mathbf{x}]\\
        =&(\mathbf{H}^{(q)})^{-1}\left(\mathbf{H}-\mathbf{H}^{(q)}\right)\mathbf{H}^{-1}\mathbb{E}_{(\mathbf{x},y)\sim\mathcal{D}}[y\mathbf{x}]\\
        =&(\mathbf{H}^{(q)})^{-1}\left(\mathbf{H}-\mathbf{H}^{(q)}\right)\mathbf{w}^*\\
        =&-(\mathbf{H}^{(q)})^{-1}\mathbf{D}\mathbf{w}^*\\
        =&-\left(\mathbf{H}+\mathbf{D}\right)^{-1}\mathbf{D}\mathbf{w}^*.
    \end{aligned}
\end{equation}
\subsection{Deviation of the Update Rule}
\label{Deviation of the Update Rule}
In this section, we derive the evolution of parameter deviation $\boldsymbol{\eta}_{t}:=\mathbf{w}_{t}-{\mathbf{w}^{(q)}}^*$.
\begin{lemma}[\rm{Error propagation}]
\label{update rule for eta_t}
    \begin{equation*}
        \boldsymbol{\eta}_t = \left(\mathbf{I}-\frac{1}{B}\gamma {\mathcal{Q}_d(\mathbf{X}_t)}^\top \mathcal{Q}_d(\mathbf{X}_t) \right)\boldsymbol{\eta}_{t-1}+\gamma\frac{1}{B}{\mathcal{Q}_d(\mathbf{X}_t)}^\top \left[\boldsymbol{\xi}_t+\boldsymbol{\epsilon}_t^{(o)}-\boldsymbol{\epsilon}_t^{(a)}-\mathcal{Q}_d(\mathbf{X}_t)\boldsymbol{\epsilon}_{t-1}^{(p)}\right],
    \end{equation*}
    where the quantization errors are
    \begin{equation*}
        \begin{aligned}
            \boldsymbol{\epsilon}_t^{(o)}:=&\mathcal{Q}_o\left(\mathcal{Q}_l(\mathbf{y}_t)-\mathcal{Q}_a\left(\mathcal{Q}_d(\mathbf{X}_t)\mathcal{Q}_p(\mathbf{w}_{t-1})\right)\right)-\left[\mathcal{Q}_l(\mathbf{y}_t)-\mathcal{Q}_a\left(\mathcal{Q}_d(\mathbf{X}_t)\mathcal{Q}_p(\mathbf{w}_{t-1})\right)\right],\\
            \boldsymbol{\epsilon}_t^{(a)}:=&\mathcal{Q}_a\left(\mathcal{Q}_d(\mathbf{X}_t)\mathcal{Q}_p(\mathbf{w}_{t-1})\right)-\mathcal{Q}_d(\mathbf{X}_t)\mathcal{Q}_p(\mathbf{w}_{t-1}),\\
            \boldsymbol{\epsilon}_{t-1}^{(p)}:=&\mathcal{Q}_p(\mathbf{w}_{t-1})-\mathbf{w}_{t-1},\\
            \boldsymbol{\xi}_t:=&\mathcal{Q}_l(\mathbf{y}_t)-\mathcal{Q}_d(\mathbf{X}_t){\mathbf{w}^{(q)}}^*.
        \end{aligned}
    \end{equation*}
\end{lemma}

\begin{proof}
    The lemma can be proved directly by the parameter update rule. By definition and the update rule of $\mathbf{w}_t$ (\ref{eq:SGD_quantized}),
    \begin{equation*}
        \begin{aligned}
            \boldsymbol{\eta}_t=&\mathbf{w}_t - {\mathbf{w}^{(q)}}^*\\
            =&\mathbf{w}_{t-1}-{\mathbf{w}^{(q)}}^*+\gamma \frac{1}{B}\mathcal{Q}_d(\mathbf{X}_t)^\top \mathcal{Q}_o\left(\mathcal{Q}_l(\mathbf{y}_t)-\mathcal{Q}_a\left(\mathcal{Q}_d(\mathbf{X}_t)\mathcal{Q}_p(\mathbf{w}_{t-1})\right)\right)\\
            =&\boldsymbol{\eta}_{t-1}+\gamma \frac{1}{B}\mathcal{Q}_d(\mathbf{X}_t)^\top \mathcal{Q}_o\left(\mathcal{Q}_l(\mathbf{y}_t)-\mathcal{Q}_a\left(\mathcal{Q}_d(\mathbf{X}_t)\mathcal{Q}_p(\mathbf{w}_{t-1})\right)\right).
        \end{aligned}
    \end{equation*}
    We then introduce quantization errors to better characterize each quantization operation $\mathcal{Q}(\cdot)$. In particular, define quantization erros:
    \begin{equation*}
        \begin{aligned}
            \boldsymbol{\epsilon}_t^{(o)}:=&\mathcal{Q}_o\left(\mathcal{Q}_l(\mathbf{y}_t)-\mathcal{Q}_a\left(\mathcal{Q}_d(\mathbf{X}_t)\mathcal{Q}_p(\mathbf{w}_{t-1})\right)\right)-\left[\mathcal{Q}_l(\mathbf{y}_t)-\mathcal{Q}_a\left(\mathcal{Q}_d(\mathbf{X}_t)\mathcal{Q}_p(\mathbf{w}_{t-1})\right)\right],\\
            \boldsymbol{\epsilon}_t^{(a)}:=&\mathcal{Q}_a\left(\mathcal{Q}_d(\mathbf{X}_t)\mathcal{Q}_p(\mathbf{w}_{t-1})\right)-\mathcal{Q}_d(\mathbf{X}_t)\mathcal{Q}_p(\mathbf{w}_{t-1}),\\
            \boldsymbol{\epsilon}_{t-1}^{(p)}:=&\mathcal{Q}_p(\mathbf{w}_{t-1})-\mathbf{w}_{t-1},\\
            \boldsymbol{\xi}_t:=&\mathcal{Q}_l(\mathbf{y}_t)-\mathcal{Q}_d(\mathbf{X}_t){\mathbf{w}^{(q)}}^*.
        \end{aligned}
    \end{equation*}
    Then the update rule for the parameter deviation can be expressed as:
    \begin{equation*}
        \begin{aligned}
            \boldsymbol{\eta}_t = &\boldsymbol{\eta}_{t-1}+\gamma \frac{1}{B}\mathcal{Q}_d(\mathbf{X}_t)^\top \mathcal{Q}_o\left(\mathcal{Q}_l(\mathbf{y}_t)-\mathcal{Q}_a\left(\mathcal{Q}_d(\mathbf{X}_t)\mathcal{Q}_p(\mathbf{w}_{t-1})\right)\right)\\
            =&\boldsymbol{\eta}_{t-1}+\gamma \mathcal{Q}_d(\mathbf{X}_t)^\top \frac{1}{B}\left[\mathcal{Q}_l(\mathbf{y}_t)-\mathcal{Q}_a\left(\mathcal{Q}_d(\mathbf{X}_t)\mathcal{Q}_p(\mathbf{w}_{t-1})\right)\right]+\gamma \frac{1}{B}\mathcal{Q}_d(\mathbf{X}_t)^\top \boldsymbol{\epsilon}_t^{(o)}\\
            =&\boldsymbol{\eta}_{t-1}+\gamma \mathcal{Q}_d(\mathbf{X}_t)^\top \frac{1}{B}\left[\mathcal{Q}_l(\mathbf{y}_t)-\mathcal{Q}_d(\mathbf{X}_t)\mathcal{Q}_p(\mathbf{w}_{t-1})\right]+\gamma \frac{1}{B}\mathcal{Q}_d(\mathbf{X}_t)^\top (\boldsymbol{\epsilon}_t^{(o)}-\boldsymbol{\epsilon}_t^{(a)})\\
            =&\boldsymbol{\eta}_{t-1}+\gamma \frac{1}{B}\mathcal{Q}_d(\mathbf{X}_t)^\top (\boldsymbol{\epsilon}_t^{(o)}-\boldsymbol{\epsilon}_t^{(a)})+\gamma \mathcal{Q}_d(\mathbf{X}_t)^\top \frac{1}{B}\\
            &\left[\mathcal{Q}_l(\mathbf{y}_t)-\mathcal{Q}_d(\mathbf{X}_t){\mathbf{w}^{(q)}}^*-\mathcal{Q}_d(\mathbf{X}_t)\boldsymbol{\eta}_{t-1}-\mathcal{Q}_d(\mathbf{X}_t)\mathcal{Q}_p(\mathbf{w}_{t-1})+\mathcal{Q}_d(\mathbf{X}_t)\mathbf{w}_{t-1}\right]\\
            =&\boldsymbol{\eta}_{t-1}+\gamma \mathcal{Q}_d(\mathbf{X}_t)^\top (\boldsymbol{\epsilon}_t^{(o)}-\boldsymbol{\epsilon}_t^{(a)})+\gamma \mathcal{Q}_d(\mathbf{X}_t)^\top \frac{1}{B}\\
            &\left[\mathcal{Q}_l(\mathbf{y}_t)-\mathcal{Q}_d(\mathbf{X}_t){\mathbf{w}^{(q)}}^*-\mathcal{Q}_d(\mathbf{X}_t)\boldsymbol{\eta}_{t-1}-\mathcal{Q}_d(\mathbf{X}_t)\boldsymbol{\epsilon}_{t-1}^{(p)}\right]\\
            =&\boldsymbol{\eta}_{t-1}+\gamma \frac{1}{B}\mathcal{Q}_d(\mathbf{X}_t)^\top (\boldsymbol{\epsilon}_t^{(o)}-\boldsymbol{\epsilon}_t^{(a)}+\boldsymbol{\xi}_t)-\gamma \mathcal{Q}_d(\mathbf{X}_t)^\top \frac{1}{B}\left[\mathcal{Q}_d(\mathbf{X}_t)\boldsymbol{\eta}_{t-1}+\mathcal{Q}_d(\mathbf{X}_t)\boldsymbol{\epsilon}_{t-1}^{(p)}\right]\\
            =&\left(\mathbf{I}-\frac{1}{B}\gamma {\mathcal{Q}_d(\mathbf{X}_t)}^\top \mathcal{Q}_d(\mathbf{X}_t) \right)\boldsymbol{\eta}_{t-1}+\gamma\frac{1}{B}{\mathcal{Q}_d(\mathbf{X}_t)}^\top \left[\boldsymbol{\xi}_t+\boldsymbol{\epsilon}_t^{(o)}-\boldsymbol{\epsilon}_t^{(a)}-\mathcal{Q}_d(\mathbf{X}_t)\boldsymbol{\epsilon}_{t-1}^{(p)}\right].
        \end{aligned}
    \end{equation*}
\end{proof}

\subsection{Decomposition of the Excess Risk}
\label{Decomposition of the Excess Risk}
In this section, we take the initial step to analyze the excess risk of averaged SGD iterate $\overline{\mathbf{w}}_N$. In particular, we define the deviation of the averaged SGD iterate as $\overline{\boldsymbol{\eta}}_{N}:=\frac{1}{N}\sum_{t=0}^{N-1}\boldsymbol{\eta}_{t}$. We decompose the excess risk as follows.
\begin{lemma}[\rm Excess risk decomposition]
\label{Excess Risk Decomposition}
    Under Assumption \ref{ass1} and Assumption \ref{ass: addd quantized},
    \begin{equation*}
        \begin{aligned}
            \mathbb{E}[\mathcal{E}(\overline{\mathbf{w}}_N)]=R_1+R_2+R_3+R_4,
        \end{aligned}
    \end{equation*}
    where
    \begin{equation*}
        \begin{aligned}
            R_1=&-\frac{1}{2}\mathbb{E}\left[\langle\overline{\mathbf{w}}_N, \mathcal{Q}_d(\mathbf{x})-\mathbf{x} \rangle^2\right],\\
            R_2=&\frac{1}{2}\langle\mathbf{H}^{(q)},\mathbb{E}[\overline{\boldsymbol{\eta}}_N\otimes\overline{\boldsymbol{\eta}}_N]\rangle,\\
            R_3=&\frac{1}{2}\mathbb{E}\left[\langle {\mathbf{w}^{(q)}}^*,\mathcal{Q}_d(\mathbf{x})-\mathbf{x}\rangle^2\right],\\
            R_4=&\frac{1}{2}\left\langle\mathbf{H},({\mathbf{w}}^*-{\mathbf{w}^{(q)}}^*)\otimes({\mathbf{w}}^*-{\mathbf{w}^{(q)}}^*)\right\rangle.
        \end{aligned}
    \end{equation*}
\end{lemma}
\begin{proof}
    By the definition of the excess risk (\ref{eq:excess_risk}),
    \begin{equation*}
        \begin{aligned}
            \mathbb{E}[\mathcal{E}(\overline{\mathbf{w}}_N)]
            =&\frac{1}{2}\mathbb{E}\left[\left(y-\langle \overline{\mathbf{w}}_N,\mathbf{x} \rangle \right)^2\right]-\frac{1}{2}\mathbb{E}\left[\left(y-\langle \mathbf{w}^*,\mathbf{x} \rangle \right)^2\right]\\
            =&\underbrace{\frac{1}{2}\mathbb{E}\left[\left(y-\langle \overline{\mathbf{w}}_N,\mathbf{x} \rangle \right)^2\right]-\frac{1}{2}\mathbb{E}\left[(\mathcal{Q}_l(y)-\langle \overline{\mathbf{w}}_N,\mathcal{Q}_d(\mathbf{x})\rangle)^2\right]}_{E_1}\\
            +&\underbrace{\frac{1}{2}\mathbb{E}\left[(\mathcal{Q}_l(y)-\langle \overline{\mathbf{w}}_N,\mathcal{Q}_d(\mathbf{x})\rangle)^2\right]-\frac{1}{2}\mathbb{E}\left[(\mathcal{Q}_l(y)-\langle {\mathbf{w}^{(q)}}^*,\mathcal{Q}_d(\mathbf{x})\rangle)^2\right]}_{E_2}\\
            +&\underbrace{\frac{1}{2}\mathbb{E}\left[(\mathcal{Q}_l(y)-\langle {\mathbf{w}^{(q)}}^*,\mathcal{Q}_d(\mathbf{x})\rangle)^2\right]-\frac{1}{2}\mathbb{E}\left[\left(y-\langle {\mathbf{w}^{(q)}}^*,\mathbf{x} \rangle \right)^2\right]}_{E_3}\\
            +&\underbrace{\frac{1}{2}\mathbb{E}\left[\left(y-\langle {\mathbf{w}^{(q)}}^*,\mathbf{x} \rangle \right)^2\right]-\frac{1}{2}\mathbb{E}\left[\left(y-\langle \mathbf{w}^*,\mathbf{x} \rangle \right)^2\right]}_{E_4},
        \end{aligned}
    \end{equation*}
    where $E_1$ captures the gap of the averaged SGD iterate between the full-precision and quantized domains, $E_2$ characterizes the distance from the averaged SGD iterate to the quantized optimal solution within the quantized domain, $E_3$ represents the mismatch of the quantized optimal solution in full-precision data space and quantized data space and $E_4$ defines the discrepancy between the averaged SGD iterate and the quantized optimal solution in the full-precision domain. 
    
    We would like to remark that the quantization operations $\mathcal{Q}_l(y)$ and $\mathcal{Q}_d(\mathbf{x})$ introduced in excess risk decomposition are independent of those quantization operators introduced in the training stage, i.e., $\overline{\mathbf{w}}_N$. Next, we analyze $E_1,E_2,E_3$ and $E_4$ respectively. These computations are mainly based on the first order optimality condition (\ref{definition w*}) and the unbiased quantization Assumption \ref{ass1}. For $E_4$,
    \begin{equation}
        \label{eq: a3}
        \begin{aligned}
            E_4=&\frac{1}{2}\mathbb{E}\left[\left(y-\langle {\mathbf{w}^{(q)}}^*,\mathbf{x} \rangle \right)^2\right]-\frac{1}{2}\mathbb{E}\left[\left(y-\langle \mathbf{w}^*,\mathbf{x} \rangle \right)^2\right]\\
            =&\frac{1}{2}\mathbb{E}\left[\langle {\mathbf{w}}^*-{\mathbf{w}^{(q)}}^*,\mathbf{x} \rangle\cdot \left(2y-\langle {\mathbf{w}}^*+{\mathbf{w}^{(q)}}^*,\mathbf{x} \rangle\right)\right]\\
            =&\frac{1}{2}\mathbb{E}\left[\langle {\mathbf{w}}^*-{\mathbf{w}^{(q)}}^*,\mathbf{x} \rangle^2\right]\\
            =&\frac{1}{2}\left({\mathbf{w}}^*-{\mathbf{w}^{(q)}}^*\right)^\top \mathbf{H} \left({\mathbf{w}}^*-{\mathbf{w}^{(q)}}^*\right)\\
            =&\frac{1}{2}\langle\mathbf{H},({\mathbf{w}}^*-{\mathbf{w}^{(q)}}^*)\otimes({\mathbf{w}}^*-{\mathbf{w}^{(q)}}^*)\rangle,
        \end{aligned}
    \end{equation}
    where the third equality uses the first order optimality condition that $\mathbb{E}_{(\mathbf{x},y)\sim\mathcal{D}}[(y-\langle{\mathbf{w}}^*,\mathbf{x}\rangle)\mathbf{x}]=\mathbf{0}$.

    For $E_2$, similarly by the first order optimality condition (\ref{definition w*}) with respect to $\mathbf{w}^{(q)^*}$, it holds
    \begin{equation}
        \label{eq: a4}
        \begin{aligned}
            E_2=&\frac{1}{2}\mathbb{E}\left[(\mathcal{Q}_l(y)-\langle \overline{\mathbf{w}}_N,\mathcal{Q}_d(\mathbf{x})\rangle)^2\right]-\frac{1}{2}\mathbb{E}\left[(\mathcal{Q}_l(y)-\langle {\mathbf{w}^{(q)}}^*,\mathcal{Q}_d(\mathbf{x})\rangle)^2\right]\\
            =&\frac{1}{2}\mathbb{E}\left[\langle {\mathbf{w}^{(q)}}^*-\overline{\mathbf{w}}_N,\mathcal{Q}_d(\mathbf{x})\rangle \cdot \left(2\mathcal{Q}_l(y)-\langle {\mathbf{w}^{(q)}}^*+\overline{\mathbf{w}}_N,\mathcal{Q}_d(\mathbf{x}) \rangle\right) \right]\\
            =&\frac{1}{2}\mathbb{E}\left[\langle {\mathbf{w}^{(q)}}^*-\overline{\mathbf{w}}_N,\mathcal{Q}_d(\mathbf{x}) \rangle^2\right]\\
            =&\frac{1}{2}\langle\mathbf{H}^{(q)},\mathbb{E}[\overline{\boldsymbol{\eta}}_N\otimes\overline{\boldsymbol{\eta}}_N]\rangle.
        \end{aligned}
    \end{equation}

    For $E_3$,
    \begin{equation*}
        \begin{aligned}
            E_3=&\frac{1}{2}\mathbb{E}\left[(\mathcal{Q}_l(y)-\langle {\mathbf{w}^{(q)}}^*,\mathcal{Q}_d(\mathbf{x})\rangle)^2\right]-\frac{1}{2}\mathbb{E}\left[\left(y-\langle {\mathbf{w}^{(q)}}^*,\mathbf{x} \rangle \right)^2\right]\\
            =&\frac{1}{2}\mathbb{E}\left[\left(\mathcal{Q}_l(y)-y-\langle {\mathbf{w}^{(q)}}^*,\mathcal{Q}_d(\mathbf{x})-\mathbf{x}\rangle\right)\cdot \left(\mathcal{Q}_l(y)+y-\langle {\mathbf{w}^{(q)}}^*,\mathcal{Q}_d(\mathbf{x})+\mathbf{x}\rangle\right)\right]\\
            =&\frac{1}{2}\mathbb{E}\left[\mathcal{Q}_l(y)^2-y^2+\langle {\mathbf{w}^{(q)}}^*,\mathcal{Q}_d(\mathbf{x})-\mathbf{x}\rangle\langle {\mathbf{w}^{(q)}}^*,\mathcal{Q}_d(\mathbf{x})+\mathbf{x}\rangle\right],
        \end{aligned}
    \end{equation*}
    where the last equality utilizes the unbiased quantization Assumption \ref{ass1}.

    For $E_1$, similarly by the unbiased quantization Assumption \ref{ass1}, it holds
    \begin{equation*}
        \begin{aligned}
            E_1=&\frac{1}{2}\mathbb{E}\left[\left(y-\langle \overline{\mathbf{w}}_N,\mathbf{x} \rangle \right)^2\right]-\frac{1}{2}\mathbb{E}\left[(\mathcal{Q}_l(y)-\langle \overline{\mathbf{w}}_N,\mathcal{Q}_d(\mathbf{x})\rangle)^2\right]\\
            =&\frac{1}{2}\mathbb{E}\left[\left(y-\mathcal{Q}_l(y)-\langle \overline{\mathbf{w}}_N ,\mathbf{x}-\mathcal{Q}_d(\mathbf{x}) \rangle\right)\cdot \left(y+\mathcal{Q}_l(y)-\langle \overline{\mathbf{w}}_N ,\mathbf{x}+\mathcal{Q}_d(\mathbf{x}) \rangle\right) \right]\\
            =&\frac{1}{2}\mathbb{E}\left[y^2-\mathcal{Q}_l(y)^2\right]+\frac{1}{2}\mathbb{E}\left[\langle \overline{\mathbf{w}}_N ,\mathbf{x}-\mathcal{Q}_d(\mathbf{x}) \rangle\langle \overline{\mathbf{w}}_N ,\mathbf{x}+\mathcal{Q}_d(\mathbf{x}) \rangle\right].
        \end{aligned}
    \end{equation*}

    Hence,
    \begin{equation}
        \label{eq: a5}
        E_1+E_3=\frac{1}{2}\mathbb{E}\left[\langle {\mathbf{w}^{(q)}}^*,\mathcal{Q}_d(\mathbf{x})-\mathbf{x}\rangle^2\right]-\frac{1}{2}\mathbb{E}\left[\langle \overline{\mathbf{w}}_N ,\mathbf{x}-\mathcal{Q}_d(\mathbf{x}) \rangle^2\right].
    \end{equation}
    Therefore, combining (\ref{eq: a3}), (\ref{eq: a4}) and (\ref{eq: a5}) we have
    \begin{equation*}
        \begin{aligned}
            \mathbb{E}[\mathcal{E}(\overline{\mathbf{w}}_N)]=&\frac{1}{2}\langle\mathbf{H}^{(q)},\mathbb{E}[\overline{\boldsymbol{\eta}}_N\otimes\overline{\boldsymbol{\eta}}_N]\rangle
            +\frac{1}{2}\langle\mathbf{H},({\mathbf{w}}^*-{\mathbf{w}^{(q)}}^*)\otimes({\mathbf{w}}^*-{\mathbf{w}^{(q)}}^*)\rangle\\
            +&\frac{1}{2}\mathbb{E}\left[\langle {\mathbf{w}^{(q)}}^*,\mathcal{Q}_d(\mathbf{x})-\mathbf{x}\rangle^2\right]-\frac{1}{2}\mathbb{E}\left[\langle \overline{\mathbf{w}}_N ,\mathbf{x}-\mathcal{Q}_d(\mathbf{x}) \rangle^2\right].
        \end{aligned}
    \end{equation*}
\end{proof}

\begin{lemma}[\rm Refine excess risk decomposition]
    \label{Refine excess risk decomposition}
    Under Assumption \ref{ass1} and Assumption \ref{ass: addd quantized}, if the stepsize $\gamma<\frac{1}{{\lambda}_1^{(q)}}$,
    \begin{equation*}
        \begin{aligned}
            \mathbb{E}[\mathcal{E}(\overline{\mathbf{w}}_N)]=&\underbrace{\frac{1}{2}\langle\mathbf{H},\mathbb{E}[\overline{\boldsymbol{\eta}}_N\otimes\overline{\boldsymbol{\eta}}_N]\rangle}_{R_N}+\mathrm{ApproxErr},
        \end{aligned}
    \end{equation*}
    where
    \begin{equation*}
        \begin{aligned}
            \mathrm{ApproxErr}=&\frac{1}{2}\langle\mathbf{H},({\mathbf{w}}^*-{\mathbf{w}^{(q)}}^*)\otimes({\mathbf{w}}^*-{\mathbf{w}^{(q)}}^*)\rangle\\
            +&\left({\mathbf{w}^{(q)}}^*\right)^\top\frac{1}{N\gamma}\left(\mathbf{I}-(\mathbf{I}-\gamma\mathbf{H}^{(q)})^N\right)(\mathbf{H}^{(q)})^{-1}\left(\mathbf{H}^{(q)}-\mathbf{H}\right){\mathbf{w}^{(q)}}^*.
        \end{aligned}
    \end{equation*}
\end{lemma}
\begin{proof}
    By Lemma \ref{Excess Risk Decomposition},
    \begin{equation*}
        \begin{aligned}
            \mathbb{E}[\mathcal{E}(\overline{\mathbf{w}}_N)]=&\frac{1}{2}\langle\mathbf{H}^{(q)},\mathbb{E}[\overline{\boldsymbol{\eta}}_N\otimes\overline{\boldsymbol{\eta}}_N]\rangle\\
            +&\frac{1}{2}\langle\mathbf{H},({\mathbf{w}}^*-{\mathbf{w}^{(q)}}^*)\otimes({\mathbf{w}}^*-{\mathbf{w}^{(q)}}^*)\rangle\\
            +&\frac{1}{2}\mathbb{E}\left[\langle {\mathbf{w}^{(q)}}^*,\mathcal{Q}_d(\mathbf{x})-\mathbf{x}\rangle^2\right]-\frac{1}{2}\mathbb{E}\left[\langle \overline{\mathbf{w}}_N ,\mathbf{x}-\mathcal{Q}_d(\mathbf{x}) \rangle^2\right].
        \end{aligned}
    \end{equation*}
    We then focus on $\frac{1}{2}\mathbb{E}\left[\langle \overline{\mathbf{w}}_N ,\mathbf{x}-\mathcal{Q}_d(\mathbf{x}) \rangle^2\right]$. Recall that
    \begin{equation*}
        \overline{\mathbf{w}}_N=\overline{\mathbf{w}}_N-{\mathbf{w}^{(q)}}^*+{\mathbf{w}^{(q)}}^*=\overline{\boldsymbol{\eta}}_N+{\mathbf{w}^{(q)}}^*,
    \end{equation*}
    we have
    \begin{equation*}
        \begin{aligned}
            \frac{1}{2}\mathbb{E}\left[\langle \overline{\mathbf{w}}_N ,\mathbf{x}-\mathcal{Q}_d(\mathbf{x}) \rangle^2\right]=&\frac{1}{2}\mathbb{E}\left[\overline{\mathbf{w}}_N^\top \left(\mathbf{H}^{(q)}-\mathbf{H}\right)\overline{\mathbf{w}}_N\right]\\
            =&\frac{1}{2}\mathbb{E}\left[\overline{\boldsymbol{\eta}}_N^\top \left(\mathbf{H}^{(q)}-\mathbf{H}\right)\overline{\boldsymbol{\eta}}_N\right]+\frac{1}{2}\mathbb{E}\left[\left({\mathbf{w}^{(q)}}^*\right)^\top \left(\mathbf{H}^{(q)}-\mathbf{H}\right){\mathbf{w}^{(q)}}^*\right]\\
            +&\mathbb{E}\left[\overline{\boldsymbol{\eta}}_N^\top \left(\mathbf{H}^{(q)}-\mathbf{H}\right){\mathbf{w}^{(q)}}^*\right].
        \end{aligned}
    \end{equation*}
    Hence,
    \begin{equation}
        \label{eq: a6}
        \begin{aligned}
            \mathbb{E}[\mathcal{E}(\overline{\mathbf{w}}_N)]=&\frac{1}{2}\langle\mathbf{H},\mathbb{E}[\overline{\boldsymbol{\eta}}_N\otimes\overline{\boldsymbol{\eta}}_N]\rangle
            +\frac{1}{2}\langle\mathbf{H},({\mathbf{w}}^*-{\mathbf{w}^{(q)}}^*)\otimes({\mathbf{w}}^*-{\mathbf{w}^{(q)}}^*)\rangle\\
            -&\mathbb{E}\left[\overline{\boldsymbol{\eta}}_N^\top \left(\mathbf{H}^{(q)}-\mathbf{H}\right){\mathbf{w}^{(q)}}^*\right].
        \end{aligned}
    \end{equation}
    Noticing that by Lemma \ref{update rule for eta_t},
    \begin{equation*}
        \boldsymbol{\eta}_t = \left(\mathbf{I}-\frac{1}{B}\gamma {\mathcal{Q}_d(\mathbf{X}_t)}^\top \mathcal{Q}_d(\mathbf{X}_t) \right)\boldsymbol{\eta}_{t-1}+\gamma\frac{1}{B}{\mathcal{Q}_d(\mathbf{X}_t)}^\top \left[\boldsymbol{\xi}_t+\boldsymbol{\epsilon}_t^{(o)}-\boldsymbol{\epsilon}_t^{(a)}-\mathcal{Q}_d(\mathbf{X}_t)\boldsymbol{\epsilon}_{t-1}^{(p)}\right],
    \end{equation*}
    it follows by Assumption \ref{ass1} that
    \begin{equation*}
        \mathbb{E}\left[\boldsymbol{\eta}_t\right]=\mathbb{E}\left[\mathbb{E}\left[\boldsymbol{\eta}_t|\boldsymbol{\eta}_{t-1}\right]\right]=\mathbb{E}\left[\left(\mathbf{I}-\gamma\mathbf{H}^{(q)}\right)\boldsymbol{\eta}_{t-1}\right]=\left(\mathbf{I}-\gamma\mathbf{H}^{(q)}\right)\mathbb{E}\left[\boldsymbol{\eta}_{t-1}\right]=\left(\mathbf{I}-\gamma\mathbf{H}^{(q)}\right)^t\boldsymbol{\eta}_0.
    \end{equation*}
    Hence,
    \begin{equation*}
        \begin{aligned}
            -\mathbb{E}\left[\overline{\boldsymbol{\eta}}_N^\top \left(\mathbf{H}^{(q)}-\mathbf{H}\right){\mathbf{w}^{(q)}}^*\right]=&-\boldsymbol{\eta}_0^\top\frac{1}{N}\sum_{t=0}^{N-1}\left(\mathbf{I}-\gamma\mathbf{H}^{(q)}\right)^t\left(\mathbf{H}^{(q)}-\mathbf{H}\right){\mathbf{w}^{(q)}}^*\\
            =&\left({\mathbf{w}^{(q)}}^*\right)^\top\frac{1}{N}\sum_{t=0}^{N-1}\left(\mathbf{I}-\gamma\mathbf{H}^{(q)}\right)^t\left(\mathbf{H}^{(q)}-\mathbf{H}\right){\mathbf{w}^{(q)}}^*\\
            =&\left({\mathbf{w}^{(q)}}^*\right)^\top\frac{1}{N\gamma}\left(\mathbf{I}-(\mathbf{I}-\gamma\mathbf{H}^{(q)})^N\right)(\mathbf{H}^{(q)})^{-1}\left(\mathbf{H}^{(q)}-\mathbf{H}\right){\mathbf{w}^{(q)}}^*.
        \end{aligned}
    \end{equation*}
    Combining (\ref{eq: a6}) completes the proof.
\end{proof}

\section{Analysis of Approximation Error}
\label{analyze R4}
In this section, we analyze $\mathrm{ApproxErr}$ under multiplicative quantization and additive quantization, respectively. We first apply the definition of ${\mathbf{w}^{(q)}}^*$:
\begin{equation*}
        {\mathbf{w}^{(q)}}^*=(\mathbf{H}^{(q)})^{-1}\mathbf{H}\mathbf{w}^*.
\end{equation*}

We first handle $\frac{1}{2}\langle\mathbf{H},({\mathbf{w}}^*-{\mathbf{w}^{(q)}}^*)\otimes({\mathbf{w}}^*-{\mathbf{w}^{(q)}}^*)\rangle$. Recall $\mathbf{D}=\mathbf{H}^{(q)}-\mathbf{H}$, we have
\begin{equation}
    \label{eq: b1}
    \begin{aligned}
            \frac{1}{2}\langle\mathbf{H},({\mathbf{w}}^*-{\mathbf{w}^{(q)}}^*)\otimes({\mathbf{w}}^*-{\mathbf{w}^{(q)}}^*)\rangle=&\frac{1}{2}\mathbb{E}\left[\mathrm{tr}\left(\mathbf{H}(\mathbf{H}+\mathbf{D})^{-1}\mathbf{D}\mathbf{w}^*{\mathbf{w}^*}^\top\mathbf{D}(\mathbf{H}+\mathbf{D})^{-1}\right)\right]\\
            =&\frac{1}{2}\mathbb{E}\left[\mathrm{tr}\left({\mathbf{w}^*}^\top\mathbf{D}(\mathbf{H}+\mathbf{D})^{-1}\mathbf{H}(\mathbf{H}+\mathbf{D})^{-1}\mathbf{D}\mathbf{w}^*\right)\right]\\
            =&\frac{1}{2}\|\mathbf{w}^*\|_{\mathbf{D}(\mathbf{H}+\mathbf{D})^{-1}\mathbf{H}(\mathbf{H}+\mathbf{D})^{-1}\mathbf{D}}^2.
    \end{aligned}
\end{equation}
We then handle $\left({\mathbf{w}^{(q)}}^*\right)^\top\frac{1}{N\gamma}\left(\mathbf{I}-(\mathbf{I}-\gamma\mathbf{H}^{(q)})^N\right)(\mathbf{H}^{(q)})^{-1}\left(\mathbf{H}^{(q)}-\mathbf{H}\right){\mathbf{w}^{(q)}}^*$.
\begin{equation}
    \label{eq: b2}
    \begin{aligned}
            &\left({\mathbf{w}^{(q)}}^*\right)^\top\frac{1}{N\gamma}\left(\mathbf{I}-(\mathbf{I}-\gamma\mathbf{H}^{(q)})^N\right)(\mathbf{H}^{(q)})^{-1}\left(\mathbf{H}^{(q)}-\mathbf{H}\right){\mathbf{w}^{(q)}}^*\\
            =&{\mathbf{w}^*}^\top \mathbf{H}(\mathbf{H}^{(q)})^{-1}\frac{1}{N\gamma}\left(\mathbf{I}-(\mathbf{I}-\gamma\mathbf{H}^{(q)})^N\right)(\mathbf{H}^{(q)})^{-1}\left(\mathbf{H}^{(q)}-\mathbf{H}\right)(\mathbf{H}^{(q)})^{-1}\mathbf{H}\mathbf{w}^*.
    \end{aligned}
\end{equation}
\begin{lemma}[\rm Approximation error under multiplicative quantization]
    \label{lem: b1}
    If there exists $\epsilon_d$ such that $\mathcal{Q}_d$ is $\epsilon_d$-multiplicative, under the assumptions and notations in Lemma \ref{Refine excess risk decomposition}, 
    \begin{equation*}
        \mathrm{ApproxErr}\leq \frac{\epsilon_d^2}{2(1+\epsilon_d)^2}\|\mathbf{w}^*\|_\mathbf{H}^2+\frac{\epsilon_d}{(1+\epsilon_d)^2}\left\|\mathbf{w}^*\right\|_\mathbf{H}^2.
    \end{equation*}
\end{lemma}
\begin{proof}
    Under multiplicative quantization, 
    \begin{equation*}
        \mathbf{H}^{(q)}=(1+\epsilon_d)\mathbf{H}, \quad \mathbf{D}=\epsilon_d\mathbf{H}.
    \end{equation*}
    It follows by (\ref{eq: b1}) that 
    \begin{equation*}
        \begin{aligned}
            \frac{1}{2}\langle\mathbf{H},({\mathbf{w}}^*-{\mathbf{w}^{(q)}}^*)\otimes({\mathbf{w}}^*-{\mathbf{w}^{(q)}}^*)\rangle=&\frac{1}{2}\|\mathbf{w}^*\|_{\mathbf{D}(\mathbf{H}+\mathbf{D})^{-1}\mathbf{H}(\mathbf{H}+\mathbf{D})^{-1}\mathbf{D}}^2=\frac{\epsilon_d^2}{2(1+\epsilon_d)^2}\|\mathbf{w}^*\|_\mathbf{H}^2.
        \end{aligned}
    \end{equation*}
    Similarly, by (\ref{eq: b2}),
    \begin{equation*}
        \begin{aligned}
            &\left({\mathbf{w}^{(q)}}^*\right)^\top\frac{1}{N\gamma}\left(\mathbf{I}-(\mathbf{I}-\gamma\mathbf{H}^{(q)})^N\right)(\mathbf{H}^{(q)})^{-1}\left(\mathbf{H}^{(q)}-\mathbf{H}\right){\mathbf{w}^{(q)}}^*\\
            =&{\mathbf{w}^*}^\top \mathbf{H}(\mathbf{H}^{(q)})^{-1}\frac{1}{N\gamma}\left(\mathbf{I}-(\mathbf{I}-\gamma\mathbf{H}^{(q)})^N\right)(\mathbf{H}^{(q)})^{-1}\left(\mathbf{H}^{(q)}-\mathbf{H}\right)(\mathbf{H}^{(q)})^{-1}\mathbf{H}\mathbf{w}^*\\
            =&\frac{\epsilon_d}{N\gamma(1+\epsilon_d)}{\mathbf{w}^*}^\top \mathbf{H}(\mathbf{H}^{(q)})^{-1}\left(\mathbf{I}-(\mathbf{I}-\gamma\mathbf{H}^{(q)})^N\right)(\mathbf{H}^{(q)})^{-1}\mathbf{H}\mathbf{w}^*\\
            \leq& \frac{\epsilon_d}{N\gamma(1+\epsilon_d)^2} \left\|\mathbf{w}^*\right\|_\mathbf{H}^2\left\|(\mathbf{H}^{(q)})^{-1/2}\left(\mathbf{I}-(\mathbf{I}-\gamma\mathbf{H}^{(q)})^N\right)(\mathbf{H}^{(q)})^{-1/2}\right\|\\
            \leq&\frac{\epsilon_d}{N\gamma(1+\epsilon_d)^2} \left\|\mathbf{w}^*\right\|_\mathbf{H}^2\max_i \frac{\min\{1,N\gamma\lambda_i^{(q)}\}}{{\lambda}_i^{(q)}}\\
            \leq&\frac{\epsilon_d}{(1+\epsilon_d)^2}\left\|\mathbf{w}^*\right\|_\mathbf{H}^2.
        \end{aligned}
    \end{equation*}
\end{proof}

\begin{lemma}[\rm Approximation error under additive quantization]
    \label{lem: b2}
    If there exists $\epsilon_d$ such that $\mathcal{Q}_d$ is $\epsilon_d$-additive, under the assumptions and notations in Lemma \ref{Refine excess risk decomposition}, 
    \begin{equation*}
        \mathrm{ApproxErr}\leq \frac{\epsilon_d^2}{2(\lambda_d+\epsilon_d)^2}\|\mathbf{w}^*\|_\mathbf{H}^2+\frac{\lambda_1\epsilon_d}{(\lambda_d+\epsilon_d)(\lambda_1+\epsilon_d)}\left\|\mathbf{w}^*\right\|_\mathbf{H}^2.
    \end{equation*}
\end{lemma}
\begin{proof}
    Under additive quantization, 
    \begin{equation*}
        \mathbf{H}^{(q)}=\mathbf{H}+\epsilon_d\mathbf{I}, \quad \mathbf{D}=\epsilon_d\mathbf{I}.
    \end{equation*}
    It follows by (\ref{eq: b1}) that 
    \begin{equation*}
        \begin{aligned}
            \frac{1}{2}\langle\mathbf{H},({\mathbf{w}}^*-{\mathbf{w}^{(q)}}^*)\otimes({\mathbf{w}}^*-{\mathbf{w}^{(q)}}^*)\rangle=&\frac{1}{2}\|\mathbf{w}^*\|_{\mathbf{D}(\mathbf{H}+\mathbf{D})^{-1}\mathbf{H}(\mathbf{H}+\mathbf{D})^{-1}\mathbf{D}}^2\leq\frac{\epsilon_d^2}{2(\lambda_d+\epsilon_d)^2}\|\mathbf{w}^*\|_\mathbf{H}^2.
        \end{aligned}
    \end{equation*}
    Similarly, by (\ref{eq: b2}),
    \begin{equation*}
        \begin{aligned}
            &\left({\mathbf{w}^{(q)}}^*\right)^\top\frac{1}{N\gamma}\left(\mathbf{I}-(\mathbf{I}-\gamma\mathbf{H}^{(q)})^N\right)(\mathbf{H}^{(q)})^{-1}\left(\mathbf{H}^{(q)}-\mathbf{H}\right){\mathbf{w}^{(q)}}^*\\
            =&{\mathbf{w}^*}^\top \mathbf{H}(\mathbf{H}^{(q)})^{-1}\frac{1}{N\gamma}\left(\mathbf{I}-(\mathbf{I}-\gamma\mathbf{H}^{(q)})^N\right)(\mathbf{H}^{(q)})^{-1}\left(\mathbf{H}^{(q)}-\mathbf{H}\right)(\mathbf{H}^{(q)})^{-1}\mathbf{H}\mathbf{w}^*\\
            \leq&\frac{\epsilon_d}{N\gamma(\lambda_d+\epsilon_d)}{\mathbf{w}^*}^\top \mathbf{H}(\mathbf{H}^{(q)})^{-1}\left(\mathbf{I}-(\mathbf{I}-\gamma\mathbf{H}^{(q)})^N\right)(\mathbf{H}^{(q)})^{-1}\mathbf{H}\mathbf{w}^*\\
            \leq& \frac{\lambda_1\epsilon_d}{N\gamma(\lambda_d+\epsilon_d)(\lambda_1+\epsilon_d)} \left\|\mathbf{w}^*\right\|_\mathbf{H}^2\left\|(\mathbf{H}^{(q)})^{-1/2}\left(\mathbf{I}-(\mathbf{I}-\gamma\mathbf{H}^{(q)})^N\right)(\mathbf{H}^{(q)})^{-1/2}\right\|\\
            \leq&\frac{\lambda_1\epsilon_d}{N\gamma(\lambda_d+\epsilon_d)(\lambda_1+\epsilon_d)} \left\|\mathbf{w}^*\right\|_\mathbf{H}^2\max_i \frac{\min\{1,N\gamma\lambda_i^{(q)}\}}{{\lambda}_i^{(q)}}\\
            \leq&\frac{\lambda_1\epsilon_d}{(\lambda_d+\epsilon_d)(\lambda_1+\epsilon_d)}\left\|\mathbf{w}^*\right\|_\mathbf{H}^2.
        \end{aligned}
    \end{equation*}
\end{proof}

\section{Analysis of \texorpdfstring{$R_N$}{}}
\label{analyze R2}
\subsection{Preliminary}
\label{Preliminary}
We first define the following linear operators as in \citet{zou2023benign}:
\begin{gather*}
        \mathcal{I}=\mathbf{I}\otimes\mathbf{I},\quad\mathcal{M}^{(q)}=\mathbb{E}[\mathbf{x}^{(q)}\otimes\mathbf{x}^{(q)}\otimes\mathbf{x}^{(q)}\otimes\mathbf{x}^{(q)}],\quad\widetilde{\mathcal{M}}^{(q)}=\mathbf{H}^{(q)}\otimes\mathbf{H}^{(q)},\\
        \mathcal{T}^{(q)}=\mathbf{H}^{(q)}\otimes\mathbf{I}+\mathbf{I}\otimes\mathbf{H}^{(q)}-\gamma\mathcal{M}^{(q)},\quad\widetilde{\mathcal{T}}^{(q)}=\mathbf{H}^{(q)}\otimes\mathbf{I}+\mathbf{I}\otimes\mathbf{H}^{(q)}-\gamma\mathbf{H}^{(q)}\otimes\mathbf{H}^{(q)}.
\end{gather*}
For a symmetric matrix $\mathbf{A}$, the above definitions result in:
\begin{gather*}
        \mathcal{I}\circ\mathbf{A}=\mathbf{A},\quad\mathcal{M}^{(q)}\circ\mathbf{A}=\mathbb{E}[({\mathbf{x}^{(q)}}^\top\mathbf{A}\mathbf{x}^{(q)})\mathbf{x}^{(q)}{\mathbf{x}^{(q)}}^\top],\quad\widetilde{\mathcal{M}}^{(q)}\circ\mathbf{A}=\mathbf{H}^{(q)}\mathbf{A}\mathbf{H}^{(q)},\\
        (\mathcal{I}-\gamma\mathcal{T}^{(q)})\circ\mathbf{A}=\mathbb{E}[(\mathbf{I}-\gamma\mathbf{x}^{(q)}{\mathbf{x}^{(q)}}^\top)\mathbf{A}(\mathbf{I}-\gamma\mathbf{x}^{(q)}{\mathbf{x}^{(q)}}^\top)],\quad(\mathcal{I}-\gamma\widetilde{\mathcal{T}}^{(q)})\circ\mathbf{A}=(\mathbf{I}-\gamma\mathbf{H}^{(q)})\mathbf{A}(\mathbf{I}-\gamma\mathbf{H}^{(q)}).
\end{gather*}
Further, we generalize the linear operators from \citet{zou2023benign} to account for batch size effects. For a symmetric matrix $\mathbf{A}$, we define
\begin{equation*}
    \begin{aligned}
        \mathcal{M}_B^{(q)}\circ\mathbf{A}=&\mathbb{E}\left[\frac{1}{B^2} {\mathbf{X}^{(q)}}^\top{\mathbf{X}^{(q)}} \mathbf{A} {\mathbf{X}^{(q)}}^\top{\mathbf{X}^{(q)}} \right],\\
        (\mathcal{I}-\gamma\mathcal{T}_B^{(q)})\circ\mathbf{A}=&\mathbb{E}\left[\left(\mathbf{I}-\gamma\frac{1}{B} {\mathbf{X}^{(q)}}^\top{\mathbf{X}^{(q)}}\right)\mathbf{A}\left(\mathbf{I}-\gamma\frac{1}{B} {\mathbf{X}^{(q)}}^\top{\mathbf{X}^{(q)}}\right)\right].
    \end{aligned}
\end{equation*}
\subsection{Initial Study}
\label{Initial study R2}
To analyze $R_N$, we firstly utilize the fact that
\begin{equation}
    \label{eq: c1}
    R_N=\frac{1}{2}\langle\mathbf{H},\mathbb{E}[\overline{\boldsymbol{\eta}}_N\otimes\overline{\boldsymbol{\eta}}_N]\rangle\leq \mu_{\max}\left(\mathbf{H}(\mathbf{H}^{(q)})^{-1}\right)\frac{1}{2}\langle\mathbf{H}^{(q)},\mathbb{E}[\overline{\boldsymbol{\eta}}_N\otimes\overline{\boldsymbol{\eta}}_N]\rangle\leq \underbrace{\frac{1}{2}\langle\mathbf{H}^{(q)},\mathbb{E}[\overline{\boldsymbol{\eta}}_N\otimes\overline{\boldsymbol{\eta}}_N]\rangle}_{R_N^{(0)}}.
\end{equation}
We secondly substitute $\overline{\boldsymbol{\eta}}_{N}$ with the summation of $\boldsymbol{\eta}_t$. This step mainly based on the propagation in Lemma \ref{update rule for eta_t}, the unbiased quantization Assumption \ref{ass1} and the first order optimality condition (\ref{definition w*}). We summarize as the following lemma.
\begin{lemma}
    \label{Initial study of R_2}
    Under Assumption \ref{ass1} and Assumption \ref{ass: addd quantized},
    \begin{equation*}
        R_N^{(0)}\leq \frac{1}{N^2}\cdot\sum_{t=0}^{N-1}\sum_{k=t}^{N-1}\left\langle(\mathbf{I}-\gamma\mathbf{H}^{(q)})^{k-t}\mathbf{H}^{(q)},\mathbb{E}[\boldsymbol{\eta}_t\otimes\boldsymbol{\eta}_t]\right\rangle.
    \end{equation*}
\end{lemma}
\begin{proof}
    Recall that by (\ref{eq: c1}),
    \begin{equation*}
        R_N^{(0)}=\frac{1}{2}\langle\mathbf{H}^{(q)},\mathbb{E}[\overline{\boldsymbol{\eta}}_N\otimes\overline{\boldsymbol{\eta}}_N]\rangle,
    \end{equation*}
    we then focus on $\mathbb{E}[\overline{\boldsymbol{\eta}}_{N}\otimes\overline{\boldsymbol{\eta}}_{N}]$. By definition $\overline{\boldsymbol{\eta}}_{N}=\frac{1}{N}\sum_{t=0}^{N-1}\boldsymbol{\eta}_{t}$,
    \begin{equation}
    \label{D0}
        \begin{aligned}
            \mathbb{E}[\overline{\boldsymbol{\eta}}_{N}\otimes\overline{\boldsymbol{\eta}}_{N}]
            =&\frac{1}{N^2}\cdot\left(\sum_{0\leq k\leq t\leq N-1}\mathbb{E}[\boldsymbol{\eta}_t\otimes\boldsymbol{\eta}_k]+\sum_{0\leq t<k\leq N-1}\mathbb{E}[\boldsymbol{\eta}_t\otimes\boldsymbol{\eta}_k]\right) \\
            \preceq&\frac{1}{N^2}\cdot\left(\sum_{0\leq k\leq t\leq N-1}\mathbb{E}\left[\mathbb{E}[\boldsymbol{\eta}_t\otimes\boldsymbol{\eta}_k|\boldsymbol{\eta}_k]\right]+\sum_{0\leq t\leq k\leq N-1}\mathbb{E}\left[\mathbb{E}[\boldsymbol{\eta}_t\otimes\boldsymbol{\eta}_k|\boldsymbol{\eta}_t]\right]\right).
        \end{aligned}
    \end{equation}
    Note that by the unbiased Assumption \ref{ass1},
    \begin{equation*}
        \mathbb{E}\left[\gamma{\mathcal{Q}_d(\mathbf{X}_t)}^\top \left(\boldsymbol{\epsilon}_t^{(o)}-\boldsymbol{\epsilon}_t^{(a)}-\mathcal{Q}_d(\mathbf{X}_t)\boldsymbol{\epsilon}_{t-1}^{(p)}\right)\bigg|\boldsymbol{\eta}_{t-1}\right]=\boldsymbol{0}.
    \end{equation*}
    Further, by the optimality (\ref{definition w*}),
    \begin{equation*}
        \mathbb{E}\left[\gamma{\mathcal{Q}_d(\mathbf{X}_t)}^\top \boldsymbol{\xi}_t\bigg|\boldsymbol{\eta}_{t-1}\right]=\mathbb{E}\left[\gamma{\mathcal{Q}_d(\mathbf{X}_t)}^\top \left[\mathcal{Q}_l(\mathbf{y}_t)-\mathcal{Q}_d(\mathbf{X}_t){\mathbf{w}^{(q)}}^*\right]\bigg|\boldsymbol{\eta}_{t-1}\right]=\boldsymbol{0}.
    \end{equation*}
    Hence, by Lemma \ref{update rule for eta_t},
    \begin{equation}
        \label{conditional update}
        \mathbb{E}\left[\boldsymbol{\eta}_t|\boldsymbol{\eta}_{t-1}\right]=\left(\mathbf{I}-\gamma \mathbf{H}^{(q)}\right)\boldsymbol{\eta}_{t-1}.
    \end{equation}
    Therefore, by (\ref{D0}) and (\ref{conditional update}),
    \begin{equation}
    \label{etaN etaN}
        \begin{aligned}
            & \mathbb{E}[\overline{\boldsymbol{\eta}}_{N}\otimes\overline{\boldsymbol{\eta}}_{N}] \\
            \preceq &\frac{1}{N^2}\cdot\left(\sum_{0\leq k\leq t\leq N-1}\mathbb{E}\left[\mathbb{E}[\boldsymbol{\eta}_t\otimes\boldsymbol{\eta}_k|\boldsymbol{\eta}_k]\right]+\sum_{0\leq t\leq k\leq N-1}\mathbb{E}\left[\mathbb{E}[\boldsymbol{\eta}_t\otimes\boldsymbol{\eta}_k|\boldsymbol{\eta}_t]\right]\right)\\
            =&\frac{1}{N^2}\cdot\left(\sum_{0\leq k\leq t\leq N-1}(\mathbf{I}-\gamma\mathbf{H}^{(q)})^{t-k}\mathbb{E}[\boldsymbol{\eta}_k\otimes\boldsymbol{\eta}_k]+\sum_{0\leq t\leq k\leq N-1}\mathbb{E}[\boldsymbol{\eta}_t\otimes\boldsymbol{\eta}_t](\mathbf{I}-\gamma\mathbf{H}^{(q)})^{k-t}\right) \\
            =&\frac{1}{N^2}\cdot\sum_{t=0}^{N-1}\sum_{k=t}^{N-1}\left((\mathbf{I}-\gamma\mathbf{H}^{(q)})^{k-t}\mathbb{E}[\boldsymbol{\eta}_t\otimes\boldsymbol{\eta}_t]+\mathbb{E}[\boldsymbol{\eta}_t\otimes\boldsymbol{\eta}_t](\mathbf{I}-\gamma\mathbf{H}^{(q)})^{k-t}\right).
        \end{aligned}
    \end{equation}
    Applying (\ref{etaN etaN}) into $R_N^{(0)}$, we have
    \begin{equation*}
        \begin{aligned}
            R_N^{(0)}=&\frac{1}{2}\langle\mathbf{H}^{(q)},\mathbb{E}[\overline{\boldsymbol{\eta}}_N\otimes\overline{\boldsymbol{\eta}}_N]\rangle\\
            \leq&\frac{1}{2N^2}\cdot\sum_{t=0}^{N-1}\sum_{k=t}^{N-1}\left\langle\mathbf{H}^{(q)},(\mathbf{I}-\gamma\mathbf{H}^{(q)})^{k-t}\mathbb{E}[\boldsymbol{\eta}_t\otimes\boldsymbol{\eta}_t]+\mathbb{E}[\boldsymbol{\eta}_t\otimes\boldsymbol{\eta}_t](\mathbf{I}-\gamma\mathbf{H}^{(q)})^{k-t}\right\rangle \\
            =&\frac{1}{N^2}\cdot\sum_{t=0}^{N-1}\sum_{k=t}^{N-1}\left\langle(\mathbf{I}-\gamma\mathbf{H}^{(q)})^{k-t}\mathbf{H}^{(q)},\mathbb{E}[\boldsymbol{\eta}_t\otimes\boldsymbol{\eta}_t]\right\rangle,
        \end{aligned}
    \end{equation*}
    where the last equality holds since $\mathbf{H}^{(q)}$ and $(\mathbf{I}-\gamma\mathbf{H}^{(q)})^{k-t}$ commute. This completes the proof.
\end{proof}

Lemma \ref{Initial study of R_2} implies that, to bound $R_N^{(0)}$, the main goal is to bound $\mathbb{E}[\boldsymbol{\eta}_t\otimes\boldsymbol{\eta}_t]$. Recall that by Lemma \ref{update rule for eta_t},
    \begin{equation*}
        \boldsymbol{\eta}_t = \left(\mathbf{I}-\frac{1}{B}\gamma {\mathcal{Q}_d(\mathbf{X}_t)}^\top \mathcal{Q}_d(\mathbf{X}_t) \right)\boldsymbol{\eta}_{t-1}+\gamma\frac{1}{B}{\mathcal{Q}_d(\mathbf{X}_t)}^\top \left[\boldsymbol{\xi}_t+\boldsymbol{\epsilon}_t^{(o)}-\boldsymbol{\epsilon}_t^{(a)}-\mathcal{Q}_d(\mathbf{X}_t)\boldsymbol{\epsilon}_{t-1}^{(p)}\right].
    \end{equation*}
Denote
\begin{equation*}
    \boldsymbol{\eta}_t^{\rm bias}=\left(\mathbf{I}-\frac{1}{B}\gamma {\mathcal{Q}_d(\mathbf{X}_t)}^\top \mathcal{Q}_d(\mathbf{X}_t) \right)\boldsymbol{\eta}_{t-1}^{\rm bias},\quad \boldsymbol{\eta}_{0}^{\rm bias}=\boldsymbol{\eta}_0,
\end{equation*}
\begin{equation*}
    \boldsymbol{\eta}_t^{\rm var}=\left(\mathbf{I}-\frac{1}{B}\gamma {\mathcal{Q}_d(\mathbf{X}_t)}^\top \mathcal{Q}_d(\mathbf{X}_t) \right)\boldsymbol{\eta}_{t-1}^{\rm var}+\gamma\frac{1}{B}{\mathcal{Q}_d(\mathbf{X}_t)}^\top \left[\boldsymbol{\xi}_t+\boldsymbol{\epsilon}_t^{(o)}-\boldsymbol{\epsilon}_t^{(a)}-\mathcal{Q}_d(\mathbf{X}_t)\boldsymbol{\epsilon}_{t-1}^{(p)}\right],
\end{equation*}
with $\boldsymbol{\eta}_0^{\rm var}=\boldsymbol{0}$. Then 
\begin{equation*}
    \boldsymbol{\eta}_t=\boldsymbol{\eta}_t^{\rm var}+\boldsymbol{\eta}_t^{\rm bias},
\end{equation*}
and
\begin{equation}
    \mathbb{E}\left[\boldsymbol{\eta}_t\otimes\boldsymbol{\eta}_t\right] \preceq 2\left(\underbrace{\mathbb{E}\left[\boldsymbol{\eta}_t^{\rm bias}\otimes\boldsymbol{\eta}_t^{\rm bias}\right]}_{\mathbf{B}_t}+\underbrace{\mathbb{E}\left[\boldsymbol{\eta}_t^{\rm var}\otimes\boldsymbol{\eta}_t^{\rm var}\right]}_{\mathbf{C}_t}\right).
\end{equation}
Regarding $\mathbf{B}_t$,
\begin{equation}
    \label{eq: bt}
    \mathbf{B}_t=\mathbb{E}\left[\left(\mathbf{I}-\gamma\frac{1}{B} {\mathcal{Q}_d(\mathbf{X})}^\top \mathcal{Q}_d(\mathbf{X}) \right)\mathbf{B}_{t-1}\left(\mathbf{I}-\gamma\frac{1}{B} {\mathcal{Q}_d(\mathbf{X})}^\top \mathcal{Q}_d(\mathbf{X}) \right)\right].
\end{equation}
Regarding $\mathbf{C}_t$, by the unbiased quantization Assumption \ref{ass1} and $\boldsymbol{\eta}_0^{\rm var}=\boldsymbol{0}$, it holds,
\begin{equation}
    \label{eq: ct}
    \mathbf{C}_t=\mathbb{E}\left[\left(\mathbf{I}-\gamma\frac{1}{B} {\mathcal{Q}_d(\mathbf{X})}^\top \mathcal{Q}_d(\mathbf{X}) \right)\mathbf{C}_{t-1}\left(\mathbf{I}-\gamma\frac{1}{B} {\mathcal{Q}_d(\mathbf{X})}^\top \mathcal{Q}_d(\mathbf{X}) \right)\right]+\gamma^2 \boldsymbol{\Sigma}_t,
\end{equation}
where
    \begin{equation*}
        \begin{aligned}
            \boldsymbol{\Sigma}_t:=&\frac{1}{B^2}\mathbb{E}\left[{\mathcal{Q}_d(\mathbf{X}_t)}^\top\left[\boldsymbol{\xi}_t+\boldsymbol{\epsilon}_t^{(o)}-\boldsymbol{\epsilon}_t^{(a)}-\mathcal{Q}_d(\mathbf{X}_t)\boldsymbol{\epsilon}_{t-1}^{(p)}\right]\left[\boldsymbol{\xi}_t+\boldsymbol{\epsilon}_t^{(o)}-\boldsymbol{\epsilon}_t^{(a)}-\mathcal{Q}_d(\mathbf{X}_t)\boldsymbol{\epsilon}_{t-1}^{(p)}\right]^\top{\mathcal{Q}_d(\mathbf{X}_t)}\right]\\
            =&\underbrace{\frac{1}{B^2}\mathbb{E}\left[{\mathcal{Q}_d(\mathbf{X}_t)}^\top \boldsymbol{\xi}_t\boldsymbol{\xi}_t^\top \mathcal{Q}_d(\mathbf{X}_t)\right]}_{\boldsymbol{\Sigma}_t^\xi}+\underbrace{\frac{1}{B^2}\mathbb{E}\left[{\mathcal{Q}_d(\mathbf{X}_t)}^\top \boldsymbol{\epsilon}_t^{(o)}{\boldsymbol{\epsilon}_t^{(o)}}^\top \mathcal{Q}_d(\mathbf{X}_t)\right]}_{\boldsymbol{\Sigma}_t^{\epsilon^{(o)}}}\\
            +&\underbrace{\frac{1}{B^2}\mathbb{E}\left[{\mathcal{Q}_d(\mathbf{X}_t)}^\top \boldsymbol{\epsilon}_t^{(a)}{\boldsymbol{\epsilon}_t^{(a)}}^\top \mathcal{Q}_d(\mathbf{X}_t)\right]}_{\boldsymbol{\Sigma}_t^{\epsilon^{(a)}}}+\underbrace{\frac{1}{B^2}\mathbb{E}\left[{\mathcal{Q}_d(\mathbf{X}_t)}^\top {\mathcal{Q}_d(\mathbf{X}_t)}\boldsymbol{\epsilon}_{t-1}^{(p)}{\boldsymbol{\epsilon}_{t-1}^{(p)}}^\top {\mathcal{Q}_d(\mathbf{X}_t)}^\top\mathcal{Q}_d(\mathbf{X}_t)\right]}_{\boldsymbol{\Sigma}_t^{\epsilon^{(p)}}}.
        \end{aligned}
    \end{equation*}
    We then summarize the update rule for $\mathbb{E}\left[\boldsymbol{\eta}_t\otimes\boldsymbol{\eta}_t \right]$ as follows.

\begin{lemma}[\rm Update rule under general quantization]
    \label{update rule for eta_t^2}
    Under Assumption \ref{ass1}, Assumption \ref{ass: addd quantized}, Assumption \ref{ass2}, and Assumption \ref{ass3},
    \begin{gather*}
            \mathbf{C}_t
            \preceq \mathbb{E}\left[\left(\mathbf{I}-\gamma\frac{1}{B} {\mathcal{Q}_d(\mathbf{X})}^\top \mathcal{Q}_d(\mathbf{X}) \right)\mathbf{C}_{t-1}\left(\mathbf{I}-\gamma\frac{1}{B} {\mathcal{Q}_d(\mathbf{X})}^\top \mathcal{Q}_d(\mathbf{X}) \right)\right]+\gamma^2 {\sigma_G^{(q)}}^2\mathbf{H}^{(q)},\\
            \mathbf{B}_t
            = \mathbb{E}\left[\left(\mathbf{I}-\gamma\frac{1}{B} {\mathcal{Q}_d(\mathbf{X})}^\top \mathcal{Q}_d(\mathbf{X}) \right)\mathbf{B}_{t-1}\left(\mathbf{I}-\gamma\frac{1}{B} {\mathcal{Q}_d(\mathbf{X})}^\top \mathcal{Q}_d(\mathbf{X}) \right)\right],
    \end{gather*}
    where
    \begin{equation*}
        \begin{aligned}
            {\sigma_G^{(q)}}^2=&\frac{\sup_t \left\{\left\|\mathbb{E}\left[\boldsymbol{\epsilon}_t^{(o)}{\boldsymbol{\epsilon}_t^{(o)}}^\top|\mathbf{o}_t\right]+\mathbb{E}\left[\boldsymbol{\epsilon}_t^{(a)}{\boldsymbol{\epsilon}_t^{(a)}}^\top|\mathbf{a}_t\right] \right\|\right\}}{B}\\
            +&\alpha_B\sup_t\mathbb{E}_{\mathbf{w}_{t-1}}\left[\mathrm{tr}\left(\mathbf{H}^{(q)} \mathbb{E}\left[\boldsymbol{\epsilon}_{t-1}^{(p)}{\boldsymbol{\epsilon}_{t-1}^{(p)}}^\top \big| \mathbf{w}_{t-1}\right]\right)\right]
        +\frac{\sigma^2}{B},
        \end{aligned}
    \end{equation*}
    with $\mathbf{a}_t=\mathcal{Q}_d(\mathbf{X}_t)\mathcal{Q}_p(\mathbf{w}_{t-1}),\ \mathbf{o}_t=\mathcal{Q}_l(\mathbf{y}_t)-\mathcal{Q}_a\left(\mathcal{Q}_d(\mathbf{X}_t)\mathcal{Q}_p(\mathbf{w}_{t-1})\right)$ and $\|\cdot\|$ denoting the spectral norm.
\end{lemma}
\begin{proof}
    We cope with each term in $\boldsymbol{\Sigma}_t$ to provide an upper bound. For $\boldsymbol{\Sigma}_t^{\epsilon^{(p)}}$,
    \begin{equation*}
        \begin{aligned}
            \boldsymbol{\Sigma}_t^{\epsilon^{(p)}}=&\frac{1}{B^2}\mathbb{E}\left[{\mathcal{Q}_d(\mathbf{X}_t)}^\top {\mathcal{Q}_d(\mathbf{X}_t)}\boldsymbol{\epsilon}_{t-1}^{(p)}{\boldsymbol{\epsilon}_{t-1}^{(p)}}^\top {\mathcal{Q}_d(\mathbf{X}_t)}^\top\mathcal{Q}_d(\mathbf{X}_t)\right]\\
            \preceq &\alpha_B \sup_t\mathbb{E}_{\mathbf{w}_{t-1}}\left[\mathrm{tr}\left(\mathbf{H}^{(q)} \mathbb{E}\left[\boldsymbol{\epsilon}_{t-1}^{(p)}{\boldsymbol{\epsilon}_{t-1}^{(p)}}^\top \big| \mathbf{w}_{t-1}\right]\right)\right]\mathbf{H}^{(q)},
        \end{aligned}
    \end{equation*}
    where the inequality holds by Assumption \ref{ass2}. For $\boldsymbol{\Sigma}_t^\xi$,
    \begin{equation}
    \label{Sigma epsilon t}
        \begin{aligned}
            \boldsymbol{\Sigma}_t^\xi=&\frac{1}{B^2}\mathbb{E}\left[{\mathcal{Q}_d(\mathbf{X}_t)}^\top \boldsymbol{\xi}_t\boldsymbol{\xi}_t^\top \mathcal{Q}_d(\mathbf{X}_t)\right]\\
            =&\frac{1}{B^2}\mathbb{E}\left[\sum_{i=1}^B\sum_{j=1}^B {{\mathcal{Q}_d(\mathbf{X}_t)}^i}^\top \boldsymbol{\xi}_t^i \left({{\mathcal{Q}_d(\mathbf{X}_t)}^j}^\top \boldsymbol{\xi}_t^j\right)^\top\right]\\
            =&\frac{1}{B^2}\sum_{i=1}^B\mathbb{E}\left[{{\mathcal{Q}_d(\mathbf{X}_t)}^i}^\top \boldsymbol{\xi}_t^i \left({{\mathcal{Q}_d(\mathbf{X}_t)}^i}^\top \boldsymbol{\xi}_t^i\right)^\top\right]\\
            =&\frac{1}{B}\cdot\mathbb{E}\left[{{\mathcal{Q}_d(\mathbf{x})}} \xi \left({{\mathcal{Q}_d(\mathbf{x}}}) \xi\right)^\top\right]\\
            =&\frac{1}{B}\cdot\mathbb{E}\left[\xi^2\mathcal{Q}_d(\mathbf{x})\mathcal{Q}_d(\mathbf{x})^\top\right]\\
            \preceq&\frac{\sigma^2}{B}\cdot\mathbf{H}^{(q)},
        \end{aligned}
    \end{equation}
    where the third equality holds as samples are independent and data quantization is applied to each sample independently, the inequality holds by Assumption \ref{ass3}. For $\boldsymbol{\Sigma}_t^{\epsilon^{(o)}}+\boldsymbol{\Sigma}_t^{\epsilon^{(a)}}$,
    \begin{equation*}
        \begin{aligned}
            \boldsymbol{\Sigma}_t^{\epsilon^{(o)}}+\boldsymbol{\Sigma}_t^{\epsilon^{(a)}}=&\frac{1}{B^2}\mathbb{E}\left[{\mathcal{Q}_d(\mathbf{X}_t)}^\top (\boldsymbol{\epsilon}_t^{(o)}{\boldsymbol{\epsilon}_t^{(o)}}^\top+\boldsymbol{\epsilon}_t^{(a)}{\boldsymbol{\epsilon}_t^{(a)}}^\top) \mathcal{Q}_d(\mathbf{X}_t)\right]\\
            =&\frac{1}{B^2}\mathbb{E}\left[{\mathcal{Q}_d(\mathbf{X}_t)}^\top \left(\mathbb{E}\left[\boldsymbol{\epsilon}_t^{(o)}{\boldsymbol{\epsilon}_t^{(o)}}^\top|\mathbf{o}_t\right] +\mathbb{E}\left[\boldsymbol{\epsilon}_t^{(a)}{\boldsymbol{\epsilon}_t^{(a)}}^\top|\mathbf{a}_t\right] \right)\mathcal{Q}_d(\mathbf{X}_t)\right]\\
            \preceq &\frac{1}{B^2}\mathbb{E}\left[\left(\left\|\mathbb{E}\left[\boldsymbol{\epsilon}_t^{(o)}{\boldsymbol{\epsilon}_t^{(o)}}^\top|\mathbf{o}_t\right] +\mathbb{E}\left[\boldsymbol{\epsilon}_t^{(a)}{\boldsymbol{\epsilon}_t^{(a)}}^\top|\mathbf{a}_t\right] \right\|\right){\mathcal{Q}_d(\mathbf{X}_t)}^\top  \mathcal{Q}_d(\mathbf{X}_t)\right]\\
            \preceq & \frac{1}{B^2}\sup_t \left[\left\|\mathbb{E}\left[\boldsymbol{\epsilon}_t^{(o)}{\boldsymbol{\epsilon}_t^{(o)}}^\top|\mathbf{o}_t\right]+\mathbb{E}\left[\boldsymbol{\epsilon}_t^{(a)}{\boldsymbol{\epsilon}_t^{(a)}}^\top|\mathbf{a}_t\right] \right\|\right] \mathbb{E}\left[{\mathcal{Q}_d(\mathbf{X}_t)}^\top  \mathcal{Q}_d(\mathbf{X}_t)\right]\\
            =&\frac{1}{B}\sup_t \left[\left\|\mathbb{E}\left[\boldsymbol{\epsilon}_t^{(o)}{\boldsymbol{\epsilon}_t^{(o)}}^\top|\mathbf{o}_t\right]+\mathbb{E}\left[\boldsymbol{\epsilon}_t^{(a)}{\boldsymbol{\epsilon}_t^{(a)}}^\top|\mathbf{a}_t\right] \right\|\right] \mathbf{H}^{(q)},
        \end{aligned}
    \end{equation*}
    where $\|\cdot\|$ represents the matrix spectral norm.

    Combining the upper bounds for $\boldsymbol{\Sigma}_t^{\epsilon^{(p)}}$, $\boldsymbol{\Sigma}_t^\xi$, $\boldsymbol{\Sigma}_t^{\epsilon^{(o)}}+\boldsymbol{\Sigma}_t^{\epsilon^{(a)}}$, (\ref{eq: bt}) and (\ref{eq: ct}) immediately completes the proof.
\end{proof}

For multiplicative quantization, the explicit dependence of the conditional expectations on $\mathbf{w}_t$ renders Lemma \ref{update rule for eta_t^2} inapplicable to the update rule for $\mathbb{E}[\boldsymbol{\eta}_t\otimes\boldsymbol{\eta}_t]$. We thus propose the following alternative update rule.

\begin{lemma}[\rm Update rule under multiplicative quantization]
    \label{update rule for eta_t^2 multiplicative}
    If there exist $\epsilon_d,\epsilon_l,\epsilon_p,\epsilon_a$ and $\epsilon_o$ such that for any $i\in \{d,l,p,a,o\}$, quantization $\mathcal{Q}_i$ is $\epsilon_i$-multiplicative, then under Assumption \ref{ass1}, Assumption \ref{ass: addd quantized}, Assumption \ref{ass2}, and Assumption \ref{ass3}, it holds
    \begin{equation*}
        \begin{aligned}
            \mathbf{C}_t
            \preceq&\mathbb{E}\left[\left(\mathbf{I}-\frac{1}{B}\gamma {\mathcal{Q}_d(\mathbf{X})}^\top \mathcal{Q}_d(\mathbf{X}) \right)\mathbf{C}_{t-1}\left(\mathbf{I}-\frac{1}{B}\gamma {\mathcal{Q}_d(\mathbf{X})}^\top \mathcal{Q}_d(\mathbf{X}) \right)\right]\\
            +&\tilde{\epsilon}\mathbb{E}\left[\frac{\gamma}{B}{\mathcal{Q}_d(\mathbf{X})}^\top {\mathcal{Q}_d(\mathbf{X})}(\mathbf{B}_{t-1}+\mathbf{C}_{t-1})\frac{\gamma}{B}{\mathcal{Q}_d(\mathbf{X})}^\top\mathcal{Q}_d(\mathbf{X})\right]
            +\gamma^2 {\sigma_M^{(q)}}^2\mathbf{H}^{(q)},\\
            \mathbf{B}_t=& \mathbb{E}\left[\left(\mathbf{I}-\frac{1}{B}\gamma {\mathcal{Q}_d(\mathbf{X})}^\top \mathcal{Q}_d(\mathbf{X}) \right)\mathbf{B}_{t-1}\left(\mathbf{I}-\frac{1}{B}\gamma {\mathcal{Q}_d(\mathbf{X})}^\top \mathcal{Q}_d(\mathbf{X}) \right)\right],
        \end{aligned}
    \end{equation*}
    where
    \begin{gather*}
            \tilde{\epsilon}=8\epsilon_o(1+\epsilon_p)(1+\epsilon_a)+4\epsilon_p+4\epsilon_a(1+\epsilon_p),\\
            {\sigma_M^{(q)}}^2=\frac{(1+4\epsilon_o)\sigma^2}{B} + \frac{\|\mathbf{w}^*\|_\mathbf{H}^2}{1+\epsilon_d}\alpha_B\left(4\epsilon_o[(1+\epsilon_p)(1+\epsilon_a)+1]+2\epsilon_a(1+\epsilon_p)+2\epsilon_p \right).
    \end{gather*}
\end{lemma}
\begin{proof}
    To complete the proof, we merely need to derive the upper bound for $\boldsymbol{\Sigma}_t=\boldsymbol{\Sigma}_t^\xi+\boldsymbol{\Sigma}_t^{\epsilon^{(a)}}+\boldsymbol{\Sigma}_t^{\epsilon^{(o)}}+\boldsymbol{\Sigma}_t^{\epsilon^{(p)}}$. Regarding $\boldsymbol{\Sigma}_t^\xi$, by the computation in the proof of Lemma \ref{update rule for eta_t^2}, i.e., (\ref{Sigma epsilon t}),
    \begin{equation}
        \label{sigma xi t}\boldsymbol{\Sigma}_t^\xi\preceq \frac{\sigma^2}{B}\mathbf{H}^{(q)}.
    \end{equation}
    Regarding $\boldsymbol{\Sigma}_t^{\epsilon^{(p)}}$,
    \begin{equation}
    \label{bound for sigma tp}
        \begin{aligned}
            \boldsymbol{\Sigma}_t^{\epsilon^{(p)}}=&\frac{1}{B^2}\mathbb{E}\left[{\mathcal{Q}_d(\mathbf{X}_t)}^\top {\mathcal{Q}_d(\mathbf{X}_t)}\mathbb{E}\left[\boldsymbol{\epsilon}_{t-1}^{(p)}{\boldsymbol{\epsilon}_{t-1}^{(p)}}^\top \big| \mathbf{w}_{t-1}\right]{\mathcal{Q}_d(\mathbf{X}_t)}^\top\mathcal{Q}_d(\mathbf{X}_t)\right]\\
            =&\frac{\epsilon_p}{B^2}\mathbb{E}\left[{\mathcal{Q}_d(\mathbf{X}_t)}^\top {\mathcal{Q}_d(\mathbf{X}_t)}\mathbf{w}_{t-1}\mathbf{w}_{t-1}^\top{\mathcal{Q}_d(\mathbf{X}_t)}^\top\mathcal{Q}_d(\mathbf{X}_t)\right]\\
            \preceq &\frac{2\epsilon_p}{B^2}\mathbb{E}\left[{\mathcal{Q}_d(\mathbf{X}_t)}^\top {\mathcal{Q}_d(\mathbf{X}_t)}\boldsymbol{\eta}_{t-1}{\boldsymbol{\eta}_{t-1}}^\top{\mathcal{Q}_d(\mathbf{X}_t)}^\top\mathcal{Q}_d(\mathbf{X}_t)\right]\\
            +&\frac{2\epsilon_p}{B^2}\mathbb{E}\left[{\mathcal{Q}_d(\mathbf{X}_t)}^\top {\mathcal{Q}_d(\mathbf{X}_t)}{\mathbf{w}^{(q)}}^*{{\mathbf{w}^{(q)}}^*}^\top{\mathcal{Q}_d(\mathbf{X}_t)}^\top\mathcal{Q}_d(\mathbf{X}_t)\right].
        \end{aligned}
    \end{equation}
    Regarding $\boldsymbol{\Sigma}_t^{\epsilon^{(a)}}$, 
    \begin{equation}
    \label{sigma a t mul}
        \begin{aligned}
            \boldsymbol{\Sigma}_t^{\epsilon^{(a)}}=&\frac{1}{B^2}\mathbb{E}\left[{\mathcal{Q}_d(\mathbf{X}_t)}^\top \boldsymbol{\epsilon}_t^{(a)}{\boldsymbol{\epsilon}_t^{(a)}}^\top \mathcal{Q}_d(\mathbf{X}_t)\right]\\
            =&\frac{\epsilon_a}{B^2}\mathbb{E}\left[{\mathcal{Q}_d(\mathbf{X}_t)}^\top \mathcal{Q}_d(\mathbf{X}_t)\mathbf{w}_{t-1}^{(q)}{\mathbf{w}_{t-1}^{(q)}}^\top{\mathcal{Q}_d(\mathbf{X}_t)}^\top  \mathcal{Q}_d(\mathbf{X}_t)\right]\\
            =&\frac{(1+\epsilon_p)\epsilon_a}{B^2}\mathbb{E}\left[{\mathcal{Q}_d(\mathbf{X}_t)}^\top \mathcal{Q}_d(\mathbf{X}_t)\mathbf{w}_{t-1}{\mathbf{w}_{t-1}}^\top{\mathcal{Q}_d(\mathbf{X}_t)}^\top  \mathcal{Q}_d(\mathbf{X}_t)\right]\\
            \preceq &\frac{2(1+\epsilon_p)\epsilon_a}{B^2}\mathbb{E}\left[{\mathcal{Q}_d(\mathbf{X}_t)}^\top \mathcal{Q}_d(\mathbf{X}_t)\boldsymbol{\eta}_{t-1}{\boldsymbol{\eta}_{t-1}}^\top{\mathcal{Q}_d(\mathbf{X}_t)}^\top  \mathcal{Q}_d(\mathbf{X}_t)\right]\\
            +&\frac{2(1+\epsilon_p)\epsilon_a}{B^2}\mathbb{E}\left[{\mathcal{Q}_d(\mathbf{X}_t)}^\top \mathcal{Q}_d(\mathbf{X}_t){\mathbf{w}^{(q)}}^*{{\mathbf{w}^{(q)}}^*}^\top{\mathcal{Q}_d(\mathbf{X}_t)}^\top  \mathcal{Q}_d(\mathbf{X}_t)\right].
        \end{aligned}
    \end{equation}
    Regarding $\boldsymbol{\Sigma}_t^{\epsilon^{(o)}}$,
    similar to $\boldsymbol{\Sigma}_t^{\epsilon^{(a)}}$, it holds
    \begin{equation*}
        \begin{aligned}
            \boldsymbol{\Sigma}_t^{\epsilon^{(o)}}=&\frac{1}{B^2}\mathbb{E}\left[{\mathcal{Q}_d(\mathbf{X}_t)}^\top \boldsymbol{\epsilon}_t^{(o)}{\boldsymbol{\epsilon}_t^{(o)}}^\top \mathcal{Q}_d(\mathbf{X}_t)\right]\\
            =&\frac{\epsilon_o}{B^2}\mathbb{E}\left[{\mathcal{Q}_d(\mathbf{X}_t)}^\top \mathbf{o}_t{\mathbf{o}_t}^\top \mathcal{Q}_d(\mathbf{X}_t)\right]\\
            \preceq&\frac{2\epsilon_o}{B^2}\mathbb{E}\left[{\mathcal{Q}_d(\mathbf{X}_t)}^\top \mathcal{Q}_l(\mathbf{y}_t)\mathcal{Q}_l(\mathbf{y}_t)^\top \mathcal{Q}_d(\mathbf{X}_t)\right]
            +\frac{2\epsilon_o}{B^2}\mathbb{E}\left[{\mathcal{Q}_d(\mathbf{X}_t)}^\top \mathcal{Q}_a(\mathbf{a}_t)\mathcal{Q}_a(\mathbf{a}_t)^\top \mathcal{Q}_d(\mathbf{X}_t)\right].
        \end{aligned}
    \end{equation*}
    For the second term,
    \begin{equation*}
        \begin{aligned}
            &\frac{2\epsilon_o}{B^2}\mathbb{E}\left[{\mathcal{Q}_d(\mathbf{X}_t)}^\top \mathcal{Q}_a(\mathbf{a}_t)\mathcal{Q}_a(\mathbf{a}_t)^\top \mathcal{Q}_d(\mathbf{X}_t)\right]\\
            \preceq&\frac{2(1+\epsilon_a)\epsilon_o}{B^2}\mathbb{E}\left[{\mathcal{Q}_d(\mathbf{X}_t)}^\top \mathbf{a}_t\mathbf{a}_t^\top \mathcal{Q}_d(\mathbf{X}_t)\right]\\
            \preceq&\frac{4(1+\epsilon_p)(1+\epsilon_a)\epsilon_o}{B^2}\mathbb{E}\left[{\mathcal{Q}_d(\mathbf{X}_t)}^\top {\mathcal{Q}_d(\mathbf{X}_t)}\boldsymbol{\eta}_{t-1}{\boldsymbol{\eta}_{t-1}}^\top{\mathcal{Q}_d(\mathbf{X}_t)}^\top  \mathcal{Q}_d(\mathbf{X}_t)\right]\\
            +&\frac{4(1+\epsilon_p)(1+\epsilon_a)\epsilon_o}{B^2}\mathbb{E}\left[{\mathcal{Q}_d(\mathbf{X}_t)}^\top {\mathcal{Q}_d(\mathbf{X}_t)} {\mathbf{w}^{(q)}}^*{{\mathbf{w}^{(q)}}^*}^\top{\mathcal{Q}_d(\mathbf{X}_t)}^\top \mathcal{Q}_d(\mathbf{X}_t)\right].
        \end{aligned}
    \end{equation*}
    For the first term,
    \begin{equation*}
        \begin{aligned}
            \frac{2\epsilon_o}{B^2}\mathbb{E}\left[{\mathcal{Q}_d(\mathbf{X}_t)}^\top \mathcal{Q}_l(\mathbf{y}_t)\mathcal{Q}_l(\mathbf{y}_t)^\top \mathcal{Q}_d(\mathbf{X}_t)\right]
            \preceq &\frac{4\epsilon_o}{B^2}\mathbb{E}\left[{\mathcal{Q}_d(\mathbf{X}_t)}^\top \boldsymbol{\xi}_t\boldsymbol{\xi}_t^\top \mathcal{Q}_d(\mathbf{X}_t)\right]\\
            +&\frac{4\epsilon_o}{B^2}\mathbb{E}\left[{\mathcal{Q}_d(\mathbf{X}_t)}^\top {\mathcal{Q}_d(\mathbf{X}_t)} {\mathbf{w}^{(q)}}^*{{\mathbf{w}^{(q)}}^*}^\top{\mathcal{Q}_d(\mathbf{X}_t)}^\top \mathcal{Q}_d(\mathbf{X}_t)\right],
        \end{aligned}
    \end{equation*}
    where we use $\boldsymbol{\xi}_t=\mathcal{Q}_l(\mathbf{y}_t)-\mathcal{Q}_d(\mathbf{X}_t)\mathbf{w}^{(q)^*}$. 
    Further, by the bound for $\boldsymbol{\Sigma}_t^\xi$ (\ref{Sigma epsilon t}), we have
    \begin{equation*}
        \frac{1}{B^2}\mathbb{E}\left[{\mathcal{Q}_d(\mathbf{X}_t)}^\top \boldsymbol{\xi}_{t}{\boldsymbol{\xi}_{t}}^\top  \mathcal{Q}_d(\mathbf{X}_t)\right]\preceq \frac{\sigma^2}{B}\mathbf{H}^{(q)},
    \end{equation*}
    it follows that
    \begin{equation}
    \label{sigma o t mul}
        \begin{aligned}
            \boldsymbol{\Sigma}_t^{\epsilon^{(o)}} \preceq &\frac{4\epsilon_o\sigma^2}{B}\mathbf{H}^{(q)}+\frac{4(1+\epsilon_p)(1+\epsilon_a)\epsilon_o}{B^2}\mathbb{E}\left[{\mathcal{Q}_d(\mathbf{X}_t)}^\top {\mathcal{Q}_d(\mathbf{X}_t)}\boldsymbol{\eta}_{t-1}{\boldsymbol{\eta}_{t-1}}^\top{\mathcal{Q}_d(\mathbf{X}_t)}^\top  \mathcal{Q}_d(\mathbf{X}_t)\right]\\
            +&\frac{4\epsilon_o[(1+\epsilon_p)(1+\epsilon_a)+1]}{B^2}\mathbb{E}\left[{\mathcal{Q}_d(\mathbf{X}_t)}^\top {\mathcal{Q}_d(\mathbf{X}_t)} {\mathbf{w}^{(q)}}^*{{\mathbf{w}^{(q)}}^*}^\top{\mathcal{Q}_d(\mathbf{X}_t)}^\top \mathcal{Q}_d(\mathbf{X}_t)\right].
        \end{aligned}
    \end{equation}

    Further, by Assumption \ref{ass2}, it holds
    \begin{equation*}
        \frac{1}{B^2}\mathbb{E}\left[{\mathcal{Q}_d(\mathbf{X}_t)}^\top {\mathcal{Q}_d(\mathbf{X}_t)} {\mathbf{w}^{(q)}}^*{{\mathbf{w}^{(q)}}^*}^\top{\mathcal{Q}_d(\mathbf{X}_t)}^\top \mathcal{Q}_d(\mathbf{X}_t)\right]\preceq \alpha_B\mathrm{tr}\left(\mathbf{H}^{(q)}{\mathbf{w}^{(q)}}^*{{\mathbf{w}^{(q)}}^*}^\top\right)\mathbf{H}^{(q)},
    \end{equation*}
    then together with (\ref{sigma xi t}), (\ref{bound for sigma tp}), (\ref{sigma a t mul}) and (\ref{sigma o t mul}) it holds
    \begin{equation*}
        \begin{aligned}
            \boldsymbol{\Sigma}_t \preceq& \frac{(1+4\epsilon_o)\sigma^2}{B}\mathbf{H}^{(q)} + \alpha_B\left(4\epsilon_o[(1+\epsilon_p)(1+\epsilon_a)+1]+2\epsilon_a(1+\epsilon_p)+2\epsilon_p \right)\mathrm{tr}\left(\mathbf{H}^{(q)}{\mathbf{w}^{(q)}}^*{{\mathbf{w}^{(q)}}^*}^\top \right)\mathbf{H}^{(q)}\\
            +&\frac{4\epsilon_o(1+\epsilon_p)(1+\epsilon_a)+2\epsilon_p+2\epsilon_a(1+\epsilon_p)}{B^2}\mathbb{E}\left[{\mathcal{Q}_d(\mathbf{X}_t)}^\top {\mathcal{Q}_d(\mathbf{X}_t)}\boldsymbol{\eta}_{t-1}{\boldsymbol{\eta}_{t-1}}^\top{\mathcal{Q}_d(\mathbf{X}_t)}^\top\mathcal{Q}_d(\mathbf{X}_t)\right].
        \end{aligned}
    \end{equation*}
    Note that by the definition of multiplicative quantization,
    \begin{equation*}
        \begin{aligned}
            \mathrm{tr}\left(\mathbf{H}^{(q)}{\mathbf{w}^{(q)}}^*{{\mathbf{w}^{(q)}}^*}^\top \right)=\frac{\|\mathbf{w}^*\|_\mathbf{H}^2}{1+\epsilon_d},
        \end{aligned}
    \end{equation*}
    then
    \begin{equation}
    \label{sigma t mul}
        \begin{aligned}
            \boldsymbol{\Sigma}_t \preceq& \left[\frac{(1+4\epsilon_o)\sigma^2}{B} + \frac{\|\mathbf{w}^*\|_\mathbf{H}^2}{1+\epsilon_d}\alpha_B\left(4\epsilon_o[(1+\epsilon_p)(1+\epsilon_a)+1]+2\epsilon_a(1+\epsilon_p)+2\epsilon_p \right)\right]\mathbf{H}^{(q)}\\
            +&\frac{4\epsilon_o(1+\epsilon_p)(1+\epsilon_a)+2\epsilon_p+2\epsilon_a(1+\epsilon_p)}{B^2}\mathbb{E}\left[{\mathcal{Q}_d(\mathbf{X}_t)}^\top {\mathcal{Q}_d(\mathbf{X}_t)}\boldsymbol{\eta}_{t-1}{\boldsymbol{\eta}_{t-1}}^\top{\mathcal{Q}_d(\mathbf{X}_t)}^\top\mathcal{Q}_d(\mathbf{X}_t)\right].
        \end{aligned}
    \end{equation}
    Hence, by (\ref{sigma t mul}), (\ref{eq: ct}) and $\mathbb{E}\left[\boldsymbol{\eta}_t\otimes\boldsymbol{\eta}_t\right] \preceq 2(\mathbf{B}_t+\mathbf{C}_t)$, we have
    \begin{equation*}
        \begin{aligned}
            &\mathbf{C}_t\preceq\mathbb{E}\left[\left(\mathbf{I}-\frac{1}{B}\gamma {\mathcal{Q}_d(\mathbf{X})}^\top \mathcal{Q}_d(\mathbf{X}) \right)\mathbf{C}_{t-1}\left(\mathbf{I}-\frac{1}{B}\gamma {\mathcal{Q}_d(\mathbf{X})}^\top \mathcal{Q}_d(\mathbf{X}) \right)\right]\\
            +&\left[8\epsilon_o(1+\epsilon_p)(1+\epsilon_a)+4\epsilon_p+4\epsilon_a(1+\epsilon_p)\right]\mathbb{E}\left[\frac{\gamma}{B}{\mathcal{Q}_d(\mathbf{X})}^\top {\mathcal{Q}_d(\mathbf{X})}(\mathbf{B}_{t-1}+\mathbf{C}_{t-1})\frac{\gamma}{B}{\mathcal{Q}_d(\mathbf{X})}^\top\mathcal{Q}_d(\mathbf{X})\right]\\
            +&\gamma^2 \left[\frac{(1+4\epsilon_o)\sigma^2}{B} + \frac{\|\mathbf{w}^*\|_\mathbf{H}^2}{1+\epsilon_d}\alpha_B\left(4\epsilon_o[(1+\epsilon_p)(1+\epsilon_a)+1]+2\epsilon_a(1+\epsilon_p)+2\epsilon_p \right)\right]\mathbf{H}^{(q)}.
        \end{aligned}
    \end{equation*}
\end{proof}

Equipped with Lemma \ref{Initial study of R_2}, Lemma \ref{update rule for eta_t^2} and Lemma \ref{update rule for eta_t^2 multiplicative}, we are ready to derive bounds for $R_N^{(0)}$. As shown in \citet{zou2023benign}, we first perform bias-variance decomposition.

\subsection{Bias-Variance Decomposition}
\label{Bias-Variance Decomposition}
As in \citet{zou2023benign}, we perform bias-variance for excess risk, which is summarized as the following lemma. Here we slightly abuse the notations of $\mathbf{B}_t$ and $\mathbf{C}_t$.
\begin{lemma}[\rm Bias-variance decomposition under general quantization]
\label{Bias-Variance Decomposition under General Quantization}
    Under Assumption \ref{ass1}, Assumption \ref{ass: addd quantized}, Assumption \ref{ass2}, and Assumption \ref{ass3},
    \begin{equation*}
        R_N^{(0)}/2\leq \underbrace{\frac{1}{N^2}\cdot\sum_{t=0}^{N-1}\sum_{k=t}^{N-1}\left\langle(\mathbf{I}-\gamma\mathbf{H}^{(q)})^{k-t}\mathbf{H}^{(q)},\mathbf{B}_t\right\rangle}_{\mathrm{bias}}+\underbrace{\frac{1}{N^2}\cdot\sum_{t=0}^{N-1}\sum_{k=t}^{N-1}\left\langle(\mathbf{I}-\gamma\mathbf{H}^{(q)})^{k-t}\mathbf{H}^{(q)},\mathbf{C}_t\right\rangle}_{\mathrm{variance}},
    \end{equation*}
    where
    \begin{equation*}
        \mathbf{B}_t:=(\mathcal{I}-\gamma\mathcal{T}_B^{(q)})^t \circ \mathbf{B}_0, \quad \mathbf{B}_0=\mathbb{E}\left[\boldsymbol{\eta}_0 \otimes \boldsymbol{\eta}_0\right].
    \end{equation*}
    \begin{equation*}
        \mathbf{C}_t := (\mathcal{I}-\gamma \mathcal{T}_B^{(q)})\circ\mathbf{C}_{t-1}+\gamma^2 {\sigma_G^{(q)}}^2\mathbf{H}^{(q)},\quad \mathbf{C}_0=\boldsymbol{0}.
    \end{equation*}
\end{lemma}
\begin{proof}
    By Lemma \ref{Initial study of R_2},
    \begin{equation*}
        \begin{aligned}
            R_N^{(0)}\leq &\frac{1}{N^2}\cdot\sum_{t=0}^{N-1}\sum_{k=t}^{N-1}\left\langle(\mathbf{I}-\gamma\mathbf{H}^{(q)})^{k-t}\mathbf{H}^{(q)},\mathbb{E}[\boldsymbol{\eta}_t\otimes \boldsymbol{\eta}_t]\right\rangle.
        \end{aligned}
    \end{equation*}
    The proof is immediately completed by Lemma \ref{update rule for eta_t^2} and $\mathbb{E}[\boldsymbol{\eta}_t\otimes \boldsymbol{\eta}_t]\preceq 2(\mathbf{B}_t+\mathbf{C}_t).$
\end{proof}

For multiplicative quantization, we can directly deduce from Lemma \ref{Bias-Variance Decomposition under General Quantization} by the update rule under multiplicative quantization (Lemma \ref{update rule for eta_t^2 multiplicative}).
\begin{lemma}[\rm Bias-variance decomposition under multiplicative quantization]
\label{Bias-Variance Decomposition under Multiplicative Quantization}
    Under Assumption \ref{ass1}, Assumption \ref{ass: addd quantized}, Assumption \ref{ass2}, and Assumption \ref{ass3}, if there exist $\epsilon_d,\epsilon_l,\epsilon_p,\epsilon_a$ and $\epsilon_o$ such that for any $i\in \{d,l,p,a,o\}$, quantization $\mathcal{Q}_i$ is $\epsilon_i$-multiplicative, then
    \begin{equation*}
        R_N^{(0)}/2\leq \underbrace{\frac{1}{N^2}\cdot\sum_{t=0}^{N-1}\sum_{k=t}^{N-1}\left\langle(\mathbf{I}-\gamma\mathbf{H}^{(q)})^{k-t}\mathbf{H}^{(q)},\mathbf{B}_t^{(M)}\right\rangle}_{\mathrm{bias}}+\underbrace{\frac{1}{N^2}\cdot\sum_{t=0}^{N-1}\sum_{k=t}^{N-1}\left\langle(\mathbf{I}-\gamma\mathbf{H}^{(q)})^{k-t}\mathbf{H}^{(q)},\mathbf{C}_t^{(M)}\right\rangle}_{\mathrm{variance}},
    \end{equation*}
    where
    \begin{equation*}
        \mathbf{B}_t^{(M)}:=(\mathcal{I}-\gamma\mathcal{T}_B^{(q)}+\tilde{\epsilon}\gamma^2\mathcal{M}_B^{(q)})^t \circ \mathbf{B}_0^{(M)}, \quad \mathbf{B}_0^{(M)}=\mathbb{E}\left[\boldsymbol{\eta}_0 \otimes \boldsymbol{\eta}_0\right].
    \end{equation*}
    \begin{equation*}
        \mathbf{C}_t^{(M)}: = (\mathcal{I}-\gamma \mathcal{T}_B^{(q)}+\tilde{\epsilon}\gamma^2\mathcal{M}_B^{(q)})\circ\mathbf{C}_{t-1}^{(M)}+\gamma^2 {\sigma_M^{(q)}}^2\mathbf{H}^{(q)}, \quad \mathbf{C}_0^{(M)}=0.
    \end{equation*}
\end{lemma}
\begin{proof}
    By Lemma \ref{update rule for eta_t^2 multiplicative},
    \begin{equation*}
        \begin{aligned}
            \mathbf{C}_t
            \preceq&\mathbb{E}\left[\left(\mathbf{I}-\frac{1}{B}\gamma {\mathcal{Q}_d(\mathbf{X})}^\top \mathcal{Q}_d(\mathbf{X}) \right)\mathbf{C}_{t-1}\left(\mathbf{I}-\frac{1}{B}\gamma {\mathcal{Q}_d(\mathbf{X})}^\top \mathcal{Q}_d(\mathbf{X}) \right)\right]\\
            +&\tilde{\epsilon}\mathbb{E}\left[\frac{\gamma}{B}{\mathcal{Q}_d(\mathbf{X})}^\top {\mathcal{Q}_d(\mathbf{X})}(\mathbf{B}_{t-1}+\mathbf{C}_{t-1})\frac{\gamma}{B}{\mathcal{Q}_d(\mathbf{X})}^\top\mathcal{Q}_d(\mathbf{X})\right]
            +\gamma^2 {\sigma_M^{(q)}}^2\mathbf{H}^{(q)},\\
            \mathbf{B}_t=& \mathbb{E}\left[\left(\mathbf{I}-\frac{1}{B}\gamma {\mathcal{Q}_d(\mathbf{X})}^\top \mathcal{Q}_d(\mathbf{X}) \right)\mathbf{B}_{t-1}\left(\mathbf{I}-\frac{1}{B}\gamma {\mathcal{Q}_d(\mathbf{X})}^\top \mathcal{Q}_d(\mathbf{X}) \right)\right].
        \end{aligned}
    \end{equation*}
    Hence,
    \begin{equation*}
        \begin{aligned}
            \mathbb{E}\left[\boldsymbol{\eta}_t \otimes \boldsymbol{\eta}_t\right] \preceq &2(\mathbf{B}_t+\mathbf{C}_t)\\
            \preceq&2\left[(\mathcal{I}-\gamma \mathcal{T}_B^{(q)}+\tilde{\epsilon}\gamma^2\mathcal{M}_B^{(q)})\circ\left(\mathbf{B}_{t-1}+\mathbf{C}_{t-1}\right)+\gamma^2 {\sigma_M^{(q)}}^2\mathbf{H}^{(q)}\right]\\
            \preceq&2\left(\mathbf{B}_t^{(M)}+\mathbf{C}_t^{(M)}\right).
        \end{aligned}
    \end{equation*}
    Applying Lemma \ref{Initial study of R_2} completes the proof.
\end{proof}
\subsection{Bounding the Bias Error}
\label{Bounding the Bias Error}
By Lemma \ref{Bias-Variance Decomposition under General Quantization},
\begin{equation}
\label{bias express}
    \begin{aligned}
        \mathrm{bias}=&\frac{1}{N^{2}}\sum_{t=0}^{N-1}\sum_{k=t}^{N-1}\left\langle(\mathbf{I}-\gamma\mathbf{H}^{(q)})^{k-t}\mathbf{H}^{(q)},\mathbf{B}_{t}\right\rangle \\
        =&\frac{1}{\gamma N^{2}}\sum_{t=0}^{N-1}\left\langle\mathbf{I}-(\mathbf{I}-\gamma\mathbf{H}^{(q)})^{N-t},\mathbf{B}_{t}\right\rangle \\
        \leq&\frac{1}{\gamma N^{2}}\langle\mathbf{I}-(\mathbf{I}-\gamma\mathbf{H}^{(q)})^{N},\sum_{t=0}^{N-1}\mathbf{B}_{t}\rangle.
    \end{aligned}
\end{equation}
For $1 \leq n\leq N$, let $\mathbf{S}_n=\sum_{t=0}^{n-1}\mathbf{B}_t,\ \mathbf{S}_n^{(M)}=\sum_{t=0}^{n-1}\mathbf{B}_t^{(M)}$, then we only need to bound $\mathbf{S}_N$ and $\mathbf{S}_N^{(M)}$ to bound bias term under general quantization and multiplicative quantization, respectively. We first derive the update rule for $\mathbf{S}_t$ and $\mathbf{S}_t^{(M)}$.
\begin{lemma}[\rm Initial study of $\mathbf{S}_t$]
\label{(Initial Study of $S_t$)}
    For $1 \leq t\leq N$,
    \begin{equation*}
        \mathbf{S}_t \preceq (\mathcal{I}-\gamma\tilde{\mathcal{T}}^{(q)})\circ\mathbf{S}_{t-1}+\gamma^{2}\mathcal{M}_B^{(q)}\circ\mathbf{S}_{N}+\mathbf{B}_{0}.
    \end{equation*}
\end{lemma}
\begin{proof}
    By definition,
    \begin{equation}
        \label{22222}
        \begin{aligned}
            \mathbf{S}_t=&\sum_{k=0}^{t-1}(\mathcal{I}-\gamma\mathcal{T}_B^{(q)})^k\circ\mathbf{B}_0\\
            =&(\mathcal{I}-\gamma\mathcal{T}_B^{(q)})\circ\left(\sum_{k=1}^{t-1}(\mathcal{I}-\gamma\mathcal{T}_B^{(q)})^{k-1}\circ\mathbf{B}_0\right)+\mathbf{B}_0\\
            =&(\mathcal{I}-\gamma\mathcal{T}_B^{(q)})\circ\mathbf{S}_{t-1}+\mathbf{B}_0.
        \end{aligned}
    \end{equation}
    Then we convert $\mathcal{T}_B^{(q)}$ to $\widetilde{\mathcal{T}}^{(q)}$. By (\ref{22222}),
    \begin{equation*}
        \begin{aligned}
            \mathbf{S}_{t} =& (\mathcal{I}-\gamma\mathcal{T}_B^{(q)})\circ\mathbf{S}_{t-1}+\mathbf{B}_0\\
            =& (\mathcal{I}-\gamma\widetilde{\mathcal{T}}^{(q)})\circ\mathbf{S}_{t-1}+\gamma (\widetilde{\mathcal{T}}^{(q)}-{\mathcal{T}}_B^{(q)})\circ\mathbf{S}_{t-1}+\mathbf{B}_0\\
            =& (\mathcal{I}-\gamma\widetilde{\mathcal{T}}^{(q)})\circ\mathbf{S}_{t-1}+\gamma^2 ({\mathcal{M}}_B^{(q)}-\widetilde{\mathcal{M}}^{(q)})\circ\mathbf{S}_{t-1}+\mathbf{B}_0\\
            \preceq&(\mathcal{I}-\gamma\tilde{\mathcal{T}}^{(q)})\circ\mathbf{S}_{t-1}+\gamma^{2}\mathcal{M}_B^{(q)}\circ\mathbf{S}_{N}+\mathbf{B}_{0},
        \end{aligned}
    \end{equation*}
    where the third equality holds by the definition of linear operators.
\end{proof}
\begin{lemma}[\rm Initial study of $\mathbf{S}_t^{(M)}$]
\label{(Initial Study of $S_t$) mu}
    For $1 \leq t\leq N$,
    \begin{equation*}           
        \mathbf{S}_{t}^{(M)}\preceq(\mathcal{I}-\gamma\tilde{\mathcal{T}}^{(q)})\circ\mathbf{S}_{t-1}^{(M)}+(1+\tilde{\epsilon})\gamma^{2}\mathcal{M}_B^{(q)}\circ\mathbf{S}_{N}^{(M)}+\mathbf{B}_{0}.
    \end{equation*}
\end{lemma}
\begin{proof}
    The proof is similar to the proof for Lemma \ref{(Initial Study of $S_t$)}.
    \begin{equation*}
        \begin{aligned}
            \mathbf{S}_{t}^{(M)} =& (\mathcal{I}-\gamma\mathcal{T}_B^{(q)}+\tilde{\epsilon}\gamma^2\mathcal{M}_B^{(q)})\circ\mathbf{S}_{t-1}^{(M)}+\mathbf{B}_0\\
            =& (\mathcal{I}-\gamma\widetilde{\mathcal{T}}^{(q)})\circ\mathbf{S}_{t-1}^{(M)}+\gamma (\widetilde{\mathcal{T}}^{(q)}-{\mathcal{T}}_B^{(q)})\circ\mathbf{S}_{t-1}^{(M)}+\tilde{\epsilon}\gamma^2\mathcal{M}_B^{(q)}\circ \mathbf{S}_{t-1}^{(M)}+\mathbf{B}_0\\
            =& (\mathcal{I}-\gamma\widetilde{\mathcal{T}}^{(q)})\circ\mathbf{S}_{t-1}^{(M)}+\gamma^2 ((1+\tilde{\epsilon}){\mathcal{M}}_B^{(q)}-\widetilde{\mathcal{M}}^{(q)})\circ\mathbf{S}_{t-1}^{(M)}+\mathbf{B}_0\\
            \preceq&(\mathcal{I}-\gamma\tilde{\mathcal{T}}^{(q)})\circ\mathbf{S}_{t-1}^{(M)}+(1+\tilde{\epsilon})\gamma^{2}\mathcal{M}_B^{(q)}\circ\mathbf{S}_{N}^{(M)}+\mathbf{B}_{0}.
        \end{aligned}
    \end{equation*}
\end{proof}

\begin{lemma}[\rm A bound for $\mathcal{M}_B^{(q)}\circ\mathbf{S}_t$]
    For $1 \leq t\leq N$, under Assumption \ref{ass1}, Assumption \ref{ass: addd quantized}, Assumption \ref{ass2}, and Assumption \ref{ass3}, if $\gamma < \frac{1}{\alpha_B\mathrm{tr}(\mathbf{H}^{(q)})}$, then
    \label{A bound for M_B^{(q)} S_t}
    \begin{equation*}
         \mathcal{M}_B^{(q)} \circ \mathbf{S}_t \preceq \frac{\alpha_B\cdot\mathrm{tr}\left(\left[\mathcal{I}-(\mathcal{I}-\gamma\widetilde{\mathcal{T}}^{(q)})^t\right]\circ\mathbf{B}_0\right)}{\gamma(1-\gamma\alpha_B\operatorname{tr}(\mathbf{H}^{(q)}))}\cdot\mathbf{H}^{(q)}.
    \end{equation*}
\end{lemma}
\begin{proof}
    We prove by deriving a crude bound for $\mathbf{S}_t$ and applying $\mathcal{M}_B^{(q)}$ to this crude bound. Take summation via the update rule, we have
    \begin{equation*}
        \mathbf{S}_{t} =\sum_{k=0}^{t-1}(\mathcal{I}-\gamma\mathcal{T}_B^{(q)})^k\circ\mathbf{B}_0=\gamma^{-1}{\mathcal{T}^{(q)}_B}^{-1}\circ\left[\mathcal{I}-(\mathcal{I}-\gamma\mathcal{T}_B^{(q)})^t\right]\circ\mathbf{B}_0.
    \end{equation*}
    Note that
    \begin{equation*}
        \mathcal{I}-\gamma\widetilde{\mathcal{T}}^{(q)} \preceq \mathcal{I}-\gamma\mathcal{T}_B^{(q)}, \quad (\mathcal{I}-(\mathcal{I}-\gamma\mathcal{T}_B^{(q)})^t) \preceq (\mathcal{I}-(\mathcal{I}-\gamma\widetilde{\mathcal{T}}^{(q)})^t),
    \end{equation*}
    and further note that ${\mathcal{T}_B^{(q)}}^{-1}$ is a PSD mapping \footnote{${\mathcal{T}_B^{(q)}}^{-1}$ is a PSD mapping under the condition that $\gamma < \frac{1}{\alpha_B\mathrm{tr}(\mathbf{H}^{(q)})}$, which can be directly deduced by Lemma B.1 in \citet{zou2023benign}. We omit the proof here for simplicity.}, and $[\mathcal{I}-(\mathcal{I}-\gamma\widetilde{\mathcal{T}}^{(q)})^t]\circ\mathbf{B}_0$ is a PSD matrix, we obtain
    \begin{equation*}
        \mathbf{S}_t\preceq\gamma^{-1}{\mathcal{T}_B^{(q)}}^{-1}\circ(\mathcal{I}-(\mathcal{I}-\gamma\widetilde{\mathcal{T}}^{(q)})^t)\circ\mathbf{B}_0.
    \end{equation*}
    For simplicity, we denote $\mathbf{A}:=(\mathcal{I}-(\mathcal{I}-\gamma\widetilde{\mathcal{T}}^{(q)})^t)\circ\mathbf{B}_0$. We then tackle ${\mathcal{T}_B^{(q)}}^{-1}\circ\mathbf{A}$. To be specific, we apply $\widetilde{\mathcal{T}}^{(q)}$.
    \begin{equation*}
        \begin{aligned}
            \widetilde{\mathcal{T}}^{(q)}\circ{\mathcal{T}_B^{(q)}}^{-1}\circ\mathbf{A}\preceq \gamma\mathcal{M}_B^{(q)}\circ{\mathcal{T}_B^{(q)}}^{-1}\circ\mathbf{A}+\mathbf{A}.
        \end{aligned}
    \end{equation*}
    Therefore,
    \begin{equation*}
        \begin{aligned}
            {\mathcal{T}_B^{(q)}}^{-1}\circ\mathbf{A} \preceq \gamma{(\widetilde{\mathcal{T}}^{(q)})}^{-1}\circ\mathcal{M}_B^{(q)}\circ{\mathcal{T}_B^{(q)}}^{-1}\circ\mathbf{A}+(\widetilde{\mathcal{T}}^{(q)})^{-1}\circ\mathbf{A}.
        \end{aligned}
    \end{equation*}

    Then we apply $\mathcal{M}_B^{(q)}$ on both sides.
    \begin{equation}
        \begin{aligned}
        \label{11111}
            \mathcal{M}_B^{(q)} \circ ({\mathcal{T}_B^{(q)}}^{-1}\circ\mathbf{A}) &\preceq \mathcal{M}_B^{(q)} \circ\gamma{(\widetilde{\mathcal{T}}^{(q)})}^{-1}\circ\mathcal{M}_B^{(q)}\circ{\mathcal{T}_B^{(q)}}^{-1}\circ\mathbf{A}+ \mathcal{M}_B^{(q)} \circ(\widetilde{\mathcal{T}}^{(q)})^{-1}\circ\mathbf{A}\\
            &\preceq\sum_{t=0}^\infty(\gamma\mathcal{M}_B^{(q)}\circ{(\widetilde{\mathcal{T}}^{(q)})}^{-1})^t\circ(\mathcal{M}_B^{(q)}\circ{(\widetilde{\mathcal{T}}^{(q)})}^{-1}\circ\mathbf{A}) \ \text{(By recursion)}.
        \end{aligned}
    \end{equation}
    By Assumption \ref{ass2}, 
    \begin{equation*}
        \begin{aligned}
            {\mathcal{M}_B^{(q)}}\circ{(\widetilde{\mathcal{T}}^{(q)})}^{-1}\circ\mathbf{A}&\preceq\alpha_B\operatorname{tr}(\mathbf{H}^{(q)}{(\widetilde{\mathcal{T}}^{(q)})}^{-1}\circ\mathbf{A})\mathbf{H}^{(q)}\\
            &= \alpha_B \gamma\operatorname{tr}\left(\sum_{t=0}^{\infty}\mathbf{H}^{(q)}(\mathbf{I}-\gamma\mathbf{H}^{(q)})^{t}\mathbf{A}(\mathbf{I}-\gamma\mathbf{H}^{(q)})^{t}\right) \mathbf{H}^{(q)}\\
            &= \alpha_B \mathrm{tr}\left(\mathbf{H}^{(q)}(2\mathbf{H}^{(q)}-\gamma(\mathbf{H}^{(q)})^{2})^{-1}\mathbf{A}\right)\mathbf{H}^{(q)}\\
            &\preceq \alpha_B \mathrm{tr}(\mathbf{A})\mathbf{H}^{(q)},
        \end{aligned}
    \end{equation*}
    where the first equality holds by the definition of $\widetilde{\mathcal{T}}^{(q)}$ and the last inequality requires the condition that $\gamma < \frac{1}{\alpha_B\mathrm{tr}(\mathbf{H}^{(q)})}$. Hence, by (\ref{11111}), and further by ${(\widetilde{\mathcal{T}}^{(q)})}^{-1}\mathbf{H}^{(q)}\preceq\mathbf{I}$ and $\mathcal{M}_B^{(q)}\circ\mathbf{I}\preceq\alpha_B\operatorname{tr}(\mathbf{H}^{(q)})\mathbf{H}^{(q)}$, we obtain
    \begin{equation*}
        \begin{aligned}
            \mathcal{M}_B^{(q)} \circ ({\mathcal{T}_B^{(q)}}^{-1}\circ\mathbf{A}) &\preceq\sum_{t=0}^\infty(\gamma\mathcal{M}_B^{(q)}\circ{(\widetilde{\mathcal{T}}^{(q)})}^{-1})^t\circ(\mathcal{M}_B^{(q)}\circ{(\widetilde{\mathcal{T}}^{(q)})}^{-1}\circ\mathbf{A})\\
            &\preceq\alpha_B\operatorname{tr}(\mathbf{A})\sum_{t=0}^{\infty}(\gamma\alpha_B\operatorname{tr}(\mathbf{H}^{(q)}))^t\mathbf{H}^{(q)}\\
            &\preceq\frac{\alpha_B\operatorname{tr}(\mathbf{A})}{1-\gamma\alpha_B\operatorname{tr}(\mathbf{H}^{(q)})}\cdot\mathbf{H}^{(q)}.
        \end{aligned}
    \end{equation*}
    Therefore,
    \begin{equation*}
        \begin{aligned}
            \mathcal{M}_B^{(q)} \circ \mathbf{S}_t \preceq \gamma^{-1}\frac{\alpha_B\operatorname{tr}(\mathbf{A})}{1-\gamma\alpha_B\operatorname{tr}(\mathbf{H}^{(q)})}\cdot\mathbf{H}^{(q)}=\frac{\alpha_B\cdot\mathrm{tr}\left(\left[\mathcal{I}-(\mathcal{I}-\gamma\widetilde{\mathcal{T}}^{(q)})^t\right]\circ\mathbf{B}_0\right)}{\gamma(1-\gamma\alpha_B\operatorname{tr}(\mathbf{H}^{(q)}))}\cdot\mathbf{H}^{(q)}.
        \end{aligned}
    \end{equation*}
\end{proof}

\begin{lemma}[\rm A bound for $\mathcal{M}_B^{(q)}\circ\mathbf{S}_t^{(M)}$]
    \label{A bound for M_B^{(q)} S_t^{(M)}}
    For $1 \leq t\leq N$, under Assumption \ref{ass1}, Assumption \ref{ass: addd quantized}, Assumption \ref{ass2}, and Assumption \ref{ass3}, if $\gamma < \frac{1}{(1+\tilde{\epsilon})\alpha_B\mathrm{tr}(\mathbf{H}^{(q)})}$,
    \begin{equation*}
         \mathcal{M}_B^{(q)} \circ \mathbf{S}_t^{(M)} \preceq \frac{\alpha_B\cdot\mathrm{tr}\left(\left[\mathcal{I}-(\mathcal{I}-\gamma\widetilde{\mathcal{T}}^{(q)})^t\right]\circ\mathbf{B}_0\right)}{\gamma(1-(1+\tilde{\epsilon})\gamma\alpha_B\operatorname{tr}(\mathbf{H}^{(q)}))}\cdot\mathbf{H}^{(q)}.
    \end{equation*}
\end{lemma}

\begin{proof}
    The first step is to derive a crude bound for $\mathbf{S}_t^{(M)}$. Take summation via the update rule, we have \footnote{$({\mathcal{T}_B^{(q)}}-\tilde{\epsilon}\gamma \mathcal{M}_B^{(q)})^{-1}$ is a PSD mapping under the condition that $\gamma < \frac{1}{(1+\tilde{\epsilon})\alpha_B\mathrm{tr}(\mathbf{H}^{(q)})}$, which can be directly deduced by Lemma B.1 in \citet{zou2023benign}. We omit the proof here for simplicity.}
    \begin{equation*}
        \mathbf{S}_t^{(M)}=\sum_{k=0}^{t-1}(\mathcal{I}-\gamma\mathcal{T}_B^{(q)}+\tilde{\epsilon}\gamma^2\mathcal{M}_B^{(q)})^k\circ\mathbf{B}_0=\gamma^{-1}({\mathcal{T}^{(q)}_B}-\tilde{\epsilon}\gamma\mathcal{M}_B^{(q)})^{-1}\circ\left[\mathcal{I}-(\mathcal{I}-\gamma\mathcal{T}_B^{(q)}+\tilde{\epsilon}\gamma^2\mathcal{M}_B^{(q)})^t\right]\circ\mathbf{B}_0.
    \end{equation*}
    Note that
    \begin{equation*}
        \mathcal{I}-\gamma\widetilde{\mathcal{T}}^{(q)} \preceq \mathcal{I}-\gamma\mathcal{T}_B^{(q)}, \quad (\mathcal{I}-(\mathcal{I}-\gamma\mathcal{T}_B^{(q)}+\tilde{\epsilon}\gamma^2\mathcal{M}_B^{(q)})^t) \preceq (\mathcal{I}-(\mathcal{I}-\gamma\widetilde{\mathcal{T}}^{(q)}+\tilde{\epsilon}\gamma^2\mathcal{M}_B^{(q)})^t),
    \end{equation*}
    we obtain
    \begin{equation*}
        \mathbf{S}_t^{(M)}\preceq\gamma^{-1}({\mathcal{T}_B^{(q)}}-\tilde{\epsilon}\gamma \mathcal{M}_B^{(q)})^{-1}\circ(\mathcal{I}-(\mathcal{I}-\gamma\widetilde{\mathcal{T}}^{(q)}+\tilde{\epsilon}\gamma^2\mathcal{M}_B^{(q)})^t)\circ\mathbf{B}_0.
    \end{equation*}
    Denote $\mathbf{A}:=(\mathcal{I}-(\mathcal{I}-\gamma\widetilde{\mathcal{T}}^{(q)}+\tilde{\epsilon}\gamma^2\mathcal{M}_B^{(q)})^t)\circ\mathbf{B}_0$, then applying $\widetilde{\mathcal{T}}^{(q)}$
    \begin{equation*}
        \begin{aligned}
            \widetilde{\mathcal{T}}^{(q)}\circ({\mathcal{T}_B^{(q)}}-\tilde{\epsilon}\gamma \mathcal{M}_B^{(q)})^{-1}\circ\mathbf{A}&\preceq(1+\tilde{\epsilon})\gamma\mathcal{M}_B^{(q)}\circ({\mathcal{T}_B^{(q)}}-\tilde{\epsilon}\gamma \mathcal{M}_B^{(q)})^{-1}\circ\mathbf{A}+\mathbf{A}.
        \end{aligned}
    \end{equation*}
    Therefore
    \begin{equation*}
        \begin{aligned}
            ({\mathcal{T}_B^{(q)}}-\tilde{\epsilon}\gamma \mathcal{M}_B^{(q)})^{-1}\circ\mathbf{A} \preceq (1+\tilde{\epsilon})\gamma{(\widetilde{\mathcal{T}}^{(q)})}^{-1}\circ\mathcal{M}_B^{(q)}\circ({\mathcal{T}_B^{(q)}}-\tilde{\epsilon}\gamma \mathcal{M}_B^{(q)})^{-1}\circ\mathbf{A}+(\widetilde{\mathcal{T}}^{(q)})^{-1}\circ\mathbf{A}.
        \end{aligned}
    \end{equation*}

    Then we undertake the second step, applying $\mathcal{M}_B^{(q)}$ on both sides.
    \begin{equation}
        \label{11111 M}
        \mathcal{M}_B^{(q)} \circ({\mathcal{T}_B^{(q)}}-\tilde{\epsilon}\gamma \mathcal{M}_B^{(q)})^{-1}\circ\mathbf{A}\preceq\sum_{t=0}^\infty((1+\tilde{\epsilon})\gamma\mathcal{M}_B^{(q)}\circ{(\widetilde{\mathcal{T}}^{(q)})}^{-1})^t\circ(\mathcal{M}_B^{(q)}\circ{(\widetilde{\mathcal{T}}^{(q)})}^{-1}\circ\mathbf{A}).
    \end{equation}
    By Assumption \ref{ass2}, 
    \begin{equation}
    \label{d16}
        \begin{aligned}
            {\mathcal{M}_B^{(q)}}\circ{(\widetilde{\mathcal{T}}^{(q)})}^{-1}\circ\mathbf{A}&\preceq\alpha_B\operatorname{tr}(\mathbf{H}^{(q)}{(\widetilde{\mathcal{T}}^{(q)})}^{-1}\circ\mathbf{A})\mathbf{H}^{(q)}\\
            &= \alpha_B \gamma\operatorname{tr}\left(\sum_{t=0}^{\infty}\mathbf{H}^{(q)}(\mathbf{I}-\gamma\mathbf{H}^{(q)})^{t}\mathbf{A}(\mathbf{I}-\gamma\mathbf{H}^{(q)})^{t}\right) \mathbf{H}^{(q)}\\
            &= \alpha_B \mathrm{tr}\left(\mathbf{H}^{(q)}(2\mathbf{H}^{(q)}-\gamma(\mathbf{H}^{(q)})^{2})^{-1}\mathbf{A}\right)\mathbf{H}^{(q)}\\
            &\preceq \alpha_B \mathrm{tr}(\mathbf{A})\mathbf{H}^{(q)},
        \end{aligned}
    \end{equation}
    where the last inequality requires the condition that $\gamma < \frac{1}{\alpha_B\mathrm{tr}(\mathbf{H}^{(q)})}$.
    Hence, by (\ref{11111 M}), (\ref{d16}), and further by ${(\widetilde{\mathcal{T}}^{(q)})}^{-1}\mathbf{H}^{(q)}\preceq\mathbf{I}$ and $\mathcal{M}_B^{(q)}\circ\mathbf{I}\preceq\alpha_B\operatorname{tr}(\mathbf{H}^{(q)})\mathbf{H}^{(q)}$, we obtain
    \begin{equation*}
        \begin{aligned}
            \mathcal{M}_B^{(q)} \circ (({\mathcal{T}_B^{(q)}}-\tilde{\epsilon}\gamma \mathcal{M}_B^{(q)})^{-1}\circ\mathbf{A}) &\preceq\sum_{t=0}^\infty((1+\tilde{\epsilon})\gamma\mathcal{M}_B^{(q)}\circ{(\widetilde{\mathcal{T}}^{(q)})}^{-1})^t\circ(\mathcal{M}_B^{(q)}\circ{(\widetilde{\mathcal{T}}^{(q)})}^{-1}\circ\mathbf{A})\\
            &\preceq\alpha_B\operatorname{tr}(\mathbf{A})\sum_{t=0}^{\infty}((1+\tilde{\epsilon})\gamma\alpha_B\operatorname{tr}(\mathbf{H}^{(q)}))^t\mathbf{H}^{(q)}\\
            &\preceq\frac{\alpha_B\operatorname{tr}(\mathbf{A})}{1-(1+\tilde{\epsilon})\gamma\alpha_B\operatorname{tr}(\mathbf{H}^{(q)})}\cdot\mathbf{H}^{(q)}.
        \end{aligned}
    \end{equation*}
    Therefore,
    \begin{equation*}
        \begin{aligned}
            \mathcal{M}_B^{(q)} \circ \mathbf{S}_t^{(M)} \preceq \gamma^{-1}\frac{\alpha_B\operatorname{tr}(\mathbf{A})}{1-(1+\tilde{\epsilon})\gamma\alpha_B\operatorname{tr}(\mathbf{H}^{(q)})}\cdot\mathbf{H}^{(q)}\preceq \frac{\alpha_B\cdot\mathrm{tr}\left(\left[\mathcal{I}-(\mathcal{I}-\gamma\widetilde{\mathcal{T}}^{(q)})^t\right]\circ\mathbf{B}_0\right)}{\gamma(1-(1+\tilde{\epsilon})\gamma\alpha_B\operatorname{tr}(\mathbf{H}^{(q)}))}\cdot\mathbf{H}^{(q)}.
        \end{aligned}
    \end{equation*}
\end{proof}

By Lemma \ref{(Initial Study of $S_t$)}, Lemma \ref{(Initial Study of $S_t$) mu}, Lemma \ref{A bound for M_B^{(q)} S_t} and Lemma \ref{A bound for M_B^{(q)} S_t^{(M)}}, we can provide a refined bound for $\mathbf{S}_t$ and $\mathbf{S}_t^{(M)}$. Then we are ready to bound the bias error.

\begin{lemma}[\rm A bound for bias under general quantization]
    \label{bias R2 bound}
    Under Assumption \ref{ass1}, Assumption \ref{ass: addd quantized}, Assumption \ref{ass2}, and Assumption \ref{ass3}, if the stepsize satisfies $\gamma < \frac{1}{\alpha_B\mathrm{tr}(\mathbf{H}^{(q)})}$, then
    \begin{equation*}
        \begin{aligned}
            \mathrm{bias}
            \leq & \frac{2\alpha_B\left(\|\mathbf{w}_0-{\mathbf{w}^{(q)}}^*\|_{\mathbf{I}_{0:k^*}^{(q)}}^2+N\gamma\|\mathbf{w}_0-{\mathbf{w}^{(q)}}^*\|_{\mathbf{H}_{k^*:\infty}^{(q)}}^2\right)}{N\gamma(1-\gamma\alpha_B\operatorname{tr}(\mathbf{H}^{(q)}))}\cdot\left(\frac{k^*}{N}+N\gamma^2\sum_{i>k^*}(\lambda_i^{(q)})^2\right)\\
            +& \frac{1}{\gamma^2N^2}\cdot\|\mathbf{w}_0-{\mathbf{w}^{(q)}}^*\|_{(\mathbf{H}_{0:k^*}^{(q)})^{-1}}^2+\|\mathbf{w}_0-{\mathbf{w}^{(q)}}^*\|_{\mathbf{H}_{k^*:\infty}^{(q)}}^2.
        \end{aligned}
    \end{equation*}
\end{lemma}

\begin{proof}
    Recalling Lemma \ref{(Initial Study of $S_t$)}, we can derive a refined upper bound for $\mathbf{S}_t$ by Lemma \ref{A bound for M_B^{(q)} S_t}:
    \begin{equation}
    \label{Bound for S_t}
        \begin{aligned}
            \mathbf{S}_t\preceq&(\mathcal{I}-\gamma\tilde{\mathcal{T}}^{(q)})\circ\mathbf{S}_{t-1}+\gamma^{2}\mathcal{M}_B^{(q)}\circ\mathbf{S}_{N}+\mathbf{B}_{0}\\
            \preceq&(\mathcal{I}-\gamma\tilde{\mathcal{T}}^{(q)})\circ\mathbf{S}_{t-1}+\frac{\gamma\alpha_B\cdot\mathrm{tr}\left(\left[\mathcal{I}-(\mathcal{I}-\gamma\widetilde{\mathcal{T}}^{(q)})^N\right]\circ\mathbf{B}_0\right)}{(1-\gamma\alpha_B\operatorname{tr}(\mathbf{H}^{(q)}))}\cdot\mathbf{H}^{(q)}+\mathbf{B}_{0}\\
            \preceq&\sum_{k=0}^{t-1}(\mathcal{I}-\gamma\tilde{\mathcal{T}}^{(q)})^k\left(\frac{\gamma\alpha_B\cdot\mathrm{tr}\left(\left[\mathcal{I}-(\mathcal{I}-\gamma\widetilde{\mathcal{T}}^{(q)})^N\right]\circ\mathbf{B}_0\right)}{(1-\gamma\alpha_B\operatorname{tr}(\mathbf{H}^{(q)}))}\cdot\mathbf{H}^{(q)}+\mathbf{B}_{0}\right)\\
            =&\sum_{k=0}^{t-1}(\mathbf{I}-\gamma\mathbf{H}^{(q)})^k\left(\frac{\gamma\alpha_B\cdot\mathrm{tr}\left(\mathbf{B}_0-(\mathbf{I}-\gamma \mathbf{H}^{(q)})^N\mathbf{B}_0(\mathbf{I}-\gamma \mathbf{H}^{(q)})^N\right)}{(1-\gamma\alpha_B\operatorname{tr}(\mathbf{H}^{(q)}))}\cdot\mathbf{H}^{(q)}+\mathbf{B}_{0}\right)(\mathbf{I}-\gamma\mathbf{H}^{(q)})^k.
        \end{aligned}
    \end{equation}

    Before providing our upper bound for the bias error, we denote \begin{equation*}
        \mathbf{B}_{a,b}:=\mathbf{B}_a-(\mathbf{I}-\gamma\mathbf{H}^{(q)})^{b-a}\mathbf{B}_a(\mathbf{I}-\gamma\mathbf{H}^{(q)})^{b-a}.
    \end{equation*}
    Then by (\ref{bias express}) and (\ref{Bound for S_t}),
    \begin{equation*}
        \begin{aligned}
            \mathrm{bias}\leq&\frac{1}{\gamma N^{2}}\langle\mathbf{I}-(\mathbf{I}-\gamma\mathbf{H}^{(q)})^{N},\sum_{t=0}^{N-1}\mathbf{B}_{t}\rangle\\
            \leq&\frac{1}{\gamma N^{2}}\sum_{k=0}^{N-1}\left\langle\mathbf{I}-(\mathbf{I}-\gamma\mathbf{H}^{(q)})^{N},(\mathbf{I}-\gamma\mathbf{H}^{(q)})^k\left(\frac{\gamma\alpha_B\cdot\mathrm{tr}\left(\mathbf{B}_{0,N}\right)}{1-\gamma\alpha_B\operatorname{tr}(\mathbf{H}^{(q)})}\cdot\mathbf{H}^{(q)}+\mathbf{B}_{0}\right)(\mathbf{I}-\gamma\mathbf{H}^{(q)})^k\right\rangle\\
            =&\frac{1}{\gamma N^{2}}\sum_{k=0}^{N-1}\left\langle(\mathbf{I}-\gamma\mathbf{H}^{(q)})^{2k}-(\mathbf{I}-\gamma\mathbf{H}^{(q)})^{N+2k},\left(\frac{\gamma\alpha_B\cdot\mathrm{tr}\left(\mathbf{B}_{0,N}\right)}{1-\gamma\alpha_B\operatorname{tr}(\mathbf{H}^{(q)})}\cdot\mathbf{H}^{(q)}+\mathbf{B}_{0}\right)\right\rangle.
        \end{aligned}
    \end{equation*}
    Note that 
    \begin{equation*}
        \begin{aligned}
            (\mathbf{I}-\gamma\mathbf{H}^{(q)})^{2k}-(\mathbf{I}-\gamma\mathbf{H}^{(q)})^{N+2k} & =\left(\mathbf{I}-\gamma\mathbf{H}^{(q)}\right)^{k}\left(\left(\mathbf{I}-\gamma\mathbf{H}^{(q)}\right)^{k}-\left(\mathbf{I}-\gamma\mathbf{H}^{(q)}\right)^{N+k}\right) \\
            & \preceq(\mathbf{I}-\gamma\mathbf{H}^{(q)})^{k}-(\mathbf{I}-\gamma\mathbf{H}^{(q)})^{N+k},
        \end{aligned}
    \end{equation*}
    we obtain
    \begin{equation*}
        \mathrm{bias}\leq\frac{1}{\gamma N^{2}}\sum_{k=0}^{N-1}\left\langle(\mathbf{I}-\gamma\mathbf{H}^{(q)})^{k}-(\mathbf{I}-\gamma\mathbf{H}^{(q)})^{N+k},\frac{\gamma\alpha_B\cdot\mathrm{tr}\left(\mathbf{B}_{0,N}\right)}{1-\gamma\alpha_B\operatorname{tr}(\mathbf{H}^{(q)})}\cdot\mathbf{H}^{(q)}+\mathbf{B}_{0}\right\rangle.
    \end{equation*}
    
    Therefore, it suffices to upper bound the following two terms
    \begin{equation*}
        \begin{aligned}
            & I_{1}=\frac{\alpha_B\operatorname{tr}(\mathbf{B}_{0,N})}{N^2(1-\gamma\alpha\operatorname{tr}(\mathbf{H}^{(q)}))}\sum_{k=0}^{N-1}\left\langle(\mathbf{I}-\gamma\mathbf{H}^{(q)})^k-(\mathbf{I}-\gamma\mathbf{H}^{(q)})^{N+k},\mathbf{H}^{(q)}\right\rangle, \\
            & I_{2}=\frac{1}{\gamma N^2}\sum_{k=0}^{N-1}\left\langle(\mathbf{I}-\gamma\mathbf{H}^{(q)})^k-(\mathbf{I}-\gamma\mathbf{H}^{(q)})^{N+k},\mathbf{B}_0\right\rangle.
        \end{aligned}
    \end{equation*}

    Regarding $I_1$, since $\mathbf{H}^{(q)}$ and $\mathbf{I}-\gamma\mathbf{H}^{(q)}$ can be diagonalized simultaneously, 
    \begin{equation*}
        \begin{aligned}
            I_{1}& =\frac{\alpha_B\operatorname{tr}(\mathbf{B}_{0,N})}{N^2(1-\gamma\alpha_B\operatorname{tr}(\mathbf{H}^{(q)}))}\sum_{k=0}^{N-1}\sum_i\left[(1-\gamma\lambda_i^{(q)})^k-(1-\gamma\lambda_i^{(q)})^{N+k}\right]\lambda_i^{(q)}\\
            & =\frac{\alpha_B\operatorname{tr}(\mathbf{B}_{0,N})}{\gamma N^2(1-\gamma\alpha_B\operatorname{tr}(\mathbf{H}^{(q)}))}\sum_i\left[1-(1-\gamma\lambda_i^{(q)})^N\right]^2 \\
            & \leq\frac{\alpha_B\operatorname{tr}(\mathbf{B}_{0,N})}{\gamma N^2(1-\gamma\alpha_B\operatorname{tr}(\mathbf{H}^{(q)}))}\sum_i\min\left\{1,\gamma^2N^2(\lambda_i^{(q)})^2\right\}\\
            &\leq\frac{\alpha_B\operatorname{tr}(\mathbf{B}_{0,N})}{\gamma(1-\gamma\alpha_B\operatorname{tr}(\mathbf{H}^{(q)}))}\cdot\left(\frac{k^*}{N^2}+\gamma^2\sum_{i>k^*}(\lambda_i^{(q)})^2\right),
        \end{aligned}
    \end{equation*}
    where $k^*=\max\{k:\lambda_k^{(q)}\geq\frac{1}{N\gamma}\}$. Then we tackle $\mathrm{tr}(\mathbf{B}_{0,N})$.
    \begin{equation}
    \label{trB0N}
        \begin{aligned}
            \mathrm{tr}(\mathbf{B}_{0,N})&=\mathrm{tr}\left(\mathbf{B}_{0}-(\mathbf{I}-\gamma\mathbf{H}^{(q)})^{N}\mathbf{B}_{0}(\mathbf{I}-\gamma\mathbf{H}^{(q)})^{N}\right)\\
            &=\sum_{i}\left(1-(1-\gamma\lambda_{i}^{(q)})^{2N}\right)\cdot\left(\langle\mathbf{w}_{0}-{\mathbf{w}^{(q)}}^{*},\mathbf{v}_{i}^{(q)}\rangle\right)^{2}\\
            &\leq 2\sum_i\min\{1,N\gamma\lambda_i^{(q)}\}\left(\langle\mathbf{w}_0-{\mathbf{w}^{(q)}}^*,\mathbf{v}_i^{(q)}\rangle\right)^2\\
            &\leq 2\left(\|\mathbf{w}_{0}-{\mathbf{w}^{(q)}}^{*}\|_{\mathbf{I}_{0:k^{*}}}^{2}+N\gamma\|\mathbf{w}_{0}-{\mathbf{w}^{(q)}}^{*}\|_{\mathbf{H}_{k^{*}:\infty}}^{2}\right).
        \end{aligned}
    \end{equation}
    Hence,
    \begin{equation*}
        I_1\leq\frac{2\alpha_B\left(\|\mathbf{w}_0-{\mathbf{w}^{(q)}}^*\|_{\mathbf{I}_{0:k^*}^{(q)}}^2+N\gamma\|\mathbf{w}_0-{\mathbf{w}^{(q)}}^*\|_{\mathbf{H}_{k^*:\infty}^{(q)}}^2\right)}{N\gamma(1-\gamma\alpha_B\operatorname{tr}(\mathbf{H}^{(q)}))}\cdot\left(\frac{k^*}{N}+N\gamma^2\sum_{i>k^*}(\lambda_i^{(q)})^2\right).
    \end{equation*}

    Regarding $I_2$, decompose $\mathbf{H}^{(q)}=\mathbf{V}^{(q)}\mathbf{\Lambda}^{(q)}{\mathbf{V}^{(q)}}^\top$, then
    \begin{equation*}
        I_2=\frac{1}{\gamma N^2}\sum_{k=0}^{N-1}\langle(\mathbf{I}-\gamma\mathbf{\Lambda}^{(q)})^k-(\mathbf{I}-\gamma\mathbf{\Lambda}^{(q)})^{N+k},{\mathbf{V}^{(q)}}^\top\mathbf{B}_0\mathbf{V}^{(q)}\rangle.
    \end{equation*}
    Note that $\mathbf{B}_0=\boldsymbol{\eta}_0\boldsymbol{\eta}_0^\top$, it can be shown that the diagonal entries of ${\mathbf{V}^{(q)}}^\top\mathbf{B}_0\mathbf{V}^{(q)}$ are $\omega_1^2,\dots$, where $\omega_i={\mathbf{v}_{i}^{(q)}}^{\top}\boldsymbol{\eta}_{0}={\mathbf{v}_{i}^{(q)}}^{\top}(\mathbf{w}_{0}-{\mathbf{w}^{(q)}}^{*})$. Hence,
    \begin{equation*}
        \begin{aligned}
            I_{2} & =\frac{1}{\gamma N^{2}}\sum_{k=0}^{N-1}\sum_{i}\left[(1-\gamma\lambda_{i}^{(q)})^{k}-(1-\gamma\lambda_{i}^{(q)})^{N+k}\right]\omega_{i}^{2} \\
            & =\frac{1}{\gamma^{2}N^{2}}\sum_{i}\frac{\omega_{i}^{2}}{\lambda_{i}^{(q)}}\left[1-(1-\gamma\lambda_{i}^{(q)})^{N}\right]^{2} \\
            & \leq\frac{1}{\gamma^{2}N^{2}}\sum_{i}\frac{\omega_{i}^{2}}{\lambda_{i}^{(q)}}\operatorname*{min}\left\{1,\gamma^{2}N^{2}(\lambda_{i}^{(q)})^{2}\right\} \\
            & \leq\frac{1}{\gamma^{2}N^{2}}\cdot\sum_{i\leq k^{*}}\frac{\omega_{i}^{2}}{\lambda_{i}^{(q)}}+\sum_{i>k^{*}}\lambda_{i}^{(q)}\omega_{i}^{2}\\
            &=\frac{1}{\gamma^2N^2}\cdot\|\mathbf{w}_0-{\mathbf{w}^{(q)}}^*\|_{(\mathbf{H}_{0:k^*}^{(q)})^{-1}}^2+\|\mathbf{w}_0-{\mathbf{w}^{(q)}}^*\|_{\mathbf{H}_{{k^*}:\infty}^{(q)}}^2.
        \end{aligned}
    \end{equation*}

    In conclusion, if the stepsize satisfies $\gamma < \frac{1}{\alpha_B\mathrm{tr}(\mathbf{H}^{(q)})}$,
    \begin{equation*}
        \begin{aligned}
            \mathrm{bias}\leq& I_1+I_2\\
            \leq & \frac{2\alpha_B\left(\|\mathbf{w}_0-{\mathbf{w}^{(q)}}^*\|_{\mathbf{I}_{0:k^*}^{(q)}}^2+N\gamma\|\mathbf{w}_0-{\mathbf{w}^{(q)}}^*\|_{\mathbf{H}_{k^*:\infty}^{(q)}}^2\right)}{N\gamma(1-\gamma\alpha_B\operatorname{tr}(\mathbf{H}^{(q)}))}\cdot\left(\frac{k^*}{N}+N\gamma^2\sum_{i>k^*}(\lambda_i^{(q)})^2\right)\\
            +& \frac{1}{\gamma^2N^2}\cdot\|\mathbf{w}_0-{\mathbf{w}^{(q)}}^*\|_{(\mathbf{H}_{0:k^*}^{(q)})^{-1}}^2+\|\mathbf{w}_0-{\mathbf{w}^{(q)}}^*\|_{\mathbf{H}_{k^*:\infty}^{(q)}}^2.
        \end{aligned}
    \end{equation*}
\end{proof}

\begin{lemma}[\rm A bound for bias under multiplicative quantization]
    \label{bias R2 bound multiplicative}
    Under Assumption \ref{ass1}, Assumption \ref{ass: addd quantized}, Assumption \ref{ass2}, and Assumption \ref{ass3}, if the stepsize satisfies $\gamma < \frac{1}{(1+\tilde{\epsilon})\alpha_B\mathrm{tr}(\mathbf{H}^{(q)})}$, if there exist $\epsilon_d,\epsilon_l,\epsilon_p,\epsilon_a$ and $\epsilon_o$ such that for any $i\in \{d,l,p,a,o\}$, quantization $\mathcal{Q}_i$ is $\epsilon_i$-multiplicative, then
    \begin{equation*}
        \begin{aligned}
            \mathrm{bias}\leq & \frac{2(1+\tilde{\epsilon})\alpha_B\left(\|\mathbf{w}_0-{\mathbf{w}^{(q)}}^*\|_{\mathbf{I}_{0:k^*}^{(q)}}^2+N\gamma\|\mathbf{w}_0-{\mathbf{w}^{(q)}}^*\|_{\mathbf{H}_{k^*:\infty}^{(q)}}^2\right)}{N\gamma(1-(1+\tilde{\epsilon})\gamma\alpha_B\operatorname{tr}(\mathbf{H}^{(q)}))}\cdot\left(\frac{k^*}{N}+N\gamma^2\sum_{i>k^*}(\lambda_i^{(q)})^2\right)\\
            +& \frac{1}{\gamma^2N^2}\cdot\|\mathbf{w}_0-{\mathbf{w}^{(q)}}^*\|_{(\mathbf{H}_{0:k^*}^{(q)})^{-1}}^2+\|\mathbf{w}_0-{\mathbf{w}^{(q)}}^*\|_{\mathbf{H}_{k^*:\infty}^{(q)}}^2.
        \end{aligned}
    \end{equation*}
\end{lemma}

\begin{proof}
    Recalling Lemma \ref{(Initial Study of $S_t$) mu}, we can derive an upper bound for $\mathbf{S}_t$ by Lemma \ref{A bound for M_B^{(q)} S_t^{(M)}}:
    \begin{equation*}
        \begin{aligned}
            \mathbf{S}_{t}\preceq&(\mathcal{I}-\gamma\tilde{\mathcal{T}}^{(q)})\circ\mathbf{S}_{t-1}+(1+\tilde{\epsilon})\gamma^{2}\mathcal{M}_B^{(q)}\circ\mathbf{S}_{N}+\mathbf{B}_{0}\\
            \preceq &(\mathcal{I}-\gamma\tilde{\mathcal{T}}^{(q)})\circ\mathbf{S}_{t-1}+\frac{(1+\tilde{\epsilon})\gamma\alpha_B\cdot\mathrm{tr}\left(\left[\mathcal{I}-(\mathcal{I}-\gamma\widetilde{\mathcal{T}}^{(q)})^N\right]\circ\mathbf{B}_0\right)}{(1-(1+\tilde{\epsilon})\gamma\alpha_B\operatorname{tr}(\mathbf{H}^{(q)}))}\cdot\mathbf{H}^{(q)}+\mathbf{B}_{0}\\
            \preceq&\sum_{k=0}^{t-1}(\mathcal{I}-\gamma\tilde{\mathcal{T}}^{(q)})^k\left(\frac{(1+\tilde{\epsilon})\gamma\alpha_B\cdot\mathrm{tr}\left(\left[\mathcal{I}-(\mathcal{I}-\gamma\widetilde{\mathcal{T}}^{(q)})^N\right]\circ\mathbf{B}_0\right)}{(1-(1+\tilde{\epsilon})\gamma\alpha_B\operatorname{tr}(\mathbf{H}^{(q)}))}\cdot\mathbf{H}^{(q)}+\mathbf{B}_{0}\right)\\
            =&\sum_{k=0}^{t-1}(\mathbf{I}-\gamma\mathbf{H}^{(q)})^k \left(\frac{(1+\tilde{\epsilon})\gamma\alpha_B\cdot\mathrm{tr}\left(\mathbf{B}_0-(\mathbf{I}-\gamma \mathbf{H}^{(q)})^N\mathbf{B}_0(\mathbf{I}-\gamma \mathbf{H}^{(q)})^N\right)}{(1-(1+\tilde{\epsilon})\gamma\alpha_B\operatorname{tr}(\mathbf{H}^{(q)}))}\cdot\mathbf{H}^{(q)}+\mathbf{B}_{0}\right) (\mathbf{I}-\gamma\mathbf{H}^{(q)})^k.
        \end{aligned}
    \end{equation*}
    Repeat the same computation in the proof of Lemma \ref{bias R2 bound}, we obtain
    \begin{equation}
        \label{eq: c21}
        \mathrm{bias}\leq\frac{1}{\gamma N^{2}}\sum_{k=0}^{N-1}\left\langle(\mathbf{I}-\gamma\mathbf{H}^{(q)})^{k}-(\mathbf{I}-\gamma\mathbf{H}^{(q)})^{N+k},\frac{(1+\tilde{\epsilon})\gamma\alpha_B\cdot\mathrm{tr}\left(\mathbf{B}_{0,N}\right)}{1-(1+\tilde{\epsilon})\gamma\alpha_B\operatorname{tr}(\mathbf{H}^{(q)})}\cdot\mathbf{H}^{(q)}+\mathbf{B}_{0}\right\rangle.
    \end{equation}
    Therefore, it suffices to upper bound the following two terms
    \begin{equation*}
        \begin{aligned}
            & I_{1}=\frac{(1+\tilde{\epsilon})\alpha_B\operatorname{tr}(\mathbf{B}_{0,N})}{N^2(1-(1+\tilde{\epsilon})\gamma\alpha\operatorname{tr}(\mathbf{H}^{(q)}))}\sum_{k=0}^{N-1}\left\langle(\mathbf{I}-\gamma\mathbf{H}^{(q)})^k-(\mathbf{I}-\gamma\mathbf{H}^{(q)})^{N+k},\mathbf{H}^{(q)}\right\rangle, \\
            & I_{2}=\frac{1}{\gamma N^2}\sum_{k=0}^{N-1}\left\langle(\mathbf{I}-\gamma\mathbf{H}^{(q)})^k-(\mathbf{I}-\gamma\mathbf{H}^{(q)})^{N+k},\mathbf{B}_0\right\rangle.
        \end{aligned}
    \end{equation*}

    Repeating the computation in the proof of Lemma \ref{bias R2 bound},
    \begin{equation*}
        I_1\leq\frac{2(1+\tilde{\epsilon})\alpha_B\left(\|\mathbf{w}_0-{\mathbf{w}^{(q)}}^*\|_{\mathbf{I}_{0:k^*}^{(q)}}^2+N\gamma\|\mathbf{w}_0-{\mathbf{w}^{(q)}}^*\|_{\mathbf{H}_{k^*:\infty}^{(q)}}^2\right)}{N\gamma(1-(1+\tilde{\epsilon})\gamma\alpha_B\operatorname{tr}(\mathbf{H}^{(q)}))}\cdot\left(\frac{k^*}{N}+N\gamma^2\sum_{i>k^*}(\lambda_i^{(q)})^2\right).
    \end{equation*}
    \begin{equation*}
        I_2 \leq \frac{1}{\gamma^2N^2}\cdot\|\mathbf{w}_0-{\mathbf{w}^{(q)}}^*\|_{(\mathbf{H}_{0:k^*}^{(q)})^{-1}}^2+\|\mathbf{w}_0-{\mathbf{w}^{(q)}}^*\|_{\mathbf{H}_{{k^*}:\infty}^{(q)}}^2.
    \end{equation*}

    In conclusion, if the stepsize satisfies $\gamma < \frac{1}{(1+\tilde{\epsilon})\alpha_B\mathrm{tr}(\mathbf{H}^{(q)})}$,
    \begin{equation*}
        \begin{aligned}
            \mathrm{bias}\leq & \frac{2(1+\tilde{\epsilon})\alpha_B\left(\|\mathbf{w}_0-{\mathbf{w}^{(q)}}^*\|_{\mathbf{I}_{0:k^*}^{(q)}}^2+N\gamma\|\mathbf{w}_0-{\mathbf{w}^{(q)}}^*\|_{\mathbf{H}_{k^*:\infty}^{(q)}}^2\right)}{N\gamma(1-(1+\tilde{\epsilon})\gamma\alpha_B\operatorname{tr}(\mathbf{H}^{(q)}))}\cdot\left(\frac{k^*}{N}+N\gamma^2\sum_{i>k^*}(\lambda_i^{(q)})^2\right)\\
            +& \frac{1}{\gamma^2N^2}\cdot\|\mathbf{w}_0-{\mathbf{w}^{(q)}}^*\|_{(\mathbf{H}_{0:k^*}^{(q)})^{-1}}^2+\|\mathbf{w}_0-{\mathbf{w}^{(q)}}^*\|_{\mathbf{H}_{k^*:\infty}^{(q)}}^2.
        \end{aligned}
    \end{equation*}
\end{proof}

\subsection{Bounding the Variance Error}
\label{Bounding the Variance Error}
Recalling Lemma \ref{Bias-Variance Decomposition under General Quantization} and Lemma \ref{Bias-Variance Decomposition under Multiplicative Quantization}, the key part of bounding the variance error is to derive an upper bound for $\mathbf{C}_t$ and $\mathbf{C}_t^{(M)}$, where
\begin{equation*}
    \mathbf{C}_t := (\mathcal{I}-\gamma \mathcal{T}_B^{(q)})\mathbf{C}_{t-1}+\gamma^2 {\sigma_G^{(q)}}^2\mathbf{H}^{(q)},\quad \mathbf{C}_0=\boldsymbol{0}.
\end{equation*}
\begin{equation*}
    \mathbf{C}_t^{(M)}: = (\mathcal{I}-\gamma \mathcal{T}_B^{(q)}+\tilde{\epsilon}\gamma^2\mathcal{M}_B^{(q)})\mathbf{C}_{t-1}^{(M)}+\gamma^2 {\sigma_M^{(q)}}^2\mathbf{H}^{(q)}, \quad \mathbf{C}_0^{(M)}=0.
\end{equation*}
We first estimate $\mathbf{C}_t$ by converting $\mathcal{T}_B^{(q)}$ to $\widetilde{\mathcal{T}}^{(q)}$.
\begin{equation}
\label{initial Ct}
    \begin{aligned}
        \mathbf{C}_t=& (\mathcal{I}-\gamma\mathcal{T}_B^{(q)}) \circ \mathbf{C}_{t-1}+\gamma^2{\sigma_G^{(q)}}^2\mathbf{H}^{(q)}\\
        =&(\mathcal{I}-\gamma\widetilde{\mathcal{T}}^{(q)})\circ \mathbf{C}_{t-1}+\gamma(\widetilde{\mathcal{T}}^{(q)}-\mathcal{T}_B^{(q)})\circ \mathbf{C}_{t-1}+\gamma^2{\sigma_G^{(q)}}^2\mathbf{H}^{(q)}\\
        =&(\mathcal{I}-\gamma\widetilde{\mathcal{T}}^{(q)})\circ \mathbf{C}_{t-1}+\gamma^2(\mathcal{M}_B^{(q)}-\widetilde{\mathcal{M}}^{(q)})\circ \mathbf{C}_{t-1}+\gamma^2{\sigma_G^{(q)}}^2\mathbf{H}^{(q)}\\
        \preceq&(\mathcal{I}-\gamma\widetilde{\mathcal{T}}^{(q)})\circ \mathbf{C}_{t-1}+\gamma^2\mathcal{M}_B^{(q)}\circ \mathbf{C}_{t-1}+\gamma^2{\sigma_G^{(q)}}^2\mathbf{H}^{(q)}.
    \end{aligned}    
\end{equation}
Similarly,
\begin{equation}
\label{initial CtM}
    \begin{aligned}
        \mathbf{C}_t^{(M)}=&(\mathcal{I}-\gamma \mathcal{T}_B^{(q)}+\tilde{\epsilon}\gamma^2\mathcal{M}_B^{(q)})\mathbf{C}_{t-1}^{(M)}+\gamma^2 {\sigma_M^{(q)}}^2\mathbf{H}^{(q)}\\
        =&(\mathcal{I}-\gamma\widetilde{\mathcal{T}}^{(q)})\circ \mathbf{C}_{t-1}^{(M)}+\gamma(\widetilde{\mathcal{T}}^{(q)}-\mathcal{T}_B^{(q)}+\tilde{\epsilon}\gamma\mathcal{M}_B^{(q)})\circ \mathbf{C}_{t-1}^{(M)}+\gamma^2{\sigma_M^{(q)}}^2\mathbf{H}^{(q)}\\
        =&(\mathcal{I}-\gamma\widetilde{\mathcal{T}}^{(q)})\circ \mathbf{C}_{t-1}^{(M)}+\gamma^2(\mathcal{M}_B^{(q)}-\widetilde{\mathcal{M}}^{(q)}+\tilde{\epsilon}\gamma\mathcal{M}_B^{(q)})\circ \mathbf{C}_{t-1}^{(M)}+\gamma^2{\sigma_M^{(q)}}^2\mathbf{H}^{(q)}\\
        \preceq&(\mathcal{I}-\gamma\widetilde{\mathcal{T}}^{(q)})\circ \mathbf{C}_{t-1}^{(M)}+\gamma^2(1+\tilde{\epsilon})\mathcal{M}_B^{(q)}\circ \mathbf{C}_{t-1}^{(M)}+\gamma^2{\sigma_M^{(q)}}^2\mathbf{H}^{(q)}.
    \end{aligned}
\end{equation}

The following two lemmas provide upper bounds for $\mathcal{M}_B^{(q)} \circ \mathbf{C}_t$ and $\mathcal{M}_B^{(q)} \circ \mathbf{C}_t^{(M)}$.

\begin{lemma}[\rm A bound for $\mathcal{M}_B^{(q)} \circ \mathbf{C}_t$]
    \label{A bound for M_B^{(q)} C_t}
    For $t \geq 1$, under Assumption \ref{ass1}, Assumption \ref{ass: addd quantized}, Assumption \ref{ass2}, and Assumption \ref{ass3}, if the stepsize $\gamma \leq \frac{1}{\alpha_B\mathrm{tr}(\mathbf{H}^{(q)})}$, then
    \begin{equation*}
        \mathcal{M}_B^{(q)} \circ \mathbf{C}_t \preceq \frac{\alpha_B\mathrm{tr}(\mathbf{H}^{(q)})\gamma{\sigma_G^{(q)}}^2}{1-\gamma\alpha_B\mathrm{tr}(\mathbf{H}^{(q)})} \mathbf{H}^{(q)}.
    \end{equation*}
\end{lemma}

\begin{proof}
    The main goal is to derive a crude upper bound for $\mathbf{C}_t$. Denote $\boldsymbol{\Sigma}={\sigma_G^{(q)}}^2\mathbf{H}^{(q)}$.
    
    \textbf{Step 1: $\mathbf{C}_t$ is increasing.} By definition,
    \begin{equation*}
        \begin{aligned}
            \mathbf{C}_{t} & =(\mathcal{I}-\gamma\mathcal{T}_B^{(q)})\circ\mathbf{C}_{t-1}+\gamma^{2}\boldsymbol{\Sigma} \\
            & =\gamma^2\sum_{k=0}^{t-1}(\mathcal{I}-\gamma\mathcal{T}_B^{(q)})^k\circ\boldsymbol{\Sigma}\quad\text{(solving the recursion)} \\
            & =\mathbf{C}_{t-1}+\gamma^{2}(\mathcal{I}-\gamma\mathcal{T}_B^{(q)})^{t-1}\circ\mathbf{\Sigma} \\
            & \succeq\mathbf{C}_{t-1}.\quad(\mathrm{since~}\mathcal{I}-\gamma\mathcal{T}_B^{(q)}\text{ is a PSD mapping}).
        \end{aligned}
    \end{equation*}

    \textbf{Step 2: $\mathbf{C}_\infty$ exists.}
    It suffices to show that $\mathrm{tr}(\mathbf{C}_t)$ is uniformly upper bounded. To be specific, for any $t\geq 1$,
    \begin{equation*}
        \mathbf{C}_t=\gamma^2\sum_{k=0}^{t-1}(\mathcal{I}-\gamma\mathcal{T}_B^{(q)})^k\circ\boldsymbol{\Sigma}\preceq\gamma^2\sum_{t=0}^\infty(\mathcal{I}-\gamma\mathcal{T}_B^{(q)})^t\circ\boldsymbol{\Sigma}.
    \end{equation*}
    Then
    \begin{equation*}
        \mathrm{tr}(\mathbf{C}_t)\leq\gamma^2\sum_{t=0}^\infty\mathrm{tr}\left((\mathcal{I}-\gamma\mathcal{T}_B^{(q)})^t\circ\boldsymbol{\Sigma}\right):=\gamma^2\sum_{t=0}^\infty\mathrm{tr}(\mathbf{E}_t)\leq\frac{\gamma\operatorname{tr}(\boldsymbol{\Sigma})}{\lambda_d^{(q)}}<\infty,
    \end{equation*}
    where the second inequality holds by the iteration:
    \begin{equation*}
        \begin{aligned}
            \operatorname{tr}(\mathbf{E}_{t}) & =\mathrm{tr}(\mathbf{E}_{t-1})-2\gamma\mathrm{tr}(\mathbf{H}^{(q)}\mathbf{E}_{t-1})+\gamma^{2}\mathrm{tr}\left(\mathbf{E}_{t-1}\mathbb{E}\left[\frac{1}{B^2} {\mathbf{X}^{(q)}}^\top{\mathbf{X}^{(q)}}{\mathbf{X}^{(q)}}^\top{\mathbf{X}^{(q)}} \right]\right) \\
            & \leq\operatorname{tr}(\mathbf{E}_{t-1})-(2\gamma-\gamma^{2}\alpha_B\mathrm{tr}(\mathbf{H}^{(q)}))\operatorname{tr}(\mathbf{H}^{(q)}\mathbf{E}_{t-1}) \\
            & \leq\mathrm{tr}\left((\mathbf{I}-\gamma\mathbf{H}^{(q)})\mathbf{E}_{t-1}\right) \\
            & \leq(1-\gamma\lambda_{d}^{(q)})\operatorname{tr}(\mathbf{E}_{t-1}),
        \end{aligned}
    \end{equation*} 
    where the first inequality holds by Assumption \ref{ass2} and the second inequality holds if $\gamma \leq \frac{1}{\alpha_B\mathrm{tr}(\mathbf{H}^{(q)})}$.

    \textbf{Step 3: upper bound $\mathbf{C}_\infty$.} By the update rule for $\mathbf{C}_t$,
    \begin{equation*}
        \mathbf{C}_\infty=(\mathcal{I}-\gamma\mathcal{T}_B^{(q)})\circ\mathbf{C}_\infty+\gamma^2\boldsymbol{\Sigma},
    \end{equation*}
    which immediately implies 
    \begin{equation}
    \label{solution of Cinfty}
        \mathbf{C}_\infty=\gamma{\mathcal{T}_B^{(q)}}^{-1}\circ\mathbf{\Sigma}.
    \end{equation}
    We provide the upper bound by applying $\widetilde{\mathcal{T}}^{(q)}$.
    \begin{equation*}
        \begin{aligned}
            \widetilde{\mathcal{T}}^{(q)}\circ\mathbf{C}_{\infty} & =\mathcal{T}_B^{(q)}\circ\mathbf{C}_{\infty}+\gamma\mathcal{M}_B^{(q)}\circ\mathbf{C}_{\infty}-\gamma\widetilde{\mathcal{M}}^{(q)}\circ\mathbf{C}_{\infty} \\
            & =\gamma\boldsymbol{\Sigma}+\gamma\mathcal{M}_B^{(q)}\circ\mathbf{C}_\infty-\gamma\widetilde{\mathcal{M}}^{(q)}\circ\mathbf{C}_\infty\\
            & \preceq\gamma\boldsymbol{\Sigma}+\gamma\mathcal{M}_B^{(q)}\circ\mathbf{C}_\infty,
        \end{aligned}
    \end{equation*}
    where the first equality holds by the definition of $\mathcal{T}_B^{(q)}$ and $\widetilde{\mathcal{T}}^{(q)}$ and the second equality holds by (\ref{solution of Cinfty}).
    Hence,
    \begin{equation*}
        \widetilde{\mathcal{T}}^{(q)}\circ\mathbf{C}_{\infty}\preceq \gamma{\sigma_G^{(q)}}^2\mathbf{H}^{(q)}+\gamma\mathcal{M}_B^{(q)}\circ\mathbf{C}_\infty.
    \end{equation*}
    Therefore, by applying $(\widetilde{\mathcal{T}}^{(q)})^{-1}$ we have
    \begin{equation}
    \label{202020}
        \begin{aligned}
            \mathbf{C}_{\infty} & \preceq\gamma{\sigma_G^{(q)}}^2\cdot(\widetilde{\mathcal{T}}^{(q)})^{-1}\circ\mathbf{H}^{(q)}+\gamma(\widetilde{\mathcal{T}}^{(q)})^{-1}\circ\mathcal{M}_B^{(q)}\circ\mathbf{C}_\infty \\
            & \preceq\gamma{\sigma_G^{(q)}}^2\cdot\sum_{t=0}^\infty\left(\gamma(\widetilde{\mathcal{T}}^{(q)})^{-1}\circ\mathcal{M}_B^{(q)}\right)^t\circ(\widetilde{\mathcal{T}}^{(q)})^{-1}\circ\mathbf{H}^{(q)}.\quad\text{(solving the recursion)}
        \end{aligned}
    \end{equation}
    We first deal with $(\widetilde{\mathcal{T}}^{(q)})^{-1}\circ\mathbf{H}^{(q)}$. 
    \begin{equation}
    \label{101010}
        \begin{aligned}
            (\widetilde{\mathcal{T}}^{(q)})^{-1}\circ\mathbf{H}^{(q)} & =\gamma\sum_{t=0}^\infty(\mathcal{I}-\gamma\widetilde{\mathcal{T}}^{(q)})^t\circ\mathbf{H}^{(q)} \\
            & =\gamma\sum_{t=0}^\infty(\mathbf{I}-\gamma\mathbf{H}^{(q)})^t\mathbf{H}^{(q)}(\mathbf{I}- \gamma\mathbf{H}^{(q)})^t \\
            & \preceq\gamma\sum_{t=0}^\infty(\mathbf{I}-\gamma\mathbf{H}^{(q)})^t\mathbf{H}^{(q)} \\
            &= \mathbf{I},
        \end{aligned}
    \end{equation}
    where the second equality uses the definition of $\widetilde{\mathcal{T}}^{(q)}$.
    Hence, by (\ref{202020}) and (\ref{101010}),
    \begin{equation}
        \label{eq: c26}
        \begin{aligned}
            \mathbf{C}_{\infty} & \preceq\gamma{\sigma_G^{(q)}}^2\cdot\sum_{t=0}^\infty(\gamma(\widetilde{\mathcal{T}}^{(q)})^{-1}\circ\mathcal{M}_B^{(q)})^t\circ\mathbf{I} \\
            &=\gamma{\sigma_G^{(q)}}^2\cdot\sum_{t=0}^\infty(\gamma(\widetilde{\mathcal{T}}^{(q)})^{-1}\circ\mathcal{M}_B^{(q)})^{t-1}\gamma(\widetilde{\mathcal{T}}^{(q)})^{-1}\circ\mathcal{M}_B^{(q)}\circ\mathbf{I} \\
            &\preceq\gamma{\sigma_G^{(q)}}^2\cdot\sum_{t=0}^\infty(\gamma(\widetilde{\mathcal{T}}^{(q)})^{-1}\circ\mathcal{M}_B^{(q)})^{t-1}\circ\gamma\alpha_B\mathrm{tr}(\mathbf{H}^{(q)})\mathbf{I} \\
            &\preceq\gamma{\sigma_G^{(q)}}^2\cdot\sum_{t=0}^\infty\left(\gamma\alpha_B\mathrm{tr}(\mathbf{H}^{(q)})\right)^{t}\mathbf{I} \\
            &= \frac{\gamma{\sigma_G^{(q)}}^2}{1-\gamma\alpha_B\mathrm{tr}(\mathbf{H}^{(q)})}\mathbf{I},
        \end{aligned}
    \end{equation}
    where the second inequality holds by the fact that $\mathcal{M}_B^{(q)}\circ\mathbf{I}\preceq\alpha_B\operatorname{tr}(\mathbf{H}^{(q)})\mathbf{H}^{(q)}$. 
    
    Here we complete deriving a crude upper bound for $\mathbf{C}_t$:
    \begin{equation*}
        \mathbf{C}_t \preceq \mathbf{C}_\infty \preceq\frac{\gamma{\sigma_G^{(q)}}^2}{1-\gamma\alpha_B\mathrm{tr}(\mathbf{H}^{(q)})}\mathbf{I}.
    \end{equation*}
    Then by $\mathcal{M}_B^{(q)}\circ\mathbf{I}\preceq\alpha_B\operatorname{tr}(\mathbf{H}^{(q)})\mathbf{H}^{(q)}$ again,
    \begin{equation*}
        \mathcal{M}_B^{(q)} \circ \mathbf{C}_t \preceq \frac{\alpha_B\mathrm{tr}(\mathbf{H}^{(q)})\gamma{\sigma_G^{(q)}}^2}{1-\gamma\alpha_B\mathrm{tr}(\mathbf{H}^{(q)})} \mathbf{H}^{(q)}.
    \end{equation*}
\end{proof}

\begin{lemma}[\rm A bound for $\mathcal{M}_B^{(q)} \circ \mathbf{C}_t^{(M)}$]
    \label{A bound for M_B^{(q)} C_t^{(M)}}
    For $t \geq 1$, under Assumption \ref{ass1}, Assumption \ref{ass: addd quantized}, Assumption \ref{ass2}, and Assumption \ref{ass3}, if the stepsize $\gamma \leq \frac{1}{(1+\tilde{\epsilon})\alpha_B\mathrm{tr}(\mathbf{H}^{(q)})}$, then
    \begin{equation*}
        \mathcal{M}_B^{(q)} \circ \mathbf{C}_t^{(M)} \preceq \frac{\alpha_B\mathrm{tr}(\mathbf{H}^{(q)})\gamma{\sigma_M^{(q)}}^2}{1-(1+\tilde{\epsilon})\gamma\alpha_B\mathrm{tr}(\mathbf{H}^{(q)})} \mathbf{H}^{(q)}.
    \end{equation*}
\end{lemma}

\begin{proof}
    The proof idea is similar to the proof of Lemma \ref{A bound for M_B^{(q)} C_t} while the main goal is to derive a crude upper bound for $\mathbf{C}_t^{(M)}$. We deduce from the proof of Lemma \ref{A bound for M_B^{(q)} C_t} that \footnote{$({\mathcal{T}_B^{(q)}}-\tilde{\epsilon}\gamma \mathcal{M}_B^{(q)})^{-1}$ exists under the condition that $\gamma < \frac{1}{(1+\tilde{\epsilon})\alpha_B\mathrm{tr}(\mathbf{H}^{(q)})}$, which can be directly deduced by Lemma B.1 in \citet{zou2023benign}. We omit the proof here for simplicity.}
    \begin{equation}
    \label{express CinftyM}
        \mathbf{C}_t^{(M)}\preceq \mathbf{C}_\infty^{(M)}=\gamma({\mathcal{T}_B^{(q)}}-\tilde{\epsilon}\gamma\mathcal{M}_B^{(q)})^{-1}\circ{\sigma_M^{(q)}}^2\mathbf{H}^{(q)}.
    \end{equation}
    
    We provide the upper bound for $\mathbf{C}_\infty^{(M)}$ by applying $\widetilde{\mathcal{T}}^{(q)}$.
    \begin{equation*}
        \begin{aligned}
            \widetilde{\mathcal{T}}^{(q)}\circ\mathbf{C}_\infty^{(M)} & =(\mathcal{T}_B^{(q)}-\tilde{\epsilon}\gamma\mathcal{M}_B^{(q)})\circ\mathbf{C}_\infty^{(M)}+(1+\tilde{\epsilon})\gamma\mathcal{M}_B^{(q)}\circ\mathbf{C}_\infty^{(M)}-\gamma\widetilde{\mathcal{M}}^{(q)}\circ\mathbf{C}_\infty^{(M)} \\
            & =\gamma{\sigma_M^{(q)}}^2\mathbf{H}^{(q)}+(1+\tilde{\epsilon})\gamma\mathcal{M}_B^{(q)}\circ\mathbf{C}_\infty^{(M)}-\gamma\widetilde{\mathcal{M}}^{(q)}\circ\mathbf{C}_\infty^{(M)}\\
            & \preceq\gamma{\sigma_M^{(q)}}^2\mathbf{H}^{(q)}+(1+\tilde{\epsilon})\gamma\mathcal{M}_B^{(q)}\circ\mathbf{C}_\infty^{(M)},
        \end{aligned}
    \end{equation*}
    where the first equality holds by the definition of $\widetilde{\mathcal{T}}^{(q)}$ and the second equality holds by the definition of $\mathbf{C}_\infty^{(M)}$ (\ref{express CinftyM}).
    Therefore, applying $(\widetilde{\mathcal{T}}^{(q)})^{-1}$ we have
    \begin{equation}
    \label{303030}
        \begin{aligned}
            \mathbf{C}_{\infty}^{(M)} & \preceq\gamma{\sigma_M^{(q)}}^2\cdot(\widetilde{\mathcal{T}}^{(q)})^{-1}\circ\mathbf{H}^{(q)}+(1+\tilde{\epsilon})\gamma(\widetilde{\mathcal{T}}^{(q)})^{-1}\circ\mathcal{M}_B^{(q)}\circ\mathbf{C}_\infty^{(M)} \\
            & \preceq\gamma{\sigma_M^{(q)}}^2\cdot\sum_{t=0}^\infty((1+\tilde{\epsilon})\gamma(\widetilde{\mathcal{T}}^{(q)})^{-1}\circ\mathcal{M}_B^{(q)})^t\circ(\widetilde{\mathcal{T}}^{(q)})^{-1}\circ\mathbf{H}^{(q)}.\quad\text{(solving the recursion)}
        \end{aligned}
    \end{equation}
    By the computation (\ref{101010}) in the proof for Lemma \ref{A bound for M_B^{(q)} C_t},
    \begin{equation}
    \label{404040}
        \begin{aligned}
            (\widetilde{\mathcal{T}}^{(q)})^{-1}\circ\mathbf{H}^{(q)}  \preceq \mathbf{I}.
        \end{aligned}
    \end{equation}
    Hence, by (\ref{303030}) and (\ref{404040}),
    \begin{equation*}
        \begin{aligned}
            \mathbf{C}_{\infty}^{(M)} & \preceq\gamma{\sigma_M^{(q)}}^2\cdot\sum_{t=0}^\infty((1+\tilde{\epsilon})\gamma(\widetilde{\mathcal{T}}^{(q)})^{-1}\circ\mathcal{M}_B^{(q)})^t\circ\mathbf{I} \\
            &=\gamma{\sigma_M^{(q)}}^2\cdot\sum_{t=0}^\infty((1+\tilde{\epsilon})\gamma(\widetilde{\mathcal{T}}^{(q)})^{-1}\circ\mathcal{M}_B^{(q)})^{t-1}(1+\tilde{\epsilon})\gamma(\widetilde{\mathcal{T}}^{(q)})^{-1}\circ\mathcal{M}_B^{(q)}\circ\mathbf{I} \\
            &\preceq\gamma{\sigma_M^{(q)}}^2\cdot\sum_{t=0}^\infty((1+\tilde{\epsilon})\gamma(\widetilde{\mathcal{T}}^{(q)})^{-1}\circ\mathcal{M}_B^{(q)})^{t-1}\circ(1+\tilde{\epsilon})\gamma\alpha_B\mathrm{tr}(\mathbf{H}^{(q)})\mathbf{I} \\
            &\preceq\gamma{\sigma_M^{(q)}}^2\cdot\sum_{t=0}^\infty\left((1+\tilde{\epsilon})\gamma\alpha_B\mathrm{tr}(\mathbf{H}^{(q)})\right)^{t}\mathbf{I} \\
            &= \frac{\gamma{\sigma_M^{(q)}}^2}{1-(1+\tilde{\epsilon})\gamma\alpha_B\mathrm{tr}(\mathbf{H}^{(q)})}\mathbf{I},
        \end{aligned}
    \end{equation*}
    where the second inequality holds by the fact that $\mathcal{M}_B^{(q)}\circ\mathbf{I}\preceq\alpha_B\operatorname{tr}(\mathbf{H}^{(q)})\mathbf{H}^{(q)}$.
    
    Therefore, we complete the proof by
    \begin{equation*}
        \mathcal{M}_B^{(q)} \circ \mathbf{C}_t^{(M)} \preceq \frac{\alpha_B\mathrm{tr}(\mathbf{H}^{(q)})\gamma{\sigma_M^{(q)}}^2}{1-(1+\tilde{\epsilon})\gamma\alpha_B\mathrm{tr}(\mathbf{H}^{(q)})} \mathbf{H}^{(q)}.
    \end{equation*}
\end{proof}

By (\ref{initial Ct}), (\ref{initial CtM}), Lemma \ref{A bound for M_B^{(q)} C_t} and Lemma \ref{A bound for M_B^{(q)} C_t^{(M)}}, we can provide a refined bound for $\mathbf{C}_t$ and $\mathbf{C}_t^{(M)}$. Then we are ready to bound the variance error.

\begin{lemma}[\rm A bound for variance under general quantization]
    \label{variance R2 bound}
    Under Assumption \ref{ass1}, Assumption \ref{ass: addd quantized}, Assumption \ref{ass2}, and Assumption \ref{ass3}, if the stepsize satisfies $\gamma < \frac{1}{\alpha_B\mathrm{tr}(\mathbf{H}^{(q)})}$, then
    \begin{equation*}
        \begin{aligned}
            \mathrm{variance}\leq  \frac{{\sigma_G^{(q)}}^2}{1-\gamma\alpha_B\mathrm{tr}(\mathbf{H}^{(q)})}\left(\frac{k^*}{N}+N\gamma^2\cdot\sum_{i>k^*}(\lambda_i^{(q)})^2\right).
        \end{aligned}
    \end{equation*}
\end{lemma}

\begin{proof}
    We first provide a refined upper bound for $\mathbf{C}_t$. By (\ref{initial Ct}),
    \begin{equation}
    \label{bound for C_t}
        \begin{aligned}
            \mathbf{C}_t \preceq &(\mathcal{I}-\gamma\widetilde{\mathcal{T}}^{(q)})\circ \mathbf{C}_{t-1}+\gamma^2\mathcal{M}_B^{(q)}\circ \mathbf{C}_{t-1}+\gamma^2{\sigma_G^{(q)}}^2\mathbf{H}^{(q)}\\
            \preceq & (\mathcal{I}-\gamma\widetilde{\mathcal{T}}^{(q)})\circ \mathbf{C}_{t-1} + \frac{\gamma^2\alpha_B\mathrm{tr}(\mathbf{H}^{(q)})\gamma{\sigma_G^{(q)}}^2}{1-\gamma\alpha_B\mathrm{tr}(\mathbf{H}^{(q)})} \mathbf{H}^{(q)}+\gamma^2{\sigma_G^{(q)}}^2\mathbf{H}^{(q)}\\
            =& (\mathcal{I}-\gamma\widetilde{\mathcal{T}}^{(q)})\circ \mathbf{C}_{t-1} + \frac{\gamma^2{\sigma_G^{(q)}}^2}{1-\gamma\alpha_B\mathrm{tr}(\mathbf{H}^{(q)})} \mathbf{H}^{(q)}\\
            \preceq&\frac{\gamma^2{\sigma_G^{(q)}}^2}{1-\gamma\alpha_B\mathrm{tr}(\mathbf{H}^{(q)})}\cdot\sum_{k=0}^{t-1}(\mathcal{I}-\gamma\widetilde{\mathcal{T}}^{(q)})^k\circ\mathbf{H}^{(q)}\quad\text{(solving the recursion)} \\
            =&\frac{\gamma^2{\sigma_G^{(q)}}^2}{1-\gamma\alpha_B\mathrm{tr}(\mathbf{H}^{(q)})}\cdot\sum_{k=0}^{t-1}(\mathbf{I}-\gamma \mathbf{H}^{(q)})^k\mathbf{H}^{(q)}(\mathbf{I}-\gamma \mathbf{H}^{(q)})^k \\
            \preceq&\frac{\gamma^2{\sigma_G^{(q)}}^2}{1-\gamma\alpha_B\mathrm{tr}(\mathbf{H}^{(q)})}\cdot\sum_{k=0}^{t-1}(\mathbf{I}-\gamma \mathbf{H}^{(q)})^k\mathbf{H}^{(q)} \\
            =&\frac{\gamma{\sigma_G^{(q)}}^2}{1-\gamma\alpha_B\mathrm{tr}(\mathbf{H}^{(q)})}\cdot\left(\mathbf{I}-(\mathbf{I}-\gamma\mathbf{H}^{(q)})^t\right),
        \end{aligned}
    \end{equation}
    where the second inequality holds by Lemma \ref{A bound for M_B^{(q)} C_t} and the second equality holds by the definition of $\widetilde{\mathcal{T}}^{(q)}$.

    After providing a refined bound for $\mathbf{C}_t$, we are ready to bound the variance. By Lemma \ref{Bias-Variance Decomposition under General Quantization},
    \begin{equation}
        \label{eq: c31}
        \begin{aligned}
            \mathrm{variance}=&\frac{1}{N^2}\cdot\sum_{t=0}^{N-1}\sum_{k=t}^{N-1}\left\langle(\mathbf{I}-\gamma\mathbf{H}^{(q)})^{k-t}\mathbf{H}^{(q)},\mathbf{C}_t\right\rangle\\
            =&\frac{1}{\gamma N^{2}}\sum_{t=0}^{N-1}\left\langle\mathbf{I}-(\mathbf{I}-\gamma\mathbf{H}^{(q)})^{N-t},\mathbf{C}_{t}\right\rangle\\
            \leq&\frac{1}{\gamma^2 N^{2}}\frac{\gamma^2{\sigma_G^{(q)}}^2}{1-\gamma\alpha_B\mathrm{tr}(\mathbf{H}^{(q)})}\sum_{t=0}^{N-1}\left\langle\mathbf{I}-(\mathbf{I}-\gamma\mathbf{H}^{(q)})^{N-t},\mathbf{I}-(\mathbf{I}-\gamma\mathbf{H}^{(q)})^t\right\rangle\\
            =&\frac{1}{\gamma^2 N^{2}}\frac{\gamma^2{\sigma_G^{(q)}}^2}{1-\gamma\alpha_B\mathrm{tr}(\mathbf{H}^{(q)})}\sum_i\sum_{t=0}^{N-1}\left[1-(1-\gamma\lambda_i^{(q)})^{N-t}\right]\left[1-(1-\gamma\lambda_i^{(q)})^t\right]\\
            \leq&\frac{1}{\gamma^2 N^{2}}\frac{\gamma^2{\sigma_G^{(q)}}^2}{1-\gamma\alpha_B\mathrm{tr}(\mathbf{H}^{(q)})}\sum_i\sum_{t=0}^{N-1}\left[1-(1-\gamma\lambda_i^{(q)})^{N}\right]\left[1-(1-\gamma\lambda_i^{(q)})^N\right]\\
            =&\frac{1}{\gamma^2 N}\frac{\gamma^2{\sigma_G^{(q)}}^2}{1-\gamma\alpha_B\mathrm{tr}(\mathbf{H}^{(q)})}\sum_i\left[1-(1-\gamma\lambda_i^{(q)})^{N}\right]^2\\
            \leq &\frac{1}{\gamma^2 N}\frac{\gamma^2{\sigma_G^{(q)}}^2}{1-\gamma\alpha_B\mathrm{tr}(\mathbf{H}^{(q)})}\sum_i \min\left\{1,\gamma^2 N^2(\lambda_i^{(q)})^2\right\}\\
            \leq &\frac{1}{\gamma^2 N}\frac{\gamma^2{\sigma_G^{(q)}}^2}{1-\gamma\alpha_B\mathrm{tr}(\mathbf{H}^{(q)})}\left(k^*+N^2\gamma^2\cdot\sum_{i>k^*}(\lambda_i^{(q)})^2\right)\\
            =&\frac{{\sigma_G^{(q)}}^2}{1-\gamma\alpha_B\mathrm{tr}(\mathbf{H}^{(q)})}\left(\frac{k^*}{N}+N\gamma^2\cdot\sum_{i>k^*}(\lambda_i^{(q)})^2\right),
        \end{aligned}
    \end{equation}
    where the first inequality holds by (\ref{bound for C_t}) and the last inequality holds by the definition of $k^*=\max\left\{k: \lambda_k^{(q)} \geq \frac{1}{N\gamma}\right\}$. This immediately completes the proof.
\end{proof}

\begin{lemma}[\rm A bound for variance under multiplicative quantization]
    \label{variance R2 bound multipicative}
    Under Assumption \ref{ass1}, Assumption \ref{ass: addd quantized}, Assumption \ref{ass2}, and Assumption \ref{ass3}, if the stepsize satisfies $\gamma < \frac{1}{(1+\tilde{\epsilon})\alpha_B\mathrm{tr}(\mathbf{H}^{(q)})}$, if there exist $\epsilon_d,\epsilon_l,\epsilon_p,\epsilon_a$ and $\epsilon_o$ such that for any $i\in \{d,l,p,a,o\}$, quantization $\mathcal{Q}_i$ is $\epsilon_i$-multiplicative, then
    \begin{equation*}
        \begin{aligned}
            \mathrm{variance}\leq  \frac{{\sigma_M^{(q)}}^2}{1-(1+\tilde{\epsilon})\gamma\alpha_B\mathrm{tr}(\mathbf{H}^{(q)})}\left(\frac{k^*}{N}+N\gamma^2\cdot\sum_{i>k^*}(\lambda_i^{(q)})^2\right).
        \end{aligned}
    \end{equation*}
\end{lemma}

\begin{proof}
    Applying (\ref{initial CtM}), and repeating the computation in the proof of Lemma \ref{variance R2 bound},
    \begin{equation*}
        \begin{aligned}
            \mathbf{C}_t^{(M)}\preceq&(\mathcal{I}-\gamma\widetilde{\mathcal{T}}^{(q)})\circ \mathbf{C}_{t-1}^{(M)}+\gamma^2(1+\tilde{\epsilon})\mathcal{M}_B^{(q)}\circ \mathbf{C}_{t-1}^{(M)}+\gamma^2{\sigma_M^{(q)}}^2\mathbf{H}^{(q)}\\
            \preceq&(\mathcal{I}-\gamma\widetilde{\mathcal{T}}^{(q)})\circ \mathbf{C}_{t-1}^{(M)}+\gamma^2(1+\tilde{\epsilon})\frac{\alpha_B\mathrm{tr}(\mathbf{H}^{(q)})\gamma{\sigma_M^{(q)}}^2}{1-(1+\tilde{\epsilon})\gamma\alpha_B\mathrm{tr}(\mathbf{H}^{(q)})} \mathbf{H}^{(q)}+\gamma^2{\sigma_M^{(q)}}^2\mathbf{H}^{(q)}\\
            =&(\mathcal{I}-\gamma\widetilde{\mathcal{T}}^{(q)})\circ \mathbf{C}_{t-1}^{(M)}+\frac{\gamma^2{\sigma_M^{(q)}}^2\mathbf{H}^{(q)}}{1-(1+\tilde{\epsilon})\gamma\alpha_B\mathrm{tr}(\mathbf{H}^{(q)})} \mathbf{H}^{(q)}\\
            \preceq &\frac{\gamma{\sigma_M ^{(q)}}^2}{1-(1+\tilde{\epsilon})\gamma\alpha_B\mathrm{tr}(\mathbf{H}^{(q)})}\cdot\left(\mathbf{I}-(\mathbf{I}-\gamma\mathbf{H}^{(q)})^t\right),
        \end{aligned}
    \end{equation*}
    where the second inequality holds by Lemma \ref{A bound for M_B^{(q)} C_t^{(M)}} and the last inequality repeats the proof in (\ref{bound for C_t}).

    Therefore, repeating the procedure in the proof for Lemma \ref{variance R2 bound}, we directly deduce that
    \begin{equation*}
        \begin{aligned}
            \mathrm{variance}\leq  \frac{{\sigma_M^{(q)}}^2}{1-(1+\tilde{\epsilon})\gamma\alpha_B\mathrm{tr}(\mathbf{H}^{(q)})}\left(\frac{k^*}{N}+N\gamma^2\cdot\sum_{i>k^*}(\lambda_i^{(q)})^2\right),
        \end{aligned}
    \end{equation*}
    which immediately completes the proof.
\end{proof}

\begin{lemma}[\rm A bound for $R_N^{(0)}$ under general quantization]
\label{R_2 bound}
    Under Assumption \ref{ass1}, Assumption \ref{ass: addd quantized}, Assumption \ref{ass2}, and Assumption \ref{ass3}, if the stepsize satisfies $\gamma < \frac{1}{\alpha_B\mathrm{tr}(\mathbf{H}^{(q)})}$, then
    \begin{equation*}
        \begin{aligned}
            R_N^{(0)}/2 \leq &\frac{2\alpha_B\left(\|\mathbf{w}_0-{\mathbf{w}^{(q)}}^*\|_{\mathbf{I}_{0:k^*}^{(q)}}^2+N\gamma\|\mathbf{w}_0-{\mathbf{w}^{(q)}}^*\|_{\mathbf{H}_{k^*:\infty}^{(q)}}^2\right)}{N\gamma(1-\gamma\alpha_B\operatorname{tr}(\mathbf{H}^{(q)}))}\cdot\left(\frac{k^*}{N}+N\gamma^2\sum_{i>k^*}(\lambda_i^{(q)})^2\right)\\
            +& \frac{1}{\gamma^2N^2}\cdot\|\mathbf{w}_0-{\mathbf{w}^{(q)}}^*\|_{(\mathbf{H}_{0:k^*}^{(q)})^{-1}}^2+\|\mathbf{w}_0-{\mathbf{w}^{(q)}}^*\|_{\mathbf{H}_{k^*:\infty}^{(q)}}^2\\
            +& \frac{{\sigma_G^{(q)}}^2}{1-\gamma\alpha_B\mathrm{tr}(\mathbf{H}^{(q)})}\left(\frac{k^*}{N}+N\gamma^2\cdot\sum_{i>k^*}(\lambda_i^{(q)})^2\right),
        \end{aligned}
    \end{equation*}
    where $k^*=\max\left\{k: \lambda_k^{(q)} \geq \frac{1}{N\gamma}\right\}$ and
    \begin{gather*}
        {\sigma_G^{(q)}}^2=\frac{\sup_t \left\{\left\|\mathbb{E}\left[\boldsymbol{\epsilon}_t^{(o)}{\boldsymbol{\epsilon}_t^{(o)}}^\top|\mathbf{o}_t\right]+\mathbb{E}\left[\boldsymbol{\epsilon}_t^{(a)}{\boldsymbol{\epsilon}_t^{(a)}}^\top|\mathbf{a}_t\right] \right\|\right\}}{B}\\
        +\alpha_B\sup_t\mathbb{E}_{\mathbf{w}_{t-1}}\left[\mathrm{tr}\left(\mathbf{H}^{(q)} \mathbb{E}\left[\boldsymbol{\epsilon}_{t-1}^{(p)}{\boldsymbol{\epsilon}_{t-1}^{(p)}}^\top \big| \mathbf{w}_{t-1}\right]\right)\right]
        +\frac{\sigma^2}{B}.
    \end{gather*}
\end{lemma}

\begin{proof}
    The proof is immediately completed by Lemma \ref{Bias-Variance Decomposition under General Quantization}, Lemma \ref{bias R2 bound} and Lemma \ref{variance R2 bound}.
\end{proof}

\begin{lemma}[\rm A bound for $R_N^{(0)}$ under multiplicative quantization]
    \label{R_2 bound multi}
    Under Assumption \ref{ass1}, Assumption \ref{ass: addd quantized}, Assumption \ref{ass2}, and Assumption \ref{ass3}, if the stepsize satisfies $\gamma < \frac{1}{(1+\tilde{\epsilon})\alpha_B\mathrm{tr}(\mathbf{H}^{(q)})}$, if there exist $\epsilon_d,\epsilon_l,\epsilon_p,\epsilon_a$ and $\epsilon_o$ such that for any $i\in \{d,l,p,a,o\}$, quantization $\mathcal{Q}_i$ is $\epsilon_i$-multiplicative, then
    \begin{equation*}
        \begin{aligned}
            R_N^{(0)}/2\leq & \frac{2(1+\tilde{\epsilon})\alpha_B\left(\|\mathbf{w}_0-{\mathbf{w}^{(q)}}^*\|_{\mathbf{I}_{0:k^*}^{(q)}}^2+N\gamma\|\mathbf{w}_0-{\mathbf{w}^{(q)}}^*\|_{\mathbf{H}_{k^*:\infty}^{(q)}}^2\right)}{N\gamma(1-(1+\tilde{\epsilon})\gamma\alpha_B\operatorname{tr}(\mathbf{H}^{(q)}))}\cdot\left(\frac{k^*}{N}+N\gamma^2\sum_{i>k^*}(\lambda_i^{(q)})^2\right)\\
            +& \frac{1}{\gamma^2N^2}\cdot\|\mathbf{w}_0-{\mathbf{w}^{(q)}}^*\|_{(\mathbf{H}_{0:k^*}^{(q)})^{-1}}^2+\|\mathbf{w}_0-{\mathbf{w}^{(q)}}^*\|_{\mathbf{H}_{k^*:\infty}^{(q)}}^2\\
            +&\frac{{\sigma_M^{(q)}}^2}{1-(1+\tilde{\epsilon})\gamma\alpha_B\mathrm{tr}(\mathbf{H}^{(q)})}\left(\frac{k^*}{N}+N\gamma^2\cdot\sum_{i>k^*}(\lambda_i^{(q)})^2\right),
        \end{aligned}
    \end{equation*}
    where $k^*=\max\left\{k: \lambda_k^{(q)} \geq \frac{1}{N\gamma}\right\}$ and
    \begin{gather*}
        \tilde{\epsilon}=8\epsilon_o(1+\epsilon_p)(1+\epsilon_a)+4\epsilon_p+4\epsilon_a(1+\epsilon_p),\\
        {\sigma_M^{(q)}}^2=\frac{(1+4\epsilon_o)\sigma^2}{B} + \frac{\|\mathbf{w}^*\|_\mathbf{H}^2}{1+\epsilon_d}\alpha_B\left(4\epsilon_o[(1+\epsilon_p)(1+\epsilon_a)+1]+2\epsilon_a(1+\epsilon_p)+2\epsilon_p \right).
    \end{gather*}
\end{lemma}
\begin{proof}
    The proof is immediately completed by Lemma \ref{Bias-Variance Decomposition under Multiplicative Quantization}, Lemma \ref{bias R2 bound multiplicative} and Lemma \ref{variance R2 bound multipicative}.
\end{proof}

\section{Deferring Proofs}
\subsection{Proof for Theorem \ref{Theorem 4}}
\label{Proof for Theorem 4.1}
\begin{proof}
    By Lemma \ref{Refine excess risk decomposition}, (\ref{eq: b1}), (\ref{eq: b2}), (\ref{eq: c1}) and Lemma \ref{R_2 bound}, we have
    \begin{equation*}
        \begin{aligned}
            \mathbb{E}[\mathcal{E}(\overline{\mathbf{w}}_N)]\leq 2\mathrm{VarErr}+2\mathrm{BiasErr}+\mathrm{ApproxErr},
        \end{aligned}
    \end{equation*}
    where
    \begin{gather*}
        \mathrm{VarErr}=\frac{2\alpha_B\left(\frac{\|\mathbf{w}_0-{\mathbf{w}^{(q)}}^*\|_{\mathbf{I}_{0:k^*}^{(q)}}^2}{N\gamma}+\|\mathbf{w}_0-{\mathbf{w}^{(q)}}^*\|_{\mathbf{H}_{k^*:\infty}^{(q)}}^2\right)+{\sigma_G^{(q)}}^2}{1-\gamma\alpha_B\mathrm{tr}(\mathbf{H}^{(q)})}\left(\frac{k^*}{N}+N\gamma^2\cdot\sum_{i>k^*}(\lambda_i^{(q)})^2\right),\\
        \mathrm{BiasErr}=\frac{1}{\gamma^2N^2}\cdot\|\mathbf{w}_0-{\mathbf{w}^{(q)}}^*\|_{(\mathbf{H}_{0:k^*}^{(q)})^{-1}}^2+\|\mathbf{w}_0-{\mathbf{w}^{(q)}}^*\|_{\mathbf{H}_{k^*:\infty}^{(q)}}^2,\\
        \mathrm{ApproxErr}=\left\|\mathbf{w}^*\right\|_{\mathbf{D}_2^{\mathbf{H}}}^2
        +\frac{1}{2}\|\mathbf{w}^*\|_{\mathbf{D}_1^{\mathbf{H}}}^2,
    \end{gather*}
    with $\mathbf{D}=\mathbf{H}^{(q)}-\mathbf{H}$, $k^*=\max\left\{k: \lambda_k^{(q)} \geq \frac{1}{N\gamma}\right\}$,  and
    \begin{gather*}
        \mathbf{D}_2^{\mathbf{H}}=\mathbf{H}(\mathbf{H}^{(q)})^{-1}\frac{1}{N\gamma}\left(\mathbf{I}-(\mathbf{I}-\gamma\mathbf{H}^{(q)})^N\right)(\mathbf{H}^{(q)})^{-1}\mathbf{D}(\mathbf{H}^{(q)})^{-1}\mathbf{H},\ \mathbf{D}_1^{\mathbf{H}}=\mathbf{D}(\mathbf{H}^{(q)})^{-1}\mathbf{H}(\mathbf{H}^{(q)})^{-1}\mathbf{D},\\
        {\sigma_G^{(q)}}^2=\frac{\sup_t \left\{\left\|\mathbb{E}\left[\boldsymbol{\epsilon}_t^{(o)}{\boldsymbol{\epsilon}_t^{(o)}}^\top|\mathbf{o}_t\right]+\mathbb{E}\left[\boldsymbol{\epsilon}_t^{(a)}{\boldsymbol{\epsilon}_t^{(a)}}^\top|\mathbf{a}_t\right] \right\|\right\}}{B}\\
        +\alpha_B\sup_t\mathbb{E}_{\mathbf{w}_{t-1}}\left[\mathrm{tr}\left(\mathbf{H}^{(q)} \mathbb{E}\left[\boldsymbol{\epsilon}_{t-1}^{(p)}{\boldsymbol{\epsilon}_{t-1}^{(p)}}^\top \big| \mathbf{w}_{t-1}\right]\right)\right]
        +\frac{\sigma^2}{B}.
    \end{gather*}
    Let the initialization $\mathbf{w}_0=0$ completes the proof.
\end{proof}
\subsection{Proof for Theorem \ref{Theorem 1}}
\label{Proof for Theorem 4.2}
We prove a tighter excess risk bound under multiplicative quantization in this subsection:
\begin{theorem}[\rm {Multiplicative quantization}]
    \label{Theorem 1, proof}
    Under Assumption \ref{ass1}, \ref{ass: addd quantized}, \ref{ass2} and \ref{ass3}, if there exist $\epsilon_d,\epsilon_l,\epsilon_p,\epsilon_a$ and $\epsilon_o$ such that for any $i\in \{d,l,p,a,o\}$, quantization $\mathcal{Q}_i$ is $\epsilon_i$-multiplicative, and the stepsize satisfies $\gamma < \frac{1}{ \alpha_B (1+\epsilon_o)[1+\epsilon_p+\epsilon_a(1+\epsilon_p)](1+\epsilon_d)\mathrm{tr}(\mathbf{H})}$, then the excess risk can be upper bounded as follows.
    \begin{equation*}
        \mathbb{E}[\mathcal{E}(\overline{\mathbf{w}}_N)] \lesssim \mathrm{ApproxErr}+\mathrm{VarErr}+\mathrm{BiasErr},
    \end{equation*}
    where
    \begin{equation*}
        \begin{aligned}
            &\mathrm{ApproxErr}\lesssim \frac{\epsilon_d}{1+\epsilon_d}\left\|\mathbf{w}^*\right\|_\mathbf{H}^2,\quad \mathrm{BiasErr}\lesssim\frac{1}{\gamma^2N^2}\cdot\|{\mathbf{w}^{(q)}}^*\|_{(\mathbf{H}_{0:k^*}^{(q)})^{-1}}^2+\|{\mathbf{w}^{(q)}}^*\|_{\mathbf{H}_{k^*:\infty}^{(q)}}^2,\\
            &\mathrm{VarErr}\lesssim\left(\frac{k^*}{N}+N\gamma^2(1+\epsilon_d)^2\sum_{i>k^*}\lambda_i^2\right) \frac{\frac{(1+\epsilon_o)\sigma^2}{B}+\alpha_B \sigma_M^2}{1-\gamma\alpha_B(1+\epsilon_o)[1+\epsilon_p+\epsilon_a(1+\epsilon_p)](1+\epsilon_d)\mathrm{tr}\left(\mathbf{H}\right)},
        \end{aligned}
    \end{equation*}
    with 
    \begin{equation*}
        \sigma_M^2=[\epsilon_o+(1+\epsilon_o)(\epsilon_p+\epsilon_a(1+\epsilon_p))]\left\|\mathbf{w}^*\right\|_\mathbf{H}^2
    +(1+\epsilon_o)[1+\epsilon_p+\epsilon_a(1+\epsilon_p)]\left(\frac{\|{\mathbf{w}^{(q)}}^*\|_{\mathbf{I}_{0:k^*}^{(q)}}^2}{N\gamma}+\|{\mathbf{w}^{(q)}}^*\|_{\mathbf{H}_{k^*:\infty}^{(q)}}^2\right).
    \end{equation*}
\end{theorem}
\begin{proof}
    By Lemma \ref{Refine excess risk decomposition}, Lemma \ref{lem: b1}, (\ref{eq: c1}) and Lemma \ref{R_2 bound multi}, we have
    \begin{equation*}
        \begin{aligned}
            \mathbb{E}[\mathcal{E}(\overline{\mathbf{w}}_N)]\leq 2\mathrm{VarErr}+2\mathrm{BiasErr}+\mathrm{ApproxErr},
        \end{aligned}
    \end{equation*}
    where
    \begin{gather*}
        \mathrm{VarErr}=\frac{{\sigma_M^{(q)}}^2+2(1+\tilde{\epsilon})\alpha_B\left(\frac{\|\mathbf{w}_0-{\mathbf{w}^{(q)}}^*\|_{\mathbf{I}_{0:k^*}^{(q)}}^2}{N\gamma}+\|\mathbf{w}_0-{\mathbf{w}^{(q)}}^*\|_{\mathbf{H}_{k^*:\infty}^{(q)}}^2\right)}{1-(1+\tilde{\epsilon})\gamma\alpha_B\mathrm{tr}(\mathbf{H}^{(q)})}\left(\frac{k^*}{N}+N\gamma^2\cdot\sum_{i>k^*}(\lambda_i^{(q)})^2\right),\\
        \mathrm{BiasErr}=\frac{1}{\gamma^2N^2}\cdot\|\mathbf{w}_0-{\mathbf{w}^{(q)}}^*\|_{(\mathbf{H}_{0:k^*}^{(q)})^{-1}}^2+\|\mathbf{w}_0-{\mathbf{w}^{(q)}}^*\|_{\mathbf{H}_{k^*:\infty}^{(q)}}^2,\\
        \mathrm{ApproxErr}=\frac{\epsilon_d^2}{2(1+\epsilon_d)^2}\|\mathbf{w}^*\|_\mathbf{H}^2+\frac{\epsilon_d}{(1+\epsilon_d)^2}\left\|\mathbf{w}^*\right\|_\mathbf{H}^2\lesssim \frac{\epsilon_d}{1+\epsilon_d}\left\|\mathbf{w}^*\right\|_\mathbf{H}^2,
    \end{gather*}
    with $k^*=\max\left\{k: \lambda_k^{(q)} \geq \frac{1}{N\gamma}\right\}$ and
    \begin{gather*}
            \tilde{\epsilon}=8\epsilon_o(1+\epsilon_p)(1+\epsilon_a)+4\epsilon_p+4\epsilon_a(1+\epsilon_p)\lesssim \epsilon_o+(1+\epsilon_o)(\epsilon_p+\epsilon_a(1+\epsilon_p)),\\
            {\sigma_M^{(q)}}^2=\frac{(1+4\epsilon_o)\sigma^2}{B} + \frac{\|\mathbf{w}^*\|_\mathbf{H}^2}{1+\epsilon_d}\alpha_B\left(4\epsilon_o[(1+\epsilon_p)(1+\epsilon_a)+1]+2\epsilon_a(1+\epsilon_p)+2\epsilon_p \right)\\
            \lesssim \frac{(1+\epsilon_o)\sigma^2}{B} + {\|\mathbf{w}^*\|_\mathbf{H}^2}\alpha_B\left(\epsilon_o+(1+\epsilon_o)(\epsilon_p+\epsilon_a(1+\epsilon_p)) \right).
    \end{gather*}
    Let initialization $\mathbf{w}_0=0$. Regarding $\mathrm{VarErr}$, noticing that
    $1+\tilde{\epsilon}\lesssim (1+\epsilon_o)[1+\epsilon_p+\epsilon_a(1+\epsilon_p)]$,
    we have
    \begin{equation*}
        \begin{aligned}
            &{\sigma_M^{(q)}}^2+2(1+\tilde{\epsilon})\alpha_B\left(\frac{\|{\mathbf{w}^{(q)}}^*\|_{\mathbf{I}_{0:k^*}^{(q)}}^2}{N\gamma}+\|{\mathbf{w}^{(q)}}^*\|_{\mathbf{H}_{k^*:\infty}^{(q)}}^2\right)\\
            \lesssim&\frac{(1+\epsilon_o)\sigma^2}{B}+\alpha_B\left\|\mathbf{w}^*\right\|_\mathbf{H}^2\left(\epsilon_o+(1+\epsilon_o)(\epsilon_p+\epsilon_a(1+\epsilon_p)) \right)\\
            +&\alpha_B(1+\epsilon_o)[1+\epsilon_p+\epsilon_a(1+\epsilon_p)]\left(\frac{\|{\mathbf{w}^{(q)}}^*\|_{\mathbf{I}_{0:k^*}^{(q)}}^2}{N\gamma}+\|{\mathbf{w}^{(q)}}^*\|_{\mathbf{H}_{k^*:\infty}^{(q)}}^2\right).
        \end{aligned}
    \end{equation*}
    Then the proof is completed by $\mathrm{tr}(\mathbf{H}^{(q)})=(1+\epsilon_d)\mathrm{tr}(\mathbf{H})$ and $\lambda_i^{(q)}=(1+\epsilon_d)\lambda_i$.
\end{proof}
Theorem \ref{Theorem 1} can be deduced from Theorem \ref{Theorem 1, proof} by noticing that
\begin{equation*}
    \sigma_M^2\lesssim (1+\epsilon_o)[1+\epsilon_p+\epsilon_a(1+\epsilon_p)]\left\|\mathbf{w}^*\right\|_\mathbf{H}^2,
\end{equation*}
where we use
\begin{equation*}
    \frac{\|{\mathbf{w}^{(q)}}^*\|_{\mathbf{I}_{0:k^*}^{(q)}}^2}{N\gamma}+\|{\mathbf{w}^{(q)}}^*\|_{\mathbf{H}_{k^*:\infty}^{(q)}}^2\leq \|{\mathbf{w}^{(q)}}^*\|_{\mathbf{H}^{(q)}}^2=\frac{\|\mathbf{w}^*\|_\mathbf{H}^2}{1+\epsilon_d}\leq\|\mathbf{w}^*\|_\mathbf{H}^2.
\end{equation*}

\subsection{Proof for Corollary \ref{theorem new 2.5}}
\label{Proof for Theorem 4.3}
We provide a tighter excess risk bound under additive quantization in this subsection:
\begin{corollary}[\rm {Additive quantization}]
    \label{theorem new 2.5, proof}
    Under Assumption \ref{ass1}, \ref{ass: addd quantized}, \ref{ass2} and \ref{ass3}, if there exist $\epsilon_d,\epsilon_l,\epsilon_p,\epsilon_a$ and $\epsilon_o$ such that for any $i\in \{d,l,p,a,o\}$, quantization $\mathcal{Q}_i$ is $\epsilon_i$-additive, and the stepsize satisfies $\gamma < \frac{1}{\alpha_B [\mathrm{tr}(\mathbf{H})+d\epsilon_d]}$, then
    \begin{equation*}
        \mathbb{E}[\mathcal{E}(\overline{\mathbf{w}}_N)] \lesssim \mathrm{ApproxErr}+\mathrm{ VarErr}+\mathrm{BiasErr},
    \end{equation*}
    where
    \begin{align*}
            &\mathrm{ApproxErr}\lesssim \frac{\epsilon_d}{\lambda_d+\epsilon_d}\left\|\mathbf{w}^*\right\|_\mathbf{H}^2,\quad \mathrm{BiasErr}\lesssim\frac{1}{\gamma^2N^2}\cdot\|{\mathbf{w}^{(q)}}^*\|_{(\mathbf{H}_{0:k^*}^{(q)})^{-1}}^2+\|{\mathbf{w}^{(q)}}^*\|_{\mathbf{H}_{k^*:\infty}^{(q)}}^2,\\
            &\mathrm{VarErr}\lesssim \frac{\alpha_B\left(\frac{\|{\mathbf{w}^{(q)}}^*\|_{\mathbf{I}_{0:k^*}^{(q)}}^2}{N\gamma}+\|{\mathbf{w}^{(q)}}^*\|_{\mathbf{H}_{k^*:\infty}^{(q)}}^2\right)+\frac{\sigma^2+\epsilon_o+\epsilon_a}{B}+\alpha_B \epsilon_p[\mathrm{tr}(\mathbf{H})+d\epsilon_d]}{1-\gamma\alpha_B[\mathrm{tr}(\mathbf{H})+d\epsilon_d]}\left(\frac{k^*}{N}+N\gamma^2\cdot\sum_{i>k^*}(\lambda_i+\epsilon_d)^2\right).
    \end{align*}
\end{corollary}
\begin{proof}
    By Theorem \ref{Theorem 4},
    \begin{equation*}
        \mathbb{E}[\mathcal E(\overline{\mathbf{w}}_N)]\leq 2\mathrm{VarErr}+2\mathrm{BiasErr}+\mathrm{ApproxErr},
    \end{equation*}
    where
    \begin{equation*}
        \begin{aligned}
            &\mathrm{VarErr}\leq\frac{2\alpha_B\left(\frac{\|{\mathbf{w}^{(q)}}^*\|_{\mathbf{I}_{0:k^*}^{(q)}}^2}{N\gamma}+\|{\mathbf{w}^{(q)}}^*\|_{\mathbf{H}_{k^*:\infty}^{(q)}}^2\right)+{\sigma_G^{(q)}}^2}{1-\gamma\alpha_B\mathrm{tr}(\mathbf{H}^{(q)})}\left(\frac{k^*}{N}+N\gamma^2\cdot\sum_{i>k^*}(\lambda_i^{(q)})^2\right),\\
            &\mathrm{BiasErr}\leq\frac{1}{\gamma^2N^2}\cdot\|{\mathbf{w}^{(q)}}^*\|_{(\mathbf{H}_{0:k^*}^{(q)})^{-1}}^2+\|{\mathbf{w}^{(q)}}^*\|_{\mathbf{H}_{k^*:\infty}^{(q)}}^2,\\
            &\mathrm{ApproxErr}\leq\left\|\mathbf{w}^*\right\|_{\mathbf{D}(\mathbf{H}+\mathbf{D})^{-1}\mathbf{H}(\mathbf{H}+\mathbf{D})^{-1}\mathbf{D}}^2+\|\mathbf{w}^*\|_{\mathbf{D}_\mathbf{H}}^2,
        \end{aligned}
    \end{equation*}
    with ${\sigma_G^{(q)}}^2=\frac{\sigma^2+\sup_t \left\{\left\|\mathbb{E}\left[\boldsymbol{\epsilon}_t^{(o)}{\boldsymbol{\epsilon}_t^{(o)}}^\top|\mathbf{o}_t\right]+\mathbb{E}\left[\boldsymbol{\epsilon}_t^{(a)}{\boldsymbol{\epsilon}_t^{(a)}}^\top|\mathbf{a}_t\right] \right\|\right\}}{B}+\alpha_B\sup_t\mathbb{E}\left[\mathrm{tr}\left(\mathbf{H}^{(q)} \boldsymbol{\epsilon}_{t-1}^{(p)}{\boldsymbol{\epsilon}_{t-1}^{(p)}}^\top\right)\right]$ and $\mathbf{D}_{\mathbf{H}}=\mathbf{H}(\mathbf{H}^{(q)})^{-1}\frac{1}{N\gamma}\left(\mathbf{I}-(\mathbf{I}-\gamma\mathbf{H}^{(q)})^N\right)(\mathbf{H}^{(q)})^{-1}\mathbf{D}(\mathbf{H}^{(q)})^{-1}\mathbf{H}$.
    
    Under additive quantization, it holds
    \begin{equation*}
        \mathrm{tr}(\mathbf{H}^{(q)}) = \mathrm{tr}(\mathbf{H})+d\epsilon_d,\quad \sum_{i>k^*}(\lambda_i^{(q)})^2=\sum_{i>k^*}(\lambda_i+\epsilon_d)^2,
    \end{equation*}
    and
    \begin{equation*}
        {\sigma_G^{(q)}}^2=\frac{\sigma^2+\epsilon_o+\epsilon_a}{B}+\alpha_B \epsilon_p[\mathrm{tr}(\mathbf{H})+d\epsilon_d].
    \end{equation*}
    Then we have
    \begin{equation*}
        \mathrm{VarErr}\lesssim \frac{\alpha_B\left(\frac{\|{\mathbf{w}^{(q)}}^*\|_{\mathbf{I}_{0:k^*}^{(q)}}^2}{N\gamma}+\|{\mathbf{w}^{(q)}}^*\|_{\mathbf{H}_{k^*:\infty}^{(q)}}^2\right)+\frac{\sigma^2+\epsilon_o+\epsilon_a}{B}+\alpha_B \epsilon_p[\mathrm{tr}(\mathbf{H})+d\epsilon_d]}{1-\gamma\alpha_B[\mathrm{tr}(\mathbf{H})+d\epsilon_d]}\left(\frac{k^*}{N}+N\gamma^2\cdot\sum_{i>k^*}(\lambda_i+\epsilon_d)^2\right).
    \end{equation*}
    The proof is completed by Lemma \ref{lem: b2}:
    \begin{equation*}
        \mathrm{ApproxErr}\leq \frac{\epsilon_d^2}{2(\lambda_d+\epsilon_d)^2}\|\mathbf{w}^*\|_\mathbf{H}^2+\frac{\lambda_1\epsilon_d}{(\lambda_d+\epsilon_d)(\lambda_1+\epsilon_d)}\left\|\mathbf{w}^*\right\|_\mathbf{H}^2\lesssim \frac{\epsilon_d}{\lambda_d+\epsilon_d}\left\|\mathbf{w}^*\right\|_\mathbf{H}^2.
    \end{equation*}
\end{proof}
Corollary \ref{theorem new 2.5} can be deduced from Corollary \ref{theorem new 2.5, proof} by noticing that
    \begin{equation*}
        \frac{\|{\mathbf{w}^{(q)}}^*\|_{\mathbf{I}_{0:k^*}^{(q)}}^2}{N\gamma}+\|{\mathbf{w}^{(q)}}^*\|_{\mathbf{H}_{k^*:\infty}^{(q)}}^2\leq \|{\mathbf{w}^{(q)}}^*\|_{\mathbf{H}^{(q)}}^2={\mathbf{w}^*}^\top\mathbf{H}(\mathbf{H}^{(q)})^{-1}\mathbf{H}\mathbf{w}^*\leq\|\mathbf{w}^*\|_\mathbf{H}^2.
    \end{equation*}
\subsection{Proof for the Multiplicative Statement in Corollary \ref{coro:comparison}}
\label{Proof for Theorem 11}
\begin{proof}
    Recall that $k_0^*=\max\{k: \lambda_k \geq \frac{1}{N\gamma}\}$,
    \begin{equation*}
        \begin{aligned}
            R_0=&\underbrace{\left(\frac{k_0^*}{N}+N\gamma^2\cdot\sum_{i>k_0^*}\lambda_i^2\right)\frac{\alpha_B\left(\frac{1}{N\gamma}\|{\mathbf{w}}^*\|_{\mathbf{I}_{0:k_0^*}}^2+\|{\mathbf{w}}^*\|_{\mathbf{H}_{k_0^*:\infty}}^2\right)
            +\frac{\sigma^2}{B}}{1-\gamma\alpha_B\mathrm{tr}\left(\mathbf{H}\right)}}_{\rm EffectiveVar}\\
            +&\underbrace{\frac{1}{\gamma^2N^2}\cdot\|      {\mathbf{w}}^*\|_{(\mathbf{H}_{0:k_0^*})^{-1}}^2+\|{\mathbf{w}}^*\|_{\mathbf{H}_{k_0^*:\infty}}^2}_{\rm EffectiveBias},
        \end{aligned}
    \end{equation*}
    and by Theorem \ref{Theorem 1, proof},
    \begin{equation*}
        \mathbb{E}[\mathcal{E}(\overline{\mathbf{w}}_N)] \lesssim \mathrm{ApproxErr}+\mathrm{VarErr}+\mathrm{BiasErr},
    \end{equation*}
    where
    \begin{equation*}
        \begin{aligned}
            &\mathrm{ApproxErr}\lesssim \frac{\epsilon_d}{1+\epsilon_d}\left\|\mathbf{w}^*\right\|_\mathbf{H}^2,\quad \mathrm{BiasErr}\leq\frac{1}{\gamma^2N^2}\cdot\|{\mathbf{w}^{(q)}}^*\|_{(\mathbf{H}_{0:k^*}^{(q)})^{-1}}^2+\|{\mathbf{w}^{(q)}}^*\|_{\mathbf{H}_{k^*:\infty}^{(q)}}^2,\\
            &\mathrm{VarErr}\lesssim\left(\frac{k^*}{N}+N\gamma^2(1+\epsilon_d)^2\sum_{i>k^*}\lambda_i^2\right) \frac{\frac{(1+\epsilon_o)\sigma^2}{B}+\alpha_B\sigma_M^2}{1-\gamma\alpha_B(1+\epsilon_o)[1+\epsilon_p+\epsilon_a(1+\epsilon_p)](1+\epsilon_d)\mathrm{tr}\left(\mathbf{H}\right)},
        \end{aligned}
    \end{equation*}
    with 
    \begin{equation*}
        \sigma_M^2=[\epsilon_o+(1+\epsilon_o)(\epsilon_p+\epsilon_a(1+\epsilon_p))]\left\|\mathbf{w}^*\right\|_\mathbf{H}^2
        +(1+\epsilon_o)[1+\epsilon_p+\epsilon_a(1+\epsilon_p)]\left(\frac{\|{\mathbf{w}^{(q)}}^*\|_{\mathbf{I}_{0:k^*}^{(q)}}^2}{N\gamma}+\|{\mathbf{w}^{(q)}}^*\|_{\mathbf{H}_{k^*:\infty}^{(q)}}^2\right).
    \end{equation*}
    
    We then compare the upper bound of $\mathbb{E}[\mathcal{E}(\overline{\mathbf{w}}_N)]$ with $R_0$. Regarding $\mathrm{VarErr}$, we first analyze $\frac{k^*}{N}+N\gamma^2(1+\epsilon_d)^2\sum_{i>k^*}\lambda_i^2$. Note that for $k_0^*< i \leq k^*$, $\frac{1}{N\gamma(1+\epsilon_d)}\leq\lambda_i<\frac{1}{N\gamma}$, we have
    \begin{equation*}
        \begin{aligned}
            &\frac{k^*}{N}+N\gamma^2(1+\epsilon_d)^2\cdot\sum_{i>k^*}\lambda_i^2\\
            =&\frac{k_0^*}{N}+\frac{k^*-k_0^*}{N}-N\gamma^2(1+\epsilon_d)^2\cdot\sum_{k_0^*<i\leq k^*}\lambda_i^2+N\gamma^2(1+\epsilon_d)^2\cdot\sum_{i>k_0^*}\lambda_i^2\\
            \leq &\frac{k_0^*}{N}+\frac{k^*-k_0^*}{N}-N\gamma^2(1+\epsilon_d)^2(k^*-k_0^*)\frac{1}{N^2\gamma^2(1+\epsilon_d)^2}+N\gamma^2(1+\epsilon_d)^2\cdot\sum_{i>k_0^*}\lambda_i^2\\
            =&\frac{k_0^*}{N}+N\gamma^2(1+\epsilon_d)^2\cdot\sum_{i>k_0^*}\lambda_i^2.
        \end{aligned}
    \end{equation*}
    We then analyze $\frac{\|{\mathbf{w}^{(q)}}^*\|_{\mathbf{I}_{0:k^*}^{(q)}}^2}{N\gamma}+\|{\mathbf{w}^{(q)}}^*\|_{\mathbf{H}_{k^*:\infty}^{(q)}}^2$. Similarly,
    \begin{equation*}
        \begin{aligned}
            &\frac{\|{\mathbf{w}^{(q)}}^*\|_{\mathbf{I}_{0:k^*}^{(q)}}^2}{N\gamma}+\|{\mathbf{w}^{(q)}}^*\|_{\mathbf{H}_{k^*:\infty}^{(q)}}^2\\
            =&\frac{\|{\mathbf{w}^{(q)}}^*\|_{\mathbf{I}_{0:k_0^*}^{(q)}}^2}{N\gamma}+\frac{\|{\mathbf{w}^{(q)}}^*\|_{\mathbf{I}_{k_0^*:k^*}^{(q)}}^2}{N\gamma}-\|{\mathbf{w}^{(q)}}^*\|_{\mathbf{H}_{k_0^*:k^*}^{(q)}}^2+\|{\mathbf{w}^{(q)}}^*\|_{\mathbf{H}_{k_0^*:\infty}^{(q)}}^2\\
            \leq &\frac{\|{\mathbf{w}^{(q)}}^*\|_{\mathbf{I}_{0:k_0^*}^{(q)}}^2}{N\gamma}+\|{\mathbf{w}^{(q)}}^*\|_{\mathbf{H}_{k_0^*:\infty}^{(q)}}^2\\
            \leq &\frac{\|{\mathbf{w}}^*\|_{\mathbf{I}_{0:k_0^*}}^2}{N\gamma}+\|{\mathbf{w}}^*\|_{\mathbf{H}_{k_0^*:\infty}}^2.
        \end{aligned}
    \end{equation*}
    Therefore, the sufficient conditions for $\mathrm{VarErr}\lesssim \mathrm{EffectiveVar}$ are
    \begin{equation*}
        \epsilon_d\lesssim 1,\quad  \epsilon_o,\epsilon_a,\epsilon_p \lesssim  \left(\frac{\sigma^2}{B\alpha_B\|\mathbf{w}^*\|_\mathbf{H}^2}+\frac{\frac{1}{N\gamma}\|{\mathbf{w}}^*\|_{\mathbf{I}_{0:k_0^*}}^2+\|{\mathbf{w}}^*\|_{\mathbf{H}_{k_0^*:\infty}}^2}{\|\mathbf{w}^*\|_\mathbf{H}^2}\right)\wedge 1,
    \end{equation*}
    
    Secondly, we analyze $\mathrm{BiasErr}$. Similarly,
    \begin{equation}
        \label{eq: d1}
        \begin{aligned}
            & \frac{1}{\gamma^2N^2}\cdot\|{\mathbf{w}^{(q)}}^*\|_{(\mathbf{H}_{0:k^*}^{(q)})^{-1}}^2+\|{\mathbf{w}^{(q)}}^*\|_{\mathbf{H}_{k^*:\infty}^{(q)}}^2\\
            =&\frac{1}{\gamma^2N^2}\cdot\left(\|{\mathbf{w}^{(q)}}^*\|_{(\mathbf{H}_{0:k_0^*}^{(q)})^{-1}}^2+\|{\mathbf{w}^{(q)}}^*\|_{(\mathbf{H}_{k_0^*:k^*}^{(q)})^{-1}}^2\right)-\|{\mathbf{w}^{(q)}}^*\|_{\mathbf{H}_{k_0^*:k^*}^{(q)}}^2+\|{\mathbf{w}^{(q)}}^*\|_{\mathbf{H}_{k_0^*:\infty}^{(q)}}^2\\
            \leq &\frac{1}{\gamma^2N^2}\cdot\|{\mathbf{w}^{(q)}}^*\|_{(\mathbf{H}_{0:k_0^*}^{(q)})^{-1}}^2+\|{\mathbf{w}^{(q)}}^*\|_{\mathbf{H}_{k_0^*:\infty}^{(q)}}^2\\
            \leq&\frac{1}{\gamma^2N^2}\cdot\|{\mathbf{w}}^*\|_{(\mathbf{H}_{0:k_0^*})^{-1}}^2+\|{\mathbf{w}}^*\|_{\mathbf{H}_{k_0^*:\infty}}^2=\mathrm{EffectiveBias}.
        \end{aligned}
    \end{equation}
    Thirdly, the sufficient condition for $\mathrm{ApproxErr}\lesssim R_0$ is $\epsilon_d\lesssim \frac{R_0}{\|\mathbf{w}^*\|_{\mathbf{H}}^2}.$ Overall, we require
    \begin{equation*}
        \epsilon_d\lesssim 1\wedge \frac{R_0}{\|\mathbf{w}^*\|_{\mathbf{H}}^2},\quad\epsilon_o,\epsilon_a,\epsilon_p \lesssim  \left(\frac{\sigma^2}{B\alpha_B\|\mathbf{w}^*\|_\mathbf{H}^2}+\frac{\frac{1}{N\gamma}\|{\mathbf{w}}^*\|_{\mathbf{I}_{0:k_0^*}}^2+\|{\mathbf{w}}^*\|_{\mathbf{H}_{k_0^*:\infty}}^2}{\|\mathbf{w}^*\|_\mathbf{H}^2}\right)\wedge 1.
    \end{equation*}
\end{proof}

\subsection{Proof for the Additive Statement in Corollary \ref{coro:comparison}}
\label{Proof for Theorem 12}
\begin{proof}
    Recall that $k_0^*=\max\{k: \lambda_k \geq \frac{1}{N\gamma}\}$,
    \begin{equation*}
        \begin{aligned}
            R_0=&\underbrace{\left(\frac{k_0^*}{N}+N\gamma^2\cdot\sum_{i>k_0^*}\lambda_i^2\right)\frac{\alpha_B\left(\frac{1}{N\gamma}\|{\mathbf{w}}^*\|_{\mathbf{I}_{0:k_0^*}}^2+\|{\mathbf{w}}^*\|_{\mathbf{H}_{k_0^*:\infty}}^2\right)
            +\frac{\sigma^2}{B}}{1-\gamma\alpha_B\mathrm{tr}\left(\mathbf{H}\right)}}_{\rm EffectiveVar}\\
            +&\underbrace{\frac{1}{\gamma^2N^2}\cdot\|      {\mathbf{w}}^*\|_{(\mathbf{H}_{0:k_0^*})^{-1}}^2+\|{\mathbf{w}}^*\|_{\mathbf{H}_{k_0^*:\infty}}^2}_{\rm EffectiveBias},
        \end{aligned}
    \end{equation*}
    and by Corollary \ref{theorem new 2.5, proof},
    \begin{equation*}
        \mathbb{E}[\mathcal{E}(\overline{\mathbf{w}}_N)] \lesssim \mathrm{ApproxErr}+\mathrm{ VarErr}+\mathrm{BiasErr},
    \end{equation*}
    where
    \begin{align*}
        &\mathrm{ApproxErr}\lesssim \frac{\epsilon_d}{\lambda_d+\epsilon_d}\left\|\mathbf{w}^*\right\|_\mathbf{H}^2,\quad \mathrm{BiasErr}\lesssim\frac{1}{\gamma^2N^2}\cdot\|{\mathbf{w}^{(q)}}^*\|_{(\mathbf{H}_{0:k^*}^{(q)})^{-1}}^2+\|{\mathbf{w}^{(q)}}^*\|_{\mathbf{H}_{k^*:\infty}^{(q)}}^2,\\
        &\mathrm{VarErr}\lesssim \frac{\alpha_B\left(\frac{\|{\mathbf{w}^{(q)}}^*\|_{\mathbf{I}_{0:k^*}^{(q)}}^2}{N\gamma}+\|{\mathbf{w}^{(q)}}^*\|_{\mathbf{H}_{k^*:\infty}^{(q)}}^2\right)+\frac{\sigma^2+\epsilon_o+\epsilon_a}{B}+\alpha_B \epsilon_p[\mathrm{tr}(\mathbf{H})+d\epsilon_d]}{1-\gamma\alpha_B[\mathrm{tr}(\mathbf{H})+d\epsilon_d]}\left(\frac{k^*}{N}+N\gamma^2\cdot\sum_{i>k^*}(\lambda_i+\epsilon_d)^2\right).
    \end{align*}
    We then compare the upper bound of $\mathbb{E}[\mathcal{E}(\overline{\mathbf{w}}_N)]$ with $R_0$. Regarding $\mathrm{VarErr}$, we first analyze $\frac{k^*}{N}+N\gamma^2\cdot\sum_{i>k^*}(\lambda_i+\epsilon_d)^2$. Recall that for $k_0^*< i \leq k^*$, $\frac{1}{N\gamma}-\epsilon_d\leq\lambda_i<\frac{1}{N\gamma}$,
    \begin{equation*}
        \begin{aligned}
            \frac{k^*}{N}+N\gamma^2\sum_{i>k^*}(\lambda_i+\epsilon_d)^2
            =&\frac{k_0^*}{N}+\frac{k^*-k_0^*}{N}-N\gamma^2\sum_{k_0^*<i\leq k^*}(\lambda_i+\epsilon_d)^2+N\gamma^2\sum_{i>k_0^*}(\lambda_i+\epsilon_d)^2\\
            \leq&\frac{k_0^*}{N}+N\gamma^2\sum_{i>k_0^*}(\lambda_i+\epsilon_d)^2.
        \end{aligned}
    \end{equation*}
    We then analyze $\frac{\|{\mathbf{w}^{(q)}}^*\|_{\mathbf{I}_{0:k^*}^{(q)}}^2}{N\gamma}+\|{\mathbf{w}^{(q)}}^*\|_{\mathbf{H}_{k^*:\infty}^{(q)}}^2$. Similarly,
    \begin{equation}
        \label{eq: d1'}
        \begin{aligned}
            \frac{\|{\mathbf{w}^{(q)}}^*\|_{\mathbf{I}_{0:k^*}^{(q)}}^2}{N\gamma}+\|{\mathbf{w}^{(q)}}^*\|_{\mathbf{H}_{k^*:\infty}^{(q)}}^2\leq \frac{\|{\mathbf{w}}^*\|_{\mathbf{I}_{0:k_0^*}}^2}{N\gamma}+\|{\mathbf{w}}^*\|_{\mathbf{H}_{k_0^*:\infty}}^2.
        \end{aligned}
    \end{equation}
    Therefore, the sufficient conditions for $\mathrm{VarErr}\lesssim \mathrm{EffectiveVar}$ are
    \begin{gather*}
        \epsilon_d\lesssim \sqrt{\frac{\frac{k_0^*}{N}+N\gamma^2\cdot\sum_{i>k_0^*}\lambda_i^2}{N\gamma^2(d-k_0^*)}},\quad \epsilon_p\lesssim \frac{\sigma^2}{B\alpha_B[\mathrm{tr}(\mathbf{H})+d\epsilon_d]}+\frac{\frac{\|{\mathbf{w}}^*\|_{\mathbf{I}_{0:k_0^*}}^2}{N\gamma}+\|{\mathbf{w}}^*\|_{\mathbf{H}_{k_0^*:\infty}}^2}{\mathrm{tr}(\mathbf{H})+d\epsilon_d}, \\ \epsilon_a,\epsilon_o\lesssim \sigma^2+B\alpha_B\left(\frac{\|{\mathbf{w}}^*\|_{\mathbf{I}_{0:k_0^*}}^2}{N\gamma}+\|{\mathbf{w}}^*\|_{\mathbf{H}_{k_0^*:\infty}}^2\right).
    \end{gather*}

    Secondly, we analyze $\mathrm{BiasErr}$. Similarly,
    \begin{equation*}
        \begin{aligned}
            & \frac{1}{\gamma^2N^2}\cdot\|{\mathbf{w}^{(q)}}^*\|_{(\mathbf{H}_{0:k^*}^{(q)})^{-1}}^2+\|{\mathbf{w}^{(q)}}^*\|_{\mathbf{H}_{k^*:\infty}^{(q)}}^2\\
            =&\frac{1}{\gamma^2N^2}\cdot\left(\|{\mathbf{w}^{(q)}}^*\|_{(\mathbf{H}_{0:k_0^*}^{(q)})^{-1}}^2+\|{\mathbf{w}^{(q)}}^*\|_{(\mathbf{H}_{k_0^*:k^*}^{(q)})^{-1}}^2\right)-\|{\mathbf{w}^{(q)}}^*\|_{\mathbf{H}_{k_0^*:k^*}^{(q)}}^2+\|{\mathbf{w}^{(q)}}^*\|_{\mathbf{H}_{k_0^*:\infty}^{(q)}}^2\\
            \leq &\frac{1}{\gamma^2N^2}\cdot\|{\mathbf{w}^{(q)}}^*\|_{(\mathbf{H}_{0:k_0^*}^{(q)})^{-1}}^2+\|{\mathbf{w}^{(q)}}^*\|_{\mathbf{H}_{k_0^*:\infty}^{(q)}}^2\\
            \leq&\frac{1}{\gamma^2N^2}\cdot\|{\mathbf{w}}^*\|_{(\mathbf{H}_{0:k_0^*})^{-1}}^2+\|{\mathbf{w}}^*\|_{\mathbf{H}_{k_0^*:\infty}}^2=\mathrm{EffectiveBias}.
        \end{aligned}
    \end{equation*}
    Thirdly, the sufficient condition for $\mathrm{ApproxErr}\lesssim R_0$ is $\epsilon_d\lesssim \frac{R_0}{\|\mathbf{w}^*\|_{\mathbf{H}}^2}\lambda_d.$ Overall, we require
    \begin{gather*}
        \epsilon_d\lesssim \sqrt{\frac{\frac{k_0^*}{N}+N\gamma^2\cdot\sum_{i>k_0^*}\lambda_i^2}{N\gamma^2(d-k_0^*)}}\wedge \frac{R_0\lambda_d}{\|\mathbf{w}^*\|_{\mathbf{H}}^2},\quad \epsilon_a,\epsilon_o\lesssim \sigma^2+B\alpha_B\left(\frac{\|{\mathbf{w}}^*\|_{\mathbf{I}_{0:k_0^*}}^2}{N\gamma}+\|{\mathbf{w}}^*\|_{\mathbf{H}_{k_0^*:\infty}}^2\right), \\ 
        \epsilon_p\lesssim \frac{\sigma^2}{B\alpha_B[\mathrm{tr}(\mathbf{H})+d\epsilon_d]}+\frac{\frac{\|{\mathbf{w}}^*\|_{\mathbf{I}_{0:k_0^*}}^2}{N\gamma}+\|{\mathbf{w}}^*\|_{\mathbf{H}_{k_0^*:\infty}}^2}{\mathrm{tr}(\mathbf{H})+d\epsilon_d}.
    \end{gather*}
\end{proof}
\subsection{Proof for the Multiplicative Statement in Corollary \ref{coro: case study}}
\begin{proof}
    We prove by applying Theorem \ref{Theorem 1}:
    \begin{equation*}
        \mathbb{E}[\mathcal{E}(\overline{\mathbf{w}}_N)] \lesssim \mathrm{ApproxErr}+\mathrm{VarErr}+\mathrm{BiasErr},
    \end{equation*}
    where
    \begin{equation*}
        \begin{aligned}
            &\mathrm{ApproxErr}\lesssim \frac{\epsilon_d}{1+\epsilon_d}\left\|\mathbf{w}^*\right\|_\mathbf{H}^2,\quad \mathrm{BiasErr}\lesssim\frac{1}{\gamma^2N^2}\cdot\|{\mathbf{w}^{(q)}}^*\|_{(\mathbf{H}_{0:k^*}^{(q)})^{-1}}^2+\|{\mathbf{w}^{(q)}}^*\|_{\mathbf{H}_{k^*:\infty}^{(q)}}^2,\\
            &\mathrm{VarErr}\lesssim\left(\frac{k^*}{N}+N\gamma^2(1+\epsilon_d)^2\sum_{i>k^*}\lambda_i^2\right) \frac{\frac{(1+\epsilon_o)\sigma^2}{B}+\alpha_B(1+\epsilon_o) [1+\epsilon_p+\epsilon_a(1+\epsilon_p)]\left\|\mathbf{w}^*\right\|_\mathbf{H}^2}{1-\gamma\alpha_B(1+\epsilon_o)[1+\epsilon_p+\epsilon_a(1+\epsilon_p)](1+\epsilon_d)\mathrm{tr}\left(\mathbf{H}\right)}.
        \end{aligned}
    \end{equation*}

    We first deal with $\mathrm{VarErr}$ under power-law spectrum Assumption \ref{assumption 1}. Under multiplicative quantization, we can estimate $k^*$ by 
    \begin{equation*}
        (1+\epsilon_d){k^*}^{-a}\approx \frac{1}{N\gamma},
    \end{equation*}
    that is
    \begin{equation}
    \label{estimate k*}
        k^*\approx \left[N\gamma (1+\epsilon_d)\right]^{\frac{1}{a}}.
    \end{equation}
    Further, the power-law Assumption \ref{assumption 1} also implies that for any positive $k$,
    \begin{equation}
    \label{estimate tail}
        \sum_{i>k} i^{-a}\approx k^{1-a}.
    \end{equation}
    By (\ref{estimate k*}) and (\ref{estimate tail}),
    \begin{equation*}
        \begin{aligned}
            \frac{k^*}{N}+N\gamma^2(1+\epsilon_d)^2\sum_{i>k^*}\lambda_i^2\lesssim &\frac{\min\left\{d,\left[N\gamma (1+\epsilon_d)\right]^{\frac{1}{a}}+N\gamma^2(1+\epsilon_d)^2\left[N\gamma (1+\epsilon_d)\right]^{\frac{1-2a}{a}}\right\}}{N}\\
            \lesssim&\frac{\min \left\{d,\left[N\gamma (1+\epsilon_d)\right]^{\frac{1}{a}}\right\}}{N}.
        \end{aligned}
    \end{equation*}
    Moreover, under polynomial spectrum Assumption \ref{assumption 1}, 
    \begin{equation*}
        \mathrm{tr}(\mathbf{H})\eqsim 1,\quad \mathbb{E}\|\mathbf{w}^*\|_\mathbf{H}^2 \eqsim 1.
    \end{equation*}
    Therefore, under Assumption \ref{assumption 1}, by applying stepsize $\gamma < \frac{1}{2\alpha_B (1+\epsilon_o)[1+\epsilon_p+\epsilon_a(1+\epsilon_p)](1+\epsilon_d)\mathrm{tr}(\mathbf{H})}$ and taking expectation on $\mathbf{w}^*$, it holds that
    \begin{equation}
        \label{eq: d4}
        \mathbb{E}_{\mathbf{w}^*}\mathrm{VarErr}\lesssim \frac{\min \left\{d,\left[N\gamma (1+\epsilon_d)\right]^{\frac{1}{a}}\right\}}{N}(1+\epsilon_o) [1+\epsilon_p+\epsilon_a(1+\epsilon_p)],
    \end{equation}
    where we use $\sigma^2\lesssim 1$.

    We secondly deal with $\mathrm{BiasErr}$. Under Assumption \ref{assumption 1}, using (\ref{eq: d1}),
    \begin{equation}
        \label{eq: d5}
        \begin{aligned}
            \mathbb{E}_{\mathbf{w}^*}\mathrm{BiasErr}\leq&
            \mathbb{E}_{\mathbf{w}^*}\left[\frac{1}{\gamma^2N^2}\cdot\|{\mathbf{w}}^*\|_{(\mathbf{H}_{0:k_0^*})^{-1}}^2+\|{\mathbf{w}}^*\|_{\mathbf{H}_{k_0^*:\infty}}^2\right]\\
            =&\frac{1}{N^2\gamma^2}\sum_{i=1}^{k_0^*}\lambda_i^{-1}+\sum_{i>k_0^*}^{d}\lambda_i\\
            \leq& \frac{k_0^*}{N\gamma}+\sum_{i>k_0^*}^{d}\lambda_i\\
            \eqsim&\frac{k_0^*}{N\gamma}+(k_0^*)^{1-a}\\
            \lesssim&\max\left\{d^{1-a},(N\gamma)^{1/a-1}\right\}.
        \end{aligned}
    \end{equation}
    
    Therefore, together with (\ref{eq: d4}) and (\ref{eq: d5}), and taking expectation on $\mathbf{w}^*$, we have
    \begin{equation*}
        \begin{aligned}
            &\mathbb{E}[\mathcal{E}(\overline{\mathbf{w}}_N)]\lesssim  \frac{\epsilon_d}{1+\epsilon_d}+\max\left\{d^{1-a},(N\gamma)^{1/a-1}\right\}+\frac{\min \left\{d,\left[N\gamma (1+\epsilon_d)\right]^{\frac{1}{a}}\right\}}{N}(1+\epsilon_o) [1+\epsilon_p+\epsilon_a(1+\epsilon_p)].
        \end{aligned}
    \end{equation*}
    Denote $R=\mathbb{E}[\mathcal{E}(\overline{\mathbf{w}}_N)]-\frac{\epsilon_d}{1+\epsilon_d}$.
    \begin{itemize}[leftmargin=*,nosep]
        \item $d>\left[N\gamma (1+\epsilon_d)\right]^{\frac{1}{a}}$

        In this case,
        \begin{equation*}
            \begin{aligned}
                R\lesssim& (N\gamma)^{1/a-1}+\frac{\left[N\gamma (1+\epsilon_d)\right]^{\frac{1}{a}}}{N}(1+\epsilon_o) [1+\epsilon_p+\epsilon_a(1+\epsilon_p)]\\
                \lesssim& N^{1/a-1}(1+\epsilon_o) [1+\epsilon_p+\epsilon_a(1+\epsilon_p)](1+\epsilon_d)^{1/a}.
            \end{aligned}
        \end{equation*}
        \item $(N\gamma)^{1/a}<d\leq \left[N\gamma (1+\epsilon_d)\right]^{\frac{1}{a}}$

        In this case,
        \begin{equation*}
            \begin{aligned}
                R\lesssim& (N\gamma)^{1/a-1}+(1+\epsilon_o) [1+\epsilon_p+\epsilon_a(1+\epsilon_p)]\frac{d}{N}\\
                \lesssim& (N\gamma)^{1/a-1}+(1+\epsilon_o) [1+\epsilon_p+\epsilon_a(1+\epsilon_p)]\frac{\left[N\gamma (1+\epsilon_d)\right]^{\frac{1}{a}}}{N}\\
                \lesssim&N^{1/a-1}(1+\epsilon_o) [1+\epsilon_p+\epsilon_a(1+\epsilon_p)](1+\epsilon_d)^{1/a}.
            \end{aligned}
        \end{equation*}
        \item $d\leq (N\gamma)^{1/a}$

        In this case,
        \begin{equation*}
            \begin{aligned}
                R\lesssim& d^{1-a}+(1+\epsilon_o) [1+\epsilon_p+\epsilon_a(1+\epsilon_p)]\frac{d}{N}\\
                \lesssim&d^{1-a}+(1+\epsilon_o) [1+\epsilon_p+\epsilon_a(1+\epsilon_p)]N^{1/a-1}.
            \end{aligned}
        \end{equation*}
    \end{itemize}
    Overall,
    \begin{equation*}
        \mathbb{E}\left[\mathcal{E}(\overline{\mathbf{w}}_N)\right]\lesssim\frac{\epsilon_d}{1+\epsilon_d}+d^{1-a}+N^{1/a-1}(1+\epsilon_o) [1+\epsilon_p+\epsilon_a(1+\epsilon_p)](1+\epsilon_d)^{1/a}.
    \end{equation*}
\end{proof}

\subsection{Proof for the Additive Statement in Corollary \ref{coro: case study}}
\begin{proof}
    We prove by applying Corollary \ref{theorem new 2.5}:
    \begin{equation*}
        \mathbb{E}[\mathcal{E}(\overline{\mathbf{w}}_N)] \lesssim \mathrm{ApproxErr}+\mathrm{ VarErr}+\mathrm{BiasErr},
    \end{equation*}
    where
    \begin{align*}
            &\mathrm{ApproxErr}\lesssim \frac{\epsilon_d}{\lambda_d+\epsilon_d}\left\|\mathbf{w}^*\right\|_\mathbf{H}^2,\quad \mathrm{BiasErr}\lesssim\frac{1}{\gamma^2N^2}\cdot\|{\mathbf{w}^{(q)}}^*\|_{(\mathbf{H}_{0:k^*}^{(q)})^{-1}}^2+\|{\mathbf{w}^{(q)}}^*\|_{\mathbf{H}_{k^*:\infty}^{(q)}}^2,\\
            &\mathrm{VarErr}\lesssim \frac{\alpha_B\|\mathbf{w}^*\|_\mathbf{H}^2+\frac{\sigma^2+\epsilon_o+\epsilon_a}{B}+\alpha_B \epsilon_p[\mathrm{tr}(\mathbf{H})+p\epsilon_d]}{1-\gamma\alpha_B[\mathrm{tr}(\mathbf{H})+p\epsilon_d]}\left(\frac{k^*}{N}+N\gamma^2\cdot\sum_{i>k^*}(\lambda_i+\epsilon_d)^2\right).
    \end{align*}
    We first deal with $\mathrm{BiasErr}$.
    Under Assumption \ref{assumption 1}, using (\ref{eq: d1'}),
    \begin{equation}
        \label{eq: d6}
        \begin{aligned}
            \mathbb{E}_{\mathbf{w}^*}\mathrm{BiasErr}\leq&
            \mathbb{E}_{\mathbf{w}^*}\left[\frac{1}{\gamma^2N^2}\cdot\|{\mathbf{w}}^*\|_{(\mathbf{H}_{0:k_0^*})^{-1}}^2+\|{\mathbf{w}}^*\|_{\mathbf{H}_{k_0^*:\infty}}^2\right]\\
            =&\frac{1}{N^2\gamma^2}\sum_{i=1}^{k_0^*}\lambda_i^{-1}+\sum_{i>k_0^*}^{d}\lambda_i\\
            \leq& \frac{k_0^*}{N\gamma}+\sum_{i>k_0^*}^{d}\lambda_i\\
            =&\frac{k_0^*}{N\gamma}+(k_0^*)^{1-a}\\
            \lesssim&\max\left\{d^{1-a},(N\gamma)^{1/a-1}\right\}.
        \end{aligned}
    \end{equation}
    We then analyze $\mathrm{VarErr}$. If $\epsilon_d+d^{-a}\geq \frac{1}{N\gamma}$, then
    \begin{equation*}
        \begin{aligned}
            \frac{k^*}{N}+N\gamma^2\cdot\sum_{i>k^*}(\lambda_i+\epsilon_d)^2=\frac{d}{N}.
        \end{aligned}
    \end{equation*}
    Otherwise,
    \begin{equation*}
        \begin{aligned}
            &\frac{k^*}{N}+N\gamma^2\cdot\sum_{i>k^*}(\lambda_i+\epsilon_d)^2\\
            \lesssim&\frac{\left(\frac{1}{N\gamma}-\epsilon_d\right)^{-1/a}+N^2\gamma^2\left(\frac{1}{N\gamma}-\epsilon_d\right)^{-(1-2a)/a}+\epsilon_d^2N^2\gamma^2\left[d-\left(\frac{1}{N\gamma}-\epsilon_d\right)^{-1/a}\right]}{N}\\
            \lesssim&\frac{\left(\frac{1}{N\gamma}-\epsilon_d\right)^{-1/a}+\epsilon_d^2N^2\gamma^2\left[d-\left(\frac{1}{N\gamma}-\epsilon_d\right)^{-1/a}\right]}{N}.
        \end{aligned}
    \end{equation*}
    Denote $k_{\rm eff}=\left[d^{-a}\vee \left(\frac{1}{N\gamma}-\epsilon_d\right)\right]^{-1/a}$, it follows that
    \begin{equation*}
        \frac{k^*}{N}+N\gamma^2\cdot\sum_{i>k^*}(\lambda_i+\epsilon_d)^2\lesssim\frac{k_{\rm eff}+\epsilon_d^2N^2\gamma^2(d-k_{\rm eff})}{N}.
    \end{equation*}
    Hence, under Assumption \ref{assumption 1}, taking expectation on $\mathbf{w}^*$ and applying stepsize $\gamma < \frac{1}{2 \alpha_B [\mathrm{tr}(\mathbf{H})+p\epsilon_d]}$,
    \begin{equation}
        \label{eq: d8}\mathbb{E}_{\mathbf{w}^*}\mathrm{VarErr}\lesssim \left(1+\frac{\epsilon_o+\epsilon_a}{B}+\epsilon_p(1+d\epsilon_d)\right)\frac{k_{\rm eff}+\epsilon_d^2N^2\gamma^2(d-k_{\rm eff})}{N}.
    \end{equation}
    
    Therefore, together with (\ref{eq: d6}) and (\ref{eq: d8}), and taking expectation on $\mathbf{w}^*$, we have
    \begin{equation*}
        \mathbb{E}[\mathcal{E}(\overline{\mathbf{w}}_N)]\lesssim  \frac{\epsilon_d}{d^{-a}+\epsilon_d}+\max\left\{d^{1-a},(N\gamma)^{1/a-1}\right\}+\left(1+\frac{\epsilon_o+\epsilon_a}{B}+\epsilon_p(1+d\epsilon_d)\right)\frac{k_{\rm eff}+\epsilon_d^2N^2\gamma^2(d-k_{\rm eff})}{N}.
    \end{equation*}
    Denote $R=\mathbb{E}[\mathcal{E}(\overline{\mathbf{w}}_N)]-  \frac{\epsilon_d}{d^{-a}+\epsilon_d}$.
    \begin{itemize}[leftmargin=*,nosep]
        \item $d^{-a}\leq {1}/{(N\gamma)}-\epsilon_d$

        In this case, let $\epsilon_d'=d^a\epsilon_d$. Then $d^{-a}\leq {1}/{(N\gamma)}-d^{-a}\epsilon_d'$. That is, $d^{-a}\leq \frac{1}{N\gamma(1+\epsilon_d')}$.
        \begin{equation*}
            \begin{aligned}
                R\lesssim& (N\gamma)^{1/a-1}+\left(1+\frac{\epsilon_o+\epsilon_a}{B}+\epsilon_p(1+d\epsilon_d)\right)\frac{\left(\frac{1}{N\gamma}-\epsilon_d\right)^{-1/a}+\epsilon_d^2N^2\gamma^2\left(d-\left(\frac{1}{N\gamma}-\epsilon_d\right)^{-1/a}\right)}{N}\\
                =& (N\gamma)^{1/a-1}+\left(1+\frac{\epsilon_o+\epsilon_a}{B}+\epsilon_p(1+d\epsilon_d)\right)\frac{\left(\frac{1}{N\gamma}-d^{-a}\epsilon_d'\right)^{-1/a}(1-\epsilon_d^2N^2\gamma^2)+\epsilon_d^2N^2\gamma^2d}{N}\\
                \leq& (N\gamma)^{1/a-1}+\left(1+\frac{\epsilon_o+\epsilon_a}{B}+\epsilon_p(1+d\epsilon_d)\right)\frac{(1+\epsilon_d')^{1/a}(N\gamma)^{1/a}(1-\epsilon_d^2N^2\gamma^2)+\epsilon_d^2N^2\gamma^2d}{N}\\
                \leq& (N\gamma)^{1/a-1}+\left(1+\frac{\epsilon_o+\epsilon_a}{B}+\epsilon_p(1+d\epsilon_d)\right)\left[\frac{(1+\epsilon_d')^{1/a}(N\gamma)^{1/a}}{N}+\epsilon_d^2N\gamma^2d\right].
            \end{aligned}
        \end{equation*}
        We then focus on $\epsilon_d^2N\gamma^2d$. By $N\gamma \leq \frac{d^{a}}{1+\epsilon_d'}$, we have
        \begin{equation*}
            \epsilon_d^2N\gamma^2d \lesssim \frac{\epsilon_d^2}{1+\epsilon_d'}d^{1+a}=\frac{(\epsilon_d')^2}{1+\epsilon_d'}d^{1-a}.
        \end{equation*}
        Therefore,
        \begin{equation*}
            \begin{aligned}
                R\lesssim N^{1/a-1}(1+d^a\epsilon_d)^{1/a}\left(1+\frac{\epsilon_o+\epsilon_a}{B}+\epsilon_p(1+d\epsilon_d)\right)+\left(1+\frac{\epsilon_o+\epsilon_a}{B}+\epsilon_p(1+d\epsilon_d)\right)\frac{(d^a\epsilon_d)^2}{1+d^a\epsilon_d}d^{1-a}.
            \end{aligned}
        \end{equation*}
        \item $\frac{1}{N\gamma}-\epsilon_d<d^{-a}\leq \frac{1}{N\gamma}$

        In this case, $\frac{1}{N\gamma}-d^{-a}\epsilon_d'<d^{-a}$, that is, $d^{a}<N\gamma(1+\epsilon_d')$. Consequently,
        \begin{equation*}
            \begin{aligned}
                R\lesssim&(N\gamma)^{1/a-1}+\left(1+\frac{\epsilon_o+\epsilon_a}{B}+\epsilon_p(1+d\epsilon_d)\right)\frac{d}{N}\\
                \lesssim&(N\gamma)^{1/a-1}+\left(1+\frac{\epsilon_o+\epsilon_a}{B}+\epsilon_p(1+d\epsilon_d)\right)(1+\epsilon_d')^{1/a}N^{1/a-1}\\
                \lesssim&\left(1+\frac{\epsilon_o+\epsilon_a}{B}+\epsilon_p(1+d\epsilon_d)\right)(1+d^a\epsilon_d)^{1/a}N^{1/a-1}.
            \end{aligned}
        \end{equation*}
        \item $d^{-a}>\frac{1}{N\gamma}$

        In this case,
        \begin{equation*}
            \begin{aligned}
                R\lesssim&d^{1-a}+\left(1+\frac{\epsilon_o+\epsilon_a}{B}+\epsilon_p(1+d\epsilon_d)\right)\frac{d}{N}\\
                \lesssim&d^{1-a}+\left(1+\frac{\epsilon_o+\epsilon_a}{B}+\epsilon_p(1+d\epsilon_d)\right)N^{1/a-1}.
            \end{aligned}
        \end{equation*}
    \end{itemize}
    Overall,
    \begin{equation*}
        \mathbb{E}[\mathcal{E}(\overline{\mathbf{w}}_N)]\lesssim \frac{d^a\epsilon_d}{1+d^a\epsilon_d}+\left(1+\frac{\epsilon_o+\epsilon_a}{B}+\epsilon_p(1+d\epsilon_d)\right)\left[N^{\frac{1}{a}-1}(1+d^a\epsilon_d)^{\frac{1}{a}}+d^{1-a}\left(1+\frac{(d^a\epsilon_d)^2}{1+d^a\epsilon_d}\right)\right].
    \end{equation*}
\end{proof}

\section{Discussion of Assumptions}
\label{Discussion of Assumptions}
In this section, we verify Assumption \ref{ass2} and Assumption \ref{ass3} under the standard fourth moment and noise assumptions made on the full-precision data \citep{zou2023benign}.
\begin{assumption}
\label{ass2 raw data}
    Assume there exists a positive constant $\alpha_0>0$, such that for any PSD matrix $\mathbf{A}$, it holds that
    \begin{equation*}
        \mathbb{E}\left[{\mathbf{x}}{\mathbf{x}}^\top\mathbf{A}{\mathbf{x}}{\mathbf{x}}^\top\right]\preceq\alpha_0\operatorname{tr}(\mathbf{H}\mathbf{A})\mathbf{H}.
    \end{equation*}
\end{assumption}
\begin{assumption}
\label{assumption noise raw}
    Assume there exists a constant $\sigma_0^2$ such that
    \begin{equation*}
        \mathbb{E}\left[(y-\langle \mathbf{w}^*,\mathbf{x}\rangle)^2\mathbf{x}\mathbf{x}^\top\right]\preceq\sigma_0^{2}\mathbf{H}.
    \end{equation*}
\end{assumption}
We consider specific quantization schemes.
\begin{example}[\rm \textbf{Strong multiplicative quantization}]
\label{example G1}
    We consider a strong multiplicative quantization. In this case, there exist constants $\epsilon_d,\epsilon_d'$ such that
    \begin{equation*}
          \mathbb{E}\left[\boldsymbol{\epsilon}^{(d)}{\boldsymbol{\epsilon}^{(d)}}^\top\big| \mathbf{x}\right]=\epsilon_d\mathbf{x}\mathbf{x}^\top,\quad\mathbb{E}\left[\boldsymbol{\epsilon}^{(d)}{\boldsymbol{\epsilon}^{(d)}}^\top\mathbf{A}\boldsymbol{\epsilon}^{(d)}{\boldsymbol{\epsilon}^{(d)}}^\top\Big| \mathbf{x}\right]\preceq \epsilon_d' \mathbf{x}\mathbf{x}^\top \mathbf{A}\mathbf{x}\mathbf{x}^\top.
    \end{equation*}
\end{example}
\begin{example}[\rm \textbf{Strong additive quantization}]
\label{exampleG2}
    We consider a strong additive quantization. In this case, there exist constants $\epsilon_d, \epsilon_d'$ such that
    \begin{equation}                    
        \mathbb{E}\left[\boldsymbol{\epsilon}^{(d)}{\boldsymbol{\epsilon}^{(d)}}^\top\big| \mathbf{x}\right]=\epsilon_d\mathbf{I},\quad\mathbb{E}\left[\boldsymbol{\epsilon}^{(d)}{\boldsymbol{\epsilon}^{(d)}}^\top\mathbf{A}\boldsymbol{\epsilon}^{(d)}{\boldsymbol{\epsilon}^{(d)}}^\top\Big| \mathbf{x}\right]\preceq \epsilon_d' \mathrm{tr}(\mathbf{A})\mathbf{I}.
    \end{equation}
\end{example}
\subsection{Discussion of Assumption \ref{ass2}}
Under Assumption \ref{ass2 raw data}, we are ready to verify if Assumption \ref{ass2} can be satisfied. We begin by:
\begin{equation}
\label{G0}
    \begin{aligned}
        \mathbb{E}\left[{\mathbf{x}^{(q)}}{\mathbf{x}^{(q)}}^\top\mathbf{A}{\mathbf{x}^{(q)}}{\mathbf{x}^{(q)}}^\top\right]=&\mathbb{E}\left[\left({\mathbf{x}^{(q)}}^\top\mathbf{A}{\mathbf{x}^{(q)}}\right){\mathbf{x}^{(q)}}{\mathbf{x}^{(q)}}^\top\right]\\
        \preceq &2\mathbb{E}\left[\left({\mathbf{x}^{(q)}}^\top\mathbf{A}{\mathbf{x}^{(q)}}\right)({\mathbf{x}}{\mathbf{x}}^\top+\boldsymbol{\epsilon}^{(d)}{\boldsymbol{\epsilon}^{(d)}}^\top)\right]\\
        \preceq& 4\mathbb{E}\left[\left({\mathbf{x}}^\top\mathbf{A}{\mathbf{x}}+{\boldsymbol{\epsilon}^{(d)}}^\top \mathbf{A}\boldsymbol{\epsilon}^{(d)}\right)({\mathbf{x}}{\mathbf{x}}^\top+\boldsymbol{\epsilon}^{(d)}{\boldsymbol{\epsilon}^{(d)}}^\top)\right]\\
        =&4\mathbb{E}\left[\mathbf{x}{\mathbf{x}}^\top\mathbf{A}{\mathbf{x}}{\mathbf{x}}^\top\right]+4\mathbb{E}\left[\boldsymbol{\epsilon}^{(d)}{\boldsymbol{\epsilon}^{(d)}}^\top\mathbf{A}\boldsymbol{\epsilon}^{(d)}{\boldsymbol{\epsilon}^{(d)}}^\top\right]\\
        +&4\mathbb{E}\left[\left({\mathbf{x}}^\top\mathbf{A}{\mathbf{x}}\right)\boldsymbol{\epsilon}^{(d)}{\boldsymbol{\epsilon}^{(d)}}^\top\right]+4\mathbb{E}\left[\left({\boldsymbol{\epsilon}^{(d)}}^\top\mathbf{A}{\boldsymbol{\epsilon}^{(d)}}\right){\mathbf{x}}{\mathbf{x}}^\top\right].
    \end{aligned}
\end{equation}
\begin{lemma}
    Under strong multiplicative quantization \ref{example G1} and Assumption \ref{ass2 raw data}, 
    \begin{equation*}
        \begin{aligned}
            \mathbb{E}\left[{\mathbf{x}^{(q)}}{\mathbf{x}^{(q)}}^\top\mathbf{A}{\mathbf{x}^{(q)}}{\mathbf{x}^{(q)}}^\top\right]\lesssim\alpha_0{(1+\epsilon_d+\epsilon_d')}\mathrm{tr}(\mathbf{H}^{(q)}\mathbf{A})\mathbf{H}^{(q)}.
        \end{aligned}
    \end{equation*}
\end{lemma}
\begin{proof}
    We proof by (\ref{G0}).
    From Assumption \ref{ass2 raw data},
    \begin{equation}
        \label{G1}
        \mathbb{E}\left[{\mathbf{x}}{\mathbf{x}}^\top\mathbf{A}{\mathbf{x}}{\mathbf{x}}^\top\right]\preceq\alpha_0\operatorname{tr}(\mathbf{H}\mathbf{A})\mathbf{H}.
    \end{equation}
    Under strong multiplicative quantization \ref{example G1}, we have
    \begin{equation}
        \label{G2}
        \mathbb{E}\left[\left({\boldsymbol{\epsilon}^{(d)}}^\top\mathbf{A}{\boldsymbol{\epsilon}^{(d)}}\right){\mathbf{x}}{\mathbf{x}}^\top\right]=\epsilon_d\mathbb{E}\left[{\mathbf{x}}{\mathbf{x}}^\top\mathbf{A}{\mathbf{x}}{\mathbf{x}}^\top\right]\preceq\epsilon_d\alpha_0\operatorname{tr}(\mathbf{H}\mathbf{A})\mathbf{H},
    \end{equation}
    \begin{equation}
        \label{G3}
        \mathbb{E}\left[\left({\mathbf{x}}^\top\mathbf{A}{\mathbf{x}}\right)\boldsymbol{\epsilon}^{(d)}{\boldsymbol{\epsilon}^{(d)}}^\top\right]=\epsilon_d\mathbb{E}\left[{\mathbf{x}}{\mathbf{x}}^\top\mathbf{A}{\mathbf{x}}{\mathbf{x}}^\top\right]\preceq\epsilon_d\alpha_0\operatorname{tr}(\mathbf{H}\mathbf{A})\mathbf{H},
    \end{equation}
    and
    \begin{equation}
        \label{G4}
        \mathbb{E}\left[\boldsymbol{\epsilon}^{(d)}{\boldsymbol{\epsilon}^{(d)}}^\top\mathbf{A}\boldsymbol{\epsilon}^{(d)}{\boldsymbol{\epsilon}^{(d)}}^\top\right]\preceq\epsilon_d'\mathbb{E}\left[{\mathbf{x}}{\mathbf{x}}^\top\mathbf{A}{\mathbf{x}}{\mathbf{x}}^\top\right]\preceq\epsilon_d'\alpha_0\operatorname{tr}(\mathbf{H}\mathbf{A})\mathbf{H}.
    \end{equation}
    Therefore, together with (\ref{G0}), (\ref{G1}), (\ref{G2}), (\ref{G3}) and (\ref{G4}), we have
    \begin{equation*}
        \begin{aligned}
            \mathbb{E}\left[{\mathbf{x}^{(q)}}{\mathbf{x}^{(q)}}^\top\mathbf{A}{\mathbf{x}^{(q)}}{\mathbf{x}^{(q)}}^\top\right]\lesssim\alpha_0(1+\epsilon_d+\epsilon_d')\mathrm{tr}(\mathbf{H}\mathbf{A})\mathbf{H}\leq\alpha_0{(1+\epsilon_d+\epsilon_d')}\mathrm{tr}(\mathbf{H}^{(q)}\mathbf{A})\mathbf{H}^{(q)}.
        \end{aligned}
    \end{equation*}
    That is, under strong multiplicative quantization Example \ref{example G1} and fourth moment Assumption \ref{ass2 raw data} on full-precision data, Assumption \ref{ass2} is verified.
\end{proof}
\begin{lemma}
    Under strong additive quantization \ref{exampleG2} and Assumption \ref{ass2 raw data},
    \begin{equation*}
        \begin{aligned}
            \mathbb{E}\left[{\mathbf{x}^{(q)}}{\mathbf{x}^{(q)}}^\top\mathbf{A}{\mathbf{x}^{(q)}}{\mathbf{x}^{(q)}}^\top\right]\lesssim(1+\alpha_0)\left(1+\frac{\epsilon_d'}{\epsilon_d^2}\right)\mathrm{tr}(\mathbf{H}^{(q)}\mathbf{A})\mathbf{H}^{(q)}.
        \end{aligned}
    \end{equation*}
\end{lemma}
\begin{proof}
    We proof by (\ref{G0}). Under strong additive quantization \ref{exampleG2},
    \begin{equation}
        \label{G5}
        \mathbb{E}\left[\left({\boldsymbol{\epsilon}^{(d)}}^\top\mathbf{A}{\boldsymbol{\epsilon}^{(d)}}\right){\mathbf{x}}{\mathbf{x}}^\top\right]\preceq \epsilon_d\operatorname{tr}(\mathbf{A})\mathbf{H},
    \end{equation}
    \begin{equation}
        \label{G6}
        \mathbb{E}\left[\left({\mathbf{x}}^\top\mathbf{A}{\mathbf{x}}\right)\boldsymbol{\epsilon}^{(d)}{\boldsymbol{\epsilon}^{(d)}}^\top\right]\preceq \epsilon_d\operatorname{tr}(\mathbf{H}\mathbf{A})\mathbf{I},
    \end{equation}
    and
    \begin{equation}
        \label{G7}
        \mathbb{E}\left[\boldsymbol{\epsilon}^{(d)}{\boldsymbol{\epsilon}^{(d)}}^\top\mathbf{A}\boldsymbol{\epsilon}^{(d)}{\boldsymbol{\epsilon}^{(d)}}^\top\right]\preceq\epsilon_d'\operatorname{tr}(\mathbf{A})\mathbf{I}.
    \end{equation}
    
    Therefore, together with (\ref{G0}), (\ref{G5}), (\ref{G6}) and (\ref{G7}), we have
    \begin{equation*}
        \begin{aligned}
            \mathbb{E}\left[{\mathbf{x}^{(q)}}{\mathbf{x}^{(q)}}^\top\mathbf{A}{\mathbf{x}^{(q)}}{\mathbf{x}^{(q)}}^\top\right]\lesssim(1+\alpha_0)\left(1+\frac{\epsilon_d'}{\epsilon_d^2}\right)\mathrm{tr}(\mathbf{H}^{(q)}\mathbf{A})\mathbf{H}^{(q)}.
        \end{aligned}
    \end{equation*}
    That is, under strong additive quantization Example \ref{exampleG2} and fourth moment Assumption \ref{ass2 raw data} on full-precision data, Assumption \ref{ass2} is verified.
\end{proof}

\subsection{Discussion of Assumption \ref{ass3}}
Under Assumption \ref{assumption noise raw}, we are ready to verify if Assumption \ref{ass3} can be satisfied. We begin by:
\begin{equation}
\label{G8}
    \begin{aligned}
        &\mathbb{E}\left[(y^{(q)}-\langle {\mathbf{w}^{(q)}}^*,\mathbf{x}^{(q)}\rangle)^2{\mathbf{x}^{(q)}}{\mathbf{x}^{(q)}}^\top\right]\\
        = & \mathbb{E}\left[(y^{(q)}-y+y-\langle \mathbf{w}^*,\mathbf{x}\rangle+\langle \mathbf{w}^*,\mathbf{x}\rangle-\langle {\mathbf{w}^{(q)}}^*,\mathbf{x}^{(q)}\rangle)^2{\mathbf{x}^{(q)}}{\mathbf{x}^{(q)}}^\top\right]\\
        \preceq &3\mathbb{E}\left[(y^{(q)}-y)^2{\mathbf{x}^{(q)}}{\mathbf{x}^{(q)}}^\top\right]+3\mathbb{E}\left[(y-\langle \mathbf{w}^*,\mathbf{x}\rangle)^2{\mathbf{x}^{(q)}}{\mathbf{x}^{(q)}}^\top\right]\\
        +&3\mathbb{E}\left[(\langle \mathbf{w}^*,\mathbf{x}\rangle-\langle {\mathbf{w}^{(q)}}^*,\mathbf{x}^{(q)}\rangle)^2{\mathbf{x}^{(q)}}{\mathbf{x}^{(q)}}^\top\right]\\
        \preceq &3\mathbb{E}\left[(y^{(q)}-y)^2{\mathbf{x}^{(q)}}{\mathbf{x}^{(q)}}^\top\right]+3\mathbb{E}\left[(y-\langle \mathbf{w}^*,\mathbf{x}\rangle)^2{\mathbf{x}^{(q)}}{\mathbf{x}^{(q)}}^\top\right]\\
        +& 6\mathbb{E}\left[\langle {\mathbf{w}^{(q)}}^*- \mathbf{w}^*,\mathbf{x}\rangle^2{\mathbf{x}^{(q)}}{\mathbf{x}^{(q)}}^\top\right]+6\mathbb{E}\left[\langle {\mathbf{w}^{(q)}}^*,\boldsymbol{\epsilon}^{(d)}\rangle^2{\mathbf{x}^{(q)}}{\mathbf{x}^{(q)}}^\top\right].
    \end{aligned}
\end{equation}
\begin{lemma}
    Under strong multiplicative quantization \ref{example G1}, Assumption \ref{ass2 raw data}, and Assumption \ref{assumption noise raw},
    \begin{equation*}
        \begin{aligned}
            \mathbb{E}\left[(y^{(q)}-\langle {\mathbf{w}^{(q)}}^*,\mathbf{x}^{(q)}\rangle)^2{\mathbf{x}^{(q)}}{\mathbf{x}^{(q)}}^\top\right]
            \precsim \left(\sigma_0^2+\epsilon_l+\frac{1+\epsilon_d+\epsilon_d'}{1+\epsilon_d}\alpha_0\|\mathbf{w}^*\|_\mathbf{H}^2\right)\mathbf{H}^{(q)}.
        \end{aligned}
    \end{equation*}
\end{lemma}
\begin{proof}
    Regarding $\mathbb{E}\left[(y-\langle \mathbf{w}^*,\mathbf{x}\rangle)^2{\mathbf{x}^{(q)}}{\mathbf{x}^{(q)}}^\top\right]$,
    \begin{equation}
    \label{G9}
        \begin{aligned}
            \mathbb{E}\left[(y-\langle \mathbf{w}^*,\mathbf{x}\rangle)^2{\mathbf{x}^{(q)}}{\mathbf{x}^{(q)}}^\top\right] \preceq& 2\mathbb{E}\left[(y-\langle \mathbf{w}^*,\mathbf{x}\rangle)^2{\mathbf{x}}{\mathbf{x}}^\top\right]+2\mathbb{E}\left[(y-\langle \mathbf{w}^*,\mathbf{x}\rangle)^2{\boldsymbol{\epsilon}^{(d)}}{\boldsymbol{\epsilon}^{(d)}}^\top\right]\\
            \preceq & 2(1+\epsilon_d)\mathbb{E}\left[(y-\langle \mathbf{w}^*,\mathbf{x}\rangle)^2{\mathbf{x}}{\mathbf{x}}^\top\right]\\
            \preceq & 2(1+\epsilon_d)\sigma_0^2\mathbf{H},
        \end{aligned}
    \end{equation}
    where the second inequality holds by the definition of Example \ref{example G1} and the last inequality holds by Assumption \ref{assumption noise raw}. Regarding $\mathbb{E}\left[\langle {\mathbf{w}^{(q)}}^*,\boldsymbol{\epsilon}^{(d)}\rangle^2{\mathbf{x}^{(q)}}{\mathbf{x}^{(q)}}^\top\right]$,
    \begin{equation}
    \label{G10}
        \begin{aligned}
            &\mathbb{E}\left[\langle {\mathbf{w}^{(q)}}^*,\boldsymbol{\epsilon}^{(d)}\rangle^2{\mathbf{x}^{(q)}}{\mathbf{x}^{(q)}}^\top\right]\\
            =&\mathbb{E}\left[{\boldsymbol{\epsilon}^{(d)}}^\top{\mathbf{w}^{(q)}}^*{{\mathbf{w}^{(q)}}^*}^\top \boldsymbol{\epsilon}^{(d)}{\mathbf{x}^{(q)}}{\mathbf{x}^{(q)}}^\top\right]\\
            \preceq & 2\mathbb{E}\left[{\boldsymbol{\epsilon}^{(d)}}^\top{\mathbf{w}^{(q)}}^*{{\mathbf{w}^{(q)}}^*}^\top \boldsymbol{\epsilon}^{(d)}{\mathbf{x}}{\mathbf{x}}^\top\right]+2\mathbb{E}\left[{\boldsymbol{\epsilon}^{(d)}}^\top{\mathbf{w}^{(q)}}^*{{\mathbf{w}^{(q)}}^*}^\top \boldsymbol{\epsilon}^{(d)}{\boldsymbol{\epsilon}^{(d)}}{\boldsymbol{\epsilon}^{(d)}}^\top\right]\\
            \preceq & 2\epsilon_d\alpha_0\mathrm{tr}({\mathbf{w}^{(q)}}^*{{\mathbf{w}^{(q)}}^*}^\top\mathbf{H})\mathbf{H}+2\epsilon_d' \alpha_0\mathrm{tr}({\mathbf{w}^{(q)}}^*{{\mathbf{w}^{(q)}}^*}^\top\mathbf{H})\mathbf{H},
        \end{aligned}
    \end{equation}
    where the last inequality holds by the definition of Example \ref{example G1} and Assumption \ref{ass2 raw data}. Regarding the term
    \begin{equation}
    \label{G11}
        \begin{aligned}
            \mathbb{E}\left[\langle {\mathbf{w}^{(q)}}^*- \mathbf{w}^*,\mathbf{x}\rangle^2{\mathbf{x}^{(q)}}{\mathbf{x}^{(q)}}^\top\right]\preceq &2\mathbb{E}\left[\langle {\mathbf{w}^{(q)}}^*- \mathbf{w}^*,\mathbf{x}\rangle^2{\mathbf{x}}{\mathbf{x}}^\top\right]+2\mathbb{E}\left[\langle {\mathbf{w}^{(q)}}^*- \mathbf{w}^*,\mathbf{x}\rangle^2{\boldsymbol{\epsilon}^{(d)}}{\boldsymbol{\epsilon}^{(d)}}^\top\right]\\
            \preceq &2(1+\epsilon_d)\alpha_0 \mathrm{tr}\left(({\mathbf{w}^{(q)}}^*- \mathbf{w}^*)({\mathbf{w}^{(q)}}^*- \mathbf{w}^*)^\top\mathbf{H}\right)\mathbf{H},
        \end{aligned}
    \end{equation}
    where the last inequality holds by the definition of Example \ref{example G1} and Assumption \ref{ass2 raw data}. Regarding $\mathbb{E}\left[(y^{(q)}-y)^2{\mathbf{x}^{(q)}}{\mathbf{x}^{(q)}}^\top\right]$, if we further assume that there exists a constant $C$ such that $\mathbb{E}\left[y^2 \mathbf{x}\mathbf{x}^\top\right]\preceq C\mathbf{H}$, then
    \begin{equation}
    \label{G12}
        \begin{aligned}
            \mathbb{E}\left[(y^{(q)}-y)^2{\mathbf{x}^{(q)}}{\mathbf{x}^{(q)}}^\top\right]\preceq & 2\mathbb{E}\left[(y^{(q)}-y)^2{\mathbf{x}}{\mathbf{x}}^\top\right]+2\mathbb{E}\left[(y^{(q)}-y)^2{\boldsymbol{\epsilon}^{(d)}}{\boldsymbol{\epsilon}^{(d)}}^\top\right]\\
            \preceq & 2(1+\epsilon_d)\mathbb{E}\left[(y^{(q)}-y)^2\mathbf{x}\mathbf{x}^\top\right]\\
            = & 2(1+\epsilon_d)\epsilon_l \mathbb{E}[y^2\mathbf{x}\mathbf{x}^\top]\\
            \preceq & 2(1+\epsilon_d)\epsilon_lC\mathbf{H}.
        \end{aligned}
    \end{equation}
    
    Therefore, together with (\ref{G8}), (\ref{G9}), (\ref{G10}), (\ref{G11}) and (\ref{G12}), we have
    \begin{equation*}
        \begin{aligned}
            &\mathbb{E}\left[(y^{(q)}-\langle {\mathbf{w}^{(q)}}^*,\mathbf{x}^{(q)}\rangle)^2{\mathbf{x}^{(q)}}{\mathbf{x}^{(q)}}^\top\right]\\
            \precsim & (1+\epsilon_d)\sigma_0^2\mathbf{H}+(\epsilon_d+\epsilon_d')\alpha_0\left\|{\mathbf{w}^{(q)}}^*\right\|_\mathbf{H}^2\mathbf{H}
            +(1+\epsilon_d)\alpha_0\left\|{\mathbf{w}^{(q)}}^*- \mathbf{w}^*\right\|_\mathbf{H}^2\mathbf{H}+(1+\epsilon_d)\epsilon_l\mathbf{H}\\
            \precsim& \left[(1+\epsilon_d)(\sigma_0^2+\epsilon_l)+(1+\epsilon_d+\epsilon_d')\alpha_0\|\mathbf{w}^*\|_\mathbf{H}^2\right]\mathbf{H}\\
            =&\left(\sigma_0^2+\epsilon_l+\frac{1+\epsilon_d+\epsilon_d'}{1+\epsilon_d}\alpha_0\|\mathbf{w}^*\|_\mathbf{H}^2\right)\mathbf{H}^{(q)}.
        \end{aligned}
    \end{equation*}
    That is, under strong multiplicative quantization Example \ref{example G1} and fourth moment Assumption \ref{assumption noise raw} on full-precision data, Assumption \ref{ass3} is verified.
\end{proof}
\begin{lemma}
    Under strong additive quantization \ref{exampleG2}, Assumption \ref{ass2 raw data} and Assumption \ref{assumption noise raw},
    \begin{equation*}
        \begin{aligned}
            \mathbb{E}\left[(y^{(q)}-\langle {\mathbf{w}^{(q)}}^*,\mathbf{x}^{(q)}\rangle)^2{\mathbf{x}^{(q)}}{\mathbf{x}^{(q)}}^\top\right]
            \precsim\left[\sigma_0^2+\epsilon_l+\epsilon_d(1+\frac{\epsilon_d'}{\epsilon_d^2})\left\|{\mathbf{w}}^*\right\|^2+(1+\alpha_0)\left\|{\mathbf{w}}^*\right\|_\mathbf{H}^2\right]\mathbf{H}^{(q)}.
        \end{aligned}
    \end{equation*}
\end{lemma}
\begin{proof}
    Regarding $\mathbb{E}\left[(y-\langle \mathbf{w}^*,\mathbf{x}\rangle)^2{\mathbf{x}^{(q)}}{\mathbf{x}^{(q)}}^\top\right]$, if we further assume that $\mathbb{E}\left[(y-\langle \mathbf{w}^*,\mathbf{x}\rangle)^2\right]\leq \sigma_0^2$, then
    \begin{equation}
    \label{G13}
        \begin{aligned}
            \mathbb{E}\left[(y-\langle \mathbf{w}^*,\mathbf{x}\rangle)^2{\mathbf{x}^{(q)}}{\mathbf{x}^{(q)}}^\top\right] \preceq& 2\mathbb{E}\left[(y-\langle \mathbf{w}^*,\mathbf{x}\rangle)^2{\mathbf{x}}{\mathbf{x}}^\top\right]+2\mathbb{E}\left[(y-\langle \mathbf{w}^*,\mathbf{x}\rangle)^2{\boldsymbol{\epsilon}^{(d)}}{\boldsymbol{\epsilon}^{(d)}}^\top\right]\\
            \preceq & 2\sigma_0^2 \mathbf{H}+2\epsilon_d\sigma_0^2  \mathbf{I}.
        \end{aligned}
    \end{equation}
    Regarding $\mathbb{E}\left[(y^{(q)}-y)^2{\mathbf{x}^{(q)}}{\mathbf{x}^{(q)}}^\top\right]$,
    \begin{equation}
    \label{G14}
        \mathbb{E}\left[(y^{(q)}-y)^2{\mathbf{x}^{(q)}}{\mathbf{x}^{(q)}}^\top\right] \leq \epsilon_d\epsilon_l\mathbf{I}.
    \end{equation}
    Regarding $\mathbb{E}\left[\langle {\mathbf{w}^{(q)}}^*,\boldsymbol{\epsilon}^{(d)}\rangle^2{\mathbf{x}^{(q)}}{\mathbf{x}^{(q)}}^\top\right]$,
    \begin{equation}
    \label{G15}
        \begin{aligned}
            &\mathbb{E}\left[\langle {\mathbf{w}^{(q)}}^*,\boldsymbol{\epsilon}^{(d)}\rangle^2{\mathbf{x}^{(q)}}{\mathbf{x}^{(q)}}^\top\right]\\
            =&\mathbb{E}\left[{\boldsymbol{\epsilon}^{(d)}}^\top{\mathbf{w}^{(q)}}^*{{\mathbf{w}^{(q)}}^*}^\top \boldsymbol{\epsilon}^{(d)}{\mathbf{x}^{(q)}}{\mathbf{x}^{(q)}}^\top\right]\\
            \preceq & 2\mathbb{E}\left[{\boldsymbol{\epsilon}^{(d)}}^\top{\mathbf{w}^{(q)}}^*{{\mathbf{w}^{(q)}}^*}^\top \boldsymbol{\epsilon}^{(d)}{\mathbf{x}}{\mathbf{x}}^\top\right]+2\mathbb{E}\left[{\boldsymbol{\epsilon}^{(d)}}^\top{\mathbf{w}^{(q)}}^*{{\mathbf{w}^{(q)}}^*}^\top \boldsymbol{\epsilon}^{(d)}{\boldsymbol{\epsilon}^{(d)}}{\boldsymbol{\epsilon}^{(d)}}^\top\right]\\
            \preceq & 2\epsilon_d\mathrm{tr}({\mathbf{w}^{(q)}}^*{{\mathbf{w}^{(q)}}^*}^\top)\mathbf{H}+2\epsilon_d' \mathrm{tr}({\mathbf{w}^{(q)}}^*{{\mathbf{w}^{(q)}}^*}^\top)\mathbf{I}.
        \end{aligned}
    \end{equation}
    Regarding $\mathbb{E}\left[\langle {\mathbf{w}^{(q)}}^*- \mathbf{w}^*,\mathbf{x}\rangle^2{\mathbf{x}^{(q)}}{\mathbf{x}^{(q)}}^\top\right]$,
    \begin{equation}
    \label{G16}
        \begin{aligned}
            &\mathbb{E}\left[\langle {\mathbf{w}^{(q)}}^*- \mathbf{w}^*,\mathbf{x}\rangle^2{\mathbf{x}^{(q)}}{\mathbf{x}^{(q)}}^\top\right]\\
            =&\mathbb{E}\left[\mathbf{x}^\top ({\mathbf{w}^{(q)}}^*- \mathbf{w}^*)({\mathbf{w}^{(q)}}^*- \mathbf{w}^*)^\top\mathbf{x}{\mathbf{x}^{(q)}}{\mathbf{x}^{(q)}}^\top\right]\\
            \preceq & 2\mathbb{E}\left[\mathbf{x}^\top ({\mathbf{w}^{(q)}}^*- \mathbf{w}^*)({\mathbf{w}^{(q)}}^*- \mathbf{w}^*)^\top\mathbf{x}{\mathbf{x}}{\mathbf{x}}^\top\right]+2\mathbb{E}\left[\mathbf{x}^\top ({\mathbf{w}^{(q)}}^*- \mathbf{w}^*)({\mathbf{w}^{(q)}}^*- \mathbf{w}^*)^\top\mathbf{x}{\boldsymbol{\epsilon}^{(d)}}{\boldsymbol{\epsilon}^{(d)}}^\top\right]\\
            \preceq& 2\alpha_0\mathrm{tr}\left(({\mathbf{w}^{(q)}}^*- \mathbf{w}^*)({\mathbf{w}^{(q)}}^*- \mathbf{w}^*)^\top\mathbf{H}\right)\mathbf{H}+2\epsilon_d\mathrm{tr}\left(({\mathbf{w}^{(q)}}^*- \mathbf{w}^*)({\mathbf{w}^{(q)}}^*- \mathbf{w}^*)^\top\mathbf{H}\right)\mathbf{I}.
        \end{aligned}
    \end{equation}

    Therefore, together with (\ref{G8}), (\ref{G13}), (\ref{G14}), (\ref{G15}) and (\ref{G16}), we have
    \begin{equation*}
        \begin{aligned}
            &\mathbb{E}\left[(y^{(q)}-\langle {\mathbf{w}^{(q)}}^*,\mathbf{x}^{(q)}\rangle)^2{\mathbf{x}^{(q)}}{\mathbf{x}^{(q)}}^\top\right]\\
            \precsim &(\sigma_0^2+\epsilon_l)\mathbf{H}^{(q)}+\epsilon_d(1+\frac{\epsilon_d'}{\epsilon_d^2})\left\|{\mathbf{w}^{(q)}}^*\right\|^2\mathbf{H}^{(q)}+(1+\alpha_0)\left\|{\mathbf{w}^{(q)}}^*-\mathbf{w}^*\right\|_\mathbf{H}^2\mathbf{H}^{(q)}\\
            \leq& \left[\sigma_0^2+\epsilon_l+\epsilon_d(1+\frac{\epsilon_d'}{\epsilon_d^2})\left\|{\mathbf{w}}^*\right\|^2+(1+\alpha_0)\left\|{\mathbf{w}}^*\right\|_\mathbf{H}^2\right]\mathbf{H}^{(q)}.
        \end{aligned}
    \end{equation*}
    That is, under strong additive quantization Example \ref{exampleG2} and noise Assumption \ref{assumption noise raw} on full-precision data, Assumption \ref{ass3} is verified.
\end{proof}
\section{Extension to Quantized Master Weights}
\label{sec: master weight}
In the quantized SGD algorithm (\ref{eq:SGD_quantized}), the master weight maintains full precision. In this section, we demonstrate that our theoretical framework can naturally extend to the setting where the master weight is also quantized. For simplicity, we only discuss the bounds for $R_N^{(0)}$. The theoretical bounds are presented in Theorem \ref{thm: master general}, Theorem \ref{thm: master additive} and Theorem \ref{thm: master multiplicative} for general quantization, additive quantization and multiplicative quantization, respectively. These results demonstrate that when the master weights are quantized, quantized SGD requires stricter conditions on the step size to ensure convergence. Furthermore, the final excess risk bounds incorporate additional error terms, which degrades generalization performance.

We first present the algorithm and the propagation of $\mathbb{E}\left[\boldsymbol{\eta}_t\boldsymbol{\eta}_t^\top\right]$.
Specifically, we consider
\begin{align}
    \mathbf{w}_t=\mathcal{Q}_p(\mathbf{w}_{t-1})+\gamma \frac{1}{B} \mathcal{Q}_d(\mathbf{X}_t)^\top \mathcal{Q}_o\Big(\mathcal{Q}_l(\mathbf{y}_t)-\mathcal{Q}_a\big(\mathcal{Q}_d(\mathbf{X}_t)\mathcal{Q}_p(\mathbf{w}_{t-1})\big)\Big), \quad t=1,...,N.
\end{align}
When master weight is quantized, Lemma \ref{update rule for eta_t} changes to
\begin{equation}
    \label{eq: F2}
    \boldsymbol{\eta}_t = \left(\mathbf{I}-\frac{1}{B}\gamma {\mathcal{Q}_d(\mathbf{X}_t)}^\top \mathcal{Q}_d(\mathbf{X}_t) \right)\boldsymbol{\eta}_{t-1}+\gamma\frac{1}{B}{\mathcal{Q}_d(\mathbf{X}_t)}^\top \left[\boldsymbol{\xi}_t+\boldsymbol{\epsilon}_t^{(o)}-\boldsymbol{\epsilon}_t^{(a)}-\mathcal{Q}_d(\mathbf{X}_t)\boldsymbol{\epsilon}_{t-1}^{(p)}\right]+\boldsymbol{\epsilon}_{t-1}^{(p)}.
\end{equation}
In particular, the coefficient for parameter quantization error $\boldsymbol{\epsilon}_{t-1}^{(p)}$ changes from $-\frac{1}{B}\gamma {\mathcal{Q}_d(\mathbf{X}_t)}^\top \mathcal{Q}_d(\mathbf{X}_t)$ to $\mathbf{I}-\frac{1}{B}\gamma {\mathcal{Q}_d(\mathbf{X}_t)}^\top \mathcal{Q}_d(\mathbf{X}_t)$. Therefore, we can rewrite Lemma \ref{update rule for eta_t^2} and Lemma \ref{update rule for eta_t^2 multiplicative} as follows.
\begin{lemma}[\rm Update rule under general quantization with quantized master weight]
    \label{update rule for eta_t^2, master}
    Under Assumption \ref{ass1}, Assumption \ref{ass: addd quantized}, Assumption \ref{ass2}, and Assumption \ref{ass3},
    \begin{equation*}
        \mathbb{E}\left[\boldsymbol{\eta}_t\boldsymbol{\eta}_t^\top\right]\preceq 2(\mathbf{B}_t+\mathbf{C}_t),
    \end{equation*}
    where
    \begin{gather*}
            \mathbf{C}_t
            \preceq \mathbb{E}\left[\left(\mathbf{I}-\gamma\frac{1}{B} {\mathcal{Q}_d(\mathbf{X})}^\top \mathcal{Q}_d(\mathbf{X}) \right)\mathbf{C}_{t-1}\left(\mathbf{I}-\gamma\frac{1}{B} {\mathcal{Q}_d(\mathbf{X})}^\top \mathcal{Q}_d(\mathbf{X}) \right)\right]+\gamma^2 {\sigma_G^{(q)}}^2\mathbf{H}^{(q)}+2\mathbb{E}\left[\boldsymbol{\epsilon}_{t-1}^{(p)}{\boldsymbol{\epsilon}_{t-1}^{(p)}}^\top\right],\\
            \mathbf{B}_t
            = \mathbb{E}\left[\left(\mathbf{I}-\gamma\frac{1}{B} {\mathcal{Q}_d(\mathbf{X})}^\top \mathcal{Q}_d(\mathbf{X}) \right)\mathbf{B}_{t-1}\left(\mathbf{I}-\gamma\frac{1}{B} {\mathcal{Q}_d(\mathbf{X})}^\top \mathcal{Q}_d(\mathbf{X}) \right)\right],
    \end{gather*}
    with $\mathbf{C}_0=\boldsymbol{0},\ \mathbf{B}_0=\mathbb{E}\left[\boldsymbol{\eta}_0\boldsymbol{\eta}_0^\top\right]$ and
    \begin{equation*}
        \begin{aligned}
            {\sigma_G^{(q)}}^2=&\frac{\sup_t \left\{\left\|\mathbb{E}\left[\boldsymbol{\epsilon}_t^{(o)}{\boldsymbol{\epsilon}_t^{(o)}}^\top|\mathbf{o}_t\right]+\mathbb{E}\left[\boldsymbol{\epsilon}_t^{(a)}{\boldsymbol{\epsilon}_t^{(a)}}^\top|\mathbf{a}_t\right] \right\|\right\}}{B}\\
            +&2\alpha_B\sup_t\mathbb{E}_{\mathbf{w}_{t-1}}\left[\mathrm{tr}\left(\mathbf{H}^{(q)} \mathbb{E}\left[\boldsymbol{\epsilon}_{t-1}^{(p)}{\boldsymbol{\epsilon}_{t-1}^{(p)}}^\top \big| \mathbf{w}_{t-1}\right]\right)\right]
            +\frac{\sigma^2}{B},
        \end{aligned}
    \end{equation*}
\end{lemma}
\begin{proof}
    Noticing that
    \begin{equation}
        \label{eq: F3}
        \begin{aligned}
            &\mathbb{E}\left[\left(\mathbf{I}-\frac{1}{B}\gamma {\mathcal{Q}_d(\mathbf{X}_t)}^\top \mathcal{Q}_d(\mathbf{X}_t)\right)\boldsymbol{\epsilon}_{t-1}^{(p)}{\boldsymbol{\epsilon}_{t-1}^{(p)}}^\top\left(\mathbf{I}-\frac{1}{B}\gamma {\mathcal{Q}_d(\mathbf{X}_t)}^\top \mathcal{Q}_d(\mathbf{X}_t)\right)\right]\\
            \preceq&2\frac{\gamma^2}{B^2}\mathbb{E}\left[{\mathcal{Q}_d(\mathbf{X}_t)}^\top {\mathcal{Q}_d(\mathbf{X}_t)}\boldsymbol{\epsilon}_{t-1}^{(p)}{\boldsymbol{\epsilon}_{t-1}^{(p)}}^\top {\mathcal{Q}_d(\mathbf{X}_t)}^\top\mathcal{Q}_d(\mathbf{X}_t)\right]+2\mathbb{E}\left[\boldsymbol{\epsilon}_{t-1}^{(p)}{\boldsymbol{\epsilon}_{t-1}^{(p)}}^\top\right],
        \end{aligned}
    \end{equation}
    together with (\ref{eq: F2}) and Lemma \ref{update rule for eta_t^2}, we complete the proof.
\end{proof}
\begin{lemma}[\rm Update rule under multiplicative quantization with quantized master weight]
    \label{update rule for eta_t^2 multiplicative, master}
    If there exist $\epsilon_d,\epsilon_l,\epsilon_p,\epsilon_a$ and $\epsilon_o$ such that for any $i\in \{d,l,p,a,o\}$, quantization $\mathcal{Q}_i$ is $\epsilon_i$-multiplicative, then under Assumption \ref{ass1}, Assumption \ref{ass: addd quantized}, Assumption \ref{ass2}, and Assumption \ref{ass3}, it holds
    \begin{equation*}
        \mathbb{E}\left[\boldsymbol{\eta}_t\boldsymbol{\eta}_t^\top\right]\preceq 2(\mathbf{B}_t+\mathbf{C}_t),
    \end{equation*}
    where
    \begin{equation*}
        \begin{aligned}
            \mathbf{C}_t
            \preceq&\mathbb{E}\left[\left(\mathbf{I}-\frac{1}{B}\gamma {\mathcal{Q}_d(\mathbf{X})}^\top \mathcal{Q}_d(\mathbf{X}) \right)\mathbf{C}_{t-1}\left(\mathbf{I}-\frac{1}{B}\gamma {\mathcal{Q}_d(\mathbf{X})}^\top \mathcal{Q}_d(\mathbf{X}) \right)\right]+8\epsilon_p\left(\mathbf{C}_{t-1}+\mathbf{B}_{t-1}\right)\\
            +&\tilde{\epsilon}\mathbb{E}\left[\frac{\gamma}{B}{\mathcal{Q}_d(\mathbf{X})}^\top {\mathcal{Q}_d(\mathbf{X})}(\mathbf{B}_{t-1}+\mathbf{C}_{t-1})\frac{\gamma}{B}{\mathcal{Q}_d(\mathbf{X})}^\top\mathcal{Q}_d(\mathbf{X})\right]
            +\gamma^2 {\sigma_M^{(q)}}^2\mathbf{H}^{(q)}+4\epsilon_p{\mathbf{w}^{(q)}}^*({\mathbf{w}^{(q)}}^*)^\top,\\
            \mathbf{B}_t=& \mathbb{E}\left[\left(\mathbf{I}-\frac{1}{B}\gamma {\mathcal{Q}_d(\mathbf{X})}^\top \mathcal{Q}_d(\mathbf{X}) \right)\mathbf{B}_{t-1}\left(\mathbf{I}-\frac{1}{B}\gamma {\mathcal{Q}_d(\mathbf{X})}^\top \mathcal{Q}_d(\mathbf{X}) \right)\right],
        \end{aligned}
    \end{equation*}
    with
    \begin{gather*}
            \tilde{\epsilon}=8\epsilon_o(1+\epsilon_p)(1+\epsilon_a)+8\epsilon_p+4\epsilon_a(1+\epsilon_p),\\
            {\sigma_M^{(q)}}^2=\frac{(1+4\epsilon_o)\sigma^2}{B} + \frac{\|\mathbf{w}^*\|_\mathbf{H}^2}{1+\epsilon_d}\alpha_B\left(4\epsilon_o[(1+\epsilon_p)(1+\epsilon_a)+1]+2\epsilon_a(1+\epsilon_p)+4\epsilon_p \right).
    \end{gather*}
\end{lemma}
\begin{proof}
    Under multiplicative quantization,
    \begin{equation}
        \label{eq: ffds}
        \mathbb{E}\left[\boldsymbol{\epsilon}_{t-1}^{(p)}{\boldsymbol{\epsilon}_{t-1}^{(p)}}^\top\right]=\epsilon_p\mathbb{E}\left[\mathbf{w}_{t-1}\mathbf{w}_{t-1}^\top\right]\preceq2\epsilon_p\mathbb{E}\left[\boldsymbol{\eta}_{t-1}\boldsymbol{\eta}_{t-1}^\top\right]+2\epsilon_p{\mathbf{w}^{(q)}}^*({\mathbf{w}^{(q)}}^*)^\top.
    \end{equation}
    By (\ref{eq: F2}), (\ref{eq: F3}), (\ref{eq: ffds}) and Lemma \ref{update rule for eta_t^2 multiplicative},
    the proof is immediately completed.
\end{proof}

\subsection{General Quantization}
In this section, we derive upper bounds for $R_N^{(0)}$ under general quantization.
We first perform bias-variance decomposition under general quantization. By Lemma \ref{update rule for eta_t^2, master}, we have 
\begin{equation*}
    \mathbb{E}\left[\boldsymbol{\eta}_t\boldsymbol{\eta}_t^\top\right]\preceq2\left(\mathbf{B}_t+\mathbf{C}_t\right),
\end{equation*}
where
\begin{equation*}
    \mathbf{B}_t:=(\mathcal{I}-\gamma\mathcal{T}_B^{(q)})^t \circ \mathbf{B}_0, \quad \mathbf{B}_0=\mathbb{E}\left[\boldsymbol{\eta}_0 \boldsymbol{\eta}_0^\top\right],
\end{equation*}
and
\begin{equation*}
    \mathbf{C}_t := (\mathcal{I}-\gamma \mathcal{T}_B^{(q)})\circ\mathbf{C}_{t-1}+\gamma^2 {\sigma_G^{(q)}}^2\mathbf{H}^{(q)}+2\mathbb{E}\left[\boldsymbol{\epsilon}_{t-1}^{(p)}{\boldsymbol{\epsilon}_{t-1}^{(p)}}^\top\right],\quad \mathbf{C}_0=\boldsymbol{0}.
\end{equation*}

Then by Lemma \ref{Initial study of R_2},
\begin{equation}
    \label{eq: f}
    R_N^{(0)}/2\leq \underbrace{\frac{1}{N^2}\cdot\sum_{t=0}^{N-1}\sum_{k=t}^{N-1}\left\langle(\mathbf{I}-\gamma\mathbf{H}^{(q)})^{k-t}\mathbf{H}^{(q)},\mathbf{B}_t\right\rangle}_{\mathrm{bias}}+\underbrace{\frac{1}{N^2}\cdot\sum_{t=0}^{N-1}\sum_{k=t}^{N-1}\left\langle(\mathbf{I}-\gamma\mathbf{H}^{(q)})^{k-t}\mathbf{H}^{(q)},\mathbf{C}_t\right\rangle}_{\mathrm{variance}}.
\end{equation}
Noticing that the bias error when master weight is quantized is the same as the bias error when master weight maintains full precision, we then only need to derive bounds for variance error. 
Similar to (\ref{initial Ct}),
\begin{equation}
    \label{eq: ff}
    \mathbf{C}_t\preceq (\mathcal{I}-\gamma\widetilde{\mathcal{T}}^{(q)})\circ \mathbf{C}_{t-1}+\gamma^2\mathcal{M}_B^{(q)}\circ \mathbf{C}_{t-1}+\gamma^2{\sigma_G^{(q)}}^2\mathbf{H}^{(q)}+2\mathbb{E}\left[\boldsymbol{\epsilon}_{t-1}^{(p)}{\boldsymbol{\epsilon}_{t-1}^{(p)}}^\top\right].
\end{equation}

\begin{lemma}[\rm A bound for $\mathcal{M}_B^{(q)} \circ \mathbf{C}_t$ with quantized master weight]
    \label{A bound for M_B^{(q)} C_t, master}
    For $t \geq 1$, under Assumption \ref{ass1}, Assumption \ref{ass: addd quantized}, Assumption \ref{ass2}, and Assumption \ref{ass3}, if the stepsize $\gamma \leq \frac{1}{\alpha_B\mathrm{tr}(\mathbf{H}^{(q)})}$, then
    \begin{equation*}
        \mathcal{M}_B^{(q)} \circ \mathbf{C}_t \preceq \frac{\alpha_B\mathrm{tr}(\mathbf{H}^{(q)})\left(\gamma{\sigma_G^{(q)}}^2+\frac{2d\mu}{\gamma\mathrm{tr}(\mathbf{H}^{(q)})}\right)}{1-\gamma\alpha_B\mathrm{tr}(\mathbf{H}^{(q)})} \mathbf{H}^{(q)}.
    \end{equation*}
\end{lemma}
\begin{proof}
We first derive a crude bound for $\mathbf{C}_t$. Denote $\boldsymbol{\Sigma}={\sigma_G^{(q)}}^2\mathbf{H}^{(q)}+\frac{2}{\gamma^2}\sup_t\mathbb{E}\left[\boldsymbol{\epsilon}_{t-1}^{(p)}{\boldsymbol{\epsilon}_{t-1}^{(p)}}^\top\right]$. By (\ref{solution of Cinfty}),
\begin{equation*}
    \mathbf{C}_t \preceq \mathbf{C}_\infty=\gamma{\mathcal{T}_B^{(q)}}^{-1}\circ\mathbf{\Sigma}.
\end{equation*}
Applying $\widetilde{\mathcal{T}}^{(q)}$, we have
\begin{equation*}
        \begin{aligned}
            \widetilde{\mathcal{T}}^{(q)}\circ\mathbf{C}_{\infty} & =\mathcal{T}_B^{(q)}\circ\mathbf{C}_{\infty}+\gamma\mathcal{M}_B^{(q)}\circ\mathbf{C}_{\infty}-\gamma\widetilde{\mathcal{M}}^{(q)}\circ\mathbf{C}_{\infty} \\
            & =\gamma\boldsymbol{\Sigma}+\gamma\mathcal{M}_B^{(q)}\circ\mathbf{C}_\infty-\gamma\widetilde{\mathcal{M}}^{(q)}\circ\mathbf{C}_\infty\\
            & \preceq\gamma\boldsymbol{\Sigma}+\gamma\mathcal{M}_B^{(q)}\circ\mathbf{C}_\infty\\
            &=\gamma{\sigma_G^{(q)}}^2\mathbf{H}^{(q)}+\frac{2}{\gamma}\sup_t\mathbb{E}\left[\boldsymbol{\epsilon}_{t-1}^{(p)}{\boldsymbol{\epsilon}_{t-1}^{(p)}}^\top\right]+\gamma\mathcal{M}_B^{(q)}\circ\mathbf{C}_\infty.
        \end{aligned}
\end{equation*}
Applying $(\widetilde{\mathcal{T}}^{(q)})^{-1}$ and by solving recursion, we have
\begin{equation}
    \label{eq: f4}
    \begin{aligned}
        \mathbf{C}_\infty& \preceq \gamma{\sigma_G^{(q)}}^2(\widetilde{\mathcal{T}}^{(q)})^{-1}\circ\mathbf{H}^{(q)}+\frac{2}{\gamma}(\widetilde{\mathcal{T}}^{(q)})^{-1}\circ\sup_t\mathbb{E}\left[\boldsymbol{\epsilon}_{t-1}^{(p)}{\boldsymbol{\epsilon}_{t-1}^{(p)}}^\top\right]+\gamma(\widetilde{\mathcal{T}}^{(q)})^{-1}\circ\mathcal{M}_B^{(q)}\circ\mathbf{C}_\infty\\
        & \preceq\sum_{t=0}^\infty\left(\gamma(\widetilde{\mathcal{T}}^{(q)})^{-1}\circ\mathcal{M}_B^{(q)}\right)^t\circ\left(\gamma{\sigma_G^{(q)}}^2(\widetilde{\mathcal{T}}^{(q)})^{-1}\circ\mathbf{H}^{(q)}+\frac{2}{\gamma}(\widetilde{\mathcal{T}}^{(q)})^{-1}\circ\sup_t\mathbb{E}\left[\boldsymbol{\epsilon}_{t-1}^{(p)}{\boldsymbol{\epsilon}_{t-1}^{(p)}}^\top\right]\right).
    \end{aligned}
\end{equation}
By (\ref{eq: c26}),
\begin{equation}
    \label{eq: f5}
    \sum_{t=0}^\infty\left(\gamma(\widetilde{\mathcal{T}}^{(q)})^{-1}\circ\mathcal{M}_B^{(q)}\right)^t\circ\gamma{\sigma_G^{(q)}}^2(\widetilde{\mathcal{T}}^{(q)})^{-1}\circ\mathbf{H}^{(q)}\preceq \frac{\gamma{\sigma_G^{(q)}}^2}{1-\gamma\alpha_B\mathrm{tr}(\mathbf{H}^{(q)})}\mathbf{I}.
\end{equation}
Noticing that
\begin{equation*}
    \begin{aligned}
        (\widetilde{\mathcal{T}}^{(q)})^{-1}\circ\sup_t\mathbb{E}\left[\boldsymbol{\epsilon}_{t-1}^{(p)}{\boldsymbol{\epsilon}_{t-1}^{(p)}}^\top\right]=&\gamma\sum_{t=0}^\infty(\mathcal{I}-\gamma\widetilde{\mathcal{T}}^{(q)})^t\circ\sup_t\mathbb{E}\left[\boldsymbol{\epsilon}_{t-1}^{(p)}{\boldsymbol{\epsilon}_{t-1}^{(p)}}^\top\right] \\
        =&\gamma\sum_{t=0}^\infty(\mathbf{I}-\gamma\mathbf{H}^{(q)})^t\sup_t\mathbb{E}\left[\boldsymbol{\epsilon}_{t-1}^{(p)}{\boldsymbol{\epsilon}_{t-1}^{(p)}}^\top\right](\mathbf{I}- \gamma\mathbf{H}^{(q)})^t \\
        \preceq& \sup_t\left\|\mathbb{E}\left[\boldsymbol{\epsilon}_{t-1}^{(p)}{\boldsymbol{\epsilon}_{t-1}^{(p)}}^\top\right]\right\| (\mathbf{H}^{(q)})^{-1}:=\mu\cdot(\mathbf{H}^{(q)})^{-1},
    \end{aligned}
\end{equation*}
we have
\begin{equation}
    \label{eq: f6}
    \begin{aligned}
        &\sum_{t=0}^\infty\left(\gamma(\widetilde{\mathcal{T}}^{(q)})^{-1}\circ\mathcal{M}_B^{(q)}\right)^t\circ\frac{2}{\gamma}(\widetilde{\mathcal{T}}^{(q)})^{-1}\circ\mathbb{E}\left[\boldsymbol{\epsilon}_{t-1}^{(p)}{\boldsymbol{\epsilon}_{t-1}^{(p)}}^\top\right]\\
        =&\sum_{t=0}^\infty\left(\gamma(\widetilde{\mathcal{T}}^{(q)})^{-1}\circ\mathcal{M}_B^{(q)}\right)^t\circ\frac{2\mu}{\gamma}(\mathbf{H}^{(q)})^{-1}\\
        =&\sum_{t=0}^\infty\left(\gamma(\widetilde{\mathcal{T}}^{(q)})^{-1}\circ\mathcal{M}_B^{(q)}\right)^{t-1}\gamma(\widetilde{\mathcal{T}}^{(q)})^{-1}\circ\mathcal{M}_B^{(q)}\circ\frac{2\mu}{\gamma}(\mathbf{H}^{(q)})^{-1}\\
        \preceq&\sum_{t=0}^\infty\left(\gamma(\widetilde{\mathcal{T}}^{(q)})^{-1}\circ\mathcal{M}_B^{(q)}\right)^{t-1}\gamma(\widetilde{\mathcal{T}}^{(q)})^{-1}\circ\alpha_B\mathrm{tr}(\mathbf{I})\frac{2\mu}{\gamma}\mathbf{H}^{(q)}\\
        \preceq&\sum_{t=0}^\infty\left(\gamma(\widetilde{\mathcal{T}}^{(q)})^{-1}\circ\mathcal{M}_B^{(q)}\right)^{t-1}2\mu\alpha_Bd\cdot\mathbf{I}\\
        \preceq&2d\mu \alpha_B \cdot \sum_{t=0}^\infty \left(\gamma \alpha_B\mathrm{tr}(\mathbf{H}^{(q)})\right)^{t-1}\mathbf{I}\\
        =&\frac{2d\mu}{\gamma\mathrm{tr}(\mathbf{H}^{(q)})} \cdot \sum_{t=0}^\infty \left(\gamma \alpha_B\mathrm{tr}(\mathbf{H}^{(q)})\right)^{t}\mathbf{I}\\
        =&\frac{2d\mu}{\gamma\mathrm{tr}(\mathbf{H}^{(q)})} \frac{1}{1-\gamma\alpha_B\mathrm{tr}\left(\mathbf{H}^{(q)}\right)}\mathbf{I}.
    \end{aligned}
\end{equation}

Therefore, together with (\ref{eq: f4}), (\ref{eq: f5}) and (\ref{eq: f6}), we have
\begin{equation*}
    \mathbf{C}_t\preceq\mathbf{C}_\infty\preceq \frac{\gamma{\sigma_G^{(q)}}^2+\frac{2d\mu}{\gamma\mathrm{tr}(\mathbf{H}^{(q)})}}{1-\gamma\alpha_B\mathrm{tr}(\mathbf{H}^{(q)})}\mathbf{I}.
\end{equation*}
It follows that
\begin{equation*}
    \mathcal{M}_B^{(q)} \circ \mathbf{C}_t \preceq \frac{\alpha_B\mathrm{tr}(\mathbf{H}^{(q)})\left(\gamma{\sigma_G^{(q)}}^2+\frac{2d\mu}{\gamma\mathrm{tr}(\mathbf{H}^{(q)})}\right)}{1-\gamma\alpha_B\mathrm{tr}(\mathbf{H}^{(q)})} \mathbf{H}^{(q)}.
\end{equation*}
\end{proof}

Together with (\ref{eq: ff}), we can provide a refined bound for $\mathbf{C}_t$ and we are now ready to bound the variance error.

\begin{lemma}[\rm A bound for variance under general quantization with quantized master weight]
    \label{variance R2 bound, master}
    Under Assumption \ref{ass1}, Assumption \ref{ass: addd quantized}, Assumption \ref{ass2}, and Assumption \ref{ass3}, if the stepsize satisfies $\gamma < \frac{1}{\alpha_B\mathrm{tr}(\mathbf{H}^{(q)})}$, then
    \begin{equation*}
        \mathrm{variance}\leq \frac{{\sigma_G^{(q)}}^2+\frac{2\alpha_Bd\mu}{\gamma}}{1-\gamma\alpha_B\mathrm{tr}(\mathbf{H}^{(q)})}\left(\frac{k^*}{N}+N\gamma^2\cdot\sum_{i>k^*}(\lambda_i^{(q)})^2\right)+\frac{2\mu}{\gamma}\left(\sum_{i\leq k^*}\frac{1}{N\gamma\lambda_i^{(q)}}+N\gamma\sum_{i>k^*}\lambda_i^{(q)}\right),
    \end{equation*}
    where $\mu=\sup_t\left\|\mathbb{E}\left[\boldsymbol{\epsilon}_{t-1}^{(p)}{\boldsymbol{\epsilon}_{t-1}^{(p)}}^\top\right]\right\|$.
\end{lemma}
\begin{proof}
    We first provide a refined upper bound for $\mathbf{C}_t$. By (\ref{eq: ff}) and Lemma \ref{A bound for M_B^{(q)} C_t, master},
    \begin{equation*}
        \begin{aligned}
            \mathbf{C}_t \preceq &(\mathcal{I}-\gamma\widetilde{\mathcal{T}}^{(q)})\circ \mathbf{C}_{t-1}+\gamma^2\mathcal{M}_B^{(q)}\circ \mathbf{C}_{t-1}+\gamma^2{\sigma_G^{(q)}}^2\mathbf{H}^{(q)}+2\mathbb{E}\left[\boldsymbol{\epsilon}_{t-1}^{(p)}{\boldsymbol{\epsilon}_{t-1}^{(p)}}^\top\right]\\
            \preceq & (\mathcal{I}-\gamma\widetilde{\mathcal{T}}^{(q)})\circ \mathbf{C}_{t-1} + \frac{\gamma^2\alpha_B\mathrm{tr}(\mathbf{H}^{(q)})\left(\gamma{\sigma_G^{(q)}}^2+\frac{2d\mu}{\gamma\mathrm{tr}(\mathbf{H}^{(q)})}\right)}{1-\gamma\alpha_B\mathrm{tr}(\mathbf{H}^{(q)})} \mathbf{H}^{(q)}+\gamma^2{\sigma_G^{(q)}}^2\mathbf{H}^{(q)}+2\mu\mathbf{I}\\
            =& (\mathcal{I}-\gamma\widetilde{\mathcal{T}}^{(q)})\circ \mathbf{C}_{t-1} + \frac{\gamma^2{\sigma_G^{(q)}}^2+2\gamma\alpha_Bd\mu}{1-\gamma\alpha_B\mathrm{tr}(\mathbf{H}^{(q)})} \mathbf{H}^{(q)}+2\mu\mathbf{I}\\
            =&\frac{\gamma^2{\sigma_G^{(q)}}^2+2\gamma\alpha_Bd\mu}{1-\gamma\alpha_B\mathrm{tr}(\mathbf{H}^{(q)})}\cdot\sum_{k=0}^{t-1}(\mathcal{I}-\gamma\widetilde{\mathcal{T}}^{(q)})^k\circ\mathbf{H}^{(q)}+2\mu\sum_{k=0}^{t-1}(\mathcal{I}-\gamma\widetilde{\mathcal{T}}^{(q)})^k\circ\mathbf{I}\\
            =&\frac{\gamma^2{\sigma_G^{(q)}}^2+2\gamma\alpha_Bd\mu}{1-\gamma\alpha_B\mathrm{tr}(\mathbf{H}^{(q)})}\cdot\sum_{k=0}^{t-1}(\mathbf{I}-\gamma \mathbf{H}^{(q)})^k\mathbf{H}^{(q)}(\mathbf{I}-\gamma \mathbf{H}^{(q)})^k+2\mu\sum_{k=0}^{t-1}(\mathbf{I}-\gamma \mathbf{H}^{(q)})^{2k} \\
            \preceq&\frac{\gamma^2{\sigma_G^{(q)}}^2+2\gamma\alpha_Bd\mu}{1-\gamma\alpha_B\mathrm{tr}(\mathbf{H}^{(q)})}\cdot\sum_{k=0}^{t-1}(\mathbf{I}-\gamma \mathbf{H}^{(q)})^k\mathbf{H}^{(q)}+2\mu\sum_{k=0}^{t-1}(\mathbf{I}-\gamma \mathbf{H}^{(q)})^{k} \\
            =&\underbrace{\frac{\gamma{\sigma_G^{(q)}}^2+2\alpha_Bd\mu}{1-\gamma\alpha_B\mathrm{tr}(\mathbf{H}^{(q)})}\cdot\left(\mathbf{I}-(\mathbf{I}-\gamma\mathbf{H}^{(q)})^t\right)}_{\mathbf{V}_1}+\underbrace{\frac{2\mu}{\gamma}\left(\mathbf{I}-(\mathbf{I}-\gamma\mathbf{H}^{(q)})^t\right)(\mathbf{H}^{(q)})^{-1}}_{\mathbf{V}_2}.
        \end{aligned}
    \end{equation*}
    By (\ref{eq: c31}),
    \begin{equation}
        \label{eq: f8}
        \begin{aligned}
            \frac{1}{N^2}\cdot\sum_{t=0}^{N-1}\sum_{k=t}^{N-1}\left\langle(\mathbf{I}-\gamma\mathbf{H}^{(q)})^{k-t}\mathbf{H}^{(q)},\mathbf{V}_1\right\rangle\leq\frac{{\sigma_G^{(q)}}^2+\frac{2\alpha_Bd\mu}{\gamma}}{1-\gamma\alpha_B\mathrm{tr}(\mathbf{H}^{(q)})}\left(\frac{k^*}{N}+N\gamma^2\cdot\sum_{i>k^*}(\lambda_i^{(q)})^2\right).
        \end{aligned}
    \end{equation}
    Regarding $\mathbf{V}_2$,
    \begin{equation}
        \label{eq: f9}
        \begin{aligned}
            &\frac{1}{N^2}\cdot\sum_{t=0}^{N-1}\sum_{k=t}^{N-1}\left\langle(\mathbf{I}-\gamma\mathbf{H}^{(q)})^{k-t}\mathbf{H}^{(q)},\mathbf{V}_2\right\rangle\\
            =&\frac{2\mu}{\gamma^2 N^2}\cdot\sum_{t=0}^{N-1}\left\langle\mathbf{I}-(\mathbf{I}-\gamma\mathbf{H}^{(q)})^{N-t},\left(\mathbf{I}-(\mathbf{I}-\gamma\mathbf{H}^{(q)})^t\right)(\mathbf{H}^{(q)})^{-1}\right\rangle\\
            =&\frac{2\mu}{\gamma^2 N^2}\sum_i\sum_{t=0}^{N-1}\left[1-(1-\gamma\lambda_i^{(q)})^{N-t}\right]\left[1-(1-\gamma\lambda_i^{(q)})^t\right]\frac{1}{\lambda_i^{(q)}}\\
            \leq &\frac{2\mu}{\gamma^2 N}\sum_i \min\left\{\frac{1}{\lambda_i^{(q)}},\gamma^2 N^2\lambda_i^{(q)}\right\}\\
            \leq&\frac{2\mu}{\gamma}\left(\sum_{i\leq k^*}\frac{1}{N\gamma\lambda_i^{(q)}}+N\gamma\sum_{i>k^*}\lambda_i^{(q)}\right).
        \end{aligned}
    \end{equation}
    Together with (\ref{eq: f}), (\ref{eq: f8}) and (\ref{eq: f9}), we have
    \begin{equation*}
        \mathrm{variance}\leq \frac{{\sigma_G^{(q)}}^2+\frac{2\alpha_Bd\mu}{\gamma}}{1-\gamma\alpha_B\mathrm{tr}(\mathbf{H}^{(q)})}\left(\frac{k^*}{N}+N\gamma^2\cdot\sum_{i>k^*}(\lambda_i^{(q)})^2\right)+2\frac{\mu}{\gamma}\left(\sum_{i\leq k^*}\frac{1}{N\gamma\lambda_i^{(q)}}+N\gamma\sum_{i>k^*}\lambda_i^{(q)}\right),
    \end{equation*}
    where $\mu=\sup_t\left\|\mathbb{E}\left[\boldsymbol{\epsilon}_{t-1}^{(p)}{\boldsymbol{\epsilon}_{t-1}^{(p)}}^\top\right]\right\|$.
\end{proof}

\begin{theorem}[\rm A bound for $R_N^{(0)}$ under general quantization with quantized master weight]
    \label{thm: master general}
    Under Assumption \ref{ass1}, Assumption \ref{ass: addd quantized}, Assumption \ref{ass2}, and Assumption \ref{ass3}, if the stepsize satisfies $\gamma < \frac{1}{\alpha_B\mathrm{tr}(\mathbf{H}^{(q)})}$, then
    \begin{equation*}
        \begin{aligned}
            R_N^{(0)}/2 \leq &\frac{2\alpha_B\left(\|\mathbf{w}_0-{\mathbf{w}^{(q)}}^*\|_{\mathbf{I}_{0:k^*}^{(q)}}^2+N\gamma\|\mathbf{w}_0-{\mathbf{w}^{(q)}}^*\|_{\mathbf{H}_{k^*:\infty}^{(q)}}^2\right)}{N\gamma(1-\gamma\alpha_B\operatorname{tr}(\mathbf{H}^{(q)}))}\cdot\left(\frac{k^*}{N}+N\gamma^2\sum_{i>k^*}(\lambda_i^{(q)})^2\right)\\
            +& \frac{1}{\gamma^2N^2}\cdot\|\mathbf{w}_0-{\mathbf{w}^{(q)}}^*\|_{(\mathbf{H}_{0:k^*}^{(q)})^{-1}}^2+\|\mathbf{w}_0-{\mathbf{w}^{(q)}}^*\|_{\mathbf{H}_{k^*:\infty}^{(q)}}^2\\
            +& \frac{{\sigma_G^{(q)}}^2+\frac{2\alpha_Bd\mu}{\gamma}}{1-\gamma\alpha_B\mathrm{tr}(\mathbf{H}^{(q)})}\left(\frac{k^*}{N}+N\gamma^2\cdot\sum_{i>k^*}(\lambda_i^{(q)})^2\right)+\frac{2\mu}{\gamma}\left(\sum_{i\leq k^*}\frac{1}{N\gamma\lambda_i^{(q)}}+N\gamma\sum_{i>k^*}\lambda_i^{(q)}\right),
        \end{aligned}
    \end{equation*}
    where $k^*=\max\left\{k: \lambda_k^{(q)} \geq \frac{1}{N\gamma}\right\}$, $\mu=\sup_t\left\|\mathbb{E}\left[\boldsymbol{\epsilon}_{t-1}^{(p)}{\boldsymbol{\epsilon}_{t-1}^{(p)}}^\top\right]\right\|$ and
    \begin{gather*}
        {\sigma_G^{(q)}}^2=\frac{\sup_t \left\{\left\|\mathbb{E}\left[\boldsymbol{\epsilon}_t^{(o)}{\boldsymbol{\epsilon}_t^{(o)}}^\top|\mathbf{o}_t\right]+\mathbb{E}\left[\boldsymbol{\epsilon}_t^{(a)}{\boldsymbol{\epsilon}_t^{(a)}}^\top|\mathbf{a}_t\right] \right\|\right\}}{B}\\
        +2\alpha_B\sup_t\mathbb{E}_{\mathbf{w}_{t-1}}\left[\mathrm{tr}\left(\mathbf{H}^{(q)} \mathbb{E}\left[\boldsymbol{\epsilon}_{t-1}^{(p)}{\boldsymbol{\epsilon}_{t-1}^{(p)}}^\top \big| \mathbf{w}_{t-1}\right]\right)\right]
        +\frac{\sigma^2}{B}.
    \end{gather*}
\end{theorem}
\begin{proof}
    The proof is completed by (\ref{eq: f}), Lemma \ref{bias R2 bound} and Lemma \ref{variance R2 bound, master}.
\end{proof}
Next, we deduce an upper bound for $R_N^{(0)}$ under additive quantization from Theorem \ref{thm: master general}.
\subsection{Additive Quantization}
\begin{theorem}
    [\rm A bound for $R_N^{(0)}$ under additive quantization with quantized master weight]
    \label{thm: master additive}
    Under Assumption \ref{ass1}, Assumption \ref{ass: addd quantized}, Assumption \ref{ass2}, and Assumption \ref{ass3}, if there exist $\epsilon_d,\epsilon_l,\epsilon_p,\epsilon_a$ and $\epsilon_o$ such that for any $i\in \{d,l,p,a,o\}$, quantization $\mathcal{Q}_i$ is $\epsilon_i$-additive, if the stepsize satisfies $\gamma < \frac{1}{\alpha_B[\mathrm{tr}(\mathbf{H})+d\epsilon_d]}$, then
    \begin{equation*}
        \begin{aligned}
            R_N^{(0)}/2 \leq &\frac{2\alpha_B\left(\|\mathbf{w}_0-{\mathbf{w}^{(q)}}^*\|_{\mathbf{I}_{0:k^*}^{(q)}}^2+N\gamma\|\mathbf{w}_0-{\mathbf{w}^{(q)}}^*\|_{\mathbf{H}_{k^*:\infty}^{(q)}}^2\right)}{N\gamma(1-\gamma\alpha_B[\operatorname{tr}(\mathbf{H})+d\epsilon_d])}\cdot\left(\frac{k^*}{N}+N\gamma^2\sum_{i>k^*}(\lambda_i+\epsilon_d)^2\right)\\
            +& \frac{1}{\gamma^2N^2}\cdot\|\mathbf{w}_0-{\mathbf{w}^{(q)}}^*\|_{(\mathbf{H}_{0:k^*}^{(q)})^{-1}}^2+\|\mathbf{w}_0-{\mathbf{w}^{(q)}}^*\|_{\mathbf{H}_{k^*:\infty}^{(q)}}^2\\
            +& \frac{\frac{\sigma^2+\epsilon_o+\epsilon_a}{B}+2\epsilon_p\alpha_B\left(\frac{d}{\gamma}+\mathrm{tr}(\mathbf{H})+d\epsilon_d\right)}{1-\gamma\alpha_B[\mathrm{tr}(\mathbf{H})+d\epsilon_d]}\left(\frac{k^*}{N}+N\gamma^2\cdot\sum_{i>k^*}(\lambda_i+\epsilon_d)^2\right)\\
            +&\frac{2\epsilon_p}{\gamma}\left(\sum_{i\leq k^*}\frac{1}{N\gamma(\lambda_i+\epsilon
            _d)}+N\gamma\sum_{i>k^*}(\lambda_i+\epsilon
            _d)\right).
        \end{aligned}
    \end{equation*}
\end{theorem}
\begin{proof}
    The proof is completed by Theorem \ref{thm: master general}.
\end{proof}
\subsection{Multiplicative Quantization}
In this section, we derive upper bounds for $R_N^{(0)}$ under multiplicative quantization, i.e., there exist $\epsilon_d,\epsilon_l,\epsilon_p,\epsilon_a$ and $\epsilon_o$ such that for any $i\in \{d,l,p,a,o\}$, quantization $\mathcal{Q}_i$ is $\epsilon_i$-multiplicative.
We first perform bias-variance decomposition under multiplicative quantization. 
Denote 
\begin{equation*}
    \mathbf{B}_t^{(M)}:=((1+8\epsilon_p)\mathcal{I}-\gamma\mathcal{T}_B^{(q)}+\tilde{\epsilon}\gamma^2\mathcal{M}_B^{(q)})^t \circ \mathbf{B}_0^{(M)}, \quad \mathbf{B}_0^{(M)}=\mathbb{E}\left[\boldsymbol{\eta}_0 \boldsymbol{\eta}_0^\top\right].
\end{equation*}
\begin{equation*}
    \mathbf{C}_t^{(M)} := ((1+8\epsilon_p)\mathcal{I}-\gamma \mathcal{T}_B^{(q)}+\tilde{\epsilon}\gamma^2\mathcal{M}_B^{(q)})\circ\mathbf{C}_{t-1}^{(M)}+\gamma^2 {\sigma_M^{(q)}}^2\mathbf{H}^{(q)}+4\epsilon_p{\mathbf{w}^{(q)}}^*({\mathbf{w}^{(q)}}^*)^\top,\ \mathbf{C}_0^{(M)}=\boldsymbol{0}.
\end{equation*}
By Lemma \ref{update rule for eta_t^2 multiplicative, master}, we have 
\begin{equation*}
    \mathbb{E}\left[\boldsymbol{\eta}_t\boldsymbol{\eta}_t^\top\right]\preceq2\left(\mathbf{B}_t^{(M)}+\mathbf{C}_t^{(M)}\right).
\end{equation*}
Then by Lemma \ref{Initial study of R_2},
\begin{equation}
    \label{eq: f mul}
    R_N^{(0)}/2\leq \underbrace{\frac{1}{N^2}\cdot\sum_{t=0}^{N-1}\sum_{k=t}^{N-1}\left\langle(\mathbf{I}-\gamma\mathbf{H}^{(q)})^{k-t}\mathbf{H}^{(q)},\mathbf{B}_t^{(M)}\right\rangle}_{\mathrm{bias}}+\underbrace{\frac{1}{N^2}\cdot\sum_{t=0}^{N-1}\sum_{k=t}^{N-1}\left\langle(\mathbf{I}-\gamma\mathbf{H}^{(q)})^{k-t}\mathbf{H}^{(q)},\mathbf{C}_t^{(M)}\right\rangle}_{\mathrm{variance}}.
\end{equation}
\subsubsection{Analysis of Variance Error}
Similar to (\ref{initial CtM}),
\begin{equation*}
    \begin{aligned}
        \mathbf{C}_t^{(M)}=&((1+8\epsilon_p)\mathcal{I}-\gamma \mathcal{T}_B^{(q)}+\tilde{\epsilon}\gamma^2\mathcal{M}_B^{(q)})\mathbf{C}_{t-1}^{(M)}+\gamma^2 {\sigma_M^{(q)}}^2\mathbf{H}^{(q)}+4\epsilon_p{\mathbf{w}^{(q)}}^*({\mathbf{w}^{(q)}}^*)^\top\\
        \preceq&((1+8\epsilon_p)\mathcal{I}-\gamma\widetilde{\mathcal{T}}^{(q)})\circ \mathbf{C}_{t-1}^{(M)}+\gamma^2(1+\tilde{\epsilon})\mathcal{M}_B^{(q)}\circ \mathbf{C}_{t-1}^{(M)}\\
        +&\gamma^2{\sigma_M^{(q)}}^2\mathbf{H}^{(q)}+4\epsilon_p{\mathbf{w}^{(q)}}^*({\mathbf{w}^{(q)}}^*)^\top.
    \end{aligned}
\end{equation*}
Recall that $\widetilde{\mathcal{T}}^{(q)}=\mathbf{H}^{(q)}\otimes\mathbf{I}+\mathbf{I}\otimes\mathbf{H}^{(q)}-\gamma\mathbf{H}^{(q)}\otimes\mathbf{H}^{(q)}$. 
Denote $\widetilde{\mathcal{T}}_2^{(q)}=\mathbf{H}^{(q)}\otimes\mathbf{I}+\mathbf{I}\otimes\mathbf{H}^{(q)}-\gamma/2\mathbf{H}^{(q)}\otimes\mathbf{H}^{(q)}$.
Before proceeding, we find conditions for step size such that
$8\epsilon_p\mathcal{I}\preceq\gamma \widetilde{\mathcal{T}}^{(q)}-\frac{\gamma}{2}\widetilde{\mathcal{T}}_2^{(q)}$. It suffices to restrict:
\begin{equation}
    \label{eq: stepsize}
    \gamma < \frac{2}{3\lambda_1^{(q)}},\quad 8\epsilon_p\leq \gamma \lambda_d^{(q)}-\frac{3}{4}\gamma^2{\lambda_d^{(q)}}^2.
\end{equation}
Equipped with (\ref{eq: stepsize}),
\begin{equation*}
    \mathbf{C}_t^{(M)}\preceq(\mathcal{I}-\frac{\gamma}{2}\widetilde{\mathcal{T}}_2^{(q)})\circ \mathbf{C}_{t-1}^{(M)}+\gamma^2(1+\tilde{\epsilon})\mathcal{M}_B^{(q)}\circ \mathbf{C}_{t-1}^{(M)}
    +\gamma^2{\sigma_M^{(q)}}^2\mathbf{H}^{(q)}+4\epsilon_p{\mathbf{w}^{(q)}}^*({\mathbf{w}^{(q)}}^*)^\top.
\end{equation*}
We would like to remark that, in the analysis of variance, to simplify ${\mathbf{w}^{(q)}}^*({\mathbf{w}^{(q)}}^*)^\top$, we assume the parameter prior
\begin{equation*}
    \mathbb{E}\left[\mathbf{w}^*{\mathbf{w}^*}^\top\right]=\mathbf{I},
\end{equation*}
and take expectation over $\mathbf{w}^*$ on variance error. It follows that \footnote{Actually, (\ref{initial CtM, master}) holds under the expectation of $\mathbf{w}^*$.
Slightly abusing notations, we omit $\mathbb{E}_{\mathbf{w}^*}$.}
\begin{equation}
    \label{initial CtM, master}
    \mathbf{C}_t^{(M)}\preceq(\mathcal{I}-\frac{\gamma}{2}\widetilde{\mathcal{T}}_2^{(q)})\circ \mathbf{C}_{t-1}^{(M)}+\gamma^2(1+\tilde{\epsilon})\mathcal{M}_B^{(q)}\circ \mathbf{C}_{t-1}^{(M)}
    +\gamma^2{\sigma_M^{(q)}}^2\mathbf{H}^{(q)}+4\epsilon_p\mathbf{I}.
\end{equation}

In subsequent analysis, we first derive a crude bound for $\mathbf{C}_t^{(M)}$, then establish a refined bound for $\mathbf{C}_t^{(M)}$ via (\ref{initial CtM, master}).
\begin{lemma}[\rm A bound for $\mathcal{M}_B^{(q)} \circ \mathbf{C}_t^{(M)}$ with quantized master weight]
    \label{A bound for M_B^{(q)} C_t^{(M)}, master}
    For $t \geq 1$, under Assumption \ref{ass1}, Assumption \ref{ass: addd quantized}, Assumption \ref{ass2}, and Assumption \ref{ass3}, if the step size satisfies
    \begin{equation*}
        1>2(1+\tilde{\epsilon})\gamma\alpha_B\mathrm{tr}(\mathbf{H}^{(q)}),\quad 8\epsilon_p\leq \gamma \lambda_d^{(q)}-\frac{3}{4}\gamma^2{\lambda_d^{(q)}}^2,
    \end{equation*}
    then
    \begin{equation*}
        \mathcal{M}_B^{(q)} \circ \mathbf{C}_t^{(M)} \preceq\frac{2\gamma{\sigma_M^{(q)}}^2+\frac{8d\epsilon_p}{\gamma\mathrm{tr}(\mathbf{H}^{(q)})}}{1-2(1+\tilde{\epsilon})\gamma\alpha_B\mathrm{tr}(\mathbf{H}^{(q)})}\alpha_B\mathrm{tr}(\mathbf{H}^{(q)})\mathbf{H}^{(q)}.
    \end{equation*}
\end{lemma}
\begin{proof}
    By (\ref{initial CtM, master}),
    \begin{equation*}
        \begin{aligned}
            \mathbf{C}_t^{(M)}\preceq \mathbf{C}_\infty^{(M)}\preceq \gamma\left(\frac{1}{2}\widetilde{\mathcal{T}}_2^{(q)}-\gamma(1+\widetilde{\epsilon})\mathcal{M}_B^{(q)}\right)^{-1}\circ\underbrace{\left({\sigma_M^{(q)}}^2\mathbf{H}^{(q)}+4\frac{\epsilon_p}{\gamma^2}\mathbf{I}\right)}_{\boldsymbol{\Sigma}}.
        \end{aligned}
    \end{equation*}
    Applying $\widetilde{\mathcal{T}}_2^{(q)}$,
    \begin{equation*}
        \begin{aligned}
            \widetilde{\mathcal{T}}_2^{(q)}\circ\mathbf{C}_\infty^{(M)} & =(\widetilde{\mathcal{T}}_2^{(q)}-2(1+\tilde{\epsilon})\gamma\mathcal{M}_B^{(q)})\circ\mathbf{C}_\infty^{(M)}+2(1+\tilde{\epsilon})\gamma\mathcal{M}_B^{(q)}\circ\mathbf{C}_\infty^{(M)} \\
            & \preceq2\gamma\boldsymbol{\Sigma}+2(1+\tilde{\epsilon})\gamma\mathcal{M}_B^{(q)}\circ\mathbf{C}_\infty^{(M)}.
        \end{aligned}
    \end{equation*}
    Applying $(\widetilde{\mathcal{T}}_2^{(q)})^{-1}$ we have
    \begin{equation*}
        \begin{aligned}
            \mathbf{C}_{\infty}^{(M)}  \preceq&2\gamma{\sigma_M^{(q)}}^2\cdot(\widetilde{\mathcal{T}}_2^{(q)})^{-1}\circ\mathbf{H}^{(q)}+\frac{8\epsilon_p}{\gamma}(\widetilde{\mathcal{T}}_2^{(q)})^{-1}\circ\mathbf{I}
            +2(1+\tilde{\epsilon})\gamma(\widetilde{\mathcal{T}}_2^{(q)})^{-1}\circ\mathcal{M}_B^{(q)}\circ\mathbf{C}_\infty^{(M)}\\
            \preceq&\sum_{t=0}^\infty\left(2(1+\tilde{\epsilon})\gamma(\widetilde{\mathcal{T}}_2^{(q)})^{-1}\circ\mathcal{M}_B^{(q)}\right)^t\circ(\widetilde{\mathcal{T}}_2^{(q)})^{-1}\circ\left(2\gamma{\sigma_M^{(q)}}^2\mathbf{H}^{(q)}+\frac{8\epsilon_p}{\gamma}\mathbf{I}\right).
        \end{aligned}
    \end{equation*}
    Noticing that $(\widetilde{\mathcal{T}}_2^{(q)})^{-1}\circ\mathbf{H}^{(q)}  \preceq \mathbf{I}$ and $(\widetilde{\mathcal{T}}_2^{(q)})^{-1}\circ\mathbf{I}\preceq  (\mathbf{H}^{(q)})^{-1}$,
    we have
    \begin{equation}
        \label{eq: f13}
        \begin{aligned}
            \mathbf{C}_{\infty}^{(M)}\preceq&\sum_{t=0}^\infty\left(2(1+\tilde{\epsilon})\gamma(\widetilde{\mathcal{T}}_2^{(q)})^{-1}\circ\mathcal{M}_B^{(q)}\right)^t\circ(\widetilde{\mathcal{T}}_2^{(q)})^{-1}\circ\left(2\gamma{\sigma_M^{(q)}}^2\mathbf{H}^{(q)}+\frac{8\epsilon_p}{\gamma}\mathbf{I}\right)\\
            \preceq&\sum_{t=0}^\infty\left(2(1+\tilde{\epsilon})\gamma(\widetilde{\mathcal{T}}_2^{(q)})^{-1}\circ\mathcal{M}_B^{(q)}\right)^t\circ\left(2\gamma{\sigma_M^{(q)}}^2\mathbf{I}+\frac{8\epsilon_p}{\gamma}(\mathbf{H}^{(q)})^{-1}\right).
        \end{aligned}
    \end{equation}
    Firstly,
    \begin{equation}
        \label{eq: f14}
        \begin{aligned}
            2(1+\tilde{\epsilon})\gamma(\widetilde{\mathcal{T}}_2^{(q)})^{-1}\circ\mathcal{M}_B^{(q)}\circ \mathbf{I}\preceq &2(1+\tilde{\epsilon})\gamma\alpha_B\mathrm{tr}(\mathbf{H}^{(q)})\mathbf{I}.
        \end{aligned}
    \end{equation}
    Secondly,
    \begin{equation}
        \label{eq: f15}
        \begin{aligned}
            2(1+\tilde{\epsilon})\gamma(\widetilde{\mathcal{T}}_2^{(q)})^{-1}\circ\mathcal{M}_B^{(q)}\circ (\mathbf{H}^{(q)})^{-1}\preceq&2(1+\tilde{\epsilon})\gamma\alpha_B \mathrm{tr}(\mathbf{I})\mathbf{I}.
        \end{aligned}
    \end{equation}

    Therefore, together with (\ref{eq: f13}), (\ref{eq: f14}) and (\ref{eq: f15}), we have
    \begin{equation*}
        \begin{aligned}
            \mathbf{C}_{\infty}^{(M)}\preceq&\sum_{t=0}^\infty\left(2(1+\tilde{\epsilon})\gamma(\widetilde{\mathcal{T}}_2^{(q)})^{-1}\circ\mathcal{M}_B^{(q)}\right)^{t-1}\circ \mathbf{I}
            \cdot\left[4(1+\tilde{\epsilon})\gamma^2{\sigma_{M}^{(q)}}^2\alpha_B\mathrm{tr}(\mathbf{H}^{(q)})+16(1+\tilde{\epsilon})\alpha_B d\epsilon_p\right]\\
            =&\left(2\gamma{\sigma_M^{(q)}}^2+\frac{8d\epsilon_p}{\gamma\mathrm{tr}(\mathbf{H}^{(q)})}\right) \sum_{t=0}^\infty\left(2(1+\tilde{\epsilon})\gamma(\widetilde{\mathcal{T}}_2^{(q)})^{-1}\circ\mathcal{M}_B^{(q)}\right)^{t-1}\circ \mathbf{I}\cdot2(1+\tilde{\epsilon})\gamma\alpha_B\mathrm{tr}(\mathbf{H}^{(q)})\\
            \preceq&\left(2\gamma{\sigma_M^{(q)}}^2+\frac{8d\epsilon_p}{\gamma\mathrm{tr}(\mathbf{H}^{(q)})}\right) \sum_{t=0}^\infty\left[2(1+\tilde{\epsilon})\gamma\alpha_B\mathrm{tr}(\mathbf{H}^{(q)})\right]^{t} \mathbf{I}\\
            =&\frac{2\gamma{\sigma_M^{(q)}}^2+\frac{8d\epsilon_p}{\gamma\mathrm{tr}(\mathbf{H}^{(q)})}}{1-2(1+\tilde{\epsilon})\gamma\alpha_B\mathrm{tr}(\mathbf{H}^{(q)})}\mathbf{I},
        \end{aligned}
    \end{equation*}
    where the second inequality uses (\ref{eq: f14}). At last,
    \begin{equation*}
        \mathcal{M}_B^{(q)} \circ \mathbf{C}_t^{(M)} \preceq\frac{2\gamma{\sigma_M^{(q)}}^2+\frac{8d\epsilon_p}{\gamma\mathrm{tr}(\mathbf{H}^{(q)})}}{1-2(1+\tilde{\epsilon})\gamma\alpha_B\mathrm{tr}(\mathbf{H}^{(q)})}\alpha_B\mathrm{tr}(\mathbf{H}^{(q)})\mathbf{H}^{(q)}.
    \end{equation*}
\end{proof}

\begin{lemma}[\rm A bound for variance under multiplicative quantization with quantized master weight]
    \label{lem: f6}
    Under Assumption \ref{ass1}, Assumption \ref{ass: addd quantized}, Assumption \ref{ass2}, and Assumption \ref{ass3}, if the step size satisfies
    \begin{equation*}
        1>2(1+\tilde{\epsilon})\gamma\alpha_B\mathrm{tr}(\mathbf{H}^{(q)}),\quad 8\epsilon_p\leq \gamma \lambda_d^{(q)}-\frac{3}{4}\gamma^2{\lambda_d^{(q)}}^2,
    \end{equation*}
    then
    \begin{equation*}
        \begin{aligned}
            \mathbb{E}_{\mathbf{w}^*}\mathrm{variance}\leq &\frac{4{\sigma_M^{(q)}}^2+32(1+\tilde{\epsilon}){d\epsilon_p}\alpha_B/\gamma}{1-2(1+\tilde{\epsilon})\gamma\alpha_B\mathrm{tr}(\mathbf{H}^{(q)})}\left(\frac{k^*}{N}+N\gamma^2\cdot\sum_{i>k^*}(\lambda_i^{(q)})^2\right)\\
            +
            &\frac{16\epsilon_p}{\gamma}\left(\sum_{i\leq k^*}\frac{1}{N\gamma\lambda_i^{(q)}}+N\gamma\sum_{i>k^*}\lambda_i^{(q)}\right).
        \end{aligned}
    \end{equation*}
\end{lemma}
\begin{proof}
    We first provide a refined upper bound for $\mathbf{C}_t$. 
    By (\ref{initial CtM, master}) and Lemma \ref{A bound for M_B^{(q)} C_t^{(M)}, master},
    \begin{equation}
        \label{eq: f17}
        \begin{aligned}
            \mathbf{C}_t^{(M)}\preceq&(\mathcal{I}-\frac{\gamma}{2}\widetilde{\mathcal{T}}_2^{(q)})\circ \mathbf{C}_{t-1}^{(M)}+\gamma^2(1+\tilde{\epsilon})\mathcal{M}_B^{(q)}\circ \mathbf{C}_{t-1}^{(M)}
            +\gamma^2{\sigma_M^{(q)}}^2\mathbf{H}^{(q)}+4\epsilon_p\mathbf{I}\\
            \preceq&(\mathcal{I}-\frac{\gamma}{2}\widetilde{\mathcal{T}}_2^{(q)})\circ \mathbf{C}_{t-1}^{(M)}
            +\gamma^2{\sigma_M^{(q)}}^2\mathbf{H}^{(q)}+4\epsilon_p\mathbf{I}\\
            +&\gamma^2(1+\tilde{\epsilon})\frac{2\gamma{\sigma_M^{(q)}}^2+\frac{8d\epsilon_p}{\gamma\mathrm{tr}(\mathbf{H}^{(q)})}}{1-2(1+\tilde{\epsilon})\gamma\alpha_B\mathrm{tr}(\mathbf{H}^{(q)})}\alpha_B\mathrm{tr}(\mathbf{H}^{(q)})\mathbf{H}^{(q)}\\
            =&(\mathcal{I}-\frac{\gamma}{2}\widetilde{\mathcal{T}}_2^{(q)})\circ \mathbf{C}_{t-1}^{(M)}+\frac{\gamma^2{\sigma_M^{(q)}}^2+8\gamma(1+\tilde{\epsilon}){d\epsilon_p}\alpha_B}{1-2(1+\tilde{\epsilon})\gamma\alpha_B\mathrm{tr}(\mathbf{H}^{(q)})}\mathbf{H}^{(q)}+4\epsilon_p\mathbf{I}\\
            =&\frac{\gamma^2{\sigma_M^{(q)}}^2+8\gamma(1+\tilde{\epsilon}){d\epsilon_p}\alpha_B}{1-2(1+\tilde{\epsilon})\gamma\alpha_B\mathrm{tr}(\mathbf{H}^{(q)})}\sum_{k=0}^{t-1}(\mathcal{I}-\frac{\gamma}{2}\widetilde{\mathcal{T}}_2^{(q)})^k\circ\mathbf{H}^{(q)}+4\epsilon_p\sum_{k=0}^{t-1}(\mathcal{I}-\frac{\gamma}{2}\widetilde{\mathcal{T}}_2^{(q)})^k\circ\mathbf{I}\\
            \preceq&\frac{\gamma^2{\sigma_M^{(q)}}^2+8\gamma(1+\tilde{\epsilon}){d\epsilon_p}\alpha_B}{1-2(1+\tilde{\epsilon})\gamma\alpha_B\mathrm{tr}(\mathbf{H}^{(q)})}\sum_{k=0}^{t-1}(\mathbf{I}-\frac{\gamma}{2} \mathbf{H}^{(q)})^{k}\mathbf{H}^{(q)}+4\epsilon_p\sum_{k=0}^{t-1}(\mathbf{I}-\frac{\gamma}{2} \mathbf{H}^{(q)})^{k}\\
            =&\underbrace{\frac{2\gamma{\sigma_M^{(q)}}^2+16(1+\tilde{\epsilon}){d\epsilon_p}\alpha_B}{1-2(1+\tilde{\epsilon})\gamma\alpha_B\mathrm{tr}(\mathbf{H}^{(q)})}(\mathbf{I}-(\mathbf{I}-\frac{\gamma}{2} \mathbf{H}^{(q)})^{t})}_{\mathbf{V}_1}+\underbrace{\frac{8\epsilon_p}{\gamma}(\mathbf{I}-(\mathbf{I}-\frac{\gamma}{2} \mathbf{H}^{(q)})^{t})(\mathbf{H}^{(q)})^{-1}}_{\mathbf{V}_2}.
        \end{aligned}
    \end{equation}

    Recall (\ref{eq: f mul}),
    \begin{equation*}
        \mathbb{E}_{\mathbf{w}^*}\mathrm{variance}\leq\frac{1}{N^2}\cdot\sum_{t=0}^{N-1}\sum_{k=t}^{N-1}\left\langle(\mathbf{I}-\gamma\mathbf{H}^{(q)})^{k-t}\mathbf{H}^{(q)},\mathbf{V}_1+\mathbf{V}_2\right\rangle.
    \end{equation*}
    Regarding $\mathbf{V}_1$, 
    \begin{equation}
    \label{eq: f18}
        \begin{aligned}
            &\frac{1}{N^2}\cdot\sum_{t=0}^{N-1}\sum_{k=t}^{N-1}\left\langle(\mathbf{I}-\gamma\mathbf{H}^{(q)})^{k-t}\mathbf{H}^{(q)},\mathbf{V}_1\right\rangle\\
            \leq&\frac{1}{N^2}\cdot\sum_{t=0}^{N-1}\sum_{k=t}^{N-1}\left\langle(\mathbf{I}-\frac{\gamma}{2}\mathbf{H}^{(q)})^{k-t}\mathbf{H}^{(q)},\mathbf{V}_1\right\rangle\\
            \leq&\frac{2}{N^2\gamma}\frac{2\gamma{\sigma_M^{(q)}}^2+16(1+\tilde{\epsilon}){d\epsilon_p}\alpha_B}{1-2(1+\tilde{\epsilon})\gamma\alpha_B\mathrm{tr}(\mathbf{H}^{(q)})}\sum_{t=0}^{N-1}\left\langle\mathbf{I}-(\mathbf{I}-\frac{\gamma}{2}\mathbf{H}^{(q)})^{N-t},\mathbf{I}-(\mathbf{I}-\frac{\gamma}{2} \mathbf{H}^{(q)})^{t}\right\rangle\\
            =&\frac{2}{N^2\gamma}\frac{2\gamma{\sigma_M^{(q)}}^2+16(1+\tilde{\epsilon}){d\epsilon_p}\alpha_B}{1-2(1+\tilde{\epsilon})\gamma\alpha_B\mathrm{tr}(\mathbf{H}^{(q)})}\sum_{t=0}^{N-1}\sum_i(1-(1-\frac{\gamma}{2}\lambda_i^{(q)})^{N-t})(1-(1-\frac{\gamma}{2} \lambda_i^{(q)})^{t})\\
            \leq&\frac{2}{N^2\gamma}\frac{2\gamma{\sigma_M^{(q)}}^2+16(1+\tilde{\epsilon}){d\epsilon_p}\alpha_B}{1-2(1+\tilde{\epsilon})\gamma\alpha_B\mathrm{tr}(\mathbf{H}^{(q)})}\sum_{t=0}^{N-1}\sum_i\left[1-(1-\frac{\gamma}{2}\lambda_i^{(q)})^{N}\right]^2\\
            \leq&\frac{2}{N\gamma}\frac{2\gamma{\sigma_M^{(q)}}^2+16(1+\tilde{\epsilon}){d\epsilon_p}\alpha_B}{1-2(1+\tilde{\epsilon})\gamma\alpha_B\mathrm{tr}(\mathbf{H}^{(q)})}\sum_i\min\left\{1,\frac{\gamma^2N^2}{4}{\lambda_i^{(q)}}^2\right\}\\
            \leq&\frac{2}{N\gamma}\frac{2\gamma{\sigma_M^{(q)}}^2+16(1+\tilde{\epsilon}){d\epsilon_p}\alpha_B}{1-2(1+\tilde{\epsilon})\gamma\alpha_B\mathrm{tr}(\mathbf{H}^{(q)})}\sum_i\min\left\{1,{\gamma^2N^2}{\lambda_i^{(q)}}^2\right\}\\
            =&\frac{4{\sigma_M^{(q)}}^2+32(1+\tilde{\epsilon}){d\epsilon_p}\alpha_B/\gamma}{1-2(1+\tilde{\epsilon})\gamma\alpha_B\mathrm{tr}(\mathbf{H}^{(q)})}\left(\frac{k^*}{N}+N\gamma^2\cdot\sum_{i>k^*}(\lambda_i^{(q)})^2\right).
        \end{aligned}
    \end{equation}
    
    Regarding $\mathbf{V}_2$,
    \begin{equation}
        \label{eq: f19}
        \begin{aligned}
            &\frac{1}{N^2}\cdot\sum_{t=0}^{N-1}\sum_{k=t}^{N-1}\left\langle(\mathbf{I}-\gamma\mathbf{H}^{(q)})^{k-t}\mathbf{H}^{(q)},\mathbf{V}_2\right\rangle\\
            \leq&\frac{2}{N^2\gamma}\cdot\sum_{t=0}^{N-1}\left\langle\mathbf{I}-(\mathbf{I}-\frac{\gamma}{2}\mathbf{H}^{(q)})^{N-t},\mathbf{V}_2\right\rangle\\
            =&\frac{8\epsilon_p}{\gamma}\frac{2}{N^2\gamma}\cdot\sum_{t=0}^{N-1}\left\langle\mathbf{I}-(\mathbf{I}-\frac{\gamma}{2}\mathbf{H}^{(q)})^{N-t},(\mathbf{I}-(\mathbf{I}-\frac{\gamma}{2} \mathbf{H}^{(q)})^{t})(\mathbf{H}^{(q)})^{-1}\right\rangle\\
            =&\frac{8\epsilon_p}{\gamma}\frac{2}{N^2\gamma}\cdot\sum_{t=0}^{N-1}\sum_i\left[1-(1-\frac{\gamma}{2}\lambda_i^{(q)})^{N-t}\right]\left[1-(1-\frac{\gamma}{2}\lambda_i^{(q)})^t\right]\frac{1}{\lambda_i^{(q)}}\\
            \leq&\frac{16\epsilon_p}{N\gamma^2}\sum_i\min\left\{ \frac{1}{\lambda_i^{(q)}}, N^2\gamma^2{\lambda_i^{(q)}}\right\}\\
            =&\frac{16\epsilon_p}{\gamma}\left(\sum_{i\leq k^*}\frac{1}{N\gamma\lambda_i^{(q)}}+N\gamma\sum_{i>k^*}\lambda_i^{(q)}\right).
        \end{aligned}
    \end{equation}
    
    Together with (\ref{eq: f mul}), (\ref{eq: f17}), (\ref{eq: f18}) and (\ref{eq: f19}),
    \begin{gather*}
        \mathbb{E}_{\mathbf{w}^*}\mathrm{variance}\leq \frac{4{\sigma_M^{(q)}}^2+32(1+\tilde{\epsilon}){d\epsilon_p}\alpha_B/\gamma}{1-2(1+\tilde{\epsilon})\gamma\alpha_B\mathrm{tr}(\mathbf{H}^{(q)})}\left(\frac{k^*}{N}+N\gamma^2\cdot\sum_{i>k^*}(\lambda_i^{(q)})^2\right)+\frac{16\epsilon_p}{\gamma}\left(\sum_{i\leq k^*}\frac{1}{N\gamma\lambda_i^{(q)}}+N\gamma\sum_{i>k^*}\lambda_i^{(q)}\right).
    \end{gather*}
\end{proof}

\subsubsection{Analysis of Bias Error}
Recall that
\begin{equation*}
    \mathbf{B}_t^{(M)}:=((1+8\epsilon_p)\mathcal{I}-\gamma\mathcal{T}_B^{(q)}+\tilde{\epsilon}\gamma^2\mathcal{M}_B^{(q)})^t \circ \mathbf{B}_0^{(M)}, \quad \mathbf{B}_0^{(M)}=\mathbf{B}_0=\mathbb{E}\left[\boldsymbol{\eta}_0 \boldsymbol{\eta}_0^\top\right],
\end{equation*}
From (\ref{bias express}) we deduce that
\begin{equation}
    \label{eq: bias m}
    \mathrm{bias}\leq \frac{1}{\gamma N^{2}}\left\langle\mathbf{I}-(\mathbf{I}-\gamma\mathbf{H}^{(q)})^{N},\sum_{t=0}^{N-1}\mathbf{B}_{t}^{(M)}\right\rangle.
\end{equation}
Denote $\mathbf{S}_n^{(M)}=\sum_{t=0}^{n-1}\mathbf{B}_t^{(M)}$. Motivated by Lemma \ref{(Initial Study of $S_t$) mu},
\begin{equation*}
    \begin{aligned}
        \mathbf{S}_{t}^{(M)} =& ((1+8\epsilon_p)\mathcal{I}-\gamma\mathcal{T}_B^{(q)}+\tilde{\epsilon}\gamma^2\mathcal{M}_B^{(q)})\circ\mathbf{S}_{t-1}^{(M)}+\mathbf{B}_0\\
        =& ((1+8\epsilon_p)\mathcal{I}-\gamma\widetilde{\mathcal{T}}^{(q)})\circ\mathbf{S}_{t-1}^{(M)}+\gamma (\widetilde{\mathcal{T}}^{(q)}-{\mathcal{T}}_B^{(q)})\circ\mathbf{S}_{t-1}^{(M)}+\tilde{\epsilon}\gamma^2\mathcal{M}_B^{(q)}\circ \mathbf{S}_{t-1}^{(M)}+\mathbf{B}_0\\
        =& ((1+8\epsilon_p)\mathcal{I}-\gamma\widetilde{\mathcal{T}}^{(q)})\circ\mathbf{S}_{t-1}^{(M)}+\gamma^2 ((1+\tilde{\epsilon}){\mathcal{M}}_B^{(q)}-\widetilde{\mathcal{M}}^{(q)})\circ\mathbf{S}_{t-1}^{(M)}+\mathbf{B}_0\\
        \preceq&((1+8\epsilon_p)\mathcal{I}-\gamma\tilde{\mathcal{T}}^{(q)})\circ\mathbf{S}_{t-1}^{(M)}+(1+\tilde{\epsilon})\gamma^{2}\mathcal{M}_B^{(q)}\circ\mathbf{S}_{t-1}^{(M)}+\mathbf{B}_{0}.
    \end{aligned}
\end{equation*}
Similar to the analysis of variance error, we consider
\begin{equation*}
    \gamma < \frac{2}{3\lambda_1^{(q)}},\quad 8\epsilon_p\leq \gamma \lambda_d^{(q)}-\frac{3}{4}\gamma^2{\lambda_d^{(q)}}^2.
\end{equation*}
It follows that $8\epsilon_p\mathcal{I}\preceq\gamma \widetilde{\mathcal{T}}^{(q)}-\frac{\gamma}{2}\widetilde{\mathcal{T}}_2^{(q)}$ and hence
\begin{equation}
    \label{eq: f20}
    \mathbf{S}_{t}^{(M)}\preceq(\mathcal{I}-\frac{\gamma}{2}\widetilde{\mathcal{T}}_2^{(q)})\circ\mathbf{S}_{t-1}^{(M)}+(1+\tilde{\epsilon})\gamma^{2}\mathcal{M}_B^{(q)}\circ\mathbf{S}_{t-1}^{(M)}+\mathbf{B}_{0}.
\end{equation}

\begin{lemma}[\rm A bound for $\mathcal{M}_B^{(q)}\circ\mathbf{S}_t^{(M)}$ with quantized master weight]
    \label{lem: f7}
    For $1 \leq t\leq N$, under Assumption \ref{ass1}, Assumption \ref{ass: addd quantized}, Assumption \ref{ass2}, and Assumption \ref{ass3}, if the step size satisfies
    \begin{equation*}
        1>2(1+\tilde{\epsilon})\gamma\alpha_B\mathrm{tr}(\mathbf{H}^{(q)}),\quad 8\epsilon_p\leq \gamma \lambda_d^{(q)}-\frac{3}{4}\gamma^2{\lambda_d^{(q)}}^2,
    \end{equation*}
    then
    \begin{equation*}
         \mathcal{M}_B^{(q)} \circ \mathbf{S}_t^{(M)} \preceq          \frac{2\alpha_B\cdot\mathrm{tr}\left(\left[\mathcal{I}-(\mathcal{I}-\frac{\gamma}{2}\widetilde{\mathcal{T}}_2^{(q)})^t\right]\circ\mathbf{B}_0\right)}{\gamma(1-2(1+\tilde{\epsilon})\gamma\alpha_B\operatorname{tr}(\mathbf{H}^{(q)}))}\cdot\mathbf{H}^{(q)}.
    \end{equation*}
\end{lemma}
\begin{proof}
    From (\ref{eq: f20}),
    \begin{equation*}
        \begin{aligned}
            \mathbf{S}_t^{(M)}\preceq& \sum_{k=0}^{t-1}\left(\mathcal{I}-\frac{\gamma}{2}\widetilde{\mathcal{T}}_2^{(q)}+(1+\widetilde{\epsilon})\gamma^2\mathcal{M}_B^{(q)}\right)^k\circ \mathbf{B}_0\\
            =&\frac{1}{\gamma}\left(\frac{1}{2}\widetilde{\mathcal{T}}_2^{(q)}-(1+\widetilde{\epsilon})\gamma\mathcal{M}_B^{(q)}\right)^{-1}\circ\underbrace{\left[\mathcal{I}-\left(\mathcal{I}-\frac{\gamma}{2}\widetilde{\mathcal{T}}_2^{(q)}+(1+\widetilde{\epsilon})\gamma^2\mathcal{M}_B^{(q)}\right)^t\right]\circ \mathbf{B}_0}_{\mathbf{A}}.
        \end{aligned}
    \end{equation*}
    Applying $\widetilde{\mathcal{T}}_2^{(q)}$, we have
    \begin{equation*}
        \begin{aligned}
            &\widetilde{\mathcal{T}}_2^{(q)}\circ\left(\frac{1}{2}\widetilde{\mathcal{T}}_2^{(q)}-(1+\widetilde{\epsilon})\gamma\mathcal{M}_B^{(q)}\right)^{-1}\circ\mathbf{A}\\
            =&2\mathbf{A}+2(1+\widetilde{\epsilon})\gamma\mathcal{M}_B^{(q)}\circ \left(\frac{1}{2}\widetilde{\mathcal{T}}_2^{(q)}-(1+\widetilde{\epsilon})\gamma\mathcal{M}_B^{(q)}\right)^{-1}\circ\mathbf{A}.
        \end{aligned}
    \end{equation*}
    Applying $(\widetilde{\mathcal{T}}_2^{(q)})^{-1}$, we have
    \begin{equation*}
        \begin{aligned}
            &\left(\frac{1}{2}\widetilde{\mathcal{T}}_2^{(q)}-(1+\widetilde{\epsilon})\gamma\mathcal{M}_B^{(q)}\right)^{-1}\circ\mathbf{A}\\
            =&2(\widetilde{\mathcal{T}}_2^{(q)})^{-1}\circ \mathbf{A}+2(\widetilde{\mathcal{T}}_2^{(q)})^{-1}\circ (1+\widetilde{\epsilon})\gamma\mathcal{M}_B^{(q)}\circ \left(\frac{1}{2}\widetilde{\mathcal{T}}_2^{(q)}-(1+\widetilde{\epsilon})\gamma\mathcal{M}_B^{(q)}\right)^{-1}\circ\mathbf{A}.
        \end{aligned}
    \end{equation*}
    Applying $\mathcal{M}_B^{(q)}$, we have
    \begin{equation*}
        \begin{aligned}
            &\mathcal{M}_B^{(q)}\circ \left(\frac{1}{2}\widetilde{\mathcal{T}}_2^{(q)}-(1+\widetilde{\epsilon})\gamma\mathcal{M}_B^{(q)}\right)^{-1}\circ\mathbf{A}\\
            =&2\mathcal{M}_B^{(q)}\circ(\widetilde{\mathcal{T}}_2^{(q)})^{-1}\circ (1+\widetilde{\epsilon})\gamma\mathcal{M}_B^{(q)}\circ \left(\frac{1}{2}\widetilde{\mathcal{T}}_2^{(q)}-(1+\widetilde{\epsilon})\gamma\mathcal{M}_B^{(q)}\right)^{-1}\circ\mathbf{A}+2\mathcal{M}_B^{(q)}\circ(\widetilde{\mathcal{T}}_2^{(q)})^{-1}\circ \mathbf{A}.
        \end{aligned}
    \end{equation*}
    It follows that
    \begin{equation}
        \label{eq: f21}
        \begin{aligned}
            &\mathcal{M}_B^{(q)}\circ \left(\frac{1}{2}\widetilde{\mathcal{T}}_2^{(q)}-(1+\widetilde{\epsilon})\gamma\mathcal{M}_B^{(q)}\right)^{-1}\circ\mathbf{A}\\
            =&\sum_{t=0}^\infty \left(2(1+\widetilde{\epsilon})\gamma\mathcal{M}_B^{(q)}\circ(\widetilde{\mathcal{T}}_2^{(q)})^{-1}\right)^t\circ 2\mathcal{M}_B^{(q)}\circ(\widetilde{\mathcal{T}}_2^{(q)})^{-1}\circ \mathbf{A}.
        \end{aligned}
    \end{equation}
    By Assumption \ref{ass2}, 
    \begin{equation*}
        \begin{aligned}
            {\mathcal{M}_B^{(q)}}\circ{(\widetilde{\mathcal{T}}_2^{(q)})}^{-1}\circ\mathbf{A}&\preceq\alpha_B\operatorname{tr}(\mathbf{H}^{(q)}{(\widetilde{\mathcal{T}}_2^{(q)})}^{-1}\circ\mathbf{A})\mathbf{H}^{(q)}\\
            &= \alpha_B \frac{\gamma}{2}\operatorname{tr}\left(\sum_{t=0}^{\infty}\mathbf{H}^{(q)}(\mathbf{I}-\frac{\gamma}{2}\mathbf{H}^{(q)})^{t}\mathbf{A}(\mathbf{I}-\frac{\gamma}{2}\mathbf{H}^{(q)})^{t}\right) \mathbf{H}^{(q)}\\
            &= \alpha_B \mathrm{tr}\left(\mathbf{H}^{(q)}(2\mathbf{H}^{(q)}-\frac{\gamma}{2}(\mathbf{H}^{(q)})^{2})^{-1}\mathbf{A}\right)\mathbf{H}^{(q)}\\
            &\preceq \alpha_B \mathrm{tr}(\mathbf{A})\mathbf{H}^{(q)},
        \end{aligned}
    \end{equation*}
    together with (\ref{eq: f21}), $(\widetilde{\mathcal{T}}_2^{(q)})^{-1}\circ\mathbf{H}^{(q)}  \preceq \mathbf{I}$ and $\mathcal{M}_B^{(q)}\circ\mathbf{I}\preceq\alpha_B\operatorname{tr}(\mathbf{H}^{(q)})\mathbf{H}^{(q)}$, we have
    \begin{equation*}
        \begin{aligned}
            &\mathcal{M}_B^{(q)}\circ \left(\frac{1}{2}\widetilde{\mathcal{T}}_2^{(q)}-(1+\widetilde{\epsilon})\gamma\mathcal{M}_B^{(q)}\right)^{-1}\circ\mathbf{A}\\
            \preceq&\sum_{t=0}^\infty \left(2(1+\widetilde{\epsilon})\gamma\mathcal{M}_B^{(q)}\circ(\widetilde{\mathcal{T}}_2^{(q)})^{-1}\right)^t\circ 2\alpha_B\mathrm{tr}(\mathbf{A})\mathbf{H}^{(q)}\\
            =&2\alpha_B\mathrm{tr}(\mathbf{A})\sum_{t=0}^\infty \left(2(1+\widetilde{\epsilon})\gamma\mathcal{M}_B^{(q)}\circ(\widetilde{\mathcal{T}}_2^{(q)})^{-1}\right)^{t}\circ \mathbf{H}^{(q)}\\
            \preceq&2\alpha_B\mathrm{tr}(\mathbf{A})\sum_{t=0}^\infty \left(2(1+\widetilde{\epsilon})\gamma\alpha_B\mathrm{tr}(\mathbf{H}^{(q)})\right)^{t} \mathbf{H}^{(q)}\\
            =&\frac{2\alpha_B\mathrm{tr}(\mathbf{A})}{1-2(1+\widetilde{\epsilon})\gamma\alpha_B\mathrm{tr}(\mathbf{H}^{(q)})}\mathbf{H}^{(q)}.
        \end{aligned}
    \end{equation*}

    Therefore, recall that
    \begin{equation*}
        \mathbf{A}=\left[\mathcal{I}-\left(\mathcal{I}-\frac{\gamma}{2}\widetilde{\mathcal{T}}_2^{(q)}+(1+\widetilde{\epsilon})\gamma^2\mathcal{M}_B^{(q)}\right)^t\right]\circ \mathbf{B}_0\preceq\left[\mathcal{I}-\left(\mathcal{I}-\frac{\gamma}{2}\widetilde{\mathcal{T}}_2^{(q)}\right)^t\right]\circ \mathbf{B}_0,
    \end{equation*}
    we have
    \begin{equation*}
        \begin{aligned}
            \mathcal{M}_B^{(q)} \circ \mathbf{S}_t^{(M)} \preceq \gamma^{-1}\frac{2\alpha_B\operatorname{tr}(\mathbf{A})}{1-2(1+\tilde{\epsilon})\gamma\alpha_B\operatorname{tr}(\mathbf{H}^{(q)})}\cdot\mathbf{H}^{(q)}\preceq \frac{2\alpha_B\cdot\mathrm{tr}\left(\left[\mathcal{I}-(\mathcal{I}-\frac{\gamma}{2}\widetilde{\mathcal{T}}_2^{(q)})^t\right]\circ\mathbf{B}_0\right)}{\gamma(1-2(1+\tilde{\epsilon})\gamma\alpha_B\operatorname{tr}(\mathbf{H}^{(q)}))}\cdot\mathbf{H}^{(q)}.
        \end{aligned}
    \end{equation*}
\end{proof}

Together with (\ref{eq: f20}) and Lemma \ref{lem: f7}, we are now ready to bound the bias error with quantized master weight.

\begin{lemma}[\rm A bound for bias under multiplicative quantization with quantized master weight]
    \label{bias R2 bound multiplicative, master}
    Under Assumption \ref{ass1}, Assumption \ref{ass: addd quantized}, Assumption \ref{ass2}, and Assumption \ref{ass3}, if the stepsize satisfies 
    \begin{equation*}
        1>2(1+\tilde{\epsilon})\gamma\alpha_B\mathrm{tr}(\mathbf{H}^{(q)}),\quad 8\epsilon_p\leq \gamma \lambda_d^{(q)}-\frac{3}{4}\gamma^2{\lambda_d^{(q)}}^2,
    \end{equation*}
    then
    \begin{equation*}
        \begin{aligned}
            \mathrm{bias}
            \leq& \frac{8(1+\tilde{\epsilon})\alpha_B\cdot\left(\frac{\|\mathbf{w}_{0}-{\mathbf{w}^{(q)}}^{*}\|_{\mathbf{I}_{0:k^{*}}}^{2}}{N\gamma}+\|\mathbf{w}_{0}-{\mathbf{w}^{(q)}}^{*}\|_{\mathbf{H}_{k^{*}:\infty}}^{2}\right)}{1-2(1+\tilde{\epsilon})\gamma\alpha_B\operatorname{tr}(\mathbf{H}^{(q)})}\left(\frac{k^*}{N} + {N\gamma^2}\sum_{i> k^*}{\lambda_i^{(q)}}^2\right)\\
            +& \frac{4}{\gamma^2N^2}\cdot\|\mathbf{w}_0-{\mathbf{w}^{(q)}}^*\|_{(\mathbf{H}_{0:k^*}^{(q)})^{-1}}^2+4\|\mathbf{w}_0-{\mathbf{w}^{(q)}}^*\|_{\mathbf{H}_{{k^*}:\infty}^{(q)}}^2.
        \end{aligned}
    \end{equation*}
\end{lemma}
\begin{proof}
    By (\ref{eq: f20}) and Lemma \ref{lem: f7},
    \begin{equation*}
        \begin{aligned}
            \mathbf{S}_{t}^{(M)}\preceq&(\mathcal{I}-\frac{\gamma}{2}\widetilde{\mathcal{T}}_2^{(q)})\circ\mathbf{S}_{t-1}^{(M)}+(1+\tilde{\epsilon})\gamma^{2}\mathcal{M}_B^{(q)}\circ\mathbf{S}_{t-1}^{(M)}+\mathbf{B}_{0}\\
            \preceq&(\mathcal{I}-\frac{\gamma}{2}\widetilde{\mathcal{T}}_2^{(q)})\circ\mathbf{S}_{t-1}^{(M)}+(1+\tilde{\epsilon})\gamma^{2}\frac{2\alpha_B\cdot\mathrm{tr}\left(\left[\mathcal{I}-(\mathcal{I}-\frac{\gamma}{2}\widetilde{\mathcal{T}}_2^{(q)})^N\right]\circ\mathbf{B}_0\right)}{\gamma(1-2(1+\tilde{\epsilon})\gamma\alpha_B\operatorname{tr}(\mathbf{H}^{(q)}))}\cdot\mathbf{H}^{(q)}+\mathbf{B}_{0}\\
            =&\sum_{k=0}^{t-1}(\mathcal{I}-\frac{\gamma}{2}\widetilde{\mathcal{T}}_2^{(q)})^k\circ\left((1+\tilde{\epsilon})\gamma^{2}\frac{2\alpha_B\cdot\mathrm{tr}\left(\left[\mathcal{I}-(\mathcal{I}-\frac{\gamma}{2}\widetilde{\mathcal{T}}_2^{(q)})^N\right]\circ\mathbf{B}_0\right)}{\gamma(1-2(1+\tilde{\epsilon})\gamma\alpha_B\operatorname{tr}(\mathbf{H}^{(q)}))}\cdot\mathbf{H}^{(q)}+\mathbf{B}_{0}\right)\\
            =&\sum_{k=0}^{t-1}(\mathbf{I}-\frac{\gamma}{2}\mathbf{H}^{(q)})^k\left((1+\tilde{\epsilon})\gamma^{2}\frac{2\alpha_B\cdot\mathrm{tr}\left(\left[\mathcal{I}-(\mathcal{I}-\frac{\gamma}{2}\widetilde{\mathcal{T}}_2^{(q)})^N\right]\circ\mathbf{B}_0\right)}{\gamma(1-2(1+\tilde{\epsilon})\gamma\alpha_B\operatorname{tr}(\mathbf{H}^{(q)}))}\cdot\mathbf{H}^{(q)}+\mathbf{B}_{0}\right)(\mathbf{I}-\frac{\gamma}{2}\mathbf{H}^{(q)})^k.
        \end{aligned}
    \end{equation*}

    Before providing our upper bound for the bias error, we denote \begin{equation*}
        \mathbf{B}_{a,b}:=\mathbf{B}_a-(\mathbf{I}-\frac{\gamma}{2}\mathbf{H}^{(q)})^{b-a}\mathbf{B}_a(\mathbf{I}-\frac{\gamma}{2}\mathbf{H}^{(q)})^{b-a}.
    \end{equation*}
    Recall from (\ref{eq: bias m}) that
    \begin{equation*}
        \mathrm{bias}\leq \frac{1}{\gamma N^{2}}\left\langle\mathbf{I}-(\mathbf{I}-\gamma\mathbf{H}^{(q)})^{N},\sum_{t=0}^{N-1}\mathbf{B}_{t}^{(M)}\right\rangle,
    \end{equation*}
    we have
    \begin{equation*}
        \begin{aligned}
            \mathrm{bias}\leq &\frac{2}{\gamma N^{2}}\left\langle\mathbf{I}-(\mathbf{I}-\frac{\gamma}{2}\mathbf{H}^{(q)})^{N},\sum_{t=0}^{N-1}\mathbf{B}_{t}^{(M)}\right\rangle\\
            \leq&\frac{2}{\gamma N^{2}}\left\langle\mathbf{I}-(\mathbf{I}-\frac{\gamma}{2}\mathbf{H}^{(q)})^{N},\sum_{k=0}^{N-1}(\mathbf{I}-\frac{\gamma}{2}\mathbf{H}^{(q)})^k\left(\frac{2(1+\tilde{\epsilon})\gamma\alpha_B\cdot\mathrm{tr}\left(\mathbf{B}_{0,N}\right)}{1-2(1+\tilde{\epsilon})\gamma\alpha_B\operatorname{tr}(\mathbf{H}^{(q)})}\cdot\mathbf{H}^{(q)}+\mathbf{B}_{0}\right)(\mathbf{I}-\frac{\gamma}{2}\mathbf{H}^{(q)})^k\right\rangle\\
            =&\frac{2}{\gamma N^{2}}\sum_{k=0}^{N-1}\left\langle(\mathbf{I}-\frac{\gamma}{2}\mathbf{H}^{(q)})^{2k}-(\mathbf{I}-\frac{\gamma}{2}\mathbf{H}^{(q)})^{N+2k},\frac{2(1+\tilde{\epsilon})\gamma\alpha_B\cdot\mathrm{tr}\left(\mathbf{B}_{0,N}\right)}{1-2(1+\tilde{\epsilon})\gamma\alpha_B\operatorname{tr}(\mathbf{H}^{(q)})}\cdot\mathbf{H}^{(q)}+\mathbf{B}_{0}\right\rangle.
        \end{aligned}
    \end{equation*}
    Note that 
    \begin{equation*}
        \begin{aligned}
            (\mathbf{I}-\frac{\gamma}{2}\mathbf{H}^{(q)})^{2k}-(\mathbf{I}-\frac{\gamma}{2}\mathbf{H}^{(q)})^{N+2k}  \preceq(\mathbf{I}-\frac{\gamma}{2}\mathbf{H}^{(q)})^{k}-(\mathbf{I}-\frac{\gamma}{2}\mathbf{H}^{(q)})^{N+k},
        \end{aligned}
    \end{equation*}
    we have
    \begin{equation*}
        \mathrm{bias}\leq \frac{2}{\gamma N^{2}}\sum_{k=0}^{N-1}\left\langle(\mathbf{I}-\frac{\gamma}{2}\mathbf{H}^{(q)})^{k}-(\mathbf{I}-\frac{\gamma}{2}\mathbf{H}^{(q)})^{N+k},\frac{2(1+\tilde{\epsilon})\gamma\alpha_B\cdot\mathrm{tr}\left(\mathbf{B}_{0,N}\right)}{1-2(1+\tilde{\epsilon})\gamma\alpha_B\operatorname{tr}(\mathbf{H}^{(q)})}\cdot\mathbf{H}^{(q)}+\mathbf{B}_{0}\right\rangle.
    \end{equation*}

    Therefore it suffices to bound the following two terms:
    \begin{equation*}
        I_1=\frac{2}{\gamma N^{2}}\sum_{k=0}^{N-1}\left\langle(\mathbf{I}-\frac{\gamma}{2}\mathbf{H}^{(q)})^{k}-(\mathbf{I}-\frac{\gamma}{2}\mathbf{H}^{(q)})^{N+k},\frac{2(1+\tilde{\epsilon})\gamma\alpha_B\cdot\mathrm{tr}\left(\mathbf{B}_{0,N}\right)}{1-2(1+\tilde{\epsilon})\gamma\alpha_B\operatorname{tr}(\mathbf{H}^{(q)})}\cdot\mathbf{H}^{(q)}\right\rangle,
    \end{equation*}
    \begin{equation*}
        I_2=\frac{2}{\gamma N^{2}}\sum_{k=0}^{N-1}\left\langle(\mathbf{I}-\frac{\gamma}{2}\mathbf{H}^{(q)})^{k}-(\mathbf{I}-\frac{\gamma}{2}\mathbf{H}^{(q)})^{N+k},\mathbf{B}_{0}\right\rangle.
    \end{equation*}
    Regarding $I_1$,
    \begin{equation*}
        \begin{aligned}
            &\sum_{k=0}^{N-1}\left\langle(\mathbf{I}-\frac{\gamma}{2}\mathbf{H}^{(q)})^{k}-(\mathbf{I}-\frac{\gamma}{2}\mathbf{H}^{(q)})^{N+k},\mathbf{H}^{(q)}\right\rangle\\
            =&\sum_{k=0}^{N-1}\sum_i \left[(1-\frac{\gamma}{2}\lambda_i^{(q)})^k-(1-\frac{\gamma}{2}\lambda_i^{(q)})^{N+k}\right]\lambda_i^{(q)}\\
            =&\frac{2}{\gamma}\sum_i \left[[1-(1-\frac{\gamma}{2}\lambda_i^{(q)})^N]-[(1-\frac{\gamma}{2}\lambda_i^{(q)})^N-(1-\frac{\gamma}{2}\lambda_i^{(q)})^{2N}]\right]\\
            =&\frac{2}{\gamma}\sum_i \left(1-(1-\frac{\gamma}{2}\lambda_i^{(q)})^N\right)^2\\
            \leq&\frac{2}{\gamma}\sum_i \min\left\{1, {N^2\gamma^2}{\lambda_i^{(q)}}^2\right\}\\
            \leq&\frac{2}{\gamma}\left(k^* + {N^2\gamma^2}\sum_{i> k^*}{\lambda_i^{(q)}}^2\right),
        \end{aligned}
    \end{equation*}
    Hence,
    \begin{equation*}
        I_1\leq \frac{4}{\gamma^2N^2}\frac{2(1+\tilde{\epsilon})\gamma\alpha_B\cdot\mathrm{tr}\left(\mathbf{B}_{0,N}\right)}{1-2(1+\tilde{\epsilon})\gamma\alpha_B\operatorname{tr}(\mathbf{H}^{(q)})}\left(k^* + {N^2\gamma^2}\sum_{i> k^*}{\lambda_i^{(q)}}^2\right).
    \end{equation*}
    Then we tackle $\mathrm{tr}(\mathbf{B}_{0,N})$.
    \begin{equation}
        \begin{aligned}
            \mathrm{tr}(\mathbf{B}_{0,N})&=\mathrm{tr}\left(\mathbf{B}_{0}-(\mathbf{I}-\frac{\gamma}{2}\mathbf{H}^{(q)})^{N}\mathbf{B}_{0}(\mathbf{I}-\frac{\gamma}{2}\mathbf{H}^{(q)})^{N}\right)\\
            &=\sum_{i}\left(1-(1-\frac{\gamma}{2}\lambda_{i}^{(q)})^{2N}\right)\cdot\left(\langle\mathbf{w}_{0}-{\mathbf{w}^{(q)}}^{*},\mathbf{v}_{i}^{(q)}\rangle\right)^{2}\\
            &\leq \sum_i\min\{1,N\gamma\lambda_i^{(q)}\}\left(\langle\mathbf{w}_0-{\mathbf{w}^{(q)}}^*,\mathbf{v}_i^{(q)}\rangle\right)^2\\
            &\leq \|\mathbf{w}_{0}-{\mathbf{w}^{(q)}}^{*}\|_{\mathbf{I}_{0:k^{*}}}^{2}+N\gamma\|\mathbf{w}_{0}-{\mathbf{w}^{(q)}}^{*}\|_{\mathbf{H}_{k^{*}:\infty}}^{2}.
        \end{aligned}
    \end{equation}
    Therefore,
    \begin{equation*}
        I_1\leq \frac{4}{\gamma^2N^2}\frac{2(1+\tilde{\epsilon})\gamma\alpha_B\cdot\left(\|\mathbf{w}_{0}-{\mathbf{w}^{(q)}}^{*}\|_{\mathbf{I}_{0:k^{*}}}^{2}+N\gamma\|\mathbf{w}_{0}-{\mathbf{w}^{(q)}}^{*}\|_{\mathbf{H}_{k^{*}:\infty}}^{2}\right)}{1-2(1+\tilde{\epsilon})\gamma\alpha_B\operatorname{tr}(\mathbf{H}^{(q)})}\left(k^* + {N^2\gamma^2}\sum_{i> k^*}{\lambda_i^{(q)}}^2\right).
    \end{equation*}

    Regarding $I_2$,
    \begin{equation*}
        \begin{aligned}
            I_2=&\frac{2}{\gamma N^{2}}\sum_{k=0}^{N-1}\left\langle(\mathbf{I}-\frac{\gamma}{2}\mathbf{H}^{(q)})^{k}-(\mathbf{I}-\frac{\gamma}{2}\mathbf{H}^{(q)})^{N+k},\mathbf{B}_{0}\right\rangle\\
            =&\frac{2}{\gamma N^{2}}\sum_{k=0}^{N-1}\sum_i\left[(1-\frac{\gamma}{2}\lambda_i^{(q)})^{k}-(1-\frac{\gamma}{2}\lambda_i^{(q)})^{N+k}\right]\omega_i^2\\
            =&\frac{4}{\gamma^2 N^{2}}\sum_i\left[1-(1-\frac{\gamma}{2}\lambda_i^{(q)})^{N}\right]^2\frac{\omega_i^2}{\lambda_i^{(q)}}\\
            \leq&\frac{4}{\gamma^2 N^{2}}\sum_i\frac{\omega_i^2}{\lambda_i^{(q)}}\operatorname*{min}\left\{1,\gamma^{2}N^{2}(\lambda_{i}^{(q)})^{2}\right\}\\
            =&\frac{4}{\gamma^2N^2}\cdot\|\mathbf{w}_0-{\mathbf{w}^{(q)}}^*\|_{(\mathbf{H}_{0:k^*}^{(q)})^{-1}}^2+4\|\mathbf{w}_0-{\mathbf{w}^{(q)}}^*\|_{\mathbf{H}_{{k^*}:\infty}^{(q)}}^2.
        \end{aligned}
    \end{equation*}

    Overall, 
    \begin{equation*}
        \begin{aligned}
            &\mathrm{bias}\leq I_1+I_2\\
            \leq& \frac{8(1+\tilde{\epsilon})\alpha_B\cdot\left(\frac{\|\mathbf{w}_{0}-{\mathbf{w}^{(q)}}^{*}\|_{\mathbf{I}_{0:k^{*}}}^{2}}{N\gamma}+\|\mathbf{w}_{0}-{\mathbf{w}^{(q)}}^{*}\|_{\mathbf{H}_{k^{*}:\infty}}^{2}\right)}{1-2(1+\tilde{\epsilon})\gamma\alpha_B\operatorname{tr}(\mathbf{H}^{(q)})}\left(\frac{k^*}{N} + {N\gamma^2}\sum_{i> k^*}{\lambda_i^{(q)}}^2\right)\\
            +& \frac{4}{\gamma^2N^2}\cdot\|\mathbf{w}_0-{\mathbf{w}^{(q)}}^*\|_{(\mathbf{H}_{0:k^*}^{(q)})^{-1}}^2+4\|\mathbf{w}_0-{\mathbf{w}^{(q)}}^*\|_{\mathbf{H}_{{k^*}:\infty}^{(q)}}^2.
        \end{aligned}
    \end{equation*}
\end{proof}

Based on the analysis for bias and variance error, we are ready to present the bounds for $R_N^{(0)}$ under multiplicative quantization with quantized master weight.
\begin{theorem}[\rm A bound for $R_N^{(0)}$ under multiplicative quantization with quantized master weight]
    \label{thm: master multiplicative}
    Suppose there exist $\epsilon_d,\epsilon_l,\epsilon_p,\epsilon_a$ and $\epsilon_o$ such that for any $i\in \{d,l,p,a,o\}$, quantization $\mathcal{Q}_i$ is $\epsilon_i$-multiplicative. Under Assumption \ref{ass1}, Assumption \ref{ass: addd quantized}, Assumption \ref{ass2}, and Assumption \ref{ass3}, if the step size satisfies
    \begin{equation*}
        1>2(1+\tilde{\epsilon})\gamma\alpha_B\mathrm{tr}(\mathbf{H}^{(q)}),\quad 8\epsilon_p\leq \gamma \lambda_d^{(q)}-\frac{3}{4}\gamma^2{\lambda_d^{(q)}}^2,
    \end{equation*}
    then taking expectation over $\mathbf{w}^*$ on $\mathrm{variance}$ \footnote{Here we assume that $\mathbb{E}\mathbf{w}^*{\mathbf{w}^*}^\top=\mathbf{I}$.},
    \begin{equation*}
        \begin{aligned}
            R_N^{(0)}/2\leq &\frac{4{\sigma_M^{(q)}}^2+32(1+\tilde{\epsilon}){d\epsilon_p}\alpha_B/\gamma}{1-2(1+\tilde{\epsilon})\gamma\alpha_B\mathrm{tr}(\mathbf{H}^{(q)})}\left(\frac{k^*}{N}+N\gamma^2\cdot\sum_{i>k^*}(\lambda_i^{(q)})^2\right)\\
            +
            &\frac{16\epsilon_p}{\gamma}\left(\sum_{i\leq k^*}\frac{1}{N\gamma\lambda_i^{(q)}}+N\gamma\sum_{i>k^*}\lambda_i^{(q)}\right)\\
            +& \frac{8(1+\tilde{\epsilon})\alpha_B\cdot\left(\frac{\|\mathbf{w}_{0}-{\mathbf{w}^{(q)}}^{*}\|_{\mathbf{I}_{0:k^{*}}}^{2}}{N\gamma}+\|\mathbf{w}_{0}-{\mathbf{w}^{(q)}}^{*}\|_{\mathbf{H}_{k^{*}:\infty}}^{2}\right)}{1-2(1+\tilde{\epsilon})\gamma\alpha_B\operatorname{tr}(\mathbf{H}^{(q)})}\left(\frac{k^*}{N} + {N\gamma^2}\sum_{i> k^*}{\lambda_i^{(q)}}^2\right)\\
            +& \frac{4}{\gamma^2N^2}\cdot\|\mathbf{w}_0-{\mathbf{w}^{(q)}}^*\|_{(\mathbf{H}_{0:k^*}^{(q)})^{-1}}^2+4\|\mathbf{w}_0-{\mathbf{w}^{(q)}}^*\|_{\mathbf{H}_{{k^*}:\infty}^{(q)}}^2,
        \end{aligned}
    \end{equation*}
    where
    \begin{gather*}
            \tilde{\epsilon}=8\epsilon_o(1+\epsilon_p)(1+\epsilon_a)+8\epsilon_p+4\epsilon_a(1+\epsilon_p),\\
            {\sigma_M^{(q)}}^2=\frac{(1+4\epsilon_o)\sigma^2}{B} + \frac{\mathbb{E}_{\mathbf{w}^*}\|\mathbf{w}^*\|_\mathbf{H}^2}{1+\epsilon_d}\alpha_B\left(4\epsilon_o[(1+\epsilon_p)(1+\epsilon_a)+1]+2\epsilon_a(1+\epsilon_p)+4\epsilon_p \right).
    \end{gather*}
\end{theorem}
\begin{proof}
    The proof is completed by (\ref{eq: f mul}), Lemma \ref{lem: f6} and Lemma \ref{bias R2 bound multiplicative, master}.
\end{proof}

%% file: additional_experiment_detail.tex
\section{Details of Additional Experiments}
\label{sec: exp}

\subsection{Additional Datasets}

For the supplementary experiments, we consider both synthetic and real-world datasets.

\paragraph{Synthetic dataset.}
We construct a synthetic regression dataset whose covariance spectrum follows an exponential decay. Specifically, the eigenvalues are given by
\[
    \lambda_i = e^{-i}, \qquad i = 1,2,\dots,d.
\]
This allows us to examine the behavior of our method under rapidly decaying spectral structures, complementing the polynomial-decay setting used in the main paper.

\paragraph{Real-world dataset: \texttt{Communities and Crime}.}
We additionally evaluate on the publicly available \texttt{Communities and Crime} dataset, which contains community-level statistics from across the United States. The features integrate socio-economic indicators from the 1990 U.S. Census, law-enforcement statistics from the 1990 LEMAS survey, and crime records from the 1995 FBI Uniform Crime Reporting (UCR) program.  
The task is a standard regression problem: predicting the \emph{per-capita violent crime rate} from community attributes. The dataset contains about $2000$ instances with $122$ features.

\subsection{Experimental Settings and Results}

We describe below the protocol for each additional experiment and corresponding results. In both real-world dataset and synthetic datasets, we examine how do additive vs. multiplicative quantization affect learning (generalization) performance.

\begin{itemize}[leftmargin=*]
    \item \textbf{Real-world regression (Communities and Crime).}  
    We apply both \emph{additive} and \emph{multiplicative} quantization schemes (with fixed quantization error level $\varepsilon=0.01$) to the regression task on \texttt{Communities and Crime} dataset. For each quantization method, we evaluate the resulting population risk $\mathbb{E}_{\mathbf{x},y}\left[\left(y-\langle \mathbf{w},\mathbf{x} \rangle \right)^2\right]$. As illustrated in Figure \ref{community_plot}, the results demonstrate that, unlike additive quantization, the multiplicative scheme successfully maintains the performance of full-precision SGD. This aligns with our theoretical finding that multiplicative quantization exhibits greater tolerance to quantization error level.

    \item \textbf{Effect of quantization on data spectrum.}  
    Using the same settings on \texttt{Communities and Crime} dataset, we record the resulting empirical covariance spectra to study how each quantization type perturbs the underlying eigenvalue structure. Results are shown in Figure \ref{spectrum}. It is shown that additive errors errors dramatically distort the spectrum of effective data covariance while multiplicative quantization errors largely preserve the spectral structure. This visualization corroborates the specific mechanism by which additive and multiplicative quantization lead to distinct generalization behaviors.

    \item \textbf{Sensitivity analysis on batch size and spectral decay.} To demonstrate the robustness of our findings, we conduct additional experiments varying the batch size and data spectrum. First, we extend the batch size to $B=10$ (with $d=200$) and vary the quantization error level $\varepsilon\in\{0.001,\,0.005,\,0.01\}$. Second, we replace the polynomial-decay spectrum with an exponential-decay synthetic dataset while keeping other settings identical. The results, shown in Figures \ref{bs10 mul}--\ref{bs10 int} (batch size) and Figures \ref{exp mul}--\ref{exp add} (spectral decay), consistently mirror the findings in the main paper: multiplicative quantization preserves the generalization performance of full-precision SGD across various quantization error levels, whereas additive quantization suffers from performance degradation as the error level increases.
\end{itemize}

\begin{figure}[t!]
  \centering
  \subfigure[\textbf{Real-world regression}]{
    \includegraphics[width=.45\textwidth]{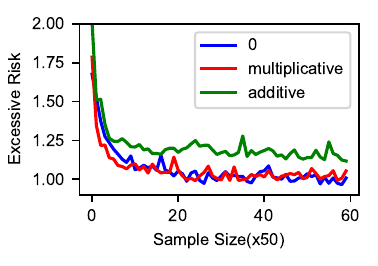}
    \label{community_plot}
  }\hspace{-.15in}
  \subfigure[\textbf{Effects on spectrum}]{
    \includegraphics[width=.45\textwidth]{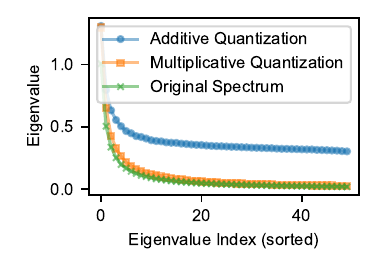}
    \label{spectrum}
  }
  \vskip -.15in
  \caption{\textbf{Comparison between multiplicative quantization and additive quantization}. (a): Real-world regression (Communities and Crime). (b): Effect of quantization on data spectrum.}\vspace{-.2in}
\end{figure}
\begin{figure}[t!]
  \centering
  \subfigure[\textbf{Multiplicative} (FP-like)]{
    \includegraphics[width=.24\textwidth]{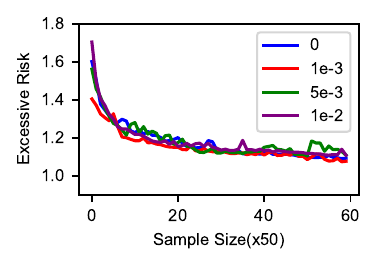}
    \label{bs10 mul}
  }\hspace{-.15in}
  \subfigure[\textbf{Additive} (INT-like) ]{
    \includegraphics[width=.24\textwidth]{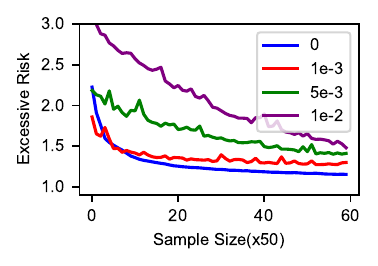}
    \label{bs10 int}
  }\hspace{-.15in}
  \subfigure[\textbf{Multiplicative} (FP-like)]{
    \includegraphics[width=.24\textwidth]{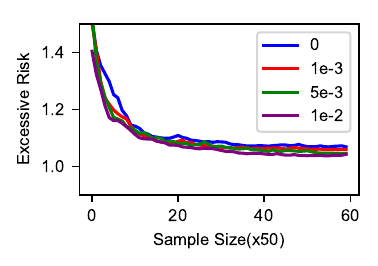}
    \label{exp mul}
  }\hspace{-.15in}
  \subfigure[\textbf{Additive} (INT-like)]{
    \includegraphics[width=.24\textwidth]{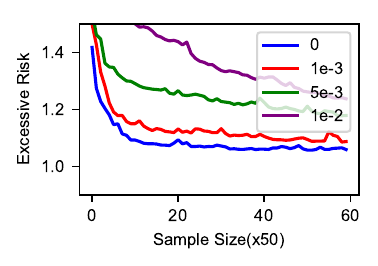}
    \label{exp add}
  }
  \vskip -.15in
  \caption{\textbf{Generalization under quantization.}
  Test risk for SGD with iterate averaging under multiplicative (FP-like) vs.\ additive (INT-like) quantization.
(a) and (b): vary the quantization level at fixed $B=10$.
(c) and (d): vary the quantization level under exponential decay.}\vspace{-.2in}
\end{figure}
